\documentclass{article}
\usepackage{iclr2022_conference,times}

\usepackage{hyperref}
\usepackage{url}

\usepackage[utf8]{inputenc} 
\usepackage[T1]{fontenc}    
\usepackage{hyperref}       
\usepackage{url}            
\usepackage{booktabs}       
\usepackage{amsfonts}       
\usepackage{nicefrac}       
\usepackage{microtype}      
\usepackage{xcolor}         

\usepackage{algorithm}
\usepackage{algpseudocode}
\usepackage{amsmath,amsfonts,amsthm,amssymb}
\usepackage{enumitem}
\usepackage{graphicx}
\usepackage[most]{tcolorbox}
\usepackage{wrapfig}

\usepackage[toc,page,header]{appendix}
\usepackage{minitoc}

\NewDocumentCommand{\codeword}{v}{%
\texttt{\textcolor{blue}{#1}}%
}

\newcommand{\shams}[1]{{\color{red}{#1}}}
\newcommand{\cgb}[1]{{\color{black}{#1}}}

\renewcommand{\shams}[1]{{\color{black}{#1}}}
\renewcommand{\cgb}[1]{{\color{black}{#1}}}

\usepackage{soul}
\setstcolor{red}

\tcbset{
    colback=yellow!10!white, colframe=black!50!black,
    highlight math style= {
        enhanced,
        colframe=black,colback=red!10!white,boxsep=0pt
    }
}

\newtheorem{theorem}{Theorem}

\newtheorem{lemma}{Lemma}
\newtheorem{defn}{Definition}
\newtheorem{corollary}{Corollary}

\newcommand{\E}[2]{\mathbb{E}_{#2}\left[#1\right]}

\newcommand{\R}[1]{\mathbb{R}^{#1}}
\renewcommand{\vec}[3]{\boldsymbol{#1}_{#2}^{#3}}
\newcommand{\tvec}[3]{\widetilde{\boldsymbol{#1}}_{#2}^{#3}}
\newcommand{\subsup}[3]{#1_{#2}^{#3}}
\newcommand{\mset}[3]{\left\{#1\right\}_{#2}^{#3}}
\newcommand{\norm}[2]{\left\Vert#1\right\Vert_{#2}}
\newcommand{\abs}[1]{\left|#1\right|}
\newcommand{\innorm}[2]{\Vert #1\Vert_{#2}}
\newcommand{\dotp}[2]{\left\langle #1, #2 \right\rangle}
\newcommand{\indotp}[2]{\langle #1, #2 \rangle}
\newcommand{\aligneqn}[1]{{\begin{align}#1\end{align}}}
\newcommand{\eqn}[1]{{\begin{equation}#1\end{equation}}}

\newcommand{\bigo}[1]{\mathcal{O}\left( #1 \right)}
\newcommand{\sbigo}[1]{\mathcal{O}\hspace{-1mm}\left( #1 \right)}

\hypersetup{
    colorlinks=true,
    linkcolor=blue,
    citecolor=blue,
    urlcolor=blue
}

\newcommand{\algName}{{\tt LBGM}}
\newcommand{\algFullName}{Look-back Gradient Multiplier}
\title{Recycling Model Updates in Federated \\Learning: Are Gradient Subspaces Low-Rank?}

\author{Sheikh Shams Azam, Seyyedali Hosseinalipour, Qiang Qiu, Christopher Brinton\\
School of ECE, Purdue University\\
\texttt{\{azam1, hosseina, qqiu, cgb\}@purdue.edu} \\
}

\iclrfinalcopy
\begin{document}

\maketitle
\doparttoc
\faketableofcontents


\vspace{-2mm}
\begin{abstract}
  In this paper, we question the rationale behind propagating large numbers of parameters through a distributed system during federated learning. We start by examining the rank characteristics of the subspace spanned by gradients \shams{across epochs} (i.e., the gradient-space) in centralized model training, and observe that \shams{this} gradient-space often consists of a few leading principal components accounting for an overwhelming majority ($95-99\%$) of the explained variance. Motivated by this, we propose the "Look-back Gradient Multiplier" ({\tt LBGM}) algorithm, which \shams{exploits} this low-rank property \shams{to enable gradient recycling} between model update rounds \shams{of federated learning, reducing transmissions of large parameters to single scalars for aggregation}. We analytically characterize the convergence behavior of {\tt LBGM}, revealing the nature of the trade-off between communication savings and model performance. Our subsequent experimental results demonstrate the improvement {\tt LBGM} obtains in communication overhead compared to \shams{conventional }federated learning \shams{on several datasets and deep learning models}. Additionally, we show that {\tt LBGM} is a general plug-and-play algorithm that can be used standalone or stacked on top of existing sparsification techniques for distributed model training.
\end{abstract}


\section{Introduction}
\vspace{-1mm}

Federated Learning (FL) \citep{konevcny2016federated} has emerged as a popular  distributed machine learning (ML) paradigm. By having each device conduct local model updates, FL substitutes raw data transmissions with model parameter transmissions,
promoting data privacy \citep{shokri2015privacy, azam2020towards} and  communication savings \citep{wang2020tackling}.
At the same time, overparameterized neural networks (NN) are becoming ubiquitous in the ML models trained by FL, e.g., in computer vision \citep{
liu2016ssd, huang2017densely} and natural language processing \citep{brown2020language, 
liu2019roberta}. While NNs can have parameters in the range of a few million (VGG \citep{simonyan2014very}, ResNet \citep{he2016deep}) to several billions (GPT-3 \citep{brown2020language}, Turing NLG \citep{turing2020nlg}), prior works have demonstrated that a majority of these parameters are often irrelevant \citep{frankle2018lottery, liu2018rethinking, han2015learning, li2016pruning} in optimization and inference. This presents an opportunity to reduce communication overhead by transmitting lighter representations of the model, conventionally achieved through compression/sparsification techniques \citep{wang2018atomo, alistarh2016qsgd, vogels2019powersgd}. 

In this work, we investigate the ``overparameterization'' of NN training, through the lens of rank characteristics of the subspace spanned by the gradients (i.e., the gradient-space) generated during stochastic gradient descent (SGD).
\shams{We start with the} fundamental question: \shams{\textit{can we observe the effect of overparameterization in NN optimization directly through the principal components analysis (PCA) of the gradients generated during SGD-based training?}} And if so: \textit{can this property be used to reduce communication overhead in FL?} Our main hypothesis is that
\begin{equation}
    \textbf{the subspaces spanned by gradients generated \cgb{across SGD epochs} are low-rank.}
    \label{hypothesis}
    \tag{\bf H1}
\end{equation}
This leads us to propose a technique that instead of propagating a million/billion-dimensional vector (i.e., gradient) over the system in each iteration of FL only requires propagating a single scalar in the majority of the iterations. Our algorithm introduces a new class of techniques based on the concept of reusing/recycling device gradients over time. Our main contributions can be summarized as follows:
\vspace{-2mm}
\begin{itemize}[leftmargin=5mm]
    \item \shams{We demonstrate the low-rank characteristics of the gradient-space by directly studying its principal components for several NN models trained on various real-world datasets.} We show that \shams{principal gradient directions (i.e., directions of principal components of the gradient-space)} can be approximated in terms of actual gradients generated during the model training process.
    \vspace{-0.8mm}
    \item Our insights lead us to develop the ``Look-back Gradient Multiplier'' ({\algName}) algorithm to \shams{significantly reduce the communication overhead in FL}. {\algName} recycles \shams{previously transmitted gradients} to represent the newly-generated gradients at each device with a single scalar.
    We further analytically investigate the convergence characteristics of this algorithm.
    \vspace{-0.8mm}
    \item Our experiments show the communication savings obtained via {\algName} in FL both as a standalone solution and a plug-and-play method used with other gradient compression techniques, e.g., top-K. We further reveal that {\algName} can be extended to distributed training, e.g., {\algName} with SignSGD \citep{bernstein2018signsgd} substantially reduces communication overhead in multi-GPU systems.
\end{itemize}


\vspace{-3mm}
\section{A Gradient-Space Odyssey} 
\label{sec:explore}
\vspace{-3mm}
We first \shams{start by directly studying the principal component analysis (PCA) of} the gradient-space of overparameterized NNs. Given a centralized ML training task (e.g., classification, regression, segmentation), we exploit principal component analysis (PCA) \citep{pearson1901liii} to answer the following: \textit{how many principal components explain the $99\%$ and $95\%$ variance (termed \textsc{n99-pca} and \textsc{n95-pca}, respectively; together \textsc{n-pca}) of all the gradients generated during model training?} 

\begin{figure}[t]
\centering
\begin{minipage}{.495\textwidth}
  \centering
  \centerline{\includegraphics[width=1.0\textwidth]{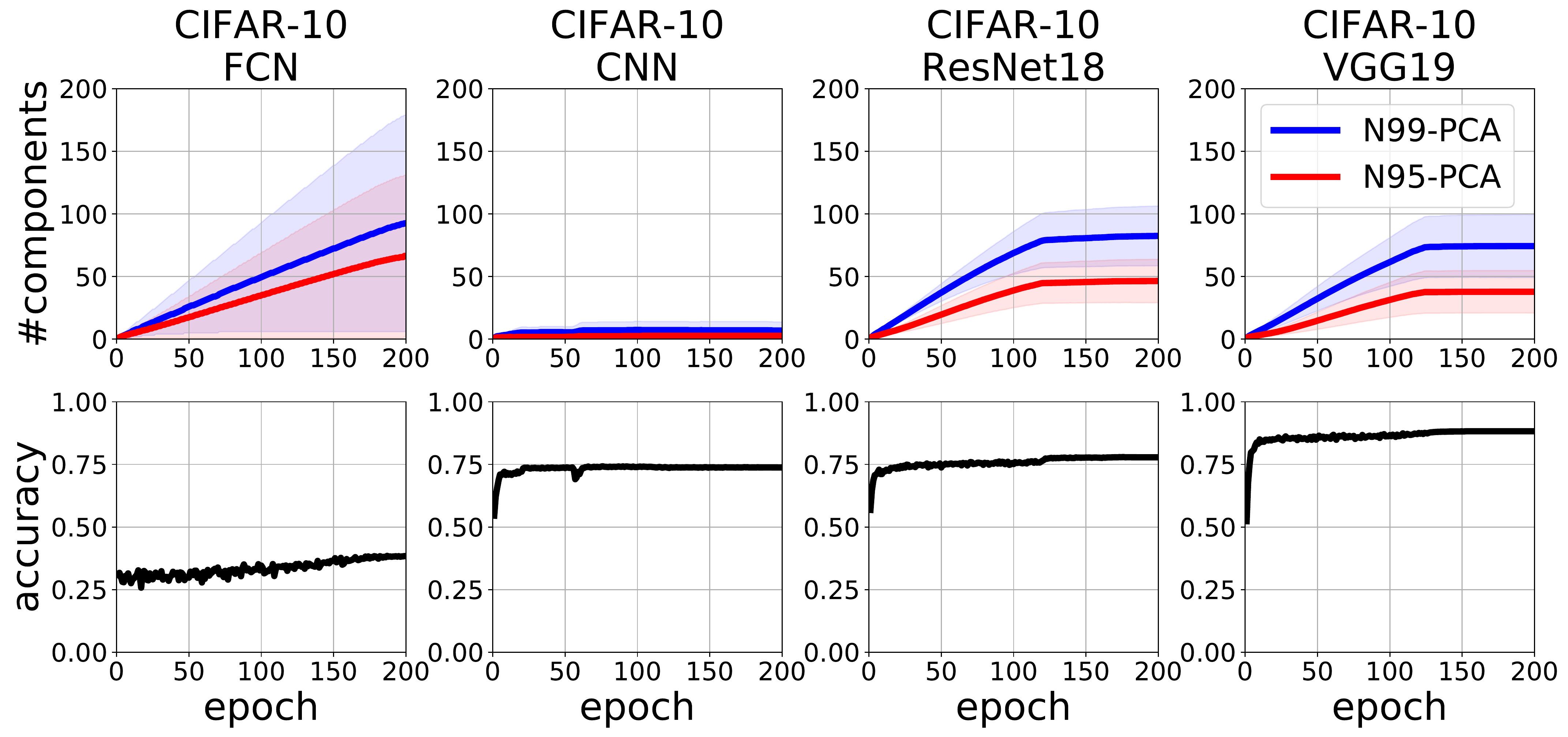}}
\end{minipage}
\begin{minipage}{.495\textwidth}
  \centering
  \centerline{\includegraphics[width=1.0\textwidth]{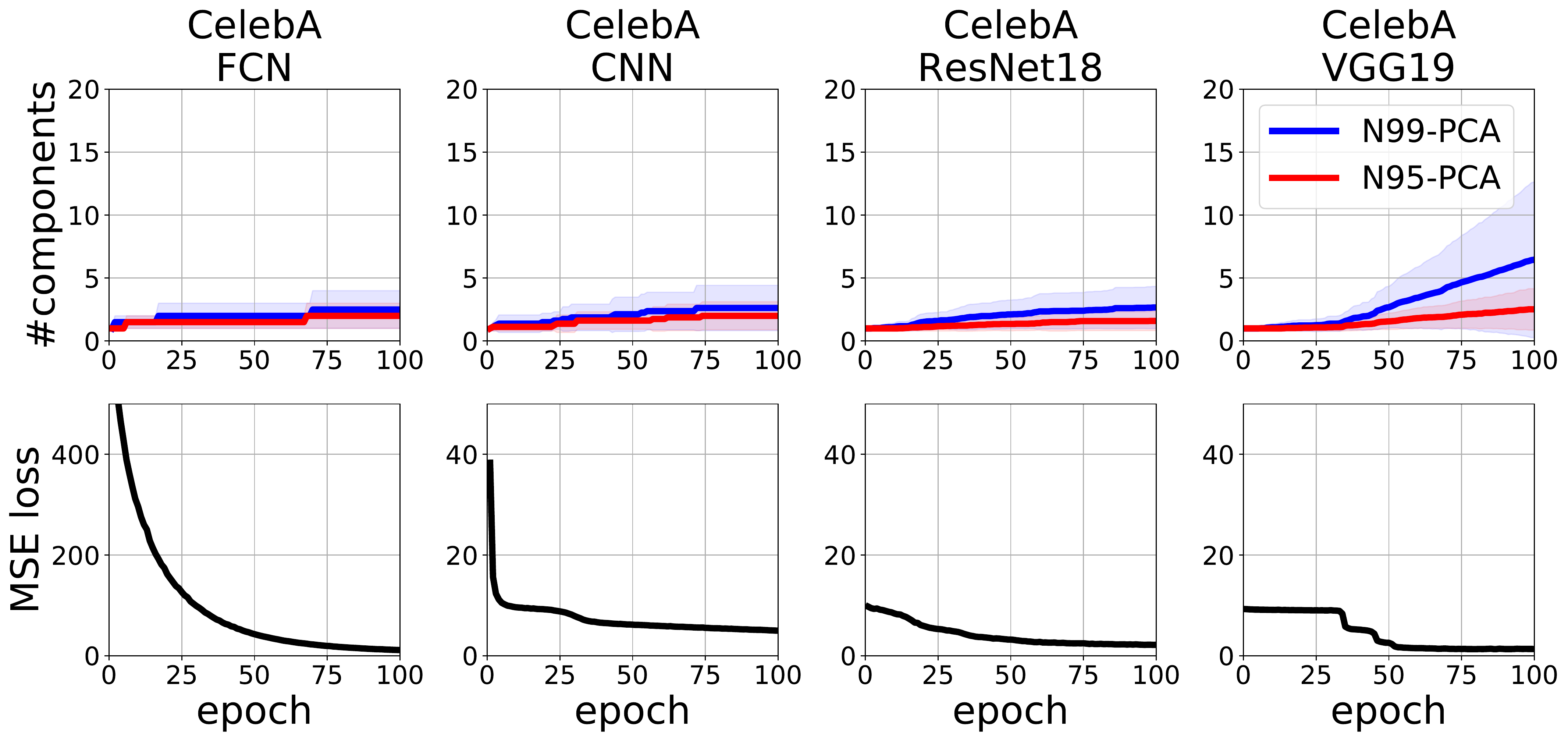}}
\end{minipage}
    \vspace{-4mm}
  \caption{\small{\textit{PCA components progression}. The top row shows the number of components that account for 99\% (N99-PCA in blue) and 95\% (N95-PCA in red) explained variance of all the gradients generated during gradient descent epochs. The bottom row shows the performance of the model on the test data. The results are presented for: (i) CIFAR-10 classification (left 4 columns), and (ii) CelebA regression (right 4 columns).}}
  \label{fig:prelim_1}
   \vspace{-4.5mm}
\end{figure}

We start with 4 different NN architectures: (i) fully-connected neural network (FCN), (ii) convolutional neural network (CNN), (iii) ResNet18 \citep{he2016deep}, and (iv) VGG19 \citep{simonyan2014very}; trained on 2 datasets:  CIFAR-10 \citep{krizhevsky2009learning} and CelebA \citep{liu2015faceattributes}, with classification and regression tasks, respectively. We then compute the \textsc{n-pca} for each epoch by applying PCA on the set of gradients accumulated until that epoch (pseudo-code in Algorithm~\ref{alg:pseudo} in \textbf{Appendix~\ref{app:pseudo}}). The results depicted in Fig.~\ref{fig:prelim_1} agree with our hypothesis \eqref{hypothesis}: both \textsc{n99-pca} and \textsc{n95-pca} are significantly lower than that the total number of gradients generated during model training, e.g., in Fig.~\ref{fig:prelim_1}\footnotemark[1]\footnotetext[1]{\shams{The addition of a learning rate scheduler (e.g., cosine annealing scheduler~\citep{loshchilov2016sgdr}) has an effect on the PCA of the gradient-space. Careful investigation of this phenomenon is left to future work.
}} 
the number of principal components (red and blue lines in top plots) are substantially lower (often as low as 10\% of number of epochs, i.e., gradients generated) for both datasets.
In \textbf{Appendix~\ref{app:prelim_expt_1}}, we further find that \eqref{hypothesis} holds in our experiments using several \shams{additional} datasets: CIFAR-100 \citep{krizhevsky2009learning}, MNIST \citep{lecun2010mnist}, FMNIST \citep{xiao2017fmnist}, CelebA \citep{liu2015faceattributes}, PascalVOC \citep{everingham2010pascal}, COCO \citep{lin2014microsoft}; models: U-Net \citep{ronneberger2015u}, SVM \citep{cortes1995support}; and tasks: segmentation and regression. Note that (especially on CIFAR-10) variations in \textsc{n-pca} across models are not necessarily related to the model performance (CNN performs almost as well as ResNet18 but has much lower \textsc{n-pca}; Fig.~\ref{fig:prelim_1} -- columns 2 \& 3) or complexity (CNN has more parameters than FCN but has lower \textsc{n-pca}; Fig.~\ref{fig:prelim_1} -- columns 1 \& 2). The fact that rank deficiency of the gradient-space is not a consequence of model complexity or performance suggests that the gradient-space of state-of-the-art large-scale ML models could be represented using a few \cgb{principal gradient directions~(PGD)}.

\textbf{N-PCA and FL.} In an ``ideal'' FL framework, if both the server and the workers/devices have the~\cgb{PGDs}, then the newly generated gradients can be transmitted by sharing their projections on the~\cgb{PGDs}, i.e., the ~\cgb{PGD} multipliers (PGM). PGMs and \cgb{PGDs} can together be used to reconstruct the device generated gradients at the server, \shams{dramatically reducing communication costs.} However, this setting is impractical since: (i) it is infeasible to obtain the~\cgb{PGDs} prior to the training, and (ii) PCA is computationally intensive. 
We thus look for an efficient online approximation of the~\cgb{PGDs}.

\begin{figure}[t]
\centering
\begin{minipage}{.495\textwidth}
  \centering
  \centerline{\includegraphics[width=1.0\textwidth]{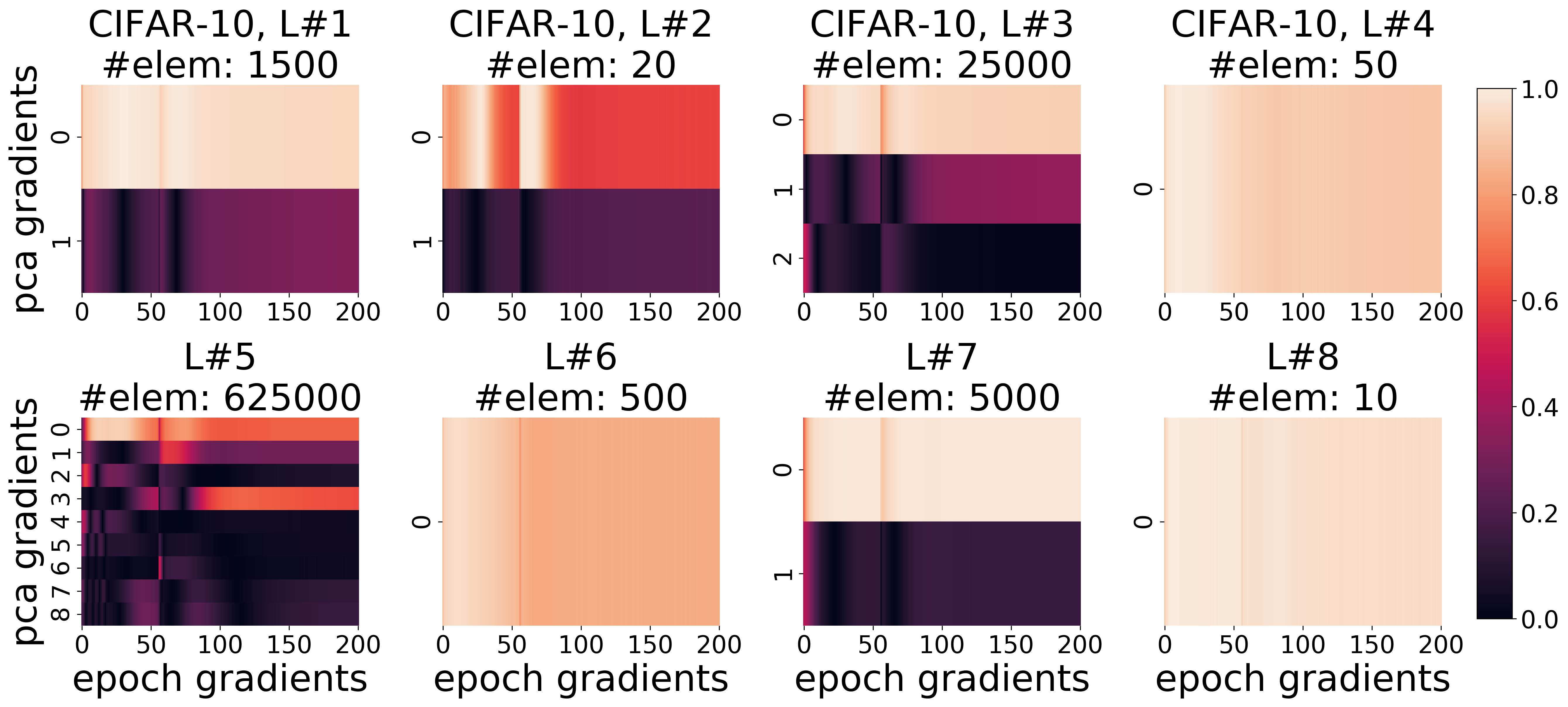}}
\end{minipage}
\begin{minipage}{.495\textwidth}
  \centering
  \centerline{\includegraphics[width=1.0\textwidth]{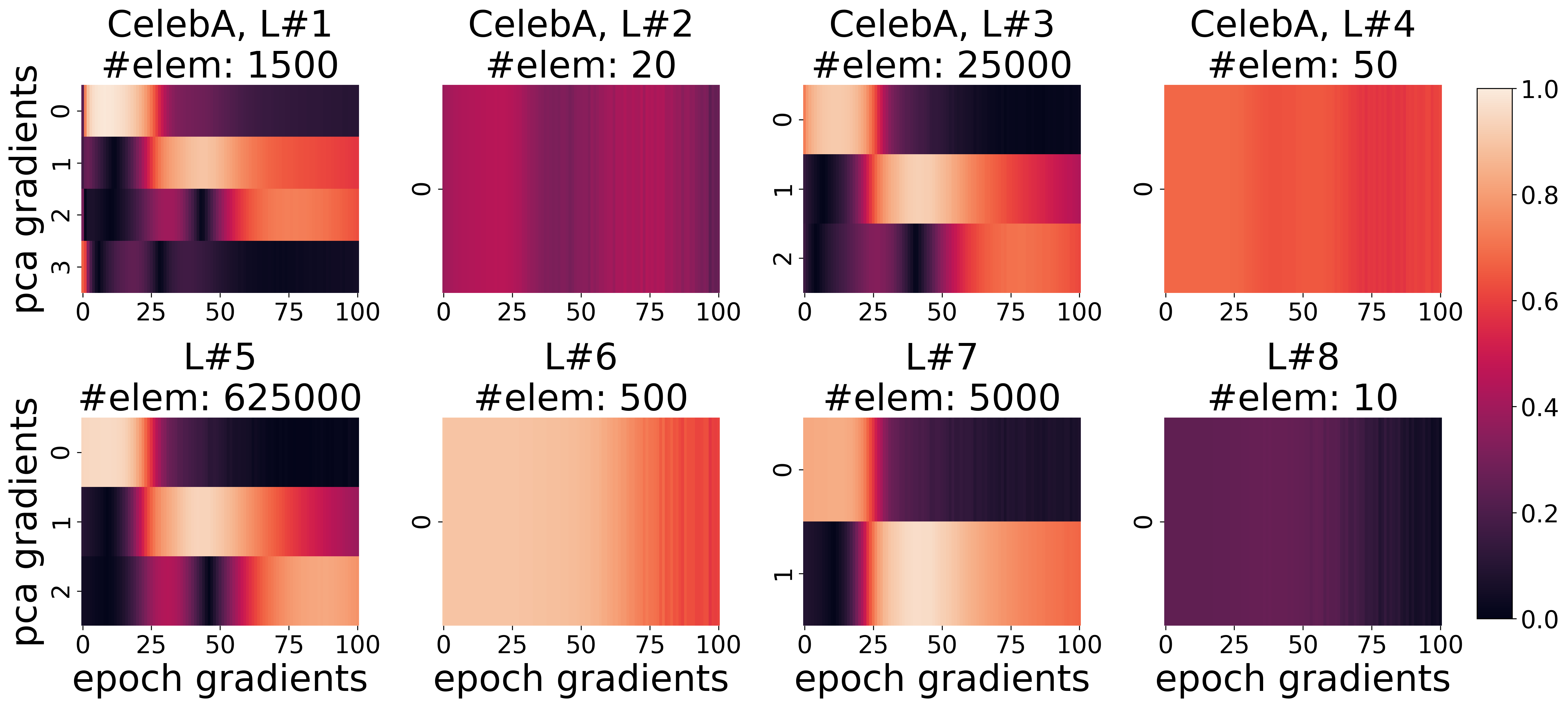}}
\end{minipage}
    \vspace{-4mm}
  \caption{\small{\textit{Overlap of actual and principal gradients}. The heatmap shows the pairwise cosine similarity \shams{between} actual (epoch) gradients and \shams{principal gradient directions (PCA gradients)}. Epoch gradients have a substantial overlap with one or more PCA gradients and consecutive epoch gradients show a gradual variation.
  \shams{This suggest that there may exist a high overlap between the gradients generated during the NN model training.}
The results are shown for a CNN classifier trained on CIFAR-10 (left 4 columns) and CelebA (right 4 columns) datasets. Each subplot is marked with \#L, the layer number of the CNN and \#elem, the number of elements in each layer.}}
  \label{fig:prelim_2}
      \vspace{-2.5mm}
\end{figure}

\textbf{Overlap of Actual Gradients and N-PCA.} To approximate~\cgb{PGDs}, we exploit an observation made in Fig.~\ref{fig:prelim_1} (and further in Appendix~\ref{app:prelim_expt_1}): \textsc{n-pca} mostly remains constant over time, suggesting that the gradients change gradually across SGD epochs.
To further investigate this, in Fig.~\ref{fig:prelim_2} (and further in Appendix~\ref{app:prelim_expt_2}) we plot the cosine similarity between~\cgb{PGDs} and actual gradients as a heatmap.\footnotemark[2] 
 We observe that (i) the similarity of actual gradients to~\cgb{PGDs} varies gradually over time, and (ii) 
 actual gradients have a high cosine similarity with one or more~\cgb{PGDs}. This leads to our second hypothesis:
 \vspace{-.3mm}
\begin{equation}
    \textbf{\cgb{PGDs} can be approximated using a subset of gradients \cgb{generated across SGD epochs}}.
    \label{hypothesis_2}
    \tag{\bf H2}
     \vspace{-.3mm}
\end{equation}
\textbf{Look-back Gradients.} 
Our observations above suggest a significant overlap among consecutive gradients generated during SGD epochs. This is further verified in Fig.~\ref{fig:prelim_3} (and in Appendix~\ref{app:prelim_expt_3}\footnotemark[3]), where we plot the pairwise cosine similarity of consecutive gradients generated during SGD epochs.\footnotemark[2] For example, consider the boxes marked B1, B2, and B3 in layer 1 (L\#1 in Fig.~\ref{fig:prelim_3}). Gradients in each box can be used to approximate other gradients within the box.
Also, interestingly, the number of such boxes that can be drawn is correlated with the corresponding number of~\cgb{PGDs}. Based on \eqref{hypothesis_2}, we next propose our \textit{Look-back Gradient Multiplier} (\algName) algorithm that utilizes a subset of actual gradients, termed ``look-back gradients'', to reuse/recycle gradients transmitted in FL.

\footnotetext[2]{Refer to Algorithm~\ref{alg:pseudo} in Appendix~\ref{app:pseudo} for the detailed pseudocode.}
\footnotetext[3]{\shams{2  of 24 experiments (Fig.~\ref{fig:prelim_3_mnist_vgg19}\&\ref{fig:prelim_3_mnist_resnet18} in Appendix~\ref{app:prelim_expt_3}) show inconsistent gradient overlaps. However, our algorithm discussed in Sec.~\ref{sec:method} still performs well on those datasets and models (see Fig.~\ref{fig:standalone_resnet18} in Appendix~\ref{app:prelim_expt_3}).}}

\begin{figure}[t]
\centering
\begin{minipage}{.495\textwidth}
  \centering
  \centerline{\includegraphics[width=1.0\textwidth]{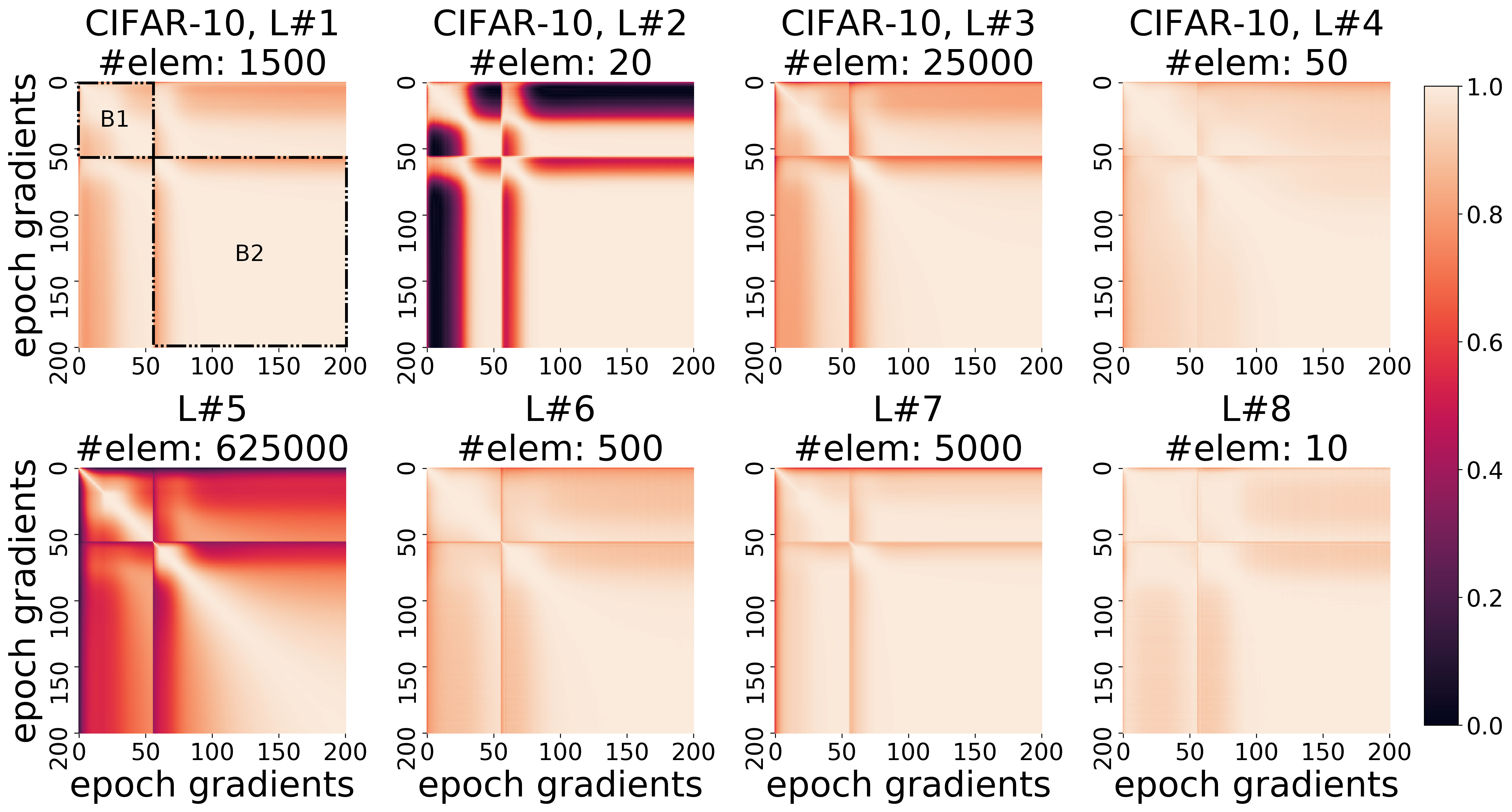}}
\end{minipage}
\begin{minipage}{.495\textwidth}
  \centering
  \vspace{2mm}
  \centerline{\includegraphics[width=1.0\textwidth]{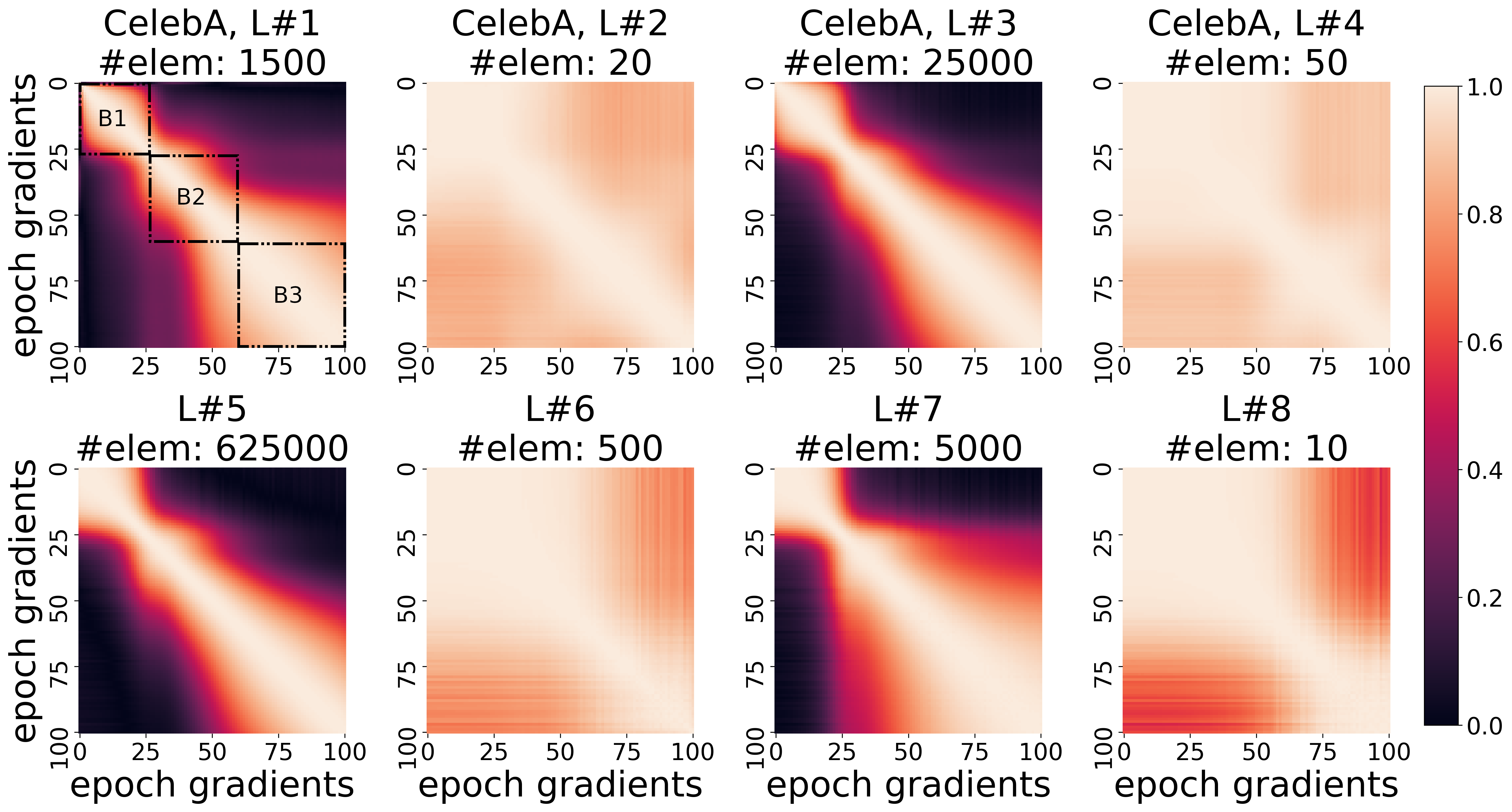}}
\end{minipage}
    \vspace{-4mm}
  \caption{\small{\textit{Similarity among consecutive gradients}. The cosine similarity  of consecutive gradients reveals a gradual change in directions of gradients over epochs. 
Thus,
  the newly generated gradients can be represented \shams{in terms of the previously generated gradients} with low approximation error. \shams{Reusing/recycling gradients can thus lead to significant communication savings during SGD-based  federated optimization.}
}}
  \label{fig:prelim_3}
    \vspace{-2mm}
\end{figure}


\vspace{-2.5mm}
\section{{\algFullName} Methodology}
\label{sec:method}
\vspace{-1.5mm}

Federated Learning (FL) considers a system of $K$ workers/devices indexed $1,...,K$, as shown in Fig.~\ref{fig:approx_fedl}. Each worker $k$ possesses a local dataset $\mathcal{D}_k$ with $n_k = |\mathcal{D}_k|$ datapoints. The goal of the system is to minimize the global loss function $F(\cdot)$ expressed through the following problem:
\eqn{
    \min_{\vec{\theta}{}{} \in \mathbb{R}^{M}}  F(\vec{\theta}{}{}) \triangleq \sum_{k=1}^K\omega_k F_k(\vec{\theta}{}{}){}{},
    \label{eqn:opt_obj}
}

\noindent where $M$ is the dimension of the model $\vec{\theta}{}{}$, $\omega_k=n_k/N$, $N=\sum_{k=1}^K n_k$, and $F_k(\vec{\theta}{}{}) = \sum_{d \in \mathcal{D}_k}f_k(\vec{\theta}{}{};d)/n_k$ is the local loss at worker $k$, with $f_k(\vec{\theta}{}{};d)$ denoting the loss function for data sample $d$ given parameter vector $\vec{\theta}{}{}$. FL tackles \eqref{eqn:opt_obj} via engaging the workers in local SGD model training on their own datasets. The local models are periodically transferred to and aggregated at the main server after $\tau$ local updates, forming a global model that is used to synchronize the workers before starting the next round of local model training.

\shams{At the start of round $t$, each model parameter $\vec{\theta}{k}{(t,0)}$ is initialized with the global model $\vec{\theta}{}{(t)}$. Thereafter,} worker $k$ updates its parameters $\vec{\theta}{k}{(t,b)}$ as: $\vec{\theta}{k}{(t, b+1)} \hspace{-0.5mm}\leftarrow\hspace{-0.5mm} \vec{\theta}{k}{(t, b)} - \eta \vec{g}{k}{}(\vec{\theta}{k}{(t, b)})$, where $ \vec{g}{k}{}(\vec{\theta}{k}{(t, b)})$ is the stochastic gradient at local step $b$, and $\eta$ is the step size. During a vanilla FL aggregation, the global model parameters are updated as $\vec{\theta}{}{(t+1)} \leftarrow \vec{\theta}{}{(t)} -  \eta \sum_{k=1}^K\omega_k\vec{g}{k}{(t)}$, where $\vec{g}{k}{(t)} = \sum_{b=0}^{\tau-1}  \vec{g}{k}{}(\vec{\theta}{k}{(t, b)})$ is the accumulated stochastic gradient (ASG) at worker $k$. \cgb{More generally, this aggregation may be conducted over a subset of workers at time $t$. 
} We define $\nabla \subsup{F}{k}{}(\vec{\theta}{k}{(t)}) = \sum_{b=0}^{\tau-1}  \nabla F_k(\vec{\theta}{k}{(t,b)})$ and $\vec{d}{k}{(t)} = \vec{g}{k}{(t)}/\tau$ as the corresponding accumulated true gradient and normalized ASG, respectively. 


\textbf{Indexing and Notations.}  In \textit{superscripts with parenthesis} expressed as tuples, the first element denotes the global aggregation round while the second element denotes the local update round, e.g., $\vec{g}{k}{(t,b)}$ is the gradient at worker $k$ at global aggregation round $t$ at local update $b$.  \textit{Superscripts without parenthesis} denote the index for look-back gradients, e.g., $\ell$ in $\subsup{\rho}{k}{(t),\ell}$ defined below.

\textbf{{{\algName}} Algorithm.}  {\algName} (see Fig.~\ref{fig:approx_fedl}) consists of three main steps: (i) workers initialize and propagate their look-back gradients (LBGs) to the server; (ii) workers estimate their look-back coefficients (LBCs), i.e., the scalar projection of subsequent ASGs on their LBGs, and the look-back phase (LBP), i.e., the angle between the ASG and the LBG; and (iii) workers update their LBGs and propagate them to the server if the LBP passes a threshold, otherwise they only transmit the scalar LBC. Thus, {\algName} propagates only a subset of actual gradients generated at the devices to the server. The intermediate global aggregation steps between two LBG propagation rounds only involve transfer of a \textit{single scalar}, i.e., the LBC, from each worker, instead of the entire ASG vector.

In {\algName}, the local model training is conducted in the same way as vanilla FL, while model aggregations at the server are conducted via the following rule:
\eqn{
    \vec{\theta}{}{(t+1)} = \vec{\theta}{}{(t)} - \eta\,\sum_{k=1}^K\omega_k \tvec{g}{k}{(t)},
    \label{eqn:lbgm_update}
    \vspace{-2mm}
}

\noindent where $\tvec{g}{k}{(t)}$ is the approximation of worker $k$'s accumulated stochastic gradient, given by Definition~\ref{lem:grad_proj}.

\begin{figure}[t]
  \centering
  \centerline{\includegraphics[width=.95\textwidth]{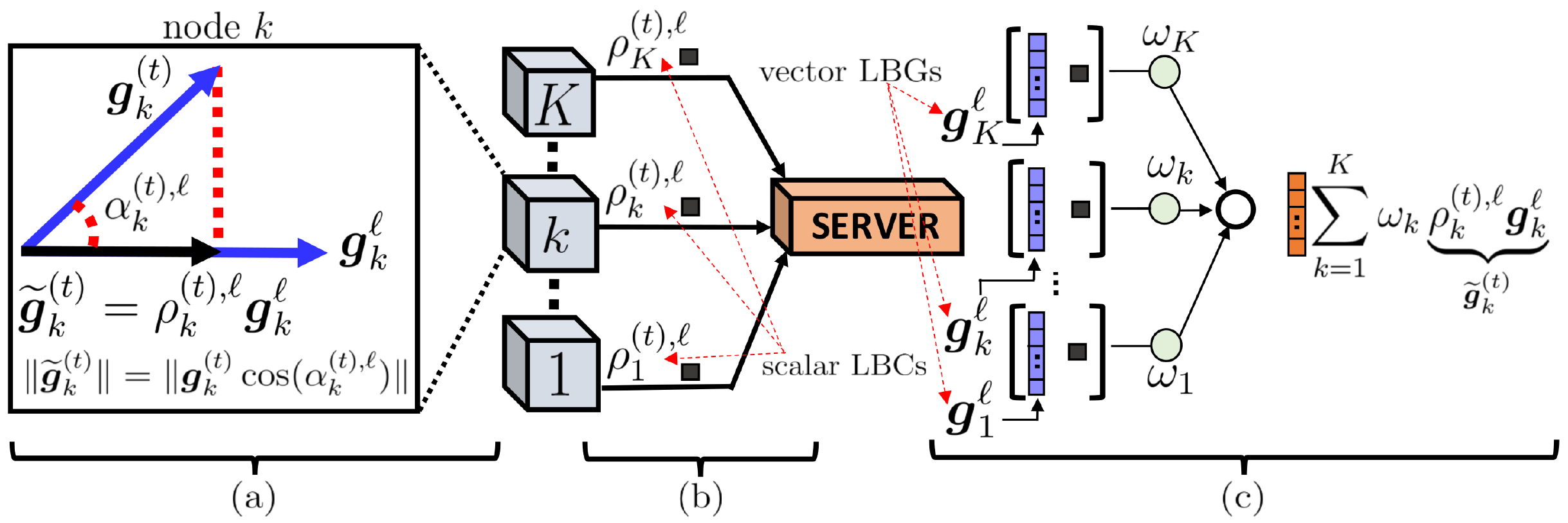}}
  \vskip -0.15in
  \caption{\small{\textit{Look-back gradient multiplier}. (a) The Look-back Coefficients (LBCs) are the projection of accumulated stochastic gradients at the workers on their Look-back Gradients (LBGs). (b) Scalar LBCs, i.e., the $\subsup{\rho}{k}{(t),\ell}$, are transmitted to the server. (c) LBG-based gradient approximations are reconstructed at the server.}}
  \label{fig:approx_fedl}
  \vspace{-3mm}
\end{figure}

\begin{defn}
\label{lem:grad_proj}
\textbf{(Gradient Approximation in {\algName})} Given the ASG $\vec{g}{k}{(t)}$, and the LBG $\vec{g}{k}{\ell}$, the gradient approximation $\tvec{g}{k}{(t)}$ recovered by the server is given by:
\eqn{
    \vspace{-1.1mm}
    \tvec{g}{k}{(t)} = \subsup{\rho}{k}{(t),\ell} \vec{g}{k}{\ell} \; \text{ for } \; \norm{\subsup{\rho}{k}{(t),\ell}\vec{d}{k}{\ell}}{} = \norm{\vec{d}{k}{(t)} \cos(\subsup{\alpha}{k}{(t),\ell})}{},
    \label{eqn:grad_proj}
    \tag{\bf D1}
    \vspace{-2mm}
}

\noindent where LBC $\subsup{\rho}{k}{(t),\ell} = { \indotp{ \vec{g}{k}{(t)} }{ \vec{g}{k}{\ell}} }\big/{ \innorm{\vec{g}{k}{\ell}}{}^2}$ is the projection of the accumulated gradient $\vec{g}{k}{(t)}$ on the LBG $\vec{g}{k}{\ell}$, and LBP {$\subsup{\alpha}{k}{(t),\ell}$ denotes the angle between $\vec{g}{k}{(t)}$ and $\vec{g}{k}{\ell}$} (see Fig.~\ref{fig:approx_fedl}(a)).
\end{defn}

{\algName} initializes the LBGs $\vec{g}{k}{\ell}$ with the first actual gradients propagated by the devices at $t=1$.
For the subsequent aggregations, each worker $k$ shares only its LBC $\subsup{\rho}{k}{(t),\ell}$ with the server if the value of LBP error $\sin^2(\subsup{\alpha}{k}{(t),\ell})$ is below a threshold, i.e., $\sin^2(\subsup{\alpha}{k}{(t),\ell}) \leq \subsup{\delta}{k}{\mathsf{threshold}}$, where $\subsup{\delta}{k}{\mathsf{threshold}}\in[0,1]$ is a tunable parameter; otherwise it updates the LBG via transmitting the entire gradient to the server.

The details of {\algName} are given in Algorithm~\ref{alg:algo_training}. 
We next conduct convergence analysis for \algName.

\begin{algorithm}[t]
{\footnotesize 
  \caption{\algName: \algFullName}
\label{alg:algo_training}
\begin{algorithmic}[1]
    \Statex \textbf{\underline{Notation:}}
    \Statex $\vec{\theta}{}{(t)}$: global model parameter at global aggregation round $t$.
    \Statex $\vec{\theta}{k}{(t, b)}$: model parameter at worker $k$, at global aggregation round $t$ and local update $b$.
    \Statex $\vec{g}{k}{(t)}$: accumulated gradient at worker $k$ at global aggregation round $t$.
    \Statex $\vec{g}{k}{\ell}$: last full gradient transmitted to server, termed look-back gradient (LBG).
    \Statex $\subsup{\alpha}{k}{(t), \ell}$: phase between the accumulated gradient $\vec{g}{k}{(t)}$ and LBG $\vec{g}{k}{\ell}$, termed look-back phase (LBP).
    \Statex \textbf{\underline{Training at worker $k$:}}
    \State Update local parameters: $\vec{\theta}{k}{(t, 0)} \leftarrow \vec{\theta}{}{(t)}$, and initialize gradient accumulator: $\vec{g}{k}{(t)} \leftarrow \vec{0}{}{}$.
    \For{$b = 0$ to ($\tau$-1)}        
        \State Sample a minibatch of  datapoints $\mathcal{B}_k$ from  $\mathcal{D}_{k}$ and compute $\vec{g}{k}{(t, b)}= \sum_{d\in \mathcal{B}_k} \nabla f_k(\vec{\theta}{k}{(t, b)};d)/|\mathcal{B}_k|$.
        
        \State Update local parameters: $\vec{\theta}{k}{(t, b+1)} \leftarrow \vec{\theta}{k}{(t, b)} - \eta \cdot \vec{g}{k}{(t, b)}$, and accumulate gradient: $\vec{g}{k}{(t)} \leftarrow \vec{g}{k}{(t)} + \vec{g}{k}{(t,b)}$.
    \EndFor
        \State Calculate the LBP error:  $ \sin^2(\subsup{\alpha}{k}{(t), \ell}) = 1 - \left( \indotp{ \vec{g}{k}{(t)} }{ \vec{g}{k}{\ell}}  \big/ \big(\innorm{ \vec{g}{k}{(t)} }{} \times \innorm{ \vec{g}{k}{\ell} }{} \big) \right)^2$ \label{line:lbp}
    \If{ $\sin^2(\subsup{\alpha}{k}{(t), \ell}) \leq \subsup{\delta}{k}{\mathsf{threshold}}$ } \label{line:lbp_cond} \Comment{checking the LBP error}
        \State Send scalar LBC to the server: $\vec{\mu}{k}{(t)} \leftarrow \subsup{\rho}{k}{(t),\ell} = \indotp{ \vec{g}{k}{(t)} }{ \vec{g}{k}{\ell} }/\innorm{ \vec{g}{k}{\ell} }{}^2$.
    \Else \Comment{updating the LBG}
        \State Send actual gradient to the server: $\vec{\mu}{k}{(t)} \leftarrow \vec{g}{k}{(t)}$. \label{line:update1}
        \State Update worker-copy of LBG: $\vec{g}{k}{\ell} \leftarrow \vec{g}{k}{(t)}$. \label{line:update2}
    \EndIf
    \Statex \textbf{\underline{Global update at the aggregation server:}}
    \State Initialize global parameter $\vec{\theta}{}{(0)}$ and broadcast it across workers.
    \For{$t = 0$ to $(T-1)$}
        \State Receive updates from workers $\{ \vec{\mu}{k}{(t)} \}_{k=1}^{K}$.
        \State Update global parameters: $\vec{\theta}{}{(t+1)} \leftarrow \vec{\theta}{}{(t)} - \eta \sum_{k=1}^{K}\omega_k \left[ s_k \cdot  \vec{\mu}{k}{(t)} \cdot \vec{g}{k}{\ell} + (1-s_k)\cdot \vec{\mu}{k}{(t)}\right]$, 
        \Statex ~~~~~~where $s_k$ is an indicator function given by, $s_k = \begin{cases} 1, \quad \text{if } 
        \vec{\mu}{k}{(t)}\text{ is a scalar} \\
        0, \quad \text{otherwise, i.e., if }\vec{\mu}{k}{(t)}\text{ is a vector} \end{cases}$.
        \State Update server-copy of LBGs: $ \vec{g}{k}{\ell} \leftarrow (1-s_k) \vec{\mu}{k}{(t)} + (s_k) \vec{g}{k}{\ell },~~\forall k$.
    \EndFor
\end{algorithmic}
}
\end{algorithm}

\textbf{Assumptions:} It is presumed that the local loss functions are bounded below: $\min_{\vec{\theta}{}{}\in\mathbb{R}^{M}} F_k(\vec{\theta}{}{}) > -\infty$, $\forall k$. Let {\small$\Vert\cdot\Vert$} denote the 2-norm. \shams{Our analysis uses} the following \shams{standard} assumptions \citep{wang2020tackling, friedlander2012hybrid, hosseinalipour2020multi, li2019convergence, stich2018local}:
\begin{enumerate}[leftmargin=5mm]
    \vspace{-2.2mm}
    \item \textit{Smoothness of Local Loss Functions:} Local loss function $F_k: \R{M} \to \R{}$, $\forall k$, is $\beta$-smooth: 
    \eqn{
        \norm{ \nabla F_k(\vec{\theta}{x}{}) - \nabla F_k(\vec{\theta}{y}{}) }{} \leq \beta\norm{ \vec{\theta}{x}{}-\vec{\theta}{y}{} }{}, \\~~\forall \vec{\theta}{x}{}, \vec{\theta}{y}{} \in \R{M}.
        \label{eqn:lipschitz}
        \tag{\bf A1}
    }
    \vspace{-3.1mm}
    \item \textit{SGD Characteristics:} Local gradients $\vec{g}{k}{}(\vec{\theta}{}{})$, $\forall k$, estimated by SGD are unbiased estimators of the true gradients $\nabla F_k(\vec{\theta}{}{})$, and have a bounded variance $\sigma^2 \geq 0$; mathematically:
    \eqn{
        \E{ \,\vec{g}{k}{}(\vec{\theta}{}{})\, }{}= \nabla F_k(\vec{\theta}{}{}) \text{, and } \E{  \norm{\vec{g}{k}{}(\vec{\theta}{}{}) - \nabla F_k(\vec{\theta}{}{})}{}^2}{} \leq \sigma^2, ~~\forall \vec{\theta}{}{}\in \R{M}.
        \label{eqn:sgd_noise}
        \tag{\bf A2}
    }
    \vspace{-3.5mm}
    \item \textit{Bounded Dissimilarity of Local Loss Functions:} For any realization of weights $\mset{\omega_k \geq 0}{k=1}{K}$, where $\sum_{k=1}^{K}\omega_k = 1$, there exist non-negative constants $\Upsilon^2 \geq 1$ and $\Gamma^2\geq 0$ such that
    \vspace{-2mm}
    \eqn{
        \sum_{k=1}^{K} \omega_k \norm{\nabla F_k(\vec{\theta}{}{})}{}^2 \leq \Upsilon^2 \norm{ \sum_{k=1}^{K}\omega_k \nabla F_k(\vec{\theta}{}{}) }{}^2 + \Gamma^2,~~\forall \vec{\theta}{}{}\in \R{M}.
        \label{eqn:bbd_diversity}
        \tag{\bf A3}
        \vspace{-.1mm}
    }
\end{enumerate}

\begin{theorem}
\label{thm:conv_char}
\textbf{(General Convergence Characteristic of {\algName})} 
    Assume \ref{eqn:lipschitz}, \ref{eqn:sgd_noise}, \ref{eqn:bbd_diversity}, and that $\eta\beta \leq \min\big\{ 1/(2\tau), 1/\big(\tau\sqrt{2(1+4\Upsilon^2)}\big) \big\}$.
    If the threshold value in step \ref{line:lbp_cond} of Algorithm~\ref{alg:algo_training} satisfies the condition $\delta^{\mathsf{threshold}}_k \leq \Delta^2/\innorm{\vec{d}{k}{(t)}}{}^2$, $\forall k$ 
    , where $\Delta^2\geq 0$ is a constant, then after $T$ rounds of global aggregations, the performance of {\algName} is characterized by the following upper bound:
    \vspace{-6mm}
    
{\small    
\aligneqn{
    \hspace{-3mm}\frac{1}{T}\hspace{-0.45mm} \sum_{t=0}^{T-1} \hspace{-0.45mm}\E{\norm{\nabla F(\vec{\theta}{}{(t)})}{}^2}{} \hspace{-0.65mm} \leq \hspace{-0.65mm}\frac{8\hspace{-0.35mm}\big[\hspace{-0.45mm}F(\vec{\theta}{}{(0)})\hspace{-0.45mm} -\hspace{-0.45mm} F^\star\hspace{-0.45mm}\big]}{\eta\tau T}\hspace{-0.5mm} \hspace{-0.5mm}+ \hspace{-0.5mm} 16 \Delta^2 +\hspace{-0.5mm} 8\eta\beta\sigma^2  \hspace{-0.5mm}+ \hspace{-0.5mm} 5\eta^2\beta^2\sigma^2(\hspace{-0.2mm}\tau-1\hspace{-0.2mm}) \hspace{-0.55mm}+ \hspace{-0.55mm}20 \eta^2\beta^2 \Gamma^2 \tau(\hspace{-0.2mm}\tau-1\hspace{-0.2mm}).\hspace{-3mm}
    \label{eqn:thm_conv_char}
}}
\end{theorem}
\vspace{-3mm}
\vspace{-0.05in}
\begin{proof} The proof is provided in Appendix~\ref{app:proof_1}.
\end{proof}
\vspace{-0.1in}

The condition on $\delta^{\mathsf{threshold}}_k$ and LBP error $\sin^2(\subsup{\alpha}{k}{(t), \ell})$ in the above theorem implies that to have a fixed bound in~\eqref{eqn:thm_conv_char}, for a fixed $\Delta^2$, a larger gradient norm $\innorm{\vec{d}{k}{(t)}}{}^2$ is associated with a tighter condition on the LBP error $\sin^2(\subsup{\alpha}{k}{(t), \ell})$ and  $\subsup{\delta}{k}{\mathsf{threshold}}$. This is intuitive because a larger gradient norm corresponds to a larger estimation error when the gradient is recovered at the server for a given LBP (see Fig.~\ref{fig:approx_fedl}). In practice, since the gradient norm $\innorm{\vec{d}{k}{(t)}}{}^2$ does not grow to infinity during model training, the condition on LBP error in Theorem~\ref{thm:conv_char}, i.e., $ \sin^2(\subsup{\alpha}{k}{(t), \ell}) \leq \Delta^2/\innorm{\vec{d}{k}{(t)}}{}^2$, can always be satisfied for any $\Delta^2\geq 0$, since transmitting actual gradients of worker $k$ makes $\subsup{\alpha}{k}{(t), \ell}=0$. Given the general convergence behavior in Theorem~\ref{thm:conv_char}, we next obtain a specific choice of step size and an upper bound on $\Delta^2$ for which {\algName} approaches a stationary point of the global loss function \eqref{eqn:opt_obj}.


\vspace{1mm}
\begin{corollary}
\label{corr:stationary_conv}
 \textbf{(Convergence of {\algName} to a Stationary Point)} 
Assuming the conditions of Theorem~\ref{thm:conv_char}, if $\Delta^2 \leq \eta$, where $\eta=1/\sqrt{\tau T}$, then {\algName} converges to a stationary point of the global loss function, with the convergence bound characterized below:
\vspace{-6mm}

{\small
\aligneqn{
\hspace{-2mm}
    \hspace{-0.5mm}\frac{1}{T} \hspace{-0.5mm} \sum_{t=0}^{T-1} \hspace{-0.2mm} \E{\norm{\nabla F(\vec{\theta}{}{(t)})}{}^2}{} \hspace{-0.5mm} \leq \hspace{-0.3mm} \sbigo{\hspace{-0.3mm}\hspace{-0.5mm}\frac{1}{\sqrt{\tau T}}\hspace{-0.5mm}\hspace{-0.3mm}} \hspace{-0.5mm} + \hspace{-0.3mm} \sbigo{\hspace{-0.3mm}\hspace{-0.5mm}\frac{\sigma^2}{\sqrt{\tau T}}\hspace{-0.5mm}\hspace{-0.3mm}} \hspace{-0.5mm}+\hspace{-0.3mm} \sbigo{\hspace{-0.3mm}\hspace{-0.5mm}\frac{1}{\sqrt{\tau T}}\hspace{-0.5mm}\hspace{-0.3mm}} \hspace{-0.5mm}+\hspace{-0.3mm} \sbigo{\hspace{-0.3mm}\hspace{-0.5mm}\frac{\sigma^2(\hspace{-0.2mm}\tau-1\hspace{-0.2mm})}{\tau T}\hspace{-0.5mm}\hspace{-0.3mm}} \hspace{-0.5mm}+\hspace{-0.3mm} \sbigo{\hspace{-0.3mm}\hspace{-0.5mm}\frac{ (\hspace{-0.2mm}\tau-1\hspace{-0.2mm})\Gamma^2}{ T}\hspace{-0.3mm}\hspace{-0.3mm}}\hspace{-0.3mm}.
}\hspace{-6mm}
}
\end{corollary}
\vspace{-3mm}
\begin{proof} The proof is provided in Appendix~\ref{app:corrolary_1}.
\end{proof}
\vspace{-0.1in}

Considering the definition of $\Delta^2$ in Theorem~\ref{thm:conv_char} and the condition imposed on it in Corollary~\ref{corr:stationary_conv}, the LBP error (i.e., $\sin^2(\subsup{\alpha}{k}{(t), \ell})$) should satisfy $\innorm{\vec{d}{k}{(t)}}{}^2 \sin^2(\subsup{\alpha}{k}{(t), \ell}) \leq \eta=1/\sqrt{\tau T}$ for {\algName} to reach a stationary point of the global loss. Since $\innorm{\vec{d}{k}{(t)}}{}^2 $ is bounded during the model training, this condition on the LBP error can always be satisfied by tuning the frequency of the full (actual) gradient transmission, i.e., updating the LBG as in lines \ref{line:update1}\&\ref{line:update2} of Algorithm~\ref{alg:algo_training}.
Specifically, since the correlation across consecutive gradients is high and drops as the gradients are sampled from distant epochs (see Fig.~\ref{fig:prelim_3}), a low value of the LBP error can be obtained via more frequent LBG transmissions to the server.

\textbf{Main Takeaways from Theorem~\ref{thm:conv_char}, Corollary~\ref{corr:stationary_conv}, and Algorithm~\ref{alg:algo_training}:} 
\vspace{-1mm}
\begin{enumerate}[leftmargin=5mm]
    \vspace{-0.25mm}
    \item \textit{Recovering Vanilla-FL Bound:} In~\eqref{eqn:thm_conv_char}, if the LBGs are always propagated by all the devices, we have $\subsup{\alpha}{k}{(t), \ell} = 0$, $\forall k$, and thus $\Delta^2 = 0$ satisfies the condition on the LBP error. Then,~\eqref{eqn:thm_conv_char} recovers the bound for vanilla FL \citep{wang2020tackling,stich2018local,wang2018cooperative}. 
    \vspace{-0.25mm}
    \item \textit{Recovering Centralized SGD Bound:} In~\eqref{eqn:thm_conv_char}, if the LBGs are always propagated by all the devices, i.e., $\Delta^2 = 0$, the local dataset sizes are equal, i.e., $w_k=1/K, \forall k$, and $\tau=1$, then \eqref{eqn:thm_conv_char} recovers the bound for centralized SGD, e.g., see \citet{friedlander2012hybrid}.
    \vspace{-0.25mm}
    \item \textit{Unifying Algorithm~\ref{alg:algo_training} and Theorem~\ref{thm:conv_char}:}\label{tk:3} The value of $\Delta^2$ in \eqref{eqn:thm_conv_char} is determined by the value of the LBP error $\sin^2(\subsup{\alpha}{k}{(t),\ell})$, which is also reflected in step \ref{line:lbp_cond} of Algorithm~\ref{alg:algo_training}. This suggests that the performance improves when the allowable threshold on $\sin^2(\subsup{\alpha}{k}{(t),\ell})$ is decreased (i.e., smaller $\Delta^2$), which is the motivation behind introducing the tunable threshold $\subsup{\delta}{k}{\mathsf{threshold}}$ in our algorithm.
    \vspace{-0.25mm}
    \item \textit{Effect of LPB Error on Convergence:} As the value of $\sin^2(\subsup{\alpha}{k}{(t),\ell})$ increases, the term in \eqref{eqn:thm_conv_char} containing $\Delta^2$ will start diverging (it can become the same order as the gradient $\innorm{\vec{d}{k}{(t)}}{}^2$). The condition in Corollary~\ref{corr:stationary_conv} on $\Delta^2$ avoids this scenario, achieving convergence to a stationary point. 
    \vspace{-0.3mm}
    \item \textit{Performance vs. Communication Overhead Trade-off:}\label{tk:5} Considering step \ref{line:lbp_cond} of Algorithm~\ref{alg:algo_training}, increasing the tolerable threshold on the LBP error increases the chance of transmitting a scalar (i.e., LBC) instead of the entire gradient to the server from each worker, leading to communication savings. However, as seen in Theorem~\ref{thm:conv_char} and the condition on $\Delta^2$ in Corollary~\ref{corr:stationary_conv}, the threshold on the LBP error cannot be increased arbitrarily since the {\algName} may show diverging behavior. 
\end{enumerate}


\vspace{-3mm}
\section{Experiments}
\vspace{-2mm}
\label{sec:expt}

\textbf{Model Settings.} We run experiments on several NN models and datasets. Specifically, we consider: \textbf{S1:} CNN on FMNIST, MNIST, CelebA, and CIFAR-10, \textbf{S2:} FCN on FMNIST and MNIST, and \textbf{S3:} ResNet18 on FMNIST, MNIST, CelebA, CIFAR-10 and CIFAR-100 for both independently and identically distributed (iid) and non-iid data distributions. We present results of \textbf{S1} (on non-iid data) in this section and defer the rest (including \textbf{S1} on iid data and U-Net on PascalVOC) to Appendix~\ref{app:addl_lbgm_expt}.

\begin{figure}[t]
\centering
\begin{minipage}{.48\textwidth}
  \centering
  \centerline{\includegraphics[width=1.0\textwidth]{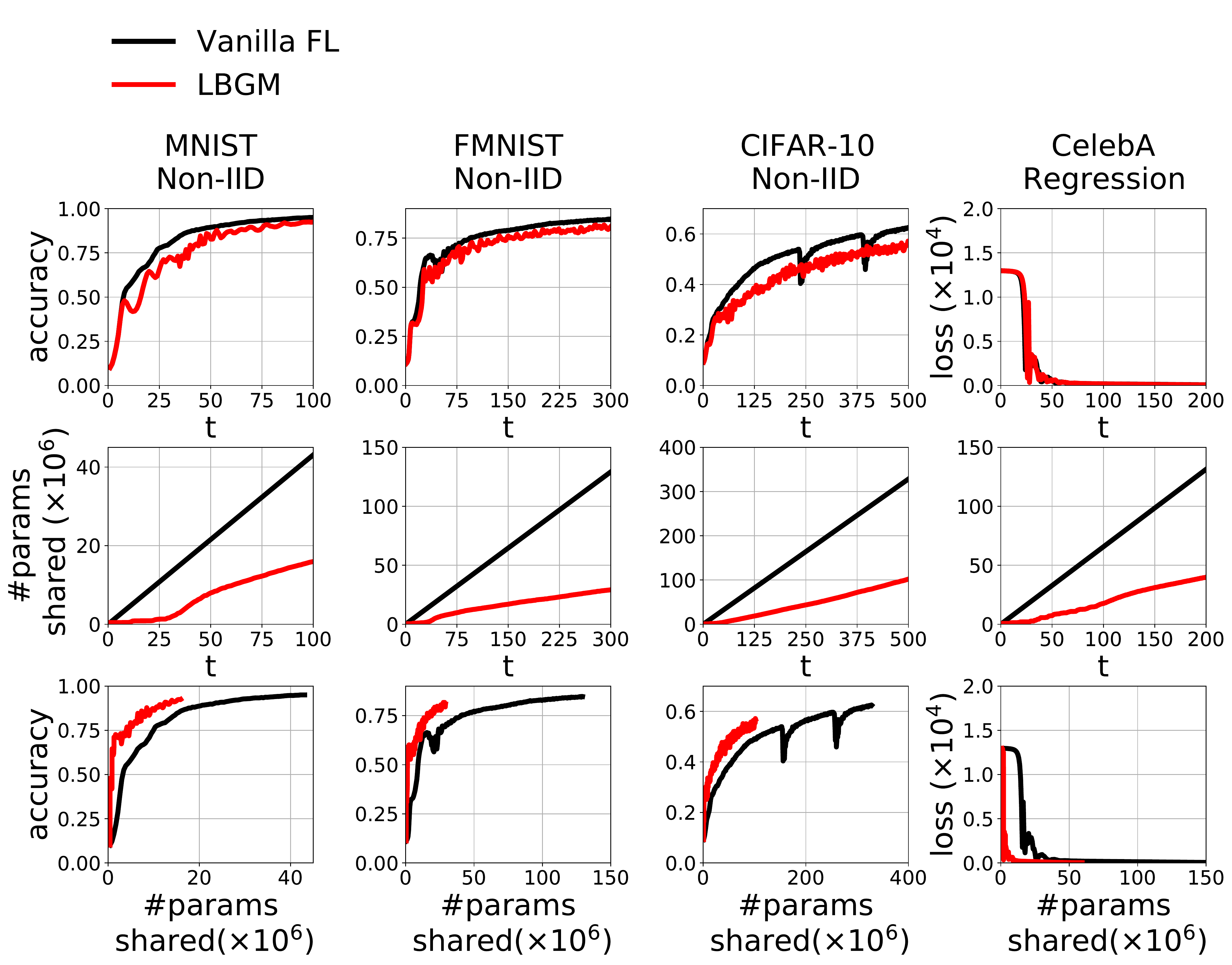}}
  \vspace{-3.5mm}
  \caption{\small{\textit{{\algName} as a Standalone Algorithm}. Irrespective of the dataset/data configuration across workers, {\algName} consistently outperforms vanilla FL in terms of the total parameters shared (middle row) while achieving comparable accuracy (top row). The bottom row shows accuracy vs. \# parameters shared.}}
  \label{fig:standalone}
\vspace{-2mm}
\end{minipage}~~~
\begin{minipage}{.48\textwidth}
  \centering
  \centerline{\includegraphics[width=1.0\textwidth]{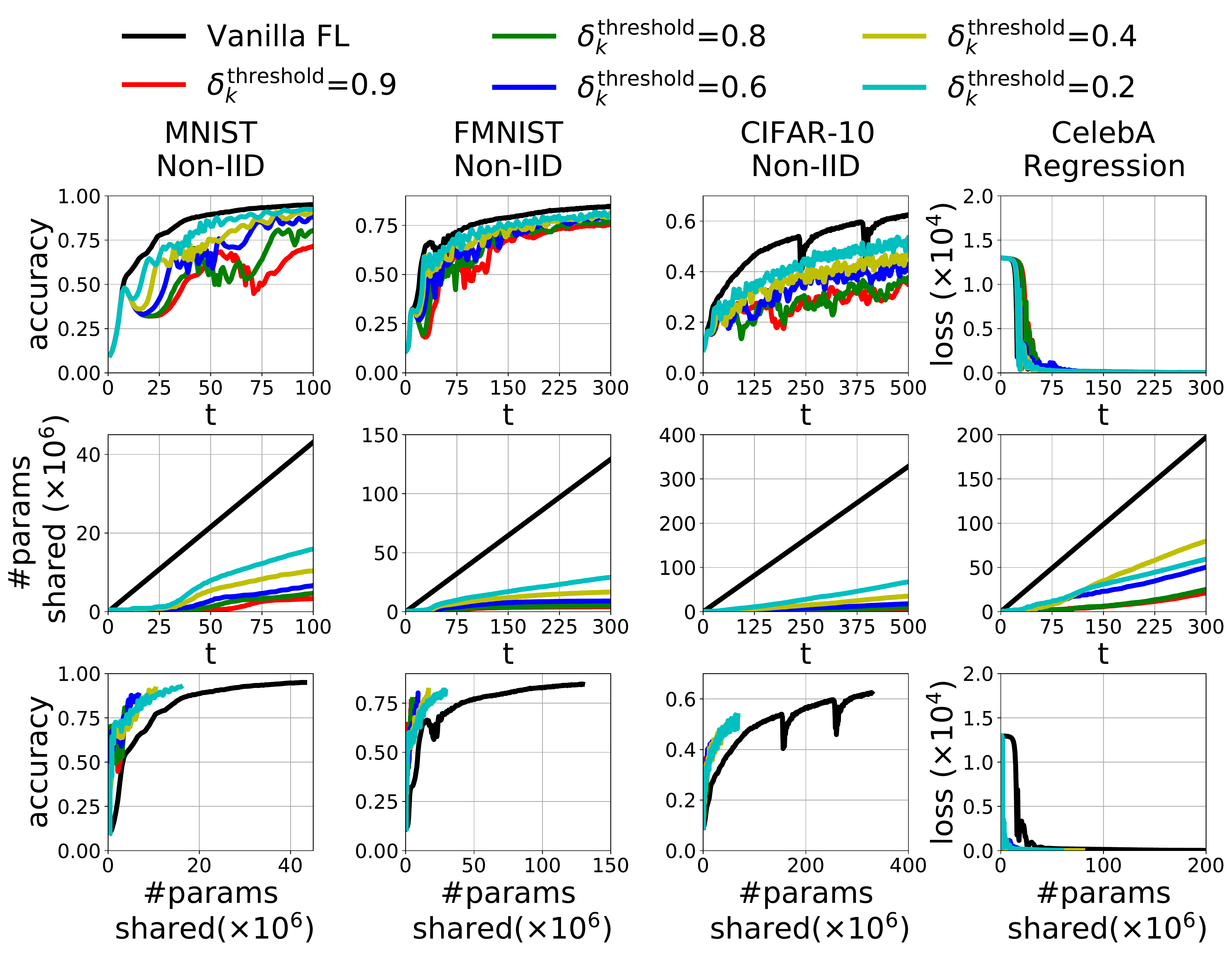}}
    \vspace{-3mm}
  \caption{\small{\textit{Effect of $\delta_k^{\mathsf{threshold}}$ on {\algName}}. As $\delta_k^{\mathsf{threshold}}$ decreases, the training may become unstable. For larger values of $\delta_k^{\mathsf{threshold}}$, {\algName} achieves communication benefits (middle row) while maintaining a performance identical to vanilla FL (top row). The bottom row shows accuracy vs. \# parameters shared.}}
  \label{fig:rho_effect}
\vspace{-3mm}
\end{minipage}
\vspace{-3mm}
\end{figure}

\vspace{-1mm}
\textbf{Properties Studied.} We specifically focus on four properties of {\algName}: \textbf{P1:} the benefits of gradient recycling by \textit{{\algName} as a standalone algorithm}, \textbf{P2:} the \textit{effect of $\subsup{\delta}{k}{\mathsf{threshold}}$} on {\algName} from Theorem~\ref{thm:conv_char}, \textbf{P3:} practical capabilities of \textit{{\algName} as a general plug-and-play algorithm} that can be stacked on top of other gradient compression techniques 
in FL training, and finally \textbf{P4:} generalizability of \textit{{\algName} to distributed learning} frameworks, e.g., multi-processor or multi-GPU ML systems.

\vspace{-1mm}
\textbf{Baselines.} For \textbf{P1} and \textbf{P2}, we compare {\algName} with vanilla FL. For \textbf{P3}, we stack {\algName} on top of top-K and ATOMO \citep{wang2018atomo}, two state-of-the-art techniques for sparsification and low-rank approximation-based gradient compression, respectively. For \textbf{P4}, we stack {\algName} on top of SignSGD \citep{bernstein2018signsgd}, a state-of-the-art method in gradient compression for distributed learning.

\vspace{-1mm}
\textbf{Implementation Details.} We consider an FL system consisting of 100 workers. We consider both the iid and non-iid data distributions among the workers. {Under the iid setting, each worker has training data from all the labels, while under the non-iid setting each worker has training data only from a subset of all labels (e.g., from 3 of 10 classes in MNIST/FMNIST).} The workers train with mini-batch sizes ranging from 128 to 512 based on the choice of dataset. We implement {\algName} with uniform $\subsup{\delta}{k}{\mathsf{threshold}}$ across workers. We also use error feedback \citep{karimireddy2019error} as standard only if top-K sparsification is used in the training. The FL system is simulated using PyTorch \citep{paszke2019pytorch} and PySyft \citep{ryffel2018generic} and trained on a 48GB Tesla-P100 GPU with 128GB RAM. All of our code and hyperparameters are available at \url{https://github.com/shams-sam/FedOptim}. \shams{Appendix~\ref{app:hyperparams} details the process of hyperparameter selection for the baselines.}

\vspace{-1mm}
\textbf{Complexity.} Compared to other gradient compression techniques, the processing overhead introduced by {\algName} is negligible. Considering Algorithm~\ref{alg:algo_training}, the calculation of LBCs and LBP errors involves inner products and division of scalars, while reconstruction of LBG-based gradient approximations at the server is no more expensive than the global aggregation step: since the global aggregation step requires averaging of local model parameters, it can be combined with gradient reconstruction. This also holds for {\algName} as a plug-and-play algorithm, as top-K and ATOMO \citep{wang2018atomo} introduce considerable computation overhead. In particular, {\algName} has $\mathcal{O}(M)$ complexity, where $M$ is the dimension of the NN parameter, which is inexpensive to plug on top of top-K ($\mathcal{O}(M\log M)$), ATOMO \citep{wang2018atomo} ($\mathcal{O}(M^2)$), and SignSGD \citep{bernstein2018signsgd} ($\mathcal{O}(M)$) methods. \shams{The corresponding space complexity of {\algName} for the server and devices is discussed in Appendix~\ref{app:storage}.}

\vspace{-1mm}
\textbf{{\tt \algName} as a Standalone Algorithm.} We first evaluate the effect of gradient recycling by {\algName} in FL. Fig.~\ref{fig:standalone} depicts the accuracy/loss values (top row) and total floating point parameters transferred over the system (middle row) across training epochs for $\smash{\subsup{\delta}{k}{\text{threshold}} = 0.2,~ \forall k}$ on diferent datasets. The parameters transferred indicates the communication overhead, leading to a corresponding performance vs. efficiency tradeoff (bottom row). For each dataset, we observe that {\algName} reduces communication overhead on the order of $10^7$ floating point parameters per worker. \shams{Similar results on other datasets and NN models are deferred to Appendix~\ref{app:standalone_expt}. We also consider {\algName} under device sampling (see Algorithm~\ref{alg:sampling} in Appendix~\ref{ssec:sampling_algo}) and present the results in Appendix~\ref{app:sampling_expt}, which are qualitatively similar.}

\begin{figure}[t]
  \centering
      \centerline{\includegraphics[width=1.0\textwidth]{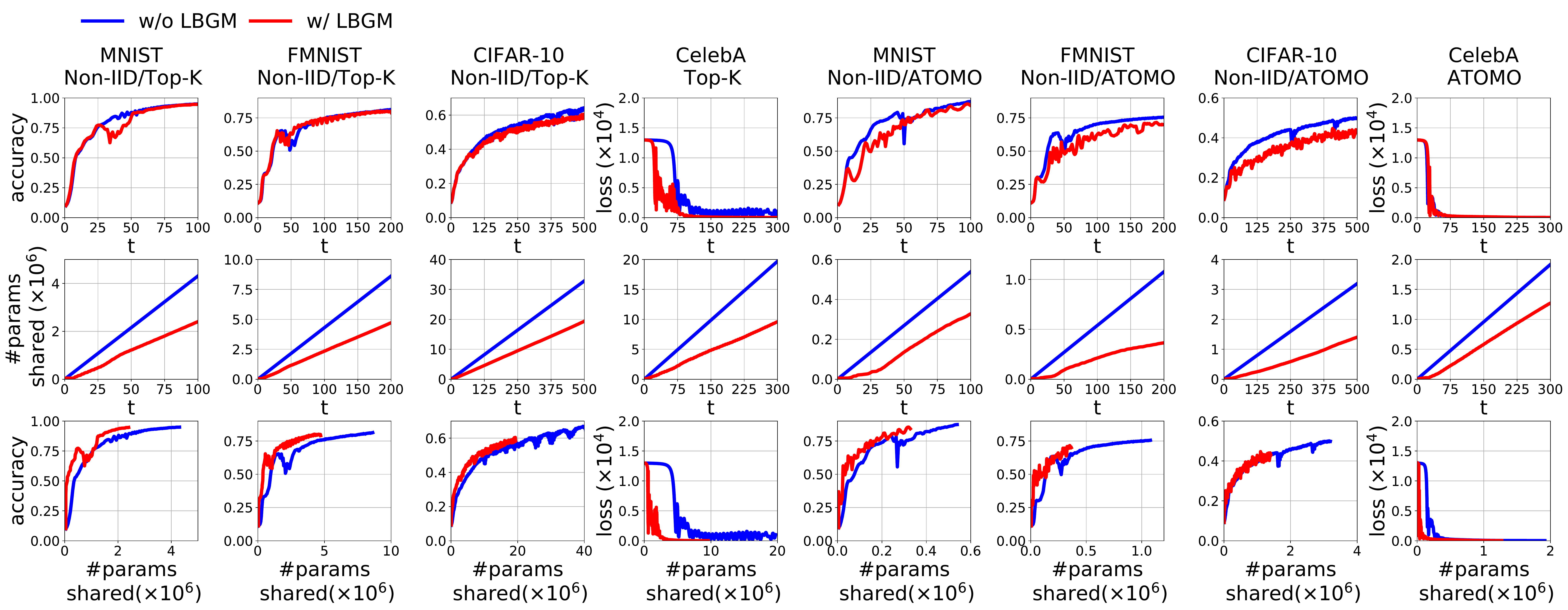}}
      \vspace{-4mm}
  \caption{\small{\textit{{\tt \algName} as a Plug-and-Play Algorithm}. {\algName} obtains substantial communication benefits when implemented on top of existing gradient compression techniques by exploiting the rank-characteristics of the gradient-space. Top-K and ATOMO are known to achieve state-of-the-art performance of their respective domains of sparsification and low-rank approximation respectively. 
  }}
  \label{fig:plugnplay}
  \vspace{-5mm}
\end{figure}

\vspace{-1mm}
\textbf{Effect of $\subsup{\delta}{k}{\mathsf{threshold}}$ on Accuracy vs. Communication Savings.} \shams{In Fig.~\ref{fig:standalone}, the drops in accuracy for the corresponding communication savings are small except for on CIFAR-10. The $14\%$ reduction in accuracy here is a result of the hyperparameter setting $\smash{\subsup{\delta}{k}{\text{threshold}} = 0.2}$. As noted in takeaway~\ref{tk:3} in Sec.~\ref{sec:method}, a decrease in the allowable threshold on the LBP error improves the accuracy; the effect of threshold value is controlled by changing $\subsup{\delta}{k}{\mathsf{threshold}}$ values in Algorithm~\ref{alg:algo_training}. Thus, we can improve the accuracy by lowering $\subsup{\delta}{k}{\text{threshold}}$: for $\subsup{\delta}{k}{\text{threshold}} = 0.05$, the accuracy drops by $4\%$ only while still retaining a communication saving of $55\%$, and for $\subsup{\delta}{k}{\text{threshold}} = 0.01$, we get a $22\%$ communication saving for a negligible drop in accuracy (by only $0.01\%$).} In Fig.~\ref{fig:rho_effect}, we analyze {\algName} under different $\subsup{\delta}{k}{\mathsf{threshold}}$ values for different datasets. A drop in model performance can be observed as we increase $\subsup{\delta}{k}{\mathsf{threshold}}$, which is accompanied by an increase in communication savings. This is consistent with takeaway \ref{tk:5} from Sec.~\ref{sec:method}, i.e., while a higher threshold requires less frequent updates of the LBGs, it reduces the convergence speed. Refer to Appendix~\ref{app:rho_effect} for additional results.

\begin{wrapfigure}{r}{0.5\textwidth}
\vspace{-16pt}
  \begin{center}
    \includegraphics[width=0.5\textwidth]{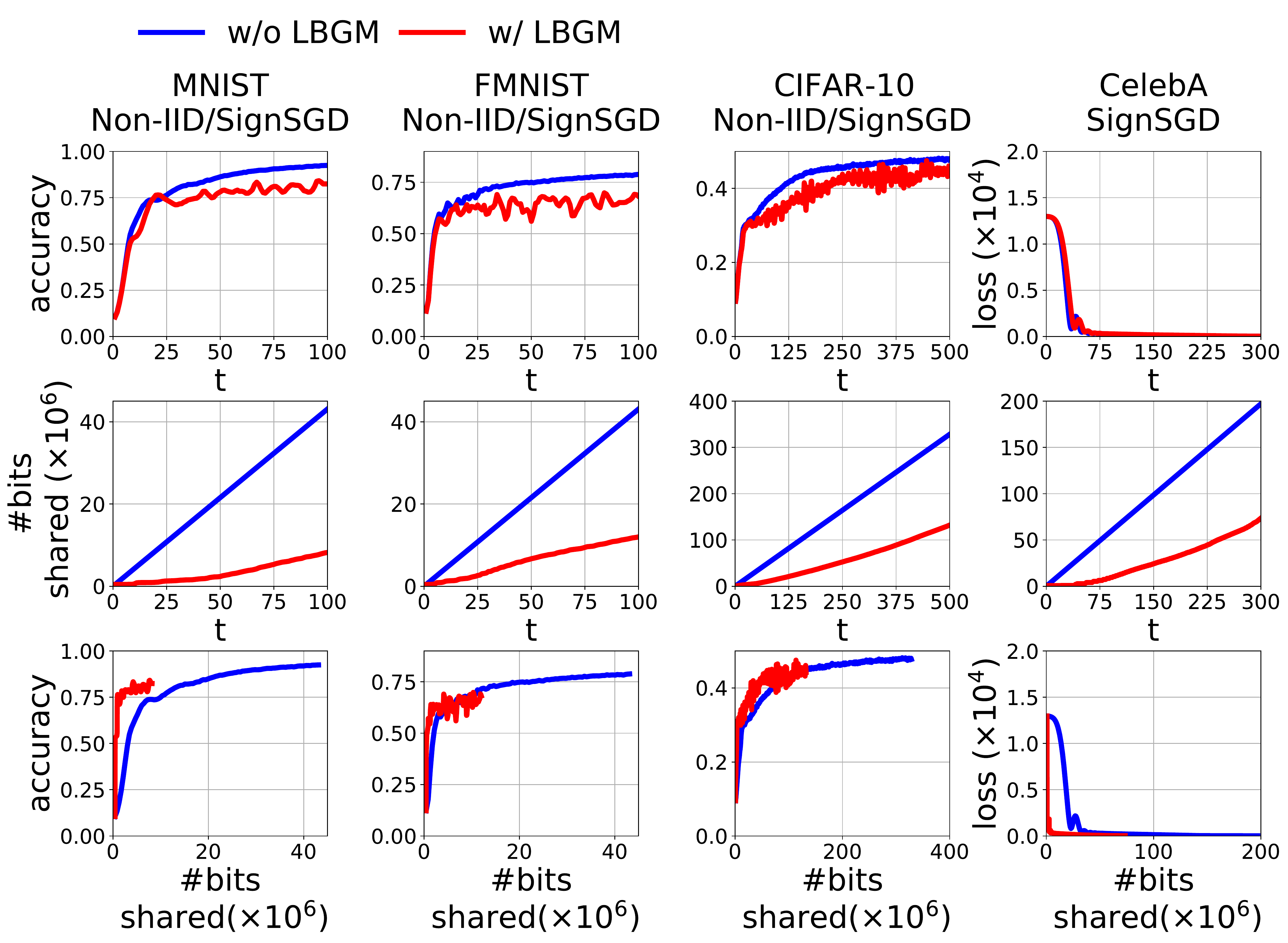}
  \end{center}
  \vspace{-16pt}
  \caption{\small{Application of {\algName} as a plug-and-play algorithm on top of SignSGD in distributed training.}}
  \label{fig:dist_training}  
  \vspace{-9pt}
\end{wrapfigure}

\vspace{-1mm}
\textbf{{\tt \algName} as a Plug-and-Play Algorithm.} For the plug-and-play setup, {\algName} follows the same steps as in Algorithm~\ref{alg:algo_training}, with the slight modification that the output of gradient compression techniques, top-K and ATOMO, are used in place of accumulated gradients $\vec{g}{k}{(t)}$ and LBGs $\vec{g}{k}{\ell}$, $\forall k$. In Fig.~\ref{fig:plugnplay}, we see that {\algName} adds on top of existing communication benefits of both top-K and ATOMO, on the order of $10^6$ and $10^5$ floating point parameters shared per worker, respectively \shams{($30-70\%$ savings across the datasets)}. \shams{The bottom row shows the accuracy/loss improvements that can be obtained for the same number of parameters transferred.} While top-K and ATOMO compress gradients through approximation, they do not alter the underlying low-rank characteristics of the gradient-space. {\algName} exploits this property to obtain substantial communication savings on top of these algorithms. Refer to Appendix~\ref{app:plugnplay_expt} for additional experiments.

\vspace{-1mm}
\textbf{Generalizability of {\algName} to Distributed Training.} {\algName} can be applied to more general distributed gradient computation settings, e.g., multi-core systems.
While \shams{heterogeneous (non-iid)} data distributions are not as much of a consideration in these settings as they are in FL (since data can be transferred/allocated across nodes), there is research interest in minimizing parameter exchange among nodes to reduce communication latency.
SignSGD \citep{bernstein2018signsgd} is known to reduce the communication requirements by several order of magnitude by converting floating-point parameters to sign bit communication. In Fig.~\ref{fig:dist_training}, we apply {\algName} as a plug-and-play addition on top of SignSGD and find that {\algName} further reduces the overall bits transferred by SignSGD on the order of $10^7$ bits \shams{($\smash{60-80\%}$ savings across the datasets)}. Refer to Appendix~\ref{app:general_expt} for additional experiments.


\vspace{-4mm}
\section{Related Work}
\vspace{-2mm}

\vspace{-1.2mm}
\shams{\textbf{NN Overparameterization Analysis.} Several prior works on NN overparameterization have focused on Hessian-based analysis.} \citet{sagun2016eigenvalues, sagun2017empirical} divide the eigenspace of the Hessians into two parts: bulk 
and edges,
and show that increasing network complexity only affects the bulk component. \citet{ghorbani2019investigation} argues that the existence of large isolated eigenvalues in the Hessian eigenspace is correlated with slow convergence. \shams{\citet{gur2018gradient} studies the overlap of gradients with Hessians and shows the Hessian edge space remains invariant during training, and thus that SGD occurs in low-rank subspaces. \citet{gur2018gradient} also suggest that the edge space cardinality is equal to the number of classification classes, which does not align with our observations in Sec.~\ref{sec:explore}.
In contrast to these, our work explores the low-rank property by studying the PCA of the gradient-space directly. \cite{li2021low}, a contemporary of ours, employs the spectral decomposition of the NN gradient space to improve centralized SGD training time. Our methodology based on Hypothesis \eqref{hypothesis_2} is more suitable for FL since having the resource-constrained workers/devices execute spectral decomposition as a component of the training process would add significant computational burden.}

\vspace{-1.2mm}
\shams{The partitioning of the gradient subspace has also been observed in the domain of continual learning~\citep{chaudhry2020continual,saha2020gradient}. However, the subspace addressed in continual learning is the one spanned by gradient with respect to data samples for the final model, which is different than the subspace we consider, i.e., the subspace of gradient updates generated during SGD epochs.}

\vspace{-1.2mm}
\textbf{Gradient Compression.} Gradient compression techniques can be broadly categorized into (i) sparsification~\citep{wangni2017gradient,sattler2019robust}, (ii) quantization~\citep{seide20141,alistarh2016qsgd}, and (iii) low-rank approximations~\citep{wang2018atomo,vogels2019powersgd, albasyoni2020optimal,haddadpour2021federated}. Our work falls under the third category, where prior works have aimed to decompose large gradient matrices as an outer product of smaller matrices to reduce communication cost. This idea was also proposed in~\citet{konevcny2016federated}, one of the pioneering works in FL. \shams{While these prior works study the low-rank property in the context of gradient compression during a single gradient transfer step, 
our work explores the low rank property of the
gradients generated across successive gradient epochs during FL. Existing techniques for gradient compression can also benefit from employing {\algName} during FL, as we show in our experiments for top-K, ATOMO, and SignSGD.} 

\vspace{-1.2mm}
\textbf{Model Compression.} Model compression techniques have also been proposed to reduce NN complexity, e.g., model distillation \citep{ba2013deep, hinton2015distilling}, model pruning \citep{lecun1990optimal, hinton2015distilling}, 
and parameter clustering
\citep{son2018clustering,cho2021dkm}. 
\citep{li2020lotteryfl} extends the lottery ticket hypothesis to the FL setting. \shams{These methods have the potential to be employed in conjunction with LBGM to reduce the size of the LBGs stored at the server.}

\vspace{-1.2mm}
\textbf{FL Communication Efficiency.} Other techniques have focused on reducing the aggregation frequency in FL. \citet{hosseinalipour2020multi,lin2021two} use peer-to-peer local network communication, while SloMo \citep{wang2019slowmo} uses momentum to delay the global aggregations.


\vspace{-3mm}
\section{Discussion \& Conclusions}
\vspace{-2mm}
\shams{In this paper, we explored the effect of overparameterization in NN optimization through the PCA of the gradient-space, and employed this to optimize the accuracy vs. communication tradeoff in FL.}
We proposed the {\algName} algorithm, which uses \shams{our hypothesis that PGDs can be approximated using a subset of gradients generated across SGD epochs}, and recycles previously generated gradients at the devices to represent the newly generated gradients. {\algName} reduces communication overhead in FL by several orders of magnitude by replacing the transmission of gradient parameter vectors with a single scalars from each device. We theoretically characterized the convergence behavior of {\algName} algorithm and experimentally substantiated our claims \shams{on several datasets and models}. Furthermore, we showed that {\algName} can be extended to further reduce latency of communication in large distributed training systems by plugging {\algName} on top of other gradient compression techniques. More generally, our work gives a novel insight to designing a class of techniques based on ``Look-back Gradients'' that can be used in distributed machine learning systems \shams{to enhance communication savings}.


{\small
\bibliographystyle{iclr2022_conference}
\bibliography{ms}

\begin{thebibliography}{57}
\providecommand{\natexlab}[1]{#1}
\providecommand{\url}[1]{\texttt{#1}}
\expandafter\ifx\csname urlstyle\endcsname\relax
  \providecommand{\doi}[1]{doi: #1}\else
  \providecommand{\doi}{doi: \begingroup \urlstyle{rm}\Url}\fi

\bibitem[tur(2020)]{turing2020nlg}
{Turing-NLG: A 17-billion-parameter language model by Microsoft}.
\newblock
  https://www.microsoft.com/en-us/research/blog/turing-nlg-a-17-billion-parameter-language-model-by-microsoft/,
  2020.

\bibitem[Albasyoni et~al.(2020)Albasyoni, Safaryan, Condat, and
  Richt{\'a}rik]{albasyoni2020optimal}
Alyazeed Albasyoni, Mher Safaryan, Laurent Condat, and Peter Richt{\'a}rik.
\newblock {Optimal Gradient Compression for Distributed and Federated
  Learning}.
\newblock \emph{arXiv preprint arXiv:2010.03246}, 2020.

\bibitem[Alistarh et~al.(2017)Alistarh, Grubic, Li, Tomioka, and
  Vojnovic]{alistarh2016qsgd}
Dan Alistarh, Demjan Grubic, Jerry Li, Ryota Tomioka, and Milan Vojnovic.
\newblock {QSGD: Communication-Efficient SGD via Gradient Quantization and
  Encoding.}
\newblock In \emph{Neural Information Processing Systems (NeurIPS)}, 2017.

\bibitem[Azam et~al.(2021)Azam, Kim, Hosseinalipour, Brinton, Joe-Wong, and
  Bagchi]{azam2020towards}
Sheikh~Shams Azam, Taejin Kim, Seyyedali Hosseinalipour, Christopher Brinton,
  Carlee Joe-Wong, and Saurabh Bagchi.
\newblock {Towards Generalized and Distributed Privacy-preserving
  Representation Learning}.
\newblock \emph{arXiv preprint arXiv:2010.01792}, 2021.

\bibitem[Ba \& Caruana(2014)Ba and Caruana]{ba2013deep}
Lei~Jimmy Ba and Rich Caruana.
\newblock {Do Deep Nets really need to be Deep?}
\newblock In \emph{Neural Information Processing Systems (NeurIPS)}, 2014.

\bibitem[Bernstein et~al.(2018)Bernstein, Wang, Azizzadenesheli, and
  Anandkumar]{bernstein2018signsgd}
Jeremy Bernstein, Yu-Xiang Wang, Kamyar Azizzadenesheli, and Animashree
  Anandkumar.
\newblock {SignSGD: Compressed Optimisation for Non-convex Problems}.
\newblock In \emph{International Conference on Machine Learning (ICML)}, 2018.

\bibitem[Brown et~al.(2020)Brown, Mann, Ryder, Subbiah, Kaplan, Dhariwal,
  Neelakantan, Shyam, Sastry, Askell, Agarwal, Herbert-Voss, Krueger, Henighan,
  Child, Ramesh, Ziegler, Wu, Winter, Hesse, Chen, Sigler, Litwin, Gray, Chess,
  Clark, Berner, McCandlish, Radford, Sutskever, and Amodei]{brown2020language}
Tom Brown, Benjamin Mann, Nick Ryder, Melanie Subbiah, Jared~D Kaplan, Prafulla
  Dhariwal, Arvind Neelakantan, Pranav Shyam, Girish Sastry, Amanda Askell,
  Sandhini Agarwal, Ariel Herbert-Voss, Gretchen Krueger, Tom Henighan, Rewon
  Child, Aditya Ramesh, Daniel Ziegler, Jeffrey Wu, Clemens Winter, Chris
  Hesse, Mark Chen, Eric Sigler, Mateusz Litwin, Scott Gray, Benjamin Chess,
  Jack Clark, Christopher Berner, Sam McCandlish, Alec Radford, Ilya Sutskever,
  and Dario Amodei.
\newblock {Language Models are Few-Shot Learners}.
\newblock In \emph{Neural Information Processing Systems (NeurIPS)}, 2020.

\bibitem[Chaudhry et~al.(2020)Chaudhry, Khan, Dokania, and
  Torr]{chaudhry2020continual}
Arslan Chaudhry, Naeemullah Khan, Puneet Dokania, and Philip Torr.
\newblock {Continual Learning in Low-rank Orthogonal Subspaces}.
\newblock In \emph{Neural Information Processing Systems (NeurIPS)}, 2020.

\bibitem[Cho et~al.(2021)Cho, Vahid, Adya, and Rastegari]{cho2021dkm}
Minsik Cho, Keivan~A Vahid, Saurabh Adya, and Mohammad Rastegari.
\newblock {DKM: Differentiable K-Means Clustering Layer for Neural Network
  Compression}.
\newblock \emph{arXiv preprint arXiv:2108.12659}, 2021.

\bibitem[Cortes \& Vapnik(1995)Cortes and Vapnik]{cortes1995support}
Corinna Cortes and Vladimir Vapnik.
\newblock {Support-vector Networks}.
\newblock \emph{Machine Learning}, 20\penalty0 (3):\penalty0 273--297, 1995.

\bibitem[Everingham et~al.(2010)Everingham, Van~Gool, Williams, Winn, and
  Zisserman]{everingham2010pascal}
Mark Everingham, Luc Van~Gool, Christopher~KI Williams, John Winn, and Andrew
  Zisserman.
\newblock {The Pascal Visual Object Classes (VOC) Challenge}.
\newblock \emph{International Journal of Computer Vision (IJCV)}, 88\penalty0
  (2):\penalty0 303--338, 2010.

\bibitem[Frankle \& Carbin(2019)Frankle and Carbin]{frankle2018lottery}
Jonathan Frankle and Michael Carbin.
\newblock {The Lottery Ticket Hypothesis: Finding Sparse, Trainable Neural
  Networks}.
\newblock In \emph{International Conference on Learning Representations
  (ICLR)}, 2019.

\bibitem[Friedlander \& Schmidt(2012)Friedlander and
  Schmidt]{friedlander2012hybrid}
Michael~P Friedlander and Mark Schmidt.
\newblock {Hybrid Deterministic-stochastic Methods for Data Fitting}.
\newblock \emph{SIAM Journal on Scientific Computing}, 34\penalty0
  (3):\penalty0 A1380--A1405, 2012.

\bibitem[Ghorbani et~al.(2019)Ghorbani, Krishnan, and
  Xiao]{ghorbani2019investigation}
Behrooz Ghorbani, Shankar Krishnan, and Ying Xiao.
\newblock {An Investigation into Neural Net Optimization via Hessian Eigenvalue
  Density}.
\newblock In \emph{International Conference on Machine Learning (ICML)}, 2019.

\bibitem[Gur-Ari et~al.(2018)Gur-Ari, Roberts, and Dyer]{gur2018gradient}
Guy Gur-Ari, Daniel~A Roberts, and Ethan Dyer.
\newblock {Gradient Descent Happens in a Tiny Subspace}.
\newblock \emph{arXiv preprint arXiv:1812.04754}, 2018.

\bibitem[Haddadpour et~al.(2021)Haddadpour, Kamani, Mokhtari, and
  Mahdavi]{haddadpour2021federated}
Farzin Haddadpour, Mohammad~Mahdi Kamani, Aryan Mokhtari, and Mehrdad Mahdavi.
\newblock {Federated Learning with Compression: Unified Analysis and Sharp
  Guarantees}.
\newblock In \emph{International Conference on Artificial Intelligence and
  Statistics (AISTATS)}, 2021.

\bibitem[Han et~al.(2015)Han, Pool, Tran, and Dally]{han2015learning}
Song Han, Jeff Pool, John Tran, and William~J Dally.
\newblock {Learning both Weights and Connections for Efficient Neural
  Networks}.
\newblock In \emph{Neural Information Processing Systems (NeurIPS)}, 2015.

\bibitem[He et~al.(2016)He, Zhang, Ren, and Sun]{he2016deep}
Kaiming He, Xiangyu Zhang, Shaoqing Ren, and Jian Sun.
\newblock {Deep Residual Learning for Image Recognition}.
\newblock In \emph{IEEE Conference on Computer Vision and Pattern Recognition
  (CVPR)}, pp.\  770--778, 2016.

\bibitem[Hinton et~al.(2015)Hinton, Vinyals, and Dean]{hinton2015distilling}
Geoffrey Hinton, Oriol Vinyals, and Jeff Dean.
\newblock {Distilling the Knowledge in a Neural Network}.
\newblock \emph{arXiv preprint arXiv:1503.02531}, 2015.

\bibitem[Hosseinalipour et~al.(2020)Hosseinalipour, Azam, Brinton, Michelusi,
  Aggarwal, Love, and Dai]{hosseinalipour2020multi}
Seyyedali Hosseinalipour, Sheikh~Shams Azam, Christopher~G Brinton, Nicolo
  Michelusi, Vaneet Aggarwal, David~J Love, and Huaiyu Dai.
\newblock {Multi-stage Hybrid Federated Learning over Large-scale Wireless Fog
  Networks}.
\newblock \emph{arXiv preprint arXiv:2007.09511}, 2020.

\bibitem[Huang et~al.(2017)Huang, Liu, Van Der~Maaten, and
  Weinberger]{huang2017densely}
Gao Huang, Zhuang Liu, Laurens Van Der~Maaten, and Kilian~Q Weinberger.
\newblock {Densely Connected Convolutional Networks}.
\newblock In \emph{IEEE Conference on Computer Vision and Pattern Recognition
  (CVPR)}, pp.\  4700--4708, 2017.

\bibitem[Karimireddy et~al.(2019)Karimireddy, Rebjock, Stich, and
  Jaggi]{karimireddy2019error}
Sai~Praneeth Karimireddy, Quentin Rebjock, Sebastian Stich, and Martin Jaggi.
\newblock {Error Feedback fixes SignSGD and other Gradient Compression
  Schemes}.
\newblock In \emph{International Conference on Machine Learning (ICML)}, 2019.

\bibitem[Kone{\v{c}}n{\`y} et~al.(2016)Kone{\v{c}}n{\`y}, McMahan, Yu,
  Richt{\'a}rik, Suresh, and Bacon]{konevcny2016federated}
Jakub Kone{\v{c}}n{\`y}, H~Brendan McMahan, Felix~X Yu, Peter Richt{\'a}rik,
  Ananda~Theertha Suresh, and Dave Bacon.
\newblock {Federated Learning: Strategies for Improving Communication
  Efficiency}.
\newblock In \emph{Neural Information Processing Systems (NeurIPS)}, 2016.

\bibitem[Krizhevsky et~al.(2009)Krizhevsky, Hinton,
  et~al.]{krizhevsky2009learning}
Alex Krizhevsky, Geoffrey Hinton, et~al.
\newblock {Learning Multiple Layers of Features from Tiny Images}.
\newblock 2009.

\bibitem[LeCun \& Cortes(2010)LeCun and Cortes]{lecun2010mnist}
Yann LeCun and Corinna Cortes.
\newblock {MNIST Handwritten Digit Database}, 2010.
\newblock URL \url{http://yann.lecun.com/exdb/mnist/}.

\bibitem[LeCun et~al.(1990)LeCun, Denker, and Solla]{lecun1990optimal}
Yann LeCun, John~S Denker, and Sara~A Solla.
\newblock {Optimal Brain Damage}.
\newblock In \emph{Neural Information Processing Systems (NeurIPS)}, pp.\
  598--605, 1990.

\bibitem[Li et~al.(2020)Li, Sun, Wang, Duan, Li, Chen, and Li]{li2020lotteryfl}
Ang Li, Jingwei Sun, Binghui Wang, Lin Duan, Sicheng Li, Yiran Chen, and Hai
  Li.
\newblock {LotteryFL: Personalized and Communication-efficient Federated
  Learning with Lottery Ticket Hypothesis on Non-iid Datasets}.
\newblock \emph{arXiv preprint arXiv:2008.03371}, 2020.

\bibitem[Li et~al.(2017)Li, Kadav, Durdanovic, Samet, and Graf]{li2016pruning}
Hao Li, Asim Kadav, Igor Durdanovic, Hanan Samet, and Hans~Peter Graf.
\newblock {Pruning Filters for Efficient Convnets}.
\newblock In \emph{International Conference on Learning Representations
  (ICLR)}, 2017.

\bibitem[Li et~al.(2021)Li, Tan, Tao, Liu, and Huang]{li2021low}
Tao Li, Lei Tan, Qinghua Tao, Yipeng Liu, and Xiaolin Huang.
\newblock {Low Dimensional Landscape Hypothesis is True: DNNs can be Trained in
  Tiny Subspaces}, 2021.

\bibitem[Li et~al.()Li, Huang, Yang, Wang, and Zhang]{li2019convergence}
Xiang Li, Kaixuan Huang, Wenhao Yang, Shusen Wang, and Zhihua Zhang.
\newblock {On the Convergence of FedAvg on Non-IID Data}.
\newblock In \emph{International Conference on Learning Representations
  (ICLR)}.

\bibitem[Lin et~al.(2021)Lin, Hosseinalipour, Azam, Brinton, and
  Michelusi]{lin2021two}
Frank Po-Chen Lin, Seyyedali Hosseinalipour, Sheikh~Shams Azam, Christopher~G
  Brinton, and Nicolo Michelusi.
\newblock {Two Timescale Hybrid Federated Learning with Cooperative D2D Local
  Model Aggregations}.
\newblock \emph{arXiv preprint arXiv:2103.10481}, 2021.

\bibitem[Lin et~al.(2014)Lin, Maire, Belongie, Hays, Perona, Ramanan,
  Doll{\'a}r, and Zitnick]{lin2014microsoft}
Tsung-Yi Lin, Michael Maire, Serge Belongie, James Hays, Pietro Perona, Deva
  Ramanan, Piotr Doll{\'a}r, and C~Lawrence Zitnick.
\newblock {Microsoft COCO: Common Objects in Context}.
\newblock In \emph{European Conference on Computer Vision (ECCV)}, pp.\
  740--755, 2014.

\bibitem[Liu et~al.(2016)Liu, Anguelov, Erhan, Szegedy, Reed, Fu, and
  Berg]{liu2016ssd}
Wei Liu, Dragomir Anguelov, Dumitru Erhan, Christian Szegedy, Scott Reed,
  Cheng-Yang Fu, and Alexander~C Berg.
\newblock {SSD: Single Shot Multibox Detector}.
\newblock In \emph{European Conference on Computer Vision (ECCV)}, pp.\
  21--37. Springer, 2016.

\bibitem[Liu et~al.(2019{\natexlab{a}})Liu, Ott, Goyal, Du, Joshi, Chen, Levy,
  Lewis, Zettlemoyer, and Stoyanov]{liu2019roberta}
Yinhan Liu, Myle Ott, Naman Goyal, Jingfei Du, Mandar Joshi, Danqi Chen, Omer
  Levy, Mike Lewis, Luke Zettlemoyer, and Veselin Stoyanov.
\newblock {RoBERTa: A Robustly Optimized BERT Pretraining Approach}.
\newblock \emph{arXiv preprint arXiv:1907.11692}, 2019{\natexlab{a}}.

\bibitem[Liu et~al.(2019{\natexlab{b}})Liu, Sun, Zhou, Huang, and
  Darrell]{liu2018rethinking}
Zhuang Liu, Mingjie Sun, Tinghui Zhou, Gao Huang, and Trevor Darrell.
\newblock {Rethinking the Value of Network Pruning}.
\newblock In \emph{International Conference on Learning Representations
  (ICLR)}, 2019{\natexlab{b}}.

\bibitem[Liu et~al.(2015)Liu, Luo, Wang, and Tang]{liu2015faceattributes}
Ziwei Liu, Ping Luo, Xiaogang Wang, and Xiaoou Tang.
\newblock {Deep Learning Face Attributes in the Wild}.
\newblock In \emph{IEEE International Conference on Computer Vision (ICCV)},
  December 2015.

\bibitem[Loshchilov \& Hutter(2016)Loshchilov and Hutter]{loshchilov2016sgdr}
Ilya Loshchilov and Frank Hutter.
\newblock {SGDR: Stochastic Gradient Descent with Warm Restarts}.
\newblock In \emph{International Conference on Learning Representations
  (ICLR)}, 2016.

\bibitem[Paszke et~al.(2019)Paszke, Gross, Massa, Lerer, Bradbury, Chanan,
  Killeen, Lin, Gimelshein, Antiga, Desmaison, Kopf, Yang, DeVito, Raison,
  Tejani, Chilamkurthy, Steiner, Fang, Bai, and Chintala]{paszke2019pytorch}
Adam Paszke, Sam Gross, Francisco Massa, Adam Lerer, James Bradbury, Gregory
  Chanan, Trevor Killeen, Zeming Lin, Natalia Gimelshein, Luca Antiga, Alban
  Desmaison, Andreas Kopf, Edward Yang, Zachary DeVito, Martin Raison, Alykhan
  Tejani, Sasank Chilamkurthy, Benoit Steiner, Lu~Fang, Junjie Bai, and Soumith
  Chintala.
\newblock {PyTorch: An Imperative Style, High-Performance Deep Learning
  Library}.
\newblock In \emph{Neural Information Processing Systems (NeurIPS)}. 2019.

\bibitem[Pearson(1901)]{pearson1901liii}
Karl Pearson.
\newblock {LIII. On Lines and Planes of Closest Fit to Systems of Points in
  Space}.
\newblock \emph{The London, Edinburgh, and Dublin Philosophical Magazine and
  Journal of Science}, 2\penalty0 (11):\penalty0 559--572, 1901.

\bibitem[Ronneberger et~al.(2015)Ronneberger, Fischer, and
  Brox]{ronneberger2015u}
Olaf Ronneberger, Philipp Fischer, and Thomas Brox.
\newblock {U-Net: Convolutional Networks for Biomedical Image Segmentation}.
\newblock In \emph{International Conference on Medical Image Computing and
  Computer-assisted Intervention}, pp.\  234--241, 2015.

\bibitem[Ryffel et~al.(2018)Ryffel, Trask, Dahl, Wagner, Mancuso, Rueckert, and
  Passerat-Palmbach]{ryffel2018generic}
Theo Ryffel, Andrew Trask, Morten Dahl, Bobby Wagner, Jason Mancuso, Daniel
  Rueckert, and Jonathan Passerat-Palmbach.
\newblock {A Generic Framework for Privacy Preserving Deep Learning}.
\newblock \emph{arXiv preprint arXiv:1811.04017}, 2018.

\bibitem[Sagun et~al.(2016)Sagun, Bottou, and LeCun]{sagun2016eigenvalues}
Levent Sagun, Leon Bottou, and Yann LeCun.
\newblock {Eigenvalues of the Hessian in Deep Learning: Singularity and
  Beyond}.
\newblock \emph{arXiv preprint arXiv:1611.07476}, 2016.

\bibitem[Sagun et~al.(2017)Sagun, Evci, Guney, Dauphin, and
  Bottou]{sagun2017empirical}
Levent Sagun, Utku Evci, V~Ugur Guney, Yann Dauphin, and Leon Bottou.
\newblock {Empirical Analysis of the Hessian of Over-parametrized Neural
  Networks}.
\newblock \emph{arXiv preprint arXiv:1706.04454}, 2017.

\bibitem[Saha et~al.(2020)Saha, Garg, and Roy]{saha2020gradient}
Gobinda Saha, Isha Garg, and Kaushik Roy.
\newblock {Gradient Projection Memory for Continual Learning}.
\newblock In \emph{International Conference on Learning Representations
  (ICLR)}, 2020.

\bibitem[Sattler et~al.(2019)Sattler, Wiedemann, M{\"u}ller, and
  Samek]{sattler2019robust}
Felix Sattler, Simon Wiedemann, Klaus-Robert M{\"u}ller, and Wojciech Samek.
\newblock {Robust and Communication-efficient Federated Learning from Non-iid
  Data}.
\newblock \emph{IEEE Transactions on Neural Networks and Learning Systems},
  31\penalty0 (9):\penalty0 3400--3413, 2019.

\bibitem[Seide et~al.(2014)Seide, Fu, Droppo, Li, and Yu]{seide20141}
Frank Seide, Hao Fu, Jasha Droppo, Gang Li, and Dong Yu.
\newblock {1-bit Stochastic Gradient Descent and its Application to
  Data-parallel Distributed Training of Speech DNNs}.
\newblock In \emph{Conference of the International Speech Communication
  Association (INTERSPEECH)}, 2014.

\bibitem[Shokri \& Shmatikov(2015)Shokri and Shmatikov]{shokri2015privacy}
Reza Shokri and Vitaly Shmatikov.
\newblock {Privacy-preserving Deep Learning}.
\newblock In \emph{ACM SIGSAC Conference on Computer and Communications
  Security (CCS)}, pp.\  1310--1321, 2015.

\bibitem[Simonyan \& Zisserman(2014)Simonyan and Zisserman]{simonyan2014very}
Karen Simonyan and Andrew Zisserman.
\newblock {Very Deep Convolutional Networks for Large-Scale Image Recognition}.
\newblock \emph{arXiv preprint arXiv:1409.1556}, 2014.

\bibitem[Son et~al.(2018)Son, Nah, and Lee]{son2018clustering}
Sanghyun Son, Seungjun Nah, and Kyoung~Mu Lee.
\newblock {Clustering Convolutional Kernels to Compress Deep Neural Networks}.
\newblock In \emph{European Conference on Computer Vision (ECCV)}, pp.\
  216--232, 2018.

\bibitem[Stich(2019)]{stich2018local}
Sebastian~U Stich.
\newblock {Local SGD Converges Fast and Communicates Little}.
\newblock In \emph{International Conference on Learning Representations
  (ICLR)}, 2019.

\bibitem[Vogels et~al.(2019)Vogels, Karimireddy, and Jaggi]{vogels2019powersgd}
Thijs Vogels, Sai~Praneeth Karimireddy, and Martin Jaggi.
\newblock {PowerSGD: Practical Low-rank Gradient Compression for Distributed
  Optimization}.
\newblock In \emph{Neural Information Processing Systems (NeurIPS)}, 2019.

\bibitem[Wang et~al.(2018)Wang, Sievert, Charles, Liu, Wright, and
  Papailiopoulos]{wang2018atomo}
Hongyi Wang, Scott Sievert, Zachary Charles, Shengchao Liu, Stephen Wright, and
  Dimitris Papailiopoulos.
\newblock {ATOMO: Communication-efficient Learning via Atomic Sparsification}.
\newblock In \emph{Neural Information Processing Systems (NeurIPS)}, 2018.

\bibitem[Wang \& Joshi(2018)Wang and Joshi]{wang2018cooperative}
Jianyu Wang and Gauri Joshi.
\newblock {Cooperative SGD: A Unified Framework for the Design and Analysis of
  Communication-efficient SGD Algorithms}.
\newblock \emph{arXiv preprint arXiv:1808.07576}, 2018.

\bibitem[Wang et~al.(2020{\natexlab{a}})Wang, Liu, Liang, Joshi, and
  Poor]{wang2020tackling}
Jianyu Wang, Qinghua Liu, Hao Liang, Gauri Joshi, and H~Vincent Poor.
\newblock {Tackling the Objective Inconsistency Problem in Heterogeneous
  Federated Optimization}.
\newblock In \emph{Neural Information Processing Systems (NeurIPS)},
  2020{\natexlab{a}}.

\bibitem[Wang et~al.(2020{\natexlab{b}})Wang, Tantia, Ballas, and
  Rabbat]{wang2019slowmo}
Jianyu Wang, Vinayak Tantia, Nicolas Ballas, and Michael Rabbat.
\newblock {SlowMo: Improving Communication-efficient Distributed SGD with Slow
  Momentum}.
\newblock In \emph{International Conference on Learning Representations
  (ICLR)}, 2020{\natexlab{b}}.

\bibitem[Wangni et~al.(2018)Wangni, Wang, Liu, and Zhang]{wangni2017gradient}
Jianqiao Wangni, Jialei Wang, Ji~Liu, and Tong Zhang.
\newblock {Gradient Sparsification for Communication-efficient Distributed
  Optimization}.
\newblock In \emph{Neural Information Processing Systems (NeurIPS)}, 2018.

\bibitem[Xiao et~al.(2017)Xiao, Rasul, and Vollgraf]{xiao2017fmnist}
Han Xiao, Kashif Rasul, and Roland Vollgraf.
\newblock {Fashion-MNIST: a Novel Image Dataset for Benchmarking Machine
  Learning Algorithms}, 2017.

\end{thebibliography}
}

\newpage

\appendix
\vspace{-1.0in}
\part{Appendix}
\parttoc


\newpage
\section{Proof of Theorem 1}
\label{app:proof_1}

We start by introducing a lemma, that is used prominently throughout the analysis that follows.


\begin{lemma}
\label{lemma:exp_norm_sum}
For a sequence of vectors $\mset{\vec{a}{i}{}}{i=1}{N}$, such that $\E{\vec{a}{i}{}|\vec{a}{i-1}{}, \vec{a}{i-2}{},\cdots,\vec{a}{1}{}}{} = \vec{0}{}{}, \forall i$,
\eqn{
    \E{ \norm{\sum_{i=1}^N \vec{a}{i}{}}{}^2 }{} = \sum_{i=1}^N\E{ \norm{ \vec{a}{i}{}}{}^2 }{}.
    \label{eqn:exp_norm_sum}
    \tag{\bf L2}
}
\end{lemma}
\begin{proof}
\aligneqn{
    \E{ \norm{\sum_{i=1}^N \vec{a}{i}{}}{}^2 }{} = \sum_{i=1}^N\E{ \norm{ \vec{a}{i}{}}{}^2 }{} + \sum_{i=1}^N\sum_{\substack{j=1\\j\neq i}}^N \E{ \vec{a}{i}{\top} \vec{a}{j}{} }{}.
}
Using the law of total expectation, assuming $i<j$, we get 
\eqn{
    \E{ \vec{a}{i}{\top} \vec{a}{j}{} }{} = \E{ \vec{a}{i}{\top}\E{ \vec{a}{j}{}|\vec{a}{i}{},\cdots,\vec{a}{1}{} }{} }{} = \vec{0}{}{},
}
which completes the proof.
\end{proof}

Next, we define a few auxiliary variables that would be referenced in the proof later: as defined in the main text, $\vec{g}{k}{(t)} = \sum_{b=0}^{\tau-1} \subsup{g}{k}{}(\vec{\theta}{k}{(t, b)})$ is the accumulated stochastic gradient at worker $k$, where $b$ ranging from $0$ to $\tau-1$ denotes the rounds of local updates. Using Assumption~\eqref{eqn:sgd_noise}, error in stochastic gradient approximation  $\subsup{g}{k}{}(\vec{\theta}{k}{(t, b)})$ can be defined as $\vec{\epsilon}{k}{(t, b)} = \subsup{g}{k}{}(\vec{\theta}{k}{(t, b)}) - \nabla\subsup{F}{k}{}(\vec{\theta}{k}{(t, b)})$. Consequently, we can write the stochastic gradient as $\subsup{g}{k}{}(\vec{\theta}{k}{(t, b)}) = \nabla\subsup{F}{k}{}(\vec{\theta}{k}{(t, b)}) + \vec{\epsilon}{k}{(t, b)}$ where $\nabla\subsup{F}{k}{}(\vec{\theta}{k}{(t, b)})$ is the true gradient. From Assumption~\eqref{eqn:sgd_noise} it follows that

\vspace{-0.2in}
\aligneqn{
    \E{\vec{\epsilon}{k}{(t, b)}}{} = \mathbf{0} \text{ and } \E{ \norm{\vec{\epsilon}{k}{(t, b)}}{}^2 }{} \leq \sigma^2.
}
We also introduce the normalized stochastic gradient $\vec{d}{k}{(t)}$ given by
\eqn{
   \vec{d}{k}{(t)} = \frac{\vec{g}{k}{(t)}}{\tau} = \frac{1}{\tau} \sum_{b=0}^{\tau-1} \subsup{\mathbf{g}}{k}{}(\vec{\theta}{k}{(t, b)}) = \frac{1}{\tau}\sum_{b=0}^{\tau-1} \nabla\subsup{F}{k}{}(\vec{\theta}{k}{(t, b)}) + \frac{1}{\tau}\sum_{b=0}^{\tau-1} \vec{\epsilon}{k}{(t, b)} = \vec{h}{k}{(t)} + \vec{\epsilon}{k}{(t)}, \label{eqn:normalized_approx_grad}
}
where we define
\aligneqn{
    \text{Cumulative Average of the true gradient: } &\vec{h}{k}{(t)} = \frac{1}{\tau}\sum_{b=0}^{\tau-1} \nabla\subsup{F}{k}{}(\vec{\theta}{k}{(t, b)}) \label{eqn:h_k_t}, \text{ and}\\
    \text{Cumulative Average of the SGD error: } &\vec{\epsilon}{k}{(t)} = \frac{1}{\tau}\sum_{b=0}^{\tau-1} \vec{\epsilon}{k}{(t, b)}.
}
Next, we evaluate the first and second moment of normalized SGD error, $\vec{\epsilon}{k}{(t)}$ as follows
\aligneqn{
    \E{\vec{\epsilon}{k}{(t)}}{} &= \E{\frac{1}{\tau}\sum_{b=0}^{\tau-1} \vec{\epsilon}{k}{(t, b)}}{} = \frac{1}{\tau}\sum_{b=0}^{\tau-1} \E{\vec{\epsilon}{k}{(t, b)}}{} = \vec{0}{}{},
    \label{eqn:e_k_t} \\
    \E{\norm{\vec{\epsilon}{k}{(t)}}{}^2}{} &= \E{\norm{\frac{1}{\tau}\sum_{b=0}^{\tau-1} \vec{\epsilon}{k}{(t, b)}}{}^2}{} = \frac{1}{\tau^2}\sum_{b=0}^{\tau-1}\E{\norm{\vec{\epsilon}{k}{(t, b)}}{}^2}{} \leq \frac{\sigma^2}{\tau} ,
    \label{eqn:e_k_t2}
}
where \eqref{eqn:e_k_t} uses Assumption~\eqref{eqn:sgd_noise} and \eqref{eqn:e_k_t2} uses Lemma~\ref{lemma:exp_norm_sum}. 

Now we can proceed to the proof of Theorem~\ref{thm:conv_char}. We start with the update rule of {\algName}, where a round of global update is given by
\eqn{
\vec{\theta}{}{(t+1)} = \vec{\theta}{}{(t)} - \eta\sum_{k=1}^K\omega_k\tvec{g}{k}{(t)} = \vec{\theta}{}{(t)} - \tau\eta\sum_{k=1}^K\omega_k\tvec{d}{k}{(t)},     
}
where $\tvec{g}{k}{(t)}$ is the approximate gradient shared by worker $k$ with the server and $\tvec{d}{k}{(t)}$ is given by
\eqn{
    \tvec{d}{k}{(t)} = \tvec{g}{k}{(t)}/\tau = \subsup{\rho}{k}{(t),\ell}\vec{g}{k}{\ell}/\tau = \subsup{\rho}{k}{(t),\ell}\vec{d}{k}{\ell}
    \label{eqn:d_l}
}
where $\vec{d}{k}{\ell}=\vec{g}{k}{\ell}/\tau$ is the normalized stochastic gradient w.r.t. the last LBG shared by worker $k$. Also, from the trigonometric relationship shown in Fig.~\ref{fig:approx_fedl}, we have,
\eqn{
    \norm{\subsup{\rho}{k}{(t),\ell}\vec{d}{k}{\ell}}{} = \norm{\vec{d}{k}{(t)} \cos(\subsup{\alpha}{k}{(t),\ell})}{}
    \label{eqn:norm_equality}
}
Similar to $ \vec{d}{k}{(t)}$ from \eqref{eqn:normalized_approx_grad},  $\vec{d}{k}{\ell}$ can be split into the cumulative average of true look-back gradient $\vec{h}{k}{\ell}$ and the corresponding cumulative average of SGD error $\vec{\epsilon}{k}{\ell}$, given by
\eqn{
    \vec{d}{k}{\ell} = \vec{h}{k}{\ell} + \vec{\epsilon}{k}{\ell}.
    \label{eqn:normalized_lbg_grad}
}
From Assumption~\eqref{eqn:lipschitz}, the global loss function is $\beta$ smooth, which implies:

\eqn{
    F(\vec{\theta}{}{(t+1)}) - F(\vec{\theta}{}{(t)}) \leq \dotp{\nabla F(\vec{\theta}{}{(t)})}{\vec{\theta}{}{(t+1)}-\vec{\theta}{}{(t)}}+\frac{\beta}{2}\norm{\vec{\theta}{}{(t+1)}-\vec{\theta}{}{(t)}}{}^2. \\
}
Taking expectation over $\vec{\epsilon}{k}{(t,b)},  k\in\{1,2,\cdots,K\}, b\in\{0,1,\cdots,\tau-1\}$,
\aligneqn{
    \E{F(\vec{\theta}{}{(t+1)})}{} &- F(\vec{\theta}{}{(t)}) \nonumber\\
    &\leq -\tau\eta
    \underbrace{\E{\dotp{\nabla F(\vec{\theta}{}{(t)})}{\sum_{k=1}^K\omega_k\cdot \tvec{d}{k}{(t)}}}{}}_{\bf Z_1} + \frac{\tau^2\eta^2\beta}{2}\underbrace{\E{\norm{\sum_{k=1}^K\omega_k \tvec{d}{k}{(t)}}{}^2}{}}_{\bf Z_2}
    \label{eqn:fl_sgd_i1}.
}
We evaluate ${\bf Z_1}$ as follows 
\aligneqn{
    {\bf Z_1} &= \E{\dotp{\nabla F(\vec{\theta}{}{(t)})}{\sum_{k=1}^K\omega_k \tvec{d}{k}{(t)}}}{} =  \E{\dotp{\nabla F(\vec{\theta}{}{(t)})}{\sum_{k=1}^K\omega_k \left(\vec{d}{k}{(t)} - \vec{d}{k}{(t)} + \tvec{d}{k}{(t)}\right)}}{} \nonumber\\
    &= \underbrace{\E{\dotp{\nabla F(\vec{\theta}{}{(t)})}{\sum_{k=1}^K\omega_k \vec{d}{k}{(t)} }}{} }_{\bf Z_{1, 1}} - \underbrace{\E{\dotp{\nabla F(\vec{\theta}{}{(t)})}{\sum_{k=1}^K\omega_k \left(\vec{d}{k}{(t)} - \tvec{d}{k}{(t)}\right)}}{} }_{\bf Z_{1,2}}, \label{eqn:t1_pre}
}
where ${\bf Z_{1,1}}$ is given by,
\aligneqn{
    {\bf Z_{1,1}} &\overset{(i)}{=} \E{\dotp{\nabla F(\vec{\theta}{}{(t)})}{\sum_{k=1}^K\omega_k \vec{h}{k}{(t)} }}{} + \E{\dotp{\nabla F(\vec{\theta}{}{(t)})}{\sum_{k=1}^K\omega_k \vec{\epsilon}{k}{(t)} }}{} \nonumber\\
    &\overset{(ii)}{=} \frac{1}{2}\norm{\nabla F(\vec{\theta}{}{(t)})}{}^2 \hspace{-0.5mm}+\hspace{-0.5mm} \frac{1}{2}\E{\norm{\sum_{k=1}^K\omega_k
    \vec{h}{k}{(t)} }{}^2}{} \hspace{-0.5mm}-\hspace{-0.5mm} \frac{1}{2}\E{\norm{\nabla F(\vec{\theta}{}{(t)}) \hspace{-0.5mm}-\hspace{-0.5mm} \sum_{k=1}^K\omega_k 
    \vec{h}{k}{(t)} }{}^2}{}\hspace{-1mm},
    \label{eqn:t1_1}
}
where $(i)$ follows from \eqref{eqn:normalized_lbg_grad} and $(ii)$ uses $2\dotp{\mathbf{a}}{\mathbf{b}} = \norm{\mathbf{a}}{}^2 + \norm{\mathbf{b}}{}^2 - \norm{\mathbf{a}-\mathbf{b}}{}^2$ for any two real vectors $\mathbf{a}$ and $\mathbf{b}$. We next upper bound the term ${\bf Z_{1,2}}$ (although ${Z_{1,2}}$ has a negative sign in \eqref{eqn:t1_pre}, ${Z_1}$ also appears with a negative sign in \eqref{eqn:fl_sgd_i1} which allows us to do the upper bound) as follows
\aligneqn{
    {\bf Z_{1,2}} &= \E{\dotp{\nabla F(\vec{\theta}{}{(t)})}{\sum_{k=1}^K\omega_k \left(\vec{d}{k}{(t)} - \tvec{d}{k}{(t)}\right)}}{} \nonumber\\
    &\overset{(i)}{\leq} \frac{1}{4} \norm{\nabla F(\vec{\theta}{}{(t)})}{}^2 + \E{\norm{ \sum_{k=1}^K\omega_k \left(\vec{d}{k}{(t)} - \tvec{d}{k}{(t)}\right) }{}^2}{}, \label{eqn:t1_2}
}
where $(i)$ follows from $\indotp{a}{b} \leq (1/4)\innorm{a}{}^2+\innorm{b}{}^2$ (result of Cauchy-Schwartz and Young's inequalities). Substituting \eqref{eqn:t1_1} and \eqref{eqn:t1_2} back in \eqref{eqn:t1_pre}, we get

\aligneqn{
     -{\bf Z_1} &\leq -\frac{1}{4}\norm{\nabla F(\vec{\theta}{}{(t)})}{}^2 \hspace{-0.5mm}-\hspace{-0.5mm} \frac{1}{2}\E{\norm{\sum_{k=1}^K\omega_k
    \vec{h}{k}{(t)} }{}^2}{} \hspace{-0.5mm}+\hspace{-0.5mm} \frac{1}{2}\E{\norm{\nabla F(\vec{\theta}{}{(t)}) \hspace{-0.5mm}-\hspace{-0.5mm} \sum_{k=1}^K\omega_k 
    \vec{h}{k}{(t)} }{}^2}{}\hspace{-1mm} \nonumber\\
    &~~~~ +  \E{\norm{ \sum_{k=1}^K\omega_k \left(\vec{d}{k}{(t)} - \tvec{d}{k}{(t)}\right) }{}^2}{}. \label{eqn:t1}
}
We next bound the term ${\bf Z_2}$ in \eqref{eqn:fl_sgd_i1} as follows 
\aligneqn{
    {\bf Z_2} &= \E{\norm{\sum_{k=1}^K\omega_k \tvec{d}{k}{(t)}}{}^2}{} = \E{\norm{ \sum_{k=1}^K\omega_k \left(\vec{d}{k}{(t)} - \vec{d}{k}{(t)} + \tvec{d}{k}{(t)}\right) }{}^2}{} \nonumber\\  
    &\overset{(i)}{\leq} 2\underbrace{\E{\norm{ \sum_{k=1}^K\omega_k \vec{d}{k}{(t)} }{}^2}{}}_{\bf Z_{2,1}} + 2\E{\norm{ \sum_{k=1}^K\omega_k \left(\vec{d}{k}{(t)} - \tvec{d}{k}{(t)}\right) }{}^2}{},
    \label{eqn:t2_pre}
}
where ${\bf Z_{2,1}}$ is given by,
\aligneqn{
    {\bf Z_{2,1}} &= \E{\norm{ \sum_{k=1}^K\omega_k \vec{d}{k}{(t)} }{}^2}{} \overset{(i)}{=} \E{\norm{ \sum_{k=1}^K\omega_k \vec{h}{k}{(t)} }{}^2}{} + \E{\norm{ \sum_{k=1}^K\omega_k \vec{\epsilon}{k}{(t)} }{}^2}{} \nonumber\\
    &\overset{(ii)}{\leq} \E{\norm{ \sum_{k=1}^K\omega_k \vec{h}{k}{(t)} }{}^2}{} + \sum_{k=1}^K\omega_k \E{\norm{ \vec{\epsilon}{k}{(t)} }{}^2}{} \overset{(iii)}{=} \E{\norm{ \sum_{k=1}^K\omega_k \vec{h}{k}{(t)} }{}^2}{} + \frac{\sigma^2}{\tau},\label{eqn:t2_1}
}
where $(i)$ follows from \eqref{eqn:normalized_approx_grad}, $(ii)$ uses Jensen's inequality:  $\innorm{\sum_{k=1}^{K} \omega_k \vec{a}{k}{} }{}^2\leq \sum_{k=1}^{K} \omega_k \norm{ \mathbf{a}_k }{}^2$, s.t. $\sum_{k=1}^{K} \omega_k =1$, and $(iii)$ uses \eqref{eqn:e_k_t2}. Plugging \eqref{eqn:t2_1} back into \eqref{eqn:t2_pre}, we get
\aligneqn{
    {\bf Z_2} \leq 2\E{\norm{ \sum_{k=1}^K\omega_k \vec{h}{k}{(t)} }{}^2}{} + \frac{2\sigma^2}{\tau} + 2\E{\norm{ \sum_{k=1}^K\omega_k \left(\vec{d}{k}{(t)} - \tvec{d}{k}{(t)}\right) }{}^2}{}.
    \label{eqn:t2}
}
Substituting \eqref{eqn:t1} and \eqref{eqn:t2} back in \eqref{eqn:fl_sgd_i1}, we get 
\aligneqn{
    &\E{F(\vec{\theta}{}{(t+1)})}{} - F(\vec{\theta}{}{(t)}) \leq  -\frac{\tau\eta}{4}\norm{\nabla F(\vec{\theta}{}{(t)})}{}^2 \hspace{-0.5mm}-\hspace{-0.5mm} \frac{\tau\eta}{2}\E{\norm{\sum_{k=1}^K\omega_k
    \vec{h}{k}{(t)} }{}^2}{} \nonumber\\
    &~~~~\hspace{-0.5mm}+\hspace{-0.5mm} \frac{\tau\eta}{2}\E{\norm{\nabla F(\vec{\theta}{}{(t)}) \hspace{-0.5mm}-\hspace{-0.5mm} \sum_{k=1}^K\omega_k 
    \vec{h}{k}{(t)} }{}^2}{}\hspace{-1mm} +  \tau\eta \E{\norm{ \sum_{k=1}^K\omega_k \left(\vec{d}{k}{(t)} - \tvec{d}{k}{(t)}\right) }{}^2}{} \nonumber\\
    &~~~~+ \tau^2\eta^2\beta\E{\norm{ \sum_{k=1}^K\omega_k \vec{h}{k}{(t)} }{}^2}{} + \tau\eta^2\beta\sigma^2 + \tau^2\eta^2\beta\E{\norm{ \sum_{k=1}^K\omega_k \left(\vec{d}{k}{(t)} - \tvec{d}{k}{(t)}\right) }{}^2}{} \nonumber\\
    &=  -\frac{\tau\eta}{4}\norm{\nabla F(\vec{\theta}{}{(t)})}{}^2 \hspace{-0.5mm}-\hspace{-0.5mm} \frac{\tau\eta}{2}(1-2\tau\eta\beta)\E{\norm{\sum_{k=1}^K\omega_k
    \vec{h}{k}{(t)} }{}^2}{} \hspace{-0.5mm}+\hspace{-0.5mm} \frac{\tau\eta}{2}\E{\norm{\nabla F(\vec{\theta}{}{(t)}) \hspace{-0.5mm}-\hspace{-0.5mm} \sum_{k=1}^K\omega_k 
    \vec{h}{k}{(t)} }{}^2}{}\hspace{-1mm} \nonumber\\
    &~~~~ +  \tau\eta(1+\tau\eta\beta) \E{\norm{ \sum_{k=1}^K\omega_k \left(\vec{d}{k}{(t)} - \tvec{d}{k}{(t)}\right) }{}^2}{} + \tau\eta^2\beta\sigma^2.
    \label{eqn:t2_2}
}
Choosing  $\tau\eta\beta \leq 1/2$, implies that $- \frac{\tau\eta}{2}(1-2\tau\eta\beta)\leq 0$ and $1+\tau\eta\beta \leq 3/2 < 2$, which results in simplification of \eqref{eqn:t2_2} as:
\aligneqn{
    \frac{\E{F(\vec{\theta}{}{(t+1)})}{} - F(\vec{\theta}{}{(t)})}{\eta\tau} &\leq -\frac{1}{4}\norm{\nabla F(\vec{\theta}{}{(t)})}{}^2 + \frac{1}{2} \E{\norm{\nabla F(\vec{\theta}{}{(t)}) \hspace{-0.5mm}-\hspace{-0.5mm} \sum_{k=1}^K\omega_k 
    \vec{h}{k}{(t)} }{}^2}{}\nonumber\\
    &~~~~+ 2\E{\norm{ \sum_{k=1}^K\omega_k \left(\vec{d}{k}{(t)} - \tvec{d}{k}{(t)}\right) }{}^2}{} + \eta\beta\sigma^2 \nonumber\\
    &\overset{(i)}{\leq} -\frac{1}{4}\norm{\nabla F(\vec{\theta}{}{(t)})}{}^2 + \frac{1}{2} \sum_{k=1}^K\omega_k \E{\norm{\nabla F_k(\vec{\theta}{k}{(t)}) \hspace{-0.5mm}-\hspace{-0.5mm} 
    \vec{h}{k}{(t)} }{}^2}{}\nonumber\\
    &~~~~+ 2\sum_{k=1}^K\omega_k \underbrace{\E{\norm{ \vec{d}{k}{(t)} - \tvec{d}{k}{(t)} }{}^2}{}}_{\bf Z_3} + \eta\beta\sigma^2,
    \label{eqn:t3}
}
where $(i)$ follows from Jensen's inequality $\innorm{\sum_{k=1}^{K} \omega_k \vec{a}{k}{} }{}^2\leq \sum_{k=1}^{K} \omega_k \norm{ \mathbf{a}_k }{}^2$, s.t. $\sum_{k=1}^{K} \omega_k =1$ and $\vec{\theta}{k}{(t)} = \vec{\theta}{}{(t)}, \forall k$ due to local synchronization. Now ${\bf Z_3}$ can be bounded as follows:

\aligneqn{
    {\bf Z_3} &= \E{\norm{ \vec{d}{k}{(t)} - \tvec{d}{k}{(t)} }{}^2}{} \overset{(i)}{=} \E{\norm{\vec{d}{k}{(t)} -
    \subsup{\rho}{k}{(t),\ell} \vec{d}{k}{\ell} }{}^2}{} \overset{(ii)}{=} \E{\norm{\vec{d}{k}{(t)} -
    \frac{ \frac{1}{\tau^2} \dotp{ \vec{g}{k}{(t)} }{ \vec{g}{k}{\ell}} }{\frac{1}{\tau^2} \norm{\vec{g}{k}{\ell}}{}^2} \vec{d}{k}{\ell} }{}^2}{} \nonumber\\
    &\overset{(iii)}{=} \E{\norm{\vec{d}{k}{(t)} -
    \frac{ \dotp{ \vec{d}{k}{(t)} }{ \vec{d}{k}{\ell}} }{ \norm{\vec{d}{k}{\ell}}{}^2} \vec{d}{k}{\ell} }{}^2}{} \overset{}{=} \E{\norm{\vec{d}{k}{(t)}}{}^2 - \frac{\dotp{ \vec{d}{k}{(t)} }{ \vec{d}{k}{\ell}}^2}{\norm{\vec{d}{k}{\ell}}{}^2}}{} \nonumber\\
    &\overset{(iv)}{=} \E{\norm{\vec{d}{k}{(t)}}{}^2 - \norm{ \vec{d}{k}{(t)} }{}^2\cos^2(\vec{\alpha}{k}{(t),\ell})}{} \overset{}{=} \E{\norm{\vec{d}{k}{(t)}}{}^2 \left(1 -\cos^2(\vec{\alpha}{k}{(t),\ell})\right)}{} \nonumber\\
    &= 2\E{\norm{\vec{d}{k}{(t)}}{}^2 \sin^2(\vec{\alpha}{k}{(t),\ell})}{} \overset{(v)}{\leq} \E{ \Delta^2}{} = \Delta^2,
    \label{eqn:lbgm_diff}
}

where $(i)$ uses \eqref{eqn:d_l}, $(ii)$ uses {\algName} definition from \eqref{eqn:grad_proj}, $(iii)$ uses the fact that $\vec{d}{k}{(t)} = \vec{g}{k}{\ell}/\tau$, $(iv)$ uses $\indotp{a}{b} = \innorm{a}{}\innorm{b}{}\cos(\alpha)$, and $(v)$ follows from the condition in the theorem. Substituting \eqref{eqn:lbgm_diff} back in \eqref{eqn:t3}, we get
\aligneqn{
    &\frac{\E{F(\vec{\theta}{}{(t+1)})}{} - F(\vec{\theta}{}{(t)})}{\eta\tau} \nonumber\\
    &\leq -\frac{1}{4}\norm{\nabla F(\vec{\theta}{}{(t)})}{}^2 + \frac{1}{2} \sum_{k=1}^K\omega_k \underbrace{\E{\norm{\nabla F_k(\vec{\theta}{k}{(t)}) \hspace{-0.5mm}-\hspace{-0.5mm} 
    \vec{h}{k}{(t)} }{}^2}{}}_{\bf Z_4} + 2\Delta^2 + \eta\beta\sigma^2,
    \label{eqn:t3_1}
}
where ${\bf Z_4}$ is given by,
\aligneqn{
    {\bf Z_4} &= \E{\norm{\nabla F_k(\vec{\theta}{k}{(t)}) \hspace{-0.5mm}-\hspace{-0.5mm} 
    \vec{h}{k}{(t)} }{}^2}{} = \frac{1}{\tau^2}\E{\norm{\sum_{b=0}^{\tau-1}\left(\nabla F_k(\vec{\theta}{k}{(t,0)}) \hspace{-0.5mm}-\hspace{-0.5mm} 
    \nabla F_k(\vec{\theta}{k}{(t,b)}) \right)}{}^2}{} \nonumber\\
    &\leq \frac{1}{\tau} \sum_{b=0}^{\tau-1} \E{\norm{\nabla F_k(\vec{\theta}{k}{(t,0)}) \hspace{-0.5mm}-\hspace{-0.5mm} 
    \nabla F_k(\vec{\theta}{k}{(t,b)}) }{}^2}{} \leq \frac{\beta^2}{\tau} \sum_{b=0}^{\tau-1} \E{\norm{\vec{\theta}{k}{(t,0)} \hspace{-0.5mm}-\hspace{-0.5mm} 
    \vec{\theta}{k}{(t,b)} }{}^2}{}.
}

Also, using the local update rule $\vec{\theta}{k}{(t,b)} \leftarrow \vec{\theta}{k}{(t,0)} - \eta \sum_{s=0}^{b-1} \mathbf{g}_k(\vec{\theta}{k}{(t,s)})$, where $\vec{\theta}{k}{(t,b)}$ is the model parameter obtained at the $b$-th local iteration of the global round $t$ at device $k$, we get:
\aligneqn{
    {\bf Z_5} &= \E{\norm{ \vec{\theta}{k}{(t,0)} - \vec{\theta}{k}{(t,b)}}{}^2}{} \nonumber\\
    &= \eta^2 \E{\norm{ \sum_{s=0}^{b-1} \mathbf{g}_k(\vec{\theta}{k}{(t,s)}) }{}^2}{} \overset{(i)}{\leq} 2\eta^2 \E{\norm{ \sum_{s=0}^{b-1} \nabla F_k(\vec{\theta}{k}{(t,s)}) }{}^2}{} + 2\eta^2 \E{\norm{ \sum_{s=0}^{b-1} \vec{\epsilon}{k}{(t,s)} }{}^2}{} \nonumber\\
    &\overset{(ii)}{\leq} 2\eta^2b \sum_{s=0}^{b-1} \E{\norm{ \nabla F_k(\vec{\theta}{k}{(t,s)}) }{}^2}{} + 2\eta^2 \sum_{s=0}^{b-1}  \E{\norm{ \vec{\epsilon}{k}{(t,s)} }{}^2}{} \nonumber\\& \leq 2\eta^2b \sum_{s=0}^{\tau-1} \E{\norm{ \nabla F_k(\vec{\theta}{k}{(t,s)}) }{}^2}{} + 2\eta^2\sigma^2b, \label{eqn:t4_3}
}

where $(i)$ uses Cauchy-Schwartz inequality and $(ii)$ uses Lemma~\ref{lemma:exp_norm_sum} and Cauchy-Schwartz inequality. Also note that
\aligneqn{
    \sum_{b=0}^{\tau-1}b &= \frac{\tau(\tau-1)}{2}.
    \label{eqn:sum_b}
}

Taking the cumulative sum of both hand sides of ${\bf Z_5}$ from \eqref{eqn:t4_3} over all batches, i.e., $\frac{1}{\tau}\sum_{b=0}^{\tau-1}$, and using \eqref{eqn:sum_b}, we get:
\aligneqn{
    \frac{1}{\tau}\sum_{b=0}^{\tau-1} \E{\norm{ \vec{\theta}{k}{(t,0)} - \vec{\theta}{k}{(t,b)}}{}^2}{} &\leq \sigma^2\eta^2(\tau-1) + \eta^2(\tau-1) \sum_{b=0}^{\tau-1} \underbrace{\E{\norm{ \nabla F_k(\vec{\theta}{k}{(t,b)}) }{}^2}{}}_{\bf Z_6}.\label{eqn:t5_1}
}

Furthermore, term ${\bf Z_6}$ can be bounded as follows:
\aligneqn{
    {\bf Z_6} &= \E{\norm{ \nabla F_k(\vec{\theta}{k}{(t,b)}) }{}^2}{} \overset{(i)}{\leq} 2\E{\norm{ \nabla F_k(\vec{\theta}{k}{(t,b)}) - \nabla F_k(\vec{\theta}{k}{(t,0)}) }{}^2}{} + 2\E{\norm{ \nabla F_k(\vec{\theta}{k}{(t,0)}) }{}^2}{} \nonumber\\
    &\overset{(ii)}{\leq} 2\beta^2\E{\norm{ \vec{\theta}{k}{(t,b)} - \vec{\theta}{k}{(t,0)} }{}^2}{} + 2\E{\norm{ \nabla F_k(\vec{\theta}{k}{(t,0)}) }{}^2}{},
    \label{eqn:t6}
}

where $(i)$ uses Cauchy-Schwartz inequality and $(ii)$ uses \eqref{eqn:lipschitz}. Replacing ${\bf Z_6}$ in \eqref{eqn:t5_1} using \eqref{eqn:t6},  we get:
\aligneqn{
    &\frac{1}{\tau}\sum_{b=0}^{\tau-1} \underbrace{\E{\norm{ \vec{\theta}{k}{(t,0)} - \vec{\theta}{k}{(t,b)}}{}^2}{}}_{\bf Z_5} \nonumber\\
    &\leq \eta^2\sigma^2(\tau-1) + 2\eta^2\beta^2(\tau-1) \sum_{b=0}^{\tau-1} \underbrace{\E{\norm{ \vec{\theta}{k}{(t,b)} - \vec{\theta}{k}{(t,0)} }{}^2}{}}_{\bf Z_5} + 2\eta^2\tau(\tau-1) \E{\norm{ \nabla F_k(\vec{\theta}{k}{(t,0)}) }{}^2}{}.
}

Note that ${\bf Z_5}$, which is originally defined in~\eqref{eqn:t4_3}, appears both in the left hand side (LHS) and right hand side (RHS) of the above expression. Rearranging the terms in the above inequality yields:
\eqn{
    \frac{1}{\tau}\sum_{b=0}^{\tau-1} \E{\norm{ \vec{\theta}{k}{(t,0)} - \vec{\theta}{k}{(t,b)}}{}^2}{} \leq \frac{\eta^2\sigma^2(\tau-1)}{1-2\eta^2\beta^2\tau(\tau-1)} + \frac{2\eta^2\tau(\tau-1)}{1-2\eta^2\beta^2\tau(\tau-1)}\E{\norm{ \nabla F_k(\vec{\theta}{k}{(t,0)}) }{}^2}{}.
}
Defining $H=2\eta^2\beta^2\tau(\tau-1)$, the above inequality can be re-written to evaluate ${\bf Z_4}$,
\aligneqn{
    {\bf Z_4} &= \frac{\beta^2}{\tau}\sum_{b=0}^{\tau-1} \E{\norm{ \vec{\theta}{}{(t,0)} - \vec{\theta}{}{(t,b)}}{}^2}{} \nonumber\\
    &\leq \frac{\eta^2\beta^2\sigma^2(\tau-1)}{1-2\eta^2\beta^2\tau(\tau-1)} + \frac{2\eta^2\beta^2\tau(\tau-1)}{1-2\eta^2\beta^2\tau(\tau-1)}\E{\norm{ \nabla F_k(\vec{\theta}{k}{(t,0)}) }{}^2}{} \nonumber\\
    &= \frac{\eta^2\beta^2\sigma^2}{1-H}(\tau-1) + \frac{H}{1-H}\E{\norm{ \nabla F_k(\vec{\theta}{k}{(t,0)}) }{}^2}{}.
}

Taking a weighted sum from the both hand sides of the above inequality across all the workers and using \eqref{eqn:bbd_diversity}, we get:
\aligneqn{
    & \frac{1}{2}\sum_{k=1}^{K} \omega_k \E{\norm{\nabla F_k(\vec{\theta}{k}{(t)}) \hspace{-0.5mm}-\hspace{-0.5mm} 
    \vec{h}{k}{(t)} }{}^2}{} \nonumber\\
    &\leq \frac{\eta^2\beta^2\sigma^2 (\tau-1)}{2(1-H)} \sum_{k=1}^{K} \omega_k + \frac{H}{2(1-H)}\sum_{k=1}^{K} \omega_k \E{\norm{ \nabla F_k(\vec{\theta}{k}{(t,0)}) }{}^2}{} \nonumber\\
    &=\frac{\eta^2 \beta^2\sigma^2(\tau-1)}{2(1-H)} + \frac{H}{2(1-H)}\sum_{k=1}^{K} \omega_k \E{\norm{ \nabla F_k(\vec{\theta}{k}{(t)}) }{}^2}{} \nonumber\\
    &\overset{(i)}{\leq}\frac{\eta^2\beta^2\sigma^2(\tau-1)}{2(1-H)} + \frac{H\Upsilon^2}{2(1-H)} \E{\norm{ \nabla F(\vec{\theta}{}{(t)}) }{}^2}{} + \frac{H\Gamma^2}{2(1-H)},
    \label{eqn:avg_diff_grad}
}
where $(i)$ follows from \eqref{eqn:bbd_diversity} and $\vec{\theta}{k}{(t)}=\vec{\theta}{}{(t)},~\forall k$ since computation occurs at the instance of global aggregation.
Next, plugging \eqref{eqn:avg_diff_grad} back in \eqref{eqn:t3_1}, we get:
\aligneqn{
 &\frac{\E{F(\vec{\theta}{}{(t+1)})}{} - F(\vec{\theta}{}{(t)})}{\eta\tau} \nonumber\\
    &\leq -\frac{1}{4}\norm{\nabla F(\vec{\theta}{}{(t)})}{}^2 + \frac{1}{2} \sum_{k=1}^K\omega_k \E{\norm{\nabla F_k(\vec{\theta}{k}{(t)}) \hspace{-0.5mm}-\hspace{-0.5mm} 
    \vec{h}{k}{(t)} }{}^2}{} + 2\Delta^2 + \eta\beta\sigma^2 \nonumber\\
    &\leq -\frac{1}{4}\norm{\nabla F(\vec{\theta}{}{(t)})}{}^2 + \frac{\eta^2\beta^2\sigma^2(\tau-1)}{2(1-H)} + \frac{H\Upsilon^2}{2(1-H)} \E{\norm{ \nabla F(\vec{\theta}{}{(t)}) }{}^2}{} + \frac{H\Gamma^2}{2(1-H)} \nonumber\\ 
    &~~~~ + 2\Delta^2 + \eta\beta\sigma^2 \nonumber\\
    &\leq -\frac{1}{4}\left( 1 - \frac{2H\Upsilon^2}{1-H}  \right) \norm{\nabla F(\vec{\theta}{}{(t)})}{}^2 + \frac{\eta^2\beta^2\sigma^2(\tau-1)}{2(1-H)} + \frac{H\Gamma^2}{2(1-H)} + 2\Delta^2 + \eta\beta\sigma^2.
}

If $H \leq \frac{1}{1+2\alpha\Upsilon^2}$ for some constant $\alpha > 1$, then it follows that $\frac{1}{1-H} \leq 1+ \frac{1}{2\alpha\Upsilon^2}$ and $\frac{2H\Upsilon^2}{1-H} \leq \frac{1}{\alpha}$. Choosing $\alpha = 2$ we can simplify the above expression as follows:

\aligneqn{
    &\frac{\E{F(\vec{\theta}{}{(t+1)})}{} - F(\vec{\theta}{}{(t)})}{\eta\tau} \nonumber\\
    &\leq -\frac{1}{8}\norm{\nabla F(\vec{\theta}{}{(t)})}{}^2 + 2\Delta^2 + \eta\beta\sigma^2 + \eta^2\beta^2\sigma^2(\tau-1) \left(\frac{1}{2}+\frac{1}{8\Upsilon^2}\right) + 2 \eta^2\beta^2 \Gamma^2 \tau(\tau-1) \left(1+\frac{1}{4\Upsilon^2}\right) \nonumber\\
    &\leq -\frac{1}{8}\norm{\nabla F(\vec{\theta}{}{(t)})}{}^2 + 2\Delta^2 + \eta\beta\sigma^2  + \frac{5}{8}\eta^2\beta^2\sigma^2(\tau-1) + \frac{5}{2} \eta^2\beta^2 \Gamma^2 \tau(\tau-1).
}
Rearranging the terms in the above inequality and taking the  average across all aggregation rounds from the both hand sides, yields:
\aligneqn{
    &\frac{1}{T} \sum_{t=0}^{T-1} \E{\norm{\nabla F(\vec{\theta}{}{(t, 0)})}{}^2}{} \nonumber\\
    &\leq \frac{8\left[\sum_{t=0}^{T-1}\E{F(\vec{\theta}{}{(t)})}{} - F(\vec{\theta}{}{(t+1)})\right]}{\eta\tau T} + 16\Delta^2 + 8\eta\beta\sigma^2 + 5\eta^2\beta^2\sigma^2(\tau-1) + 20 \eta^2\beta^2 \Gamma^2 \tau(\tau-1) \nonumber\\
    &= \frac{8\left[F(\vec{\theta}{}{(0)}) - F(\vec{\theta}{}{(T)})\right]}{\eta\tau T} + 16\Delta^2 + 8\eta\beta\sigma^2 + 5\eta^2\beta^2\sigma^2(\tau-1) + 20 \eta^2\beta^2 \Gamma^2 \tau(\tau-1) \nonumber\\
    &\leq \frac{8\left[F(\vec{\theta}{}{(0)}) - F^\star\right]}{\eta\tau T} + 16\Delta^2 + 8\eta\beta\sigma^2 + 5\eta^2\beta^2\sigma^2(\tau-1) + 20 \eta^2\beta^2 \Gamma^2 \tau(\tau-1) \label{eqn:final},
}
where we used the fact that $F$ is bounded below, since $F_k$-s are presumed to be bounded below, and $F^\star\leq F(\vec{\theta}{}{})$, $\forall \vec{\theta}{}{}\in\mathbb{R}^{M}$. This completes the proof of Theorem~\ref{thm:conv_char}.
\subsection{Condition on Learning Rate}
From the two conditions on the learning rate used in the analysis above, we have
\aligneqn{
    \eta\beta &\leq \frac{1}{2\tau} \\
    2\eta^2\beta^2\tau(\tau-1) &\leq \frac{1}{1+4\Upsilon^2}
}
We can further tighten the second constraint as,
\eqn{
    2\eta^2\beta^2\tau(\tau-1) \leq 2\eta^2\beta^2\tau^2 \leq \frac{1}{1+4\Upsilon^2}
}
Combining the two we have, 
\aligneqn{
    \eta\beta &\leq \min\left\{ \frac{1}{2\tau}, \frac{1}{\tau\sqrt{2(1+4\Upsilon^2)}} \right\}.
}


\newpage
\section{Proof of Corollary 1}
\label{app:corrolary_1}

Using \eqref{eqn:final} we have:
\aligneqn{
    &\frac{1}{T} \sum_{t=0}^{T-1} \E{\norm{\nabla F(\vec{\theta}{}{(t, 0)})}{}^2}{} \nonumber\\
    &\leq \frac{8\left[F(\vec{\theta}{}{(0)}) - F^\star\right]}{\eta\tau T} + 16\Delta^2 + 8\eta\beta\sigma^2 + 5\eta^2\beta^2\sigma^2(\tau-1) + 20 \eta^2\beta^2 \Gamma^2 \tau(\tau-1).
    \label{eqn:simple_2}
}
Next, using the assumption in the corollary statement, we have $\Delta^2 \leq \eta$. We then can upper bound the RHS of \eqref{eqn:simple_2} to get:
\aligneqn{
    &\frac{1}{T} \sum_{t=0}^{T-1} \E{\norm{\nabla F(\vec{\theta}{}{(t, 0)})}{}^2}{} \nonumber\\
    &\leq \frac{8\left[F(\vec{\theta}{}{(0)}) - F^\star\right]}{\eta\tau T} + 16\eta + 8\eta\beta\sigma^2 + 5\eta^2\beta^2\sigma^2(\tau-1) + 20 \eta^2\beta^2 \Gamma^2 \tau(\tau-1).
}
Choosing $\eta = \sqrt{\frac{1}{\tau T}}$, we get, 
\aligneqn{
    &\frac{1}{T} \sum_{t=0}^{T-1} \E{\norm{\nabla F(\vec{\theta}{}{(t, 0)})}{}^2}{} \nonumber\\
    &\leq \frac{8\left[F(\vec{\theta}{}{(0)}) - F^\star\right]}{\sqrt{\tau T}} + \frac{8\beta\sigma^2}{\sqrt{\tau T}} + \frac{16}{\sqrt{\tau T}} + \frac{5\beta^2\sigma^2(\tau-1)}{\tau T} + \frac{20 \beta^2 \Gamma^2 \tau(\tau-1)}{\tau T}.
}
We can write the above expression as, 
\aligneqn{
    &\frac{1}{T} \sum_{t=0}^{T-1} \E{\norm{\nabla F(\vec{\theta}{}{(t, 0)})}{}^2}{} \nonumber\\
    &\leq \bigo{\frac{1}{\sqrt{\tau T}}} + \bigo{\frac{\sigma^2}{\sqrt{\tau T}}} + \bigo{\frac{1}{\sqrt{\tau T}}} + \bigo{\frac{\sigma^2(\tau-1)}{\tau T})} + \bigo{\frac{ (\tau-1)\Gamma^2}{ T}}.
}
This completes the proof for Corollary~\ref{corr:stationary_conv}.


\newpage
\section{Additional Discussions}

\subsection{{\algName} Storage Considerations}
\label{app:storage}
As discussed in Section~\ref{sec:method}, {\algName} requires both server and workers to store the last copy of the workers' look-back gradient (LBG). While storing a single LBG (same size as the model) at each device might be trivial since the space complexity increases by a constant factor (i.e., the space complexity increases from $\bigo{M}$ to $\bigo{2M}=\bigo{M}$ where $M$ is the model size), storage of LBGs at the server might require more careful considerations since it scales linearly with the number of devices (i.e., space complexity increases from $\bigo{M}$ to $\bigo{KM}$ where $M$ is the model size and $K$ is the number of workers). Thus, the storage requirements can scale beyond memory capabilities of an aggregation server for a very large scale federated learning systems. We, therefore propose the following solutions for addressing the storage challenge:
\begin{itemize}[leftmargin=5mm]
    \item \textit{Storage Offloading.} In a large scale federated learning system, it is realistic to assume network hierarchy \cite{hosseinalipour2020multi}, e.g., base stations, edge servers, cloud servers, etc. In such cases the storage burden can be offloaded and distributed across the network hierarchy where the LBGs of the devices are stored.
    \item \textit{LBG Compression.} If the {\algName} is applied on top of existing compression techniques such as Top-K, ATOMO, etc., the size of LBGs to be stored at the server also gets compressed which reduces the storage burden. Alternatively, we could use parameter clustering techniques \citep{son2018clustering,cho2021dkm} to reduce LBG size at the server.
    \item \textit{LBG Clustering.} In a very large scale federated learning system, say with a billion workers, it's unrealistic to assume that all the billion clients have very different LBGs given the low rank hypothesis \eqref{hypothesis} and possible local data similarities across the workers. It should therefore be possible to cluster the LBGs at the server into a smaller number of centroids and only saving the centroids of the clusters instead of saving all the LBGs of the devices individually. The centroids can be broadcast across the devices to update the local version of the LBGs.
\end{itemize}

\subsection{Hyperparameter Tuning}
\label{app:hyperparams}

Most of the compression/communication savings baselines operate on a tradeoff between communication savings and accuracy. For hyperparameter selection, we first optimize the hyperparameters for the base algorithm such that we achieve the best possible communication saving subject to constraint that accuracy does not fall off much below the corresponding vanilla federated learning approach. For example, we optimize the value of $K$ in top-K sparsification by changing $K$ in orders of 10, i.e. $K=10\%$, $K=1\%$, $K=0.1\%$, etc. and choose the value that gives the best tradeoff between the final model accuracy and communication savings (this value is generally around $K=10\%$). Similarly, for ATOMO we consider rank-1, rank-2, and rank-3 approximations. While rank-1 approximation gives better communication savings, the corresponding accuracy falls off sharply. Rank-3 approximation gives only a marginal accuracy benefit over rank-2 approximation but adds considerably more communication burden. Thus we use rank-2 approximations in ATOMO. In the plug-and-play evaluations, the LBGM algorithm is applied on top of the base algorithms once their hyperparameters are tuned as a final step to show the additional benefits we can attain by exploiting the low-rank characteristic of the  gradient subspace. Our chosen hyperparameters can be found in our code repository: \url{https://github.com/shams-sam/FedOptim}.


\section{Additional Pseudocodes}

\subsection{Pseudocode for Preliminary Experiments}
\label{app:pseudo}
In this Appendix, we provide a psuedocode for generating the preliminary experimental results in Sec.~\ref{sec:explore}. The actual implementation of the following function calls used in the pseudocode can be found in the listed files of our code repository: \url{https://github.com/shams-sam/FedOptim}:
\begin{itemize}[leftmargin=5mm]
    \item ${\tt get\_num\_PCA\_components}$: implemented in function $\mathsf{estimate\_optimal\_ncomponents}$, file: $\mathsf{src/common/nb\_utils.py}$ of the repository. \shams{In summary, we stack the accumulated gradients over the epochs an perform singular value decomposition, after which we do the standard analysis for estimating the number of components explaining a given amount of variance in the datasets. Specifically,  we count the number of singular values that account for the $99\%$ and $95\%$ of the aggregated singular values.}
    \item ${\tt get\_PCA\_components}$: implemented in function~$\mathsf{pca\_transform}$,~file:~$\mathsf{src/common/nb\_utils.py}$ of the repository. \shams{In summary, we stack the accumulated gradients over the epochs an perform singular value decomposition, after which we do the standard analysis for recovering the principal components explaining a given amount of variance in the datasets. Specifically,  we recover the left singular vectors corresponding the singular values that account for the $99\%$ and $95\%$ of the aggregated singular values.}
    \item ${\tt cosine\_similarity}$: implemented using functions $\mathsf{sklearn.preprocessing.normalize}$ and $\mathsf{numpy.dot}$, such that ${\tt cosine\_similarity(a,b)=normalize(a).dot(normalize(b))}$ where ${\tt a}$ and ${\tt b}$ are $\mathsf{numpy.array}$, \shams{where ${\tt normalize}$ performs the vector normalization and ${\tt dot}$ is the standard vector dot product}.
    \item ${\tt plot\_1}$, ${\tt plot\_2}$, and ${\tt plot\_3}$: implemented in files $\mathsf{src/viz/prelim\_1.py}$,  $\mathsf{src/viz/prelim\_2.py}$, and  $\mathsf{src/viz/prelim\_3.py}$ respectively.
\end{itemize}
\begin{algorithm}[h!]
{\footnotesize 
  \caption{Pseudocode for Preliminary Experiments in Section~\ref{sec:explore}}
\label{alg:pseudo}
\begin{algorithmic}[1]
    \State Initialize model parameter $\vec{\theta}{}{(0)}$.
    \State Initialize ${\tt actual\_grads}$ = \{\} \Comment{store gradients for ${\tt PCA}$}
    \State Initialize ${\tt pca95\_store}$ = \{\} \Comment{store \#components accounting for 95\% variance}
    \State Initialize ${\tt pca99\_store}$ = \{\} \Comment{store \#components accounting for 99\% variance}
    \Statex
    \For{$t = 0$ to $T-1$} \Comment{training for $T$ epochs}
        \State Initialize $\vec{g}{}{(t)} \leftarrow \vec{0}{}{}$. \Comment{initialize gradient accumulator}
        \State Set $\vec{\theta}{}{(t,0)} \leftarrow \vec{\theta}{}{(t)}$.
        \For{$b=0$ to $B-1$} \Comment{$B$ minibatches per epoch}
            \State Sample a minibatch of  datapoints $\mathcal{B}$ from dataset $\mathcal{D}$. 
            \State Compute $\vec{g}{}{(t,b)}= \sum_{d\in \mathcal{B}} \nabla f(\vec{\theta}{}{(t,b)};d)/|\mathcal{B}|$.
            \State Update parameter: $\vec{\theta}{}{(t,b+1)} \leftarrow \vec{\theta}{}{(t,b)} - \eta \cdot \vec{g}{}{(t,b)}$.
            \State Accumulate gradient: $\vec{g}{}{(t)} \leftarrow \vec{g}{}{(t)} + \vec{g}{}{(t,b)}$.
        \EndFor    
        \State Set $\vec{\theta}{}{(t+1)} \leftarrow \vec{\theta}{}{(t, B)}$.
        \State ${\tt actual\_grads.append(} \vec{g}{}{(t)} {\tt )}$ \Comment{append accumulated gradient to store}
        \State ${\tt pca95\_store.append(get\_num\_PCA\_components(actual\_grads, variance=0.95))}$
        \State ${\tt pca99\_store.append(get\_num\_PCA\_components(actual\_grads, variance=0.99))}$
        \Statex
    \EndFor
    \Statex
    \State ${\tt plot\_1(pca95\_store, pca99\_store)}$ \Comment{plot of PCA component progression}
    \Statex
    \State ${\tt principal\_grads = get\_PCA\_components(actual\_grads, variance=0.99)}$
    \State ${\tt heatmap = zeros(len(actual\_grads), len(principal\_grads))}$
    \For{$i=0$ to ${\tt len(actual\_grads)-1}$ }
        \For{$j=0$ to ${\tt len(principal\_grads)-1}$ }
            \State ${\tt heatmap[i,j] = cosine\_similarity(actual\_grads[i], principal\_grads[j])}$
        \EndFor
    \EndFor
    \Statex
    \State ${\tt plot\_2(heatmap)}$ \Comment{plot of overlap of actual and principal gradients}
    \Statex
    \State ${\tt heatmap = zeros(len(actual\_grads), len(actual\_grads))}$
    \For{$i=0$ to ${\tt len(actual\_grads)-1}$ }
        \For{$j=i$ to ${\tt len(actual\_grads)-1}$ }
            \State ${\tt heatmap[i,j] = heatmap[j, i]= cosine\_similarity(actual\_grads[i], actual\_grads[j])}$
        \EndFor
    \EndFor
    \Statex
    \State ${\tt plot\_3(heatmap)}$ \Comment{plot of similarity among actual gradients}
\end{algorithmic}
}
\end{algorithm}

\clearpage
\subsection{Pseudocode for Device Sampling Experiments}
\label{ssec:sampling_algo}

The pseudocode for {\algName} with device sampling is given below. The process is largely similar to Algorithm~\ref{alg:algo_training}, except modifications in the global aggregation strategy. During global aggregation the server samples a subset $K'$ of all the available clients and receives updates from only those clients for aggregation as shown in line~\ref{line:start_agg}-\ref{line:end_agg} of the Algorithm~\ref{alg:sampling}.

In terms of the effect of sampling on the unsampled devices, as long as an unsampled device that joins at a later step of the training have a fairly good look-back gradient (i.e., its newly generated gradient are close to its look-back gradient (LBG)), there is no need for transmission of the entire parameter vector. This would often happen in practice unless the device engages in model training after a very long period of inactivity, in which case it would need to transmit its updated LBG before engaging in {\algName} communication savings.

\begin{algorithm}[h!]
{\footnotesize 
  \caption{{\algName} with Device Sampling}
\label{alg:sampling}
\begin{algorithmic}[1]
    \Statex \textbf{\underline{Notation:}}
    \Statex $\vec{\theta}{}{(t)}$: global model parameter at global aggregation round $t$.
    \Statex $\vec{\theta}{k}{(t, b)}$: model parameter at worker $k$, at global aggregation round $t$ and local update $b$.
    \Statex $\vec{g}{k}{(t)}$: accumulated gradient at worker $k$ at global aggregation round $t$.
    \Statex $\vec{g}{k}{\ell}$: last full gradient transmitted to server, termed look-back gradient (LBG).
    \Statex $\subsup{\alpha}{k}{(t), \ell}$: phase between the accumulated gradient $\vec{g}{k}{(t)}$ and LBG $\vec{g}{k}{\ell}$, termed look-back phase (LBP).
    \Statex \textbf{\underline{Training at worker $k$:}}
    \State Update local parameters: $\vec{\theta}{k}{(t, 0)} \leftarrow \vec{\theta}{}{(t)}$, and initialize gradient accumulator: $\vec{g}{k}{(t)} \leftarrow \vec{0}{}{}$.
    \For{$b = 0$ to ($\tau$-1)}        
        \State Sample a minibatch of  datapoints $\mathcal{B}_k$ from  $\mathcal{D}_{k}$ and compute $\vec{g}{k}{(t, b)}= \sum_{d\in \mathcal{B}_k} \nabla f_k(\vec{\theta}{k}{(t, b)};d)/|\mathcal{B}_k|$.
        
        \State Update local parameters: $\vec{\theta}{k}{(t, b+1)} \leftarrow \vec{\theta}{k}{(t, b)} - \eta \cdot \vec{g}{k}{(t, b)}$, and accumulate gradient: $\vec{g}{k}{(t)} \leftarrow \vec{g}{k}{(t)} + \vec{g}{k}{(t,b)}$.
    \EndFor
        \State Calculate the LBP error:  $ \sin^2(\subsup{\alpha}{k}{(t), \ell}) = 1 - \left( \indotp{ \vec{g}{k}{(t)} }{ \vec{g}{k}{\ell}}  \big/ \big(\innorm{ \vec{g}{k}{(t)} }{} \times \innorm{ \vec{g}{k}{\ell} }{} \big) \right)^2$ \label{line:lbp}
    \If{ $\sin^2(\subsup{\alpha}{k}{(t), \ell}) \leq \subsup{\delta}{k}{\mathsf{threshold}}$ } \label{line:lbp_cond} \Comment{checking the LBP error}
        \State Send scalar LBC to the server: $\vec{\mu}{k}{(t)} \leftarrow \subsup{\rho}{k}{(t),\ell} = \indotp{ \vec{g}{k}{(t)} }{ \vec{g}{k}{\ell} }/\innorm{ \vec{g}{k}{\ell} }{}^2$.
    \Else \Comment{updating the LBG}
        \State Send actual gradient to the server: $\vec{\mu}{k}{(t)} \leftarrow \vec{g}{k}{(t)}$. \label{line:update1}
        \State Update worker-copy of LBG: $\vec{g}{k}{\ell} \leftarrow \vec{g}{k}{(t)}$. \label{line:update2}
    \EndIf
    \Statex \textbf{\underline{Global update at the aggregation server:}}
    \State Initialize global parameter $\vec{\theta}{}{(0)}$ and broadcast it across workers.
    \For{$t = 0$ to $(T-1)$}
        \State \shams{sample set of indices $K'$, a random subset from the pool of devices $\{1, 2, ..., K\}$.} \label{line:start_agg}
        \State Receive updates from workers \shams{$\{ \vec{\mu}{k}{(t)} \}_{k\in K'}$}.
        \State Update global parameters: $\vec{\theta}{}{(t+1)} \leftarrow \vec{\theta}{}{(t)} - \frac{\eta}{\shams{\abs{K'}{}}} \sum_{k \in K'}\omega_k \left[ s_k \cdot  \vec{\mu}{k}{(t)} \cdot \vec{g}{k}{\ell} + (1-s_k)\cdot \vec{\mu}{k}{(t)}\right]$, \label{line:end_agg}
        \Statex ~~~~~~where $s_k$ is an indicator function given by, $s_k = \begin{cases} 1, \quad \text{if } 
        \vec{\mu}{k}{(t)}\text{ is a scalar} \\
        0, \quad \text{otherwise, i.e., if }\vec{\mu}{k}{(t)}\text{ is a vector} \end{cases}$.
        \State Update server-copy of look-back gradients (LBGs): $ \vec{g}{k}{\ell} \leftarrow (1-s_k) \vec{\mu}{k}{(t)} + (s_k) \vec{g}{k}{\ell },~~\forall k \in K'$.
    \EndFor
\end{algorithmic}
}
\end{algorithm}


\newpage
\section{Additional Preliminary Experiments}
\label{app:addl_expt}

We next present the additional experiments performed to test hypotheses, \eqref{hypothesis} \& \eqref{hypothesis_2}. As mentioned in Section~\ref{sec:explore}, we study the rank-characteristics of centralized training using SGD on multiple datasets: FMNIST, MNIST, CIFAR-10, CelebA, COCO, and PascalVOC, and model architectures: CNN, FCN, Resnet18, VGG19, and U-Net. Section~\ref{app:prelim_expt_1} presents experiments complementary to the one presented in Fig.~\ref{fig:prelim_1}, while Section~\ref{app:prelim_expt_2} \& \ref{app:prelim_expt_3} present experiments complementary to those presented in Fig.~\ref{fig:prelim_2}, \& \ref{fig:prelim_3}, respectively. \textit{Please follow the hyperlinks for the ease of navigating through the figures.}

\subsection{PCA Component Progression}
\label{app:prelim_expt_1}

Together with Fig.~\ref{fig:prelim_1}, Figs.~\ref{fig:prelim_1_mnist}-\ref{fig:prelim_1_seg} show that both N99-PCA and N95-PCA are significantly lower than that the total number of gradients calculated during model training irrespective of model/dataset/learning task for multiple datasets and models which all agree with \eqref{hypothesis}. Specifically, the principal gradients (i.e., red and blue lines in the top row of the plots) are substantially lower (often as low as 10\% of number of epochs, i.e., gradients generated) in these experiments. Refer below for the details of the figures.

\begin{enumerate}[leftmargin=5mm]
    \item Fig.~\ref{fig:prelim_1_cifar100} repeats the experiment conducted in Fig.~\ref{fig:prelim_1} on CIFAR-100 using FCN, CNN, Resnet18, \& VGG19.
    \item Fig.~\ref{fig:prelim_1_mnist} repeats the experiment conducted in Fig.~\ref{fig:prelim_1} on MNIST using FCN, CNN, Resnet18, \& VGG19.
    \item Fig.~\ref{fig:prelim_1_fmnist} repeats the experiment conducted in Fig.~\ref{fig:prelim_1} on FMNIST using FCN, CNN, Resnet18, \& VGG19.
    \item Fig.~\ref{fig:prelim_1_svm} repeats the experiment conducted in Fig.~\ref{fig:prelim_1} on CIFAR-10, FMNIST, and MNIST using SVM, suggesting that we can use {\algName} for classic classifiers that are not necessary neural networks.
    \item Fig.~\ref{fig:prelim_1_seg} repeats the experiment conducted in Fig.~\ref{fig:prelim_1} on COCO, and PascalVOC using U-Net.
\end{enumerate}

\subsection{Overlap of Actual and Principal Gradient}
\label{app:prelim_expt_2}

Next, we perform experiments summarized in Fig.~\ref{fig:prelim_2_celeba_vgg19}-\ref{fig:prelim_2_coco_unet} to further validate our observation in Fig.~\ref{fig:prelim_2}: (i) cosine similarity of actual gradients with principal gradients varies gradually over time, and (ii) actual gradients have a high cosine similarity with one or more of the principal gradients. Refer below for the details of the figures. \textit{Each subplot is marked with \#L, the layer number of the neural network, and \#elem, the number of elements in each layer.}

\shams{In the plots with dense heatmaps (a large number of gradients in along x and y axis) for larger models such as Fig.~\ref{fig:prelim_2_celeba_vgg19}, it is harder to observe the overlap among the actual gradient and the PCA gradients. However, we can still notice the number of prinicipal components (along y axis) is substantially lower than the total number of epochs gradients (along y axis). A better picture of gradient overlap with other gradients can still be seen in corresponding inter-gradient overlap plot in Fig.~\ref{fig:prelim_3_celeba_vgg19}, which is consistent with the corresponding PCA progression shown in Fig.~\ref{fig:prelim_1}. Note that the lesser number of prinicipal gradient directions (e.g., CNN in Fig.~\ref{fig:prelim_1}) implies a higher overlap among the generated gradients (e.g., CNN in Fig.~\ref{fig:prelim_3}), while a larger number of PGDs (e.g., VGG19 in Fig.~\ref{fig:prelim_1}) implies a lower overlap among generated gradients (e.g., VGG19 in Fig.~\ref{fig:prelim_3_cifar_vgg19}).}

\begin{enumerate}[leftmargin=5mm]
    \item Fig.~\ref{fig:prelim_2_celeba_vgg19} repeats the experiment conducted in Fig.~\ref{fig:prelim_2} on CelebA using VGG19.
    \item Fig.~\ref{fig:prelim_2_celeba_resnet18} repeats the experiment conducted in Fig.~\ref{fig:prelim_2} on CelebA using Resnet18.
    \item Fig.~\ref{fig:prelim_2_celeba_fcn} repeats the experiment conducted in Fig.~\ref{fig:prelim_2} on CelebA using FCN.
    \item Fig.~\ref{fig:prelim_2_celeba_cnn} repeats the experiment conducted in Fig.~\ref{fig:prelim_2} on CelebA using CNN.
    
    \item Fig.~\ref{fig:prelim_2_cifar_vgg19} repeats the experiment conducted in Fig.~\ref{fig:prelim_2} on CIFAR-10 using VGG19.
    \item Fig.~\ref{fig:prelim_2_cifar_resnet18} repeats the experiment conducted in Fig.~\ref{fig:prelim_2} on CIFAR-10 using Resnet18.
    \item Fig.~\ref{fig:prelim_2_cifar_fcn} repeats the experiment conducted in Fig.~\ref{fig:prelim_2} on CIFAR-10 using FCN.
    \item Fig.~\ref{fig:prelim_2_cifar_cnn} repeats the experiment conducted in Fig.~\ref{fig:prelim_2} on CIFAR-10 using CNN.
    
    \item Fig.~\ref{fig:prelim_2_cifar100_vgg19} repeats the experiment conducted in Fig.~\ref{fig:prelim_2} on CIFAR-100 using VGG19.
    \item Fig.~\ref{fig:prelim_2_cifar100_resnet18} repeats the experiment conducted in Fig.~\ref{fig:prelim_2} on CIFAR-100 using Resnet18.
    \item Fig.~\ref{fig:prelim_2_cifar100_fcn} repeats the experiment conducted in Fig.~\ref{fig:prelim_2} on CIFAR-100 using FCN.
    \item Fig.~\ref{fig:prelim_2_cifar100_cnn} repeats the experiment conducted in Fig.~\ref{fig:prelim_2} on CIFAR-100 using CNN.
    
    \item Fig.~\ref{fig:prelim_2_fmnist_vgg19} repeats the experiment conducted in Fig.~\ref{fig:prelim_2} on FMNIST using VGG19.
    \item Fig.~\ref{fig:prelim_2_fmnist_resnet18} repeats the experiment conducted in Fig.~\ref{fig:prelim_2} on FMNIST using Resnet18.
    \item Fig.~\ref{fig:prelim_2_fmnist_fcn} repeats the experiment conducted in Fig.~\ref{fig:prelim_2} on FMNIST using FCN.
    \item Fig.~\ref{fig:prelim_2_fmnist_cnn} repeats the experiment conducted in Fig.~\ref{fig:prelim_2} on FMNIST using CNN.
    
    \item Fig.~\ref{fig:prelim_2_mnist_vgg19} repeats the experiment conducted in Fig.~\ref{fig:prelim_2} on MNIST using VGG19.
    \item Fig.~\ref{fig:prelim_2_mnist_resnet18} repeats the experiment conducted in Fig.~\ref{fig:prelim_2} on MNIST using Resnet18.
    \item Fig.~\ref{fig:prelim_2_mnist_fcn} repeats the experiment conducted in Fig.~\ref{fig:prelim_2} on MNIST using FCN.
    \item Fig.~\ref{fig:prelim_2_mnist_cnn} repeats the experiment conducted in Fig.~\ref{fig:prelim_2} on MNIST using CNN.
    
    \item Fig.~\ref{fig:prelim_2_voc_unet} repeats the experiment conducted in Fig.~\ref{fig:prelim_2} on PascalVOC using U-Net.
    \item Fig.~\ref{fig:prelim_2_coco_unet} repeats the experiment conducted in Fig.~\ref{fig:prelim_2} on COCO using U-Net.
\end{enumerate}

\subsection{Similarity among Consecutive Generated Gradients}
\label{app:prelim_expt_3}

Furthermore, we perform experiments summarized in Fig.~\ref{fig:prelim_3_celeba_vgg19}-\ref{fig:prelim_3_coco_unet}. Together with Fig.~\ref{fig:prelim_3}, these experiments show that there is a significant overlap of consecutive gradients generated during SGD iterations, which further substantiates \eqref{hypothesis_2} and bolsters our main idea that \textit{``gradients transmitted in FL can be recycled/reused to represent the gradients generated in the subsequent iterations''}. Refer below for the details of the figures.

\begin{enumerate}[leftmargin=5mm]
    \item Fig.~\ref{fig:prelim_3_celeba_vgg19} repeats the experiment conducted in Fig.~\ref{fig:prelim_3} on CelebA using VGG19.
    \item Fig.~\ref{fig:prelim_3_celeba_resnet18} repeats the experiment conducted in Fig.~\ref{fig:prelim_3} on CelebA using Resnet18.
    \item Fig.~\ref{fig:prelim_3_celeba_fcn} repeats the experiment conducted in Fig.~\ref{fig:prelim_3} on CelebA using FCN.
    \item Fig.~\ref{fig:prelim_3_celeba_cnn} repeats the experiment conducted in Fig.~\ref{fig:prelim_3} on CelebA using CNN.
    
    \item Fig.~\ref{fig:prelim_3_cifar_vgg19} repeats the experiment conducted in Fig.~\ref{fig:prelim_3} on CIFAR-10 using VGG19.
    \item Fig.~\ref{fig:prelim_3_cifar_resnet18} repeats the experiment conducted in Fig.~\ref{fig:prelim_3} on CIFAR-10 using Resnet18.
    \item Fig.~\ref{fig:prelim_3_cifar_fcn} repeats the experiment conducted in Fig.~\ref{fig:prelim_3} on CIFAR-10 using FCN.
    \item Fig.~\ref{fig:prelim_3_cifar_cnn} repeats the experiment conducted in Fig.~\ref{fig:prelim_3} on CIFAR-10 using CNN.
    
    \item Fig.~\ref{fig:prelim_3_cifar100_vgg19} repeats the experiment conducted in Fig.~\ref{fig:prelim_3} on CIFAR-100 using VGG19.
    \item Fig.~\ref{fig:prelim_3_cifar100_resnet18} repeats the experiment conducted in Fig.~\ref{fig:prelim_3} on CIFAR-100 using Resnet18.
    \item Fig.~\ref{fig:prelim_3_cifar100_fcn} repeats the experiment conducted in Fig.~\ref{fig:prelim_3} on CIFAR-100 using FCN.
    \item Fig.~\ref{fig:prelim_3_cifar100_cnn} repeats the experiment conducted in Fig.~\ref{fig:prelim_3} on CIFAR-100 using CNN.
    
    \item Fig.~\ref{fig:prelim_3_fmnist_vgg19} repeats the experiment conducted in Fig.~\ref{fig:prelim_3} on FMNIST using VGG19.
    \item Fig.~\ref{fig:prelim_3_fmnist_resnet18} repeats the experiment conducted in Fig.~\ref{fig:prelim_3} on FMNIST using Resnet18.
    \item Fig.~\ref{fig:prelim_3_fmnist_fcn} repeats the experiment conducted in Fig.~\ref{fig:prelim_3} on FMNIST using FCN.
    \item Fig.~\ref{fig:prelim_3_fmnist_cnn} repeats the experiment conducted in Fig.~\ref{fig:prelim_3} on FMNIST using CNN.
    
    \item Fig.~\ref{fig:prelim_3_mnist_vgg19} repeats the experiment conducted in Fig.~\ref{fig:prelim_3} on MNIST using VGG19.
    \item Fig.~\ref{fig:prelim_3_mnist_resnet18} repeats the experiment conducted in Fig.~\ref{fig:prelim_3} on MNIST using Resnet18.
    \item Fig.~\ref{fig:prelim_3_mnist_fcn} repeats the experiment conducted in Fig.~\ref{fig:prelim_3} on MNIST using FCN.
    \item Fig.~\ref{fig:prelim_3_mnist_cnn} repeats the experiment conducted in Fig.~\ref{fig:prelim_3} on MNIST using CNN.
    
    \item Fig.~\ref{fig:prelim_3_voc_unet} repeats the experiment conducted in Fig.~\ref{fig:prelim_3} on PascalVOC using U-Net.
    \item Fig.~\ref{fig:prelim_3_coco_unet} repeats the experiment conducted in Fig.~\ref{fig:prelim_3} on COCO using U-Net.
\end{enumerate}


\section{Additional {\algName} Experiments}
\label{app:addl_lbgm_expt}

In this section, we present complimentary experiments to the properties studies in Section~\ref{sec:expt} of the main text. In particular, we show that our observations hold for datasets: CIFAR-10, CIFAR-100, CelebA, FMNIST, and MNIST. We also present results when using shallower models FCN or deeper models Resnet18 different from CNNs. Section~\ref{app:standalone_expt} gives further evidence of utility of {\algName} as a standalone solution. Section~\ref{app:rho_effect} lists figures that summarize the effect of changing $\subsup{\delta}{k}{\mathsf{threshold}}$ on model performance for mentioned datasets and models. In Section~\ref{app:plugnplay_expt} we list figures that summarize additional experiments to support the observations made in Fig.~\ref{fig:plugnplay} and Section~\ref{app:general_expt} lists figures for additional experiments corresponding to the observations made in Fig.~\ref{fig:dist_training}. Finally, Section~\ref{app:sampling_expt} presents results  on {\algName} algorithm corresponding to the case wherein $50\%$ of clients are randomly sampled during global aggregation. \textit{Since the two layer FCN considered is a simple classifier, it does not perform well on complex datasets such as CIFAR-10, CIFAR-100, and CelebA. Thus, the respective results are omitted for FCN on these complicated datasets and the performance of FCN is only studied for MNIST and FMNIST datasets. Similarly, the 4-layer CNN architecture does not perform well on CIFAR-100 dataset and hence the corresponding results are ommited. We also present results using U-Net architecture for semantic segmentation on PascalVOC dataset.}

\subsection{{\tt \algName} as a Standalone Algorithm.}
\label{app:standalone_expt}
\begin{enumerate}[leftmargin=5mm]
    \item Fig.~\ref{fig:standalone_cnn} shows the result of repeating the experiment conducted in Fig.~\ref{fig:standalone} (CNNs on non-iid data distribution) on CNNs for iid data distribution for datasets CIFAR-10, FMNIST, and MNIST, and U-Net for segmentation for dataset PascalVOC. 
    \item Fig.~\ref{fig:standalone_fcn} shows the result of repeating the experiment conducted in Fig.~\ref{fig:standalone} (CNNs on non-iid data distribution) on FCNs for both non-iid and iid data distributions on datasets FMNIST and MNIST.
    \item Fig.~\ref{fig:standalone_resnet18} shows the result of repeating the experiment conducted in Fig.~\ref{fig:standalone} (CNNs on non-iid data distribution) on Resnet18s for non-iid data distribution on datasets CelebA, CIFAR-10, CIFAR-100, FMNIST, and MNIST using a setup similar to that of \cite{wang2018atomo}.
\end{enumerate}

\subsection{Effect of $\subsup{\delta}{k}{\mathsf{threshold}}$ on {\algName}.}
\label{app:rho_effect}
\begin{enumerate}[leftmargin=5mm]
    \item Fig.~\ref{fig:rho_effect_cnn} shows the result of repeating the experiment conducted in Fig.~\ref{fig:rho_effect} (CNNs on non-iid data distribution) on CNNs for iid data distribution for datasets CIFAR-10, FMNIST, and MNIST, and U-Net for segmentation for dataset PascalVOC.
    \item Fig.~\ref{fig:rho_effect_fcn} shows the result of repeating the experiment conducted in Fig.~\ref{fig:rho_effect} (CNNs on non-iid data distribution) on FCNs for both non-iid and iid data distributions on datasets FMNIST and MNIST.
    \item Fig.~\ref{fig:rho_effect_resnet18} shows the result of repeating the experiment conducted in Fig.~\ref{fig:rho_effect} (CNNs on non-iid data distribution) on Resnet18s for non-iid data distribution on datasets CelebA, CIFAR-10 and CIFAR-100 using a setup similar to that of \cite{wang2018atomo}.
\end{enumerate}

\subsection{{\algName} as a Plug-and-Play Algorithm.}
\label{app:plugnplay_expt}
\begin{enumerate}[leftmargin=5mm]
    \item Fig.~\ref{fig:plugnplay_cnn} shows the result of repeating the experiment conducted in Fig.~\ref{fig:plugnplay} (CNNs on non-iid data distribution) on CNNs for iid data distribution for datasets CIFAR-10, FMNIST, and MNIST.
    \item Fig.~\ref{fig:plugnplay_fcn} shows the result of repeating the experiment conducted in Fig.~\ref{fig:plugnplay} (CNNs on non-iid data distribution) on FCNs for both non-iid and iid data distributions on datasets FMNIST and MNIST.
    \item Fig.~\ref{fig:plugnplay_resnet18} shows the result of repeating the experiment conducted in Fig.~\ref{fig:plugnplay} (CNNs on non-iid data distribution) on Resnet18s for non-iid data distribution on datasets CelebA, CIFAR-10 and CIFAR-100 using a setup similar to that of \cite{wang2018atomo}.
\end{enumerate}

\subsection{Generalizability of {\algName} to Distributed Training.}
\label{app:general_expt}
\begin{enumerate}[leftmargin=5mm]
    \item Fig.~\ref{fig:dist_training_cnn} shows the result of repeating the experiment conducted in Fig.~\ref{fig:dist_training} (CNNs on non-iid data distribution) on CNNs for iid data distribution for datasets CIFAR-10, FMNIST, and MNIST.
    \item Fig.~\ref{fig:dist_training_fcn} shows the result of repeating the experiment conducted in Fig.~\ref{fig:dist_training} (CNNs on non-iid data distribution) on FCNs for both non-iid and iid data distributions on datasets FMNIST and MNIST.
    \item Fig.~\ref{fig:dist_training_resnet18} shows the result of repeating the experiment conducted in Fig.~\ref{fig:dist_training} (CNNs on non-iid data distribution) on Resnet18s for non-iid data distribution on datasets CelebA, CIFAR-10 and CIFAR-100 using a setup similar to that of \cite{wang2018atomo}.
\end{enumerate}

\subsection{{\algName} under Client Sampling.}
\label{app:sampling_expt}

We present results with {\algName} under client sampling in this subsection. The results are qualitatively similar to those presented in Sec.~\ref{sec:expt} under ``{\algName} as Standalone Algorithm''. For example, our results on the MNIST dataset for $50\%$ client participation shows a $35\%$ and $55\%$ improvement in communication efficiency for only $0.2\%$ and $4\%$ drop in accuracy for the corresponding i.i.d and non-i.i.d cases (see column 1 of Fig.~\ref{fig:sampled_training_cnn_iid}\&\ref{fig:sampled_training_cnn_non_iid} respectively).
\begin{enumerate}[leftmargin=5mm]
    \item Fig.~\ref{fig:sampled_training_cnn_non_iid} shows the result of repeating the experiment conducted in Fig.~\ref{fig:standalone} (CNNs on non-iid data distribution) under $50\%$ client sampling using CNNs for non-iid data distribution for datasets CIFAR-10, FMNIST, and MNIST, and regression for dataset CelebA.
    \item Fig.~\ref{fig:sampled_training_cnn_iid} shows the result of repeating the experiment conducted in Fig.~\ref{fig:standalone} (CNNs on non-iid data distribution) under $50\%$ client sampling using CNNs for iid data distribution for datasets CIFAR-10, FMNIST, and MNIST.
\end{enumerate}


\begin{figure}[h!]
  \centering
  \centerline{\includegraphics[width=0.8\textwidth]{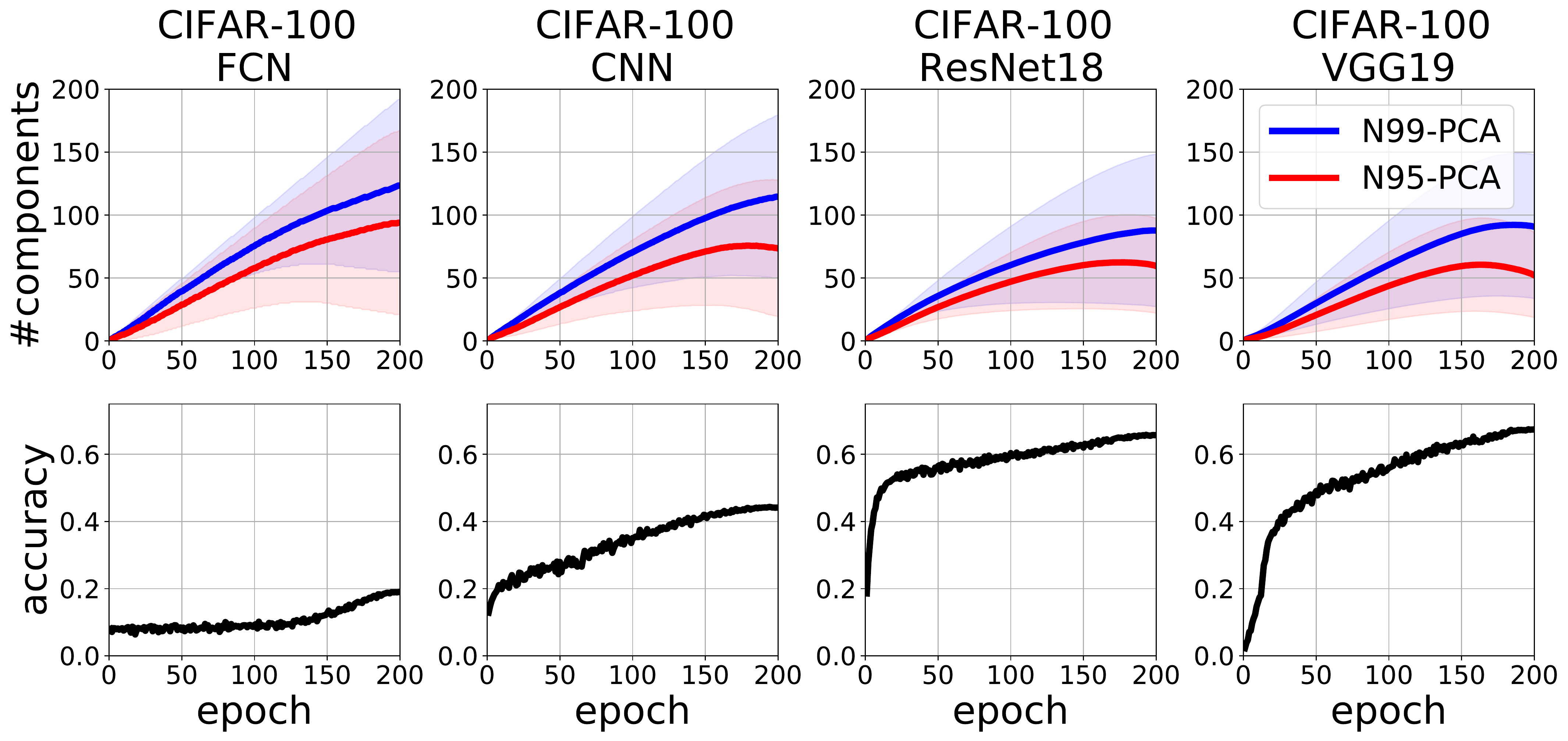}}
  \caption{\small{\textit{PCA Components Progression}. Repeat of Fig.~\ref{fig:prelim_1} on CIFAR-100 dataset.}}
  \label{fig:prelim_1_cifar100}
\end{figure}

\begin{figure}[h!]
  \centering
  \centerline{\includegraphics[width=0.8\textwidth]{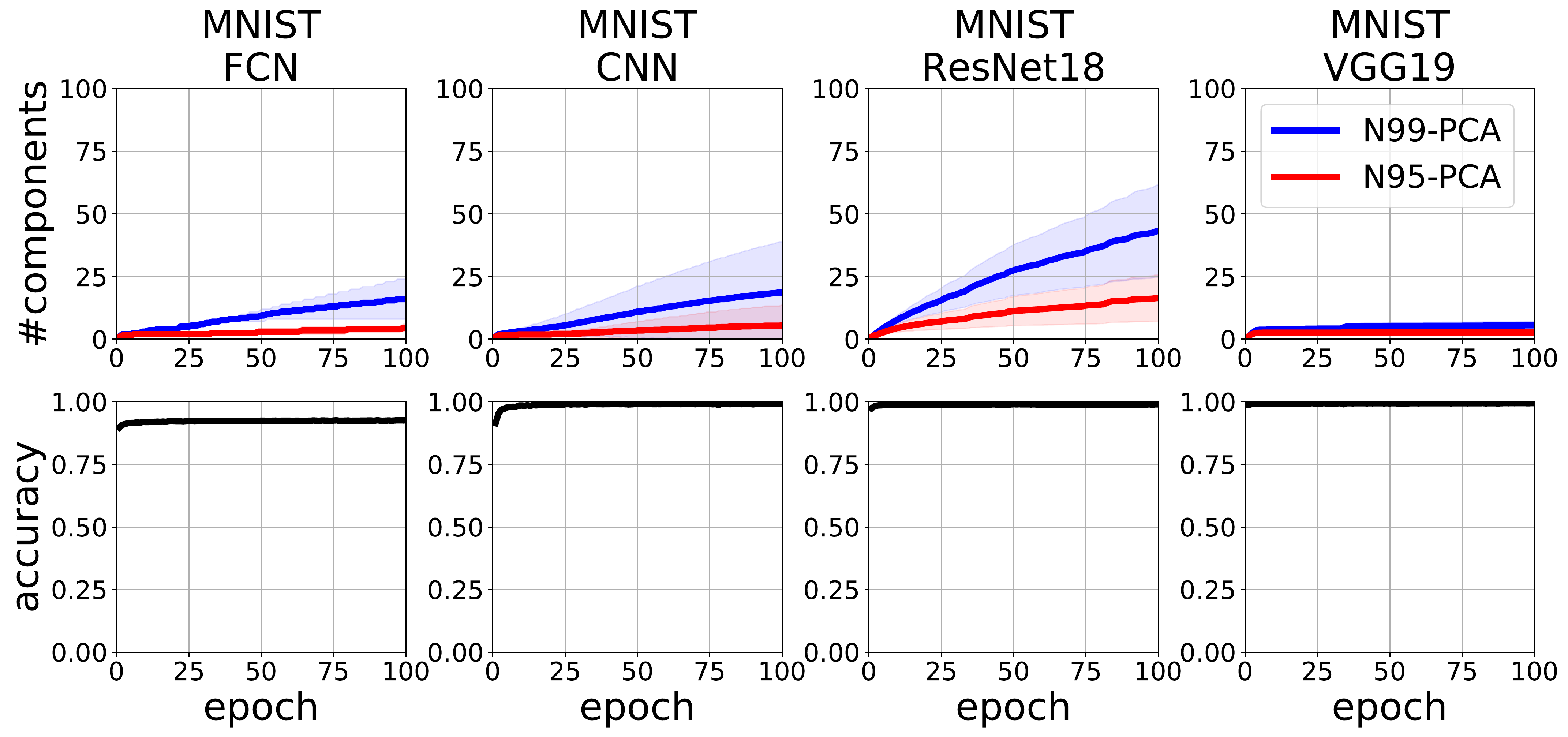}}
  \caption{\small{\textit{PCA Components Progression}. Repeat of Fig.~\ref{fig:prelim_1} on MNIST dataset.}}
  \label{fig:prelim_1_mnist}
\end{figure}

\begin{figure}[h!]
  \centering
  \centerline{\includegraphics[width=0.8\textwidth]{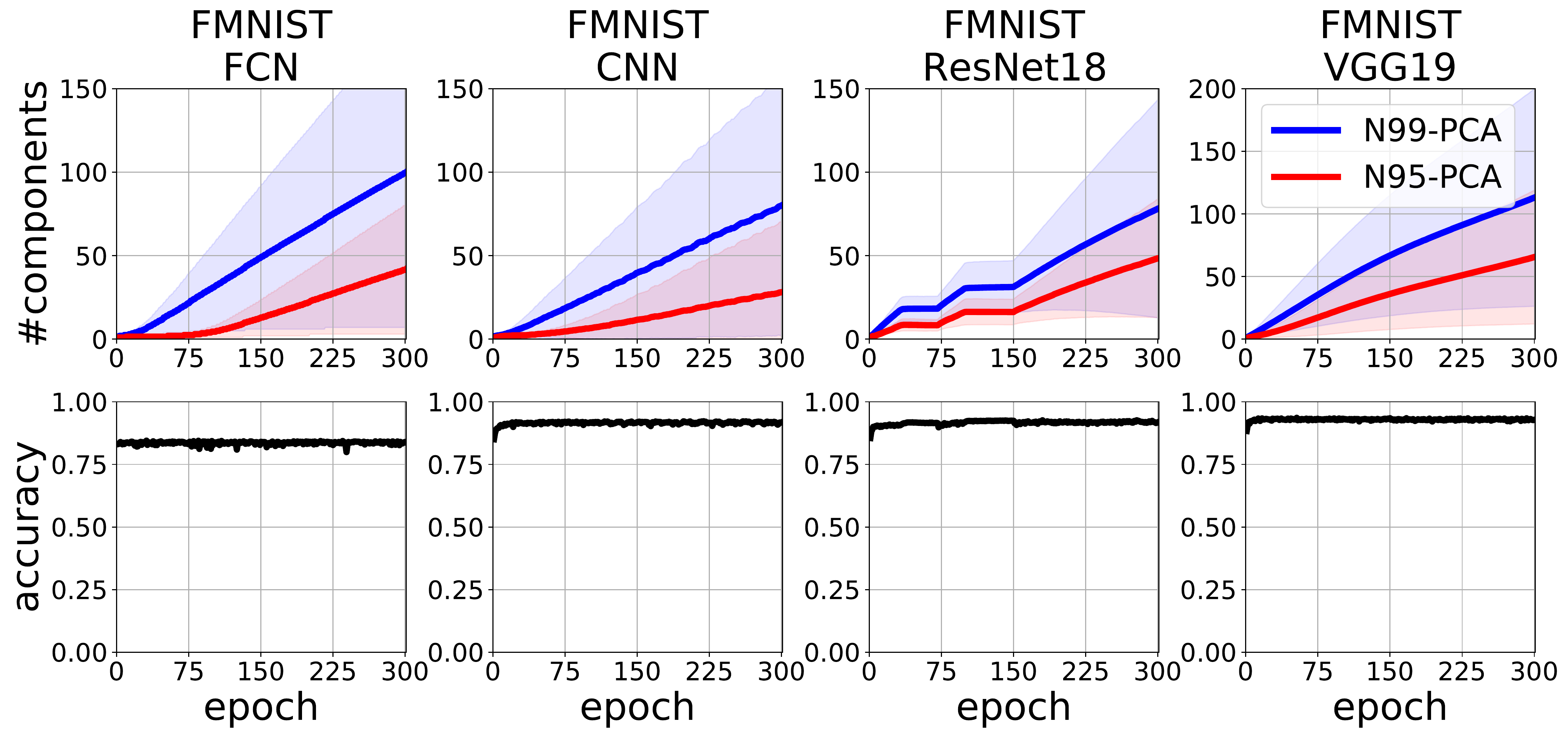}}
  \caption{\small{\textit{PCA Components Progression}. Repeat of Fig.~\ref{fig:prelim_1} on FMNIST dataset.}}
  \label{fig:prelim_1_fmnist}
\end{figure}

\begin{figure}[h!]
  \centering
  \centerline{\includegraphics[width=0.6\textwidth]{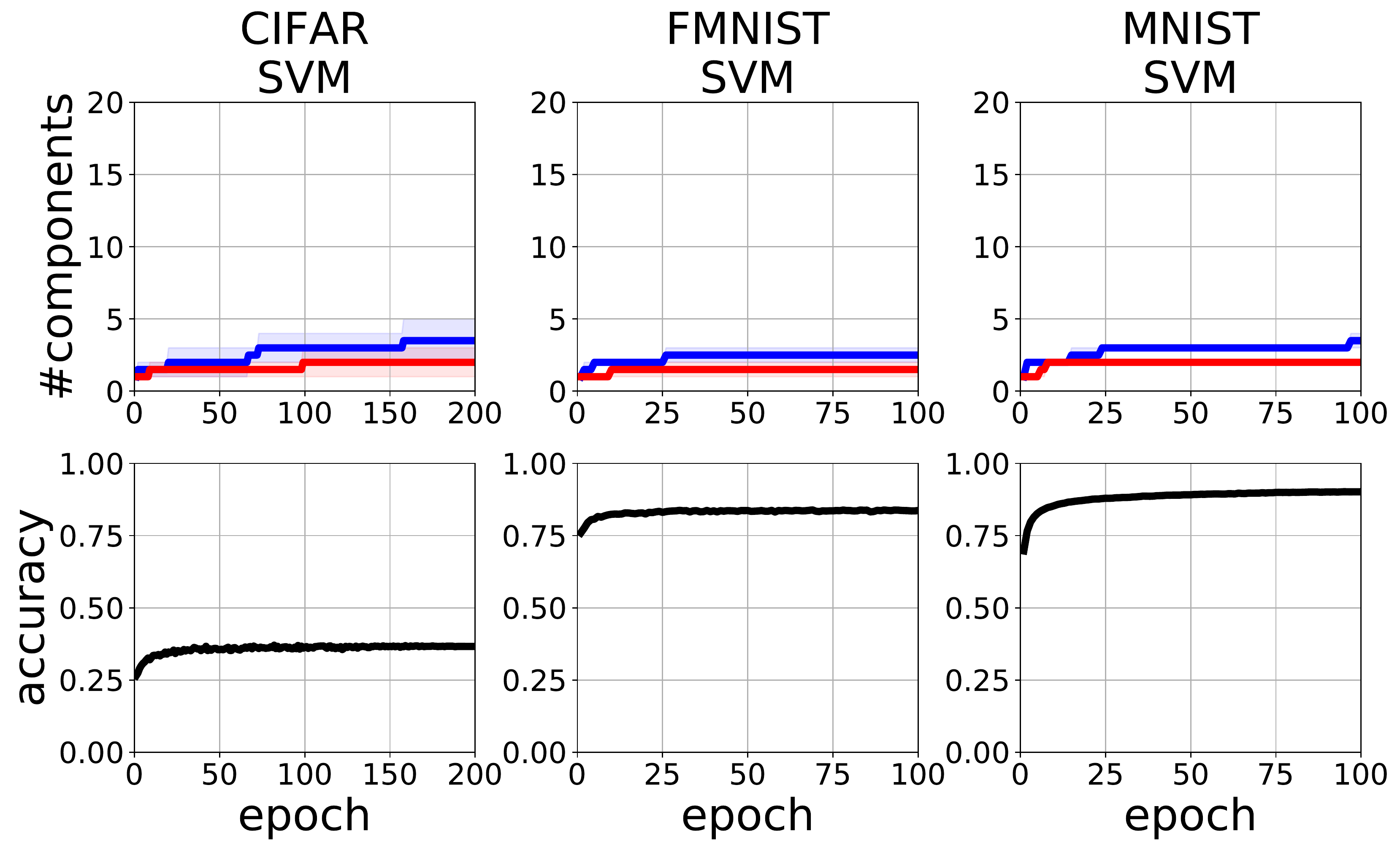}}
  \caption{\small{\textit{PCA Components Progression}. Repeat of Fig.~\ref{fig:prelim_1} on CIFAR-10, F-MNIST, and MNIST datasets but using squared SVM classifier.}}
  \label{fig:prelim_1_svm}
\end{figure}

\begin{figure}[h!]
  \centering
  \centerline{\includegraphics[width=0.5\textwidth]{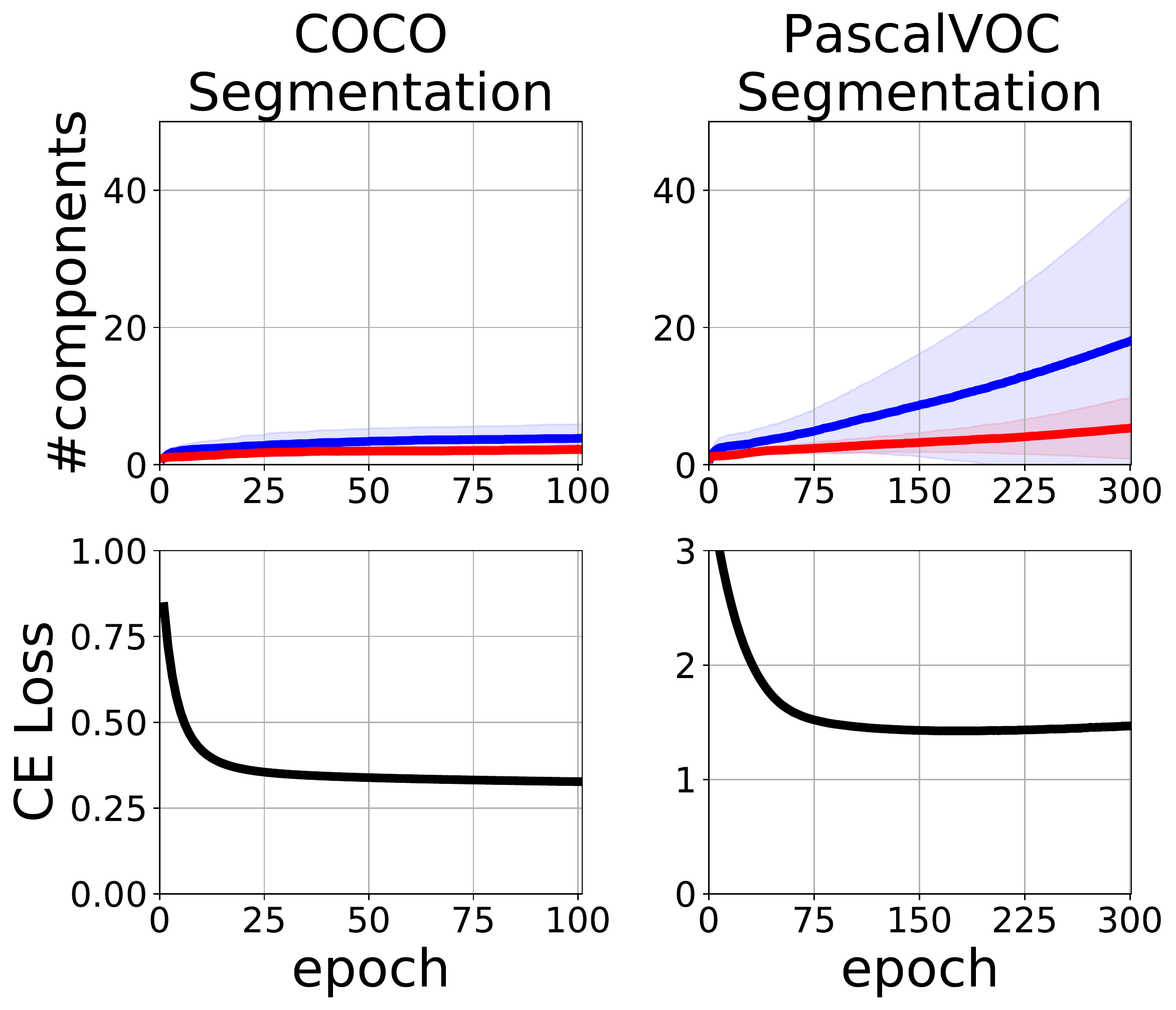}}
  \caption{\small{\textit{PCA Components Progression}. Repeat of Fig.~\ref{fig:prelim_1} on COCO, and PascalVOC datasets but using U-Net classifier.}}
  \label{fig:prelim_1_seg}
\end{figure}

\newpage

\begin{figure}[h!]
  \centering
  \centerline{\includegraphics[width=1.2\textwidth]{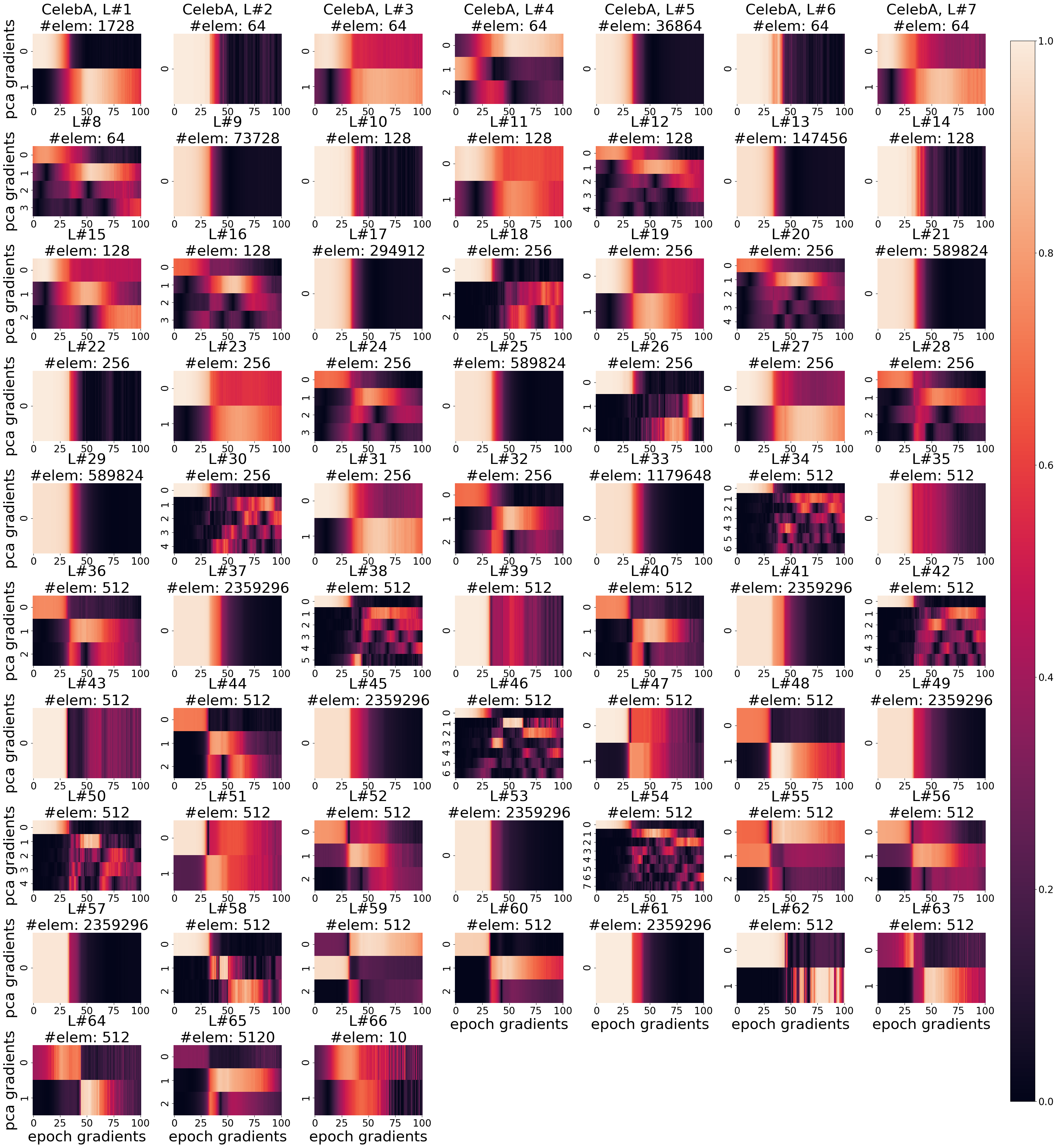}}
  \caption{\small{\textit{PCA Components Overlap with Gradient}. Repeat of Fig.~\ref{fig:prelim_2} on \textbf{VGG19} trained on \textbf{CelebA} dataset.}}
  \label{fig:prelim_2_celeba_vgg19}
\end{figure}

\begin{figure}[h!]
  \centering
  \centerline{\includegraphics[width=1.2\textwidth]{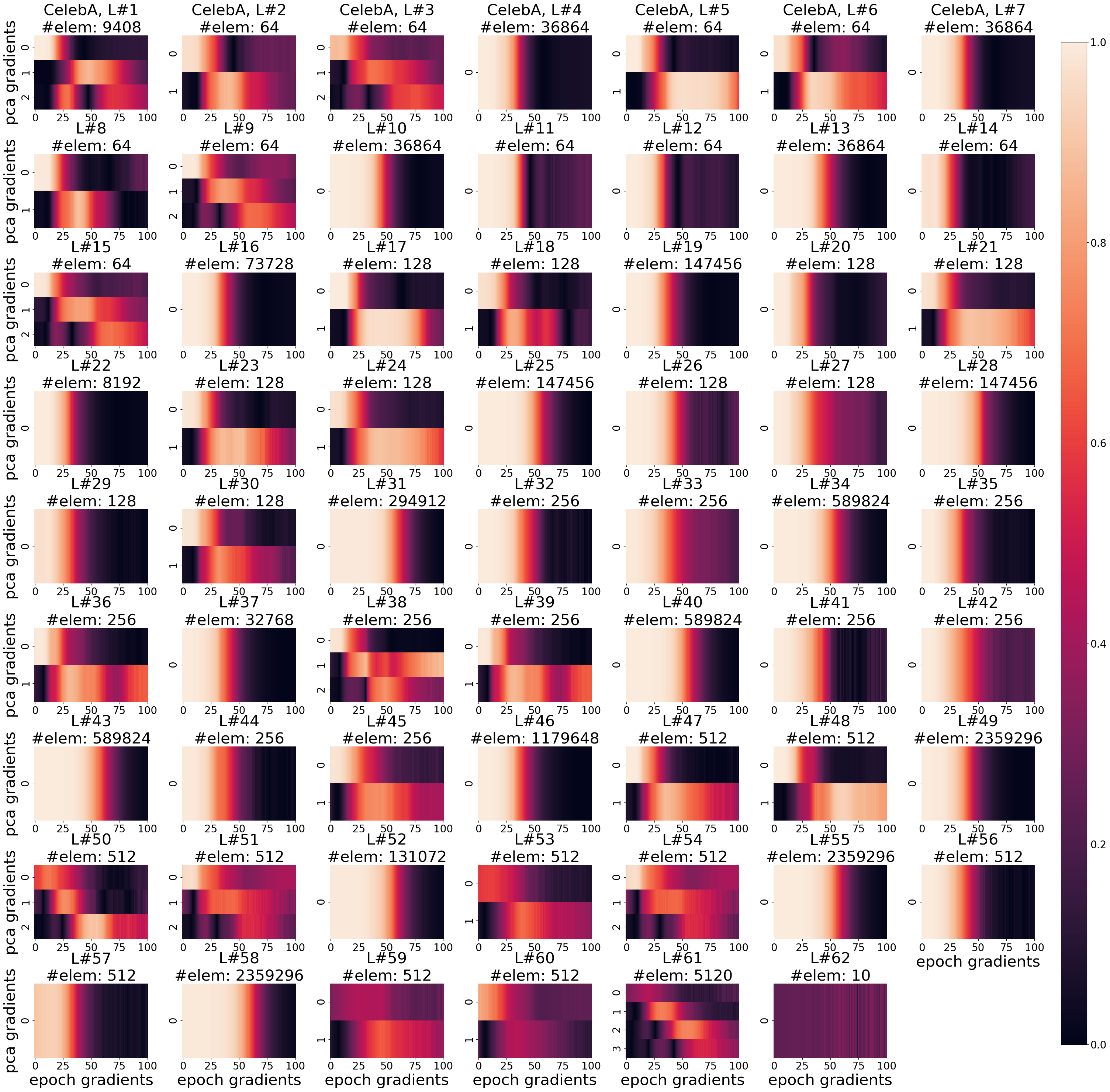}}
  \caption{\small{\textit{PCA Components Overlap with Gradient}. Repeat of Fig.~\ref{fig:prelim_2} on \textbf{ResNet18} trained on \textbf{CelebA} dataset.}}
  \label{fig:prelim_2_celeba_resnet18}
\end{figure}

\begin{figure}[h!]
  \centering
  \centerline{\includegraphics[width=0.7\textwidth]{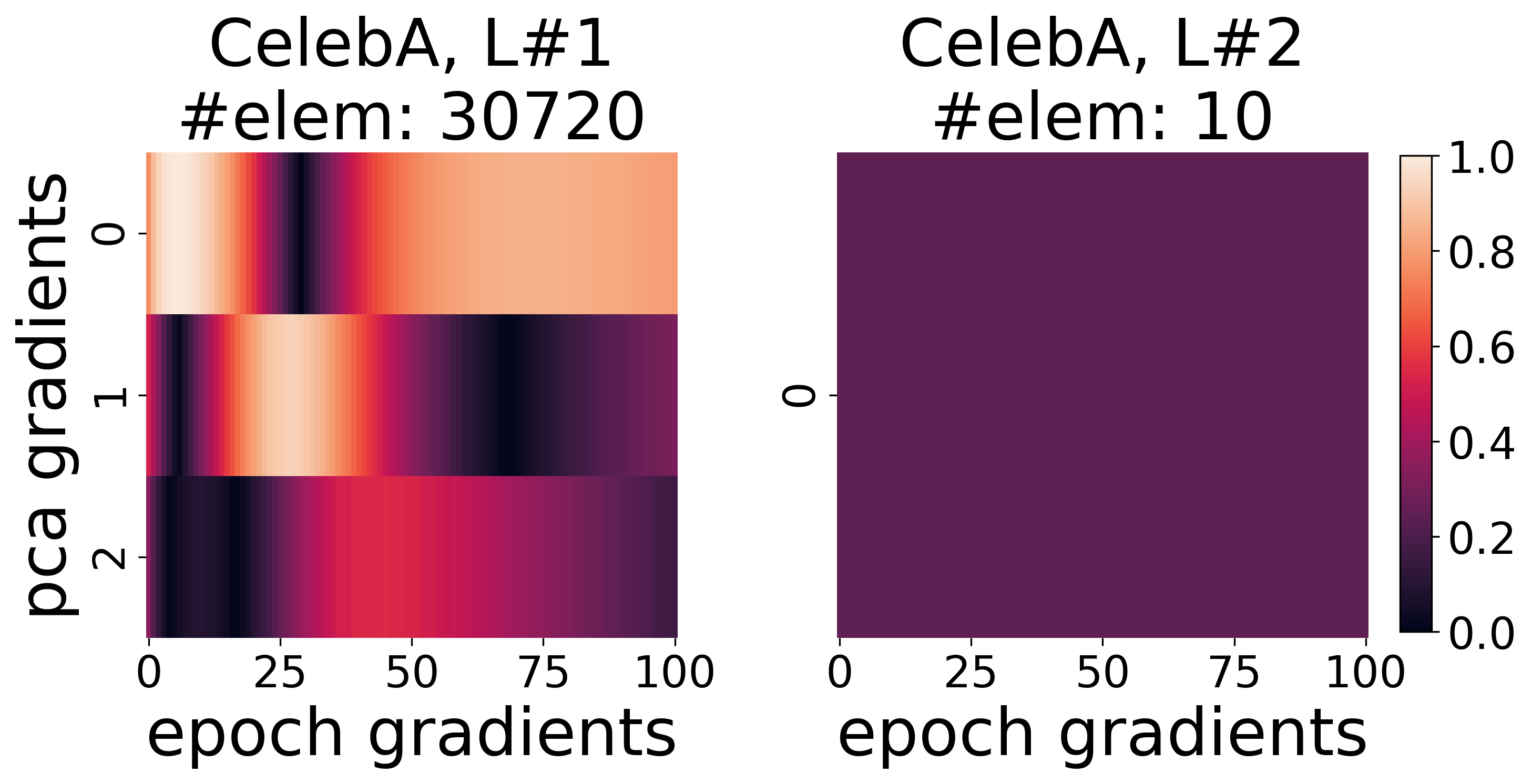}}
  \caption{\small{\textit{PCA Components Overlap with Gradient}. Repeat of Fig.~\ref{fig:prelim_2} on \textbf{FCN} trained on \textbf{CelebA} dataset.}}
  \label{fig:prelim_2_celeba_fcn}
\end{figure}

\begin{figure}[h!]
  \centering
  \centerline{\includegraphics[width=1.0\textwidth]{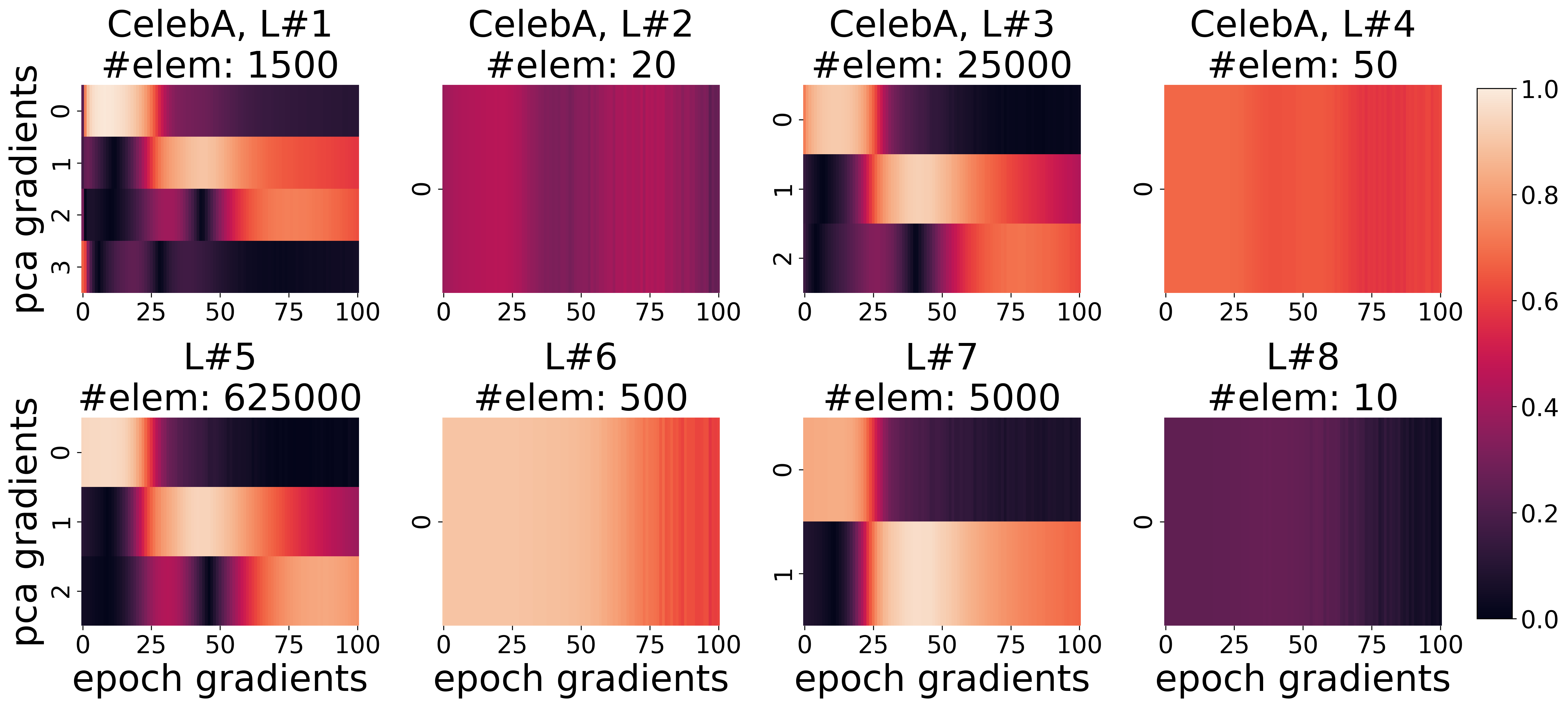}}
  \caption{\small{\textit{PCA Components Overlap with Gradient}. Fig.~\ref{fig:prelim_2} on \textbf{CNN} trained on \textbf{CelebA} dataset.}}
  \label{fig:prelim_2_celeba_cnn}
\end{figure}

\clearpage

\begin{figure}[h!]
  \centering
  \centerline{\includegraphics[width=1.2\textwidth]{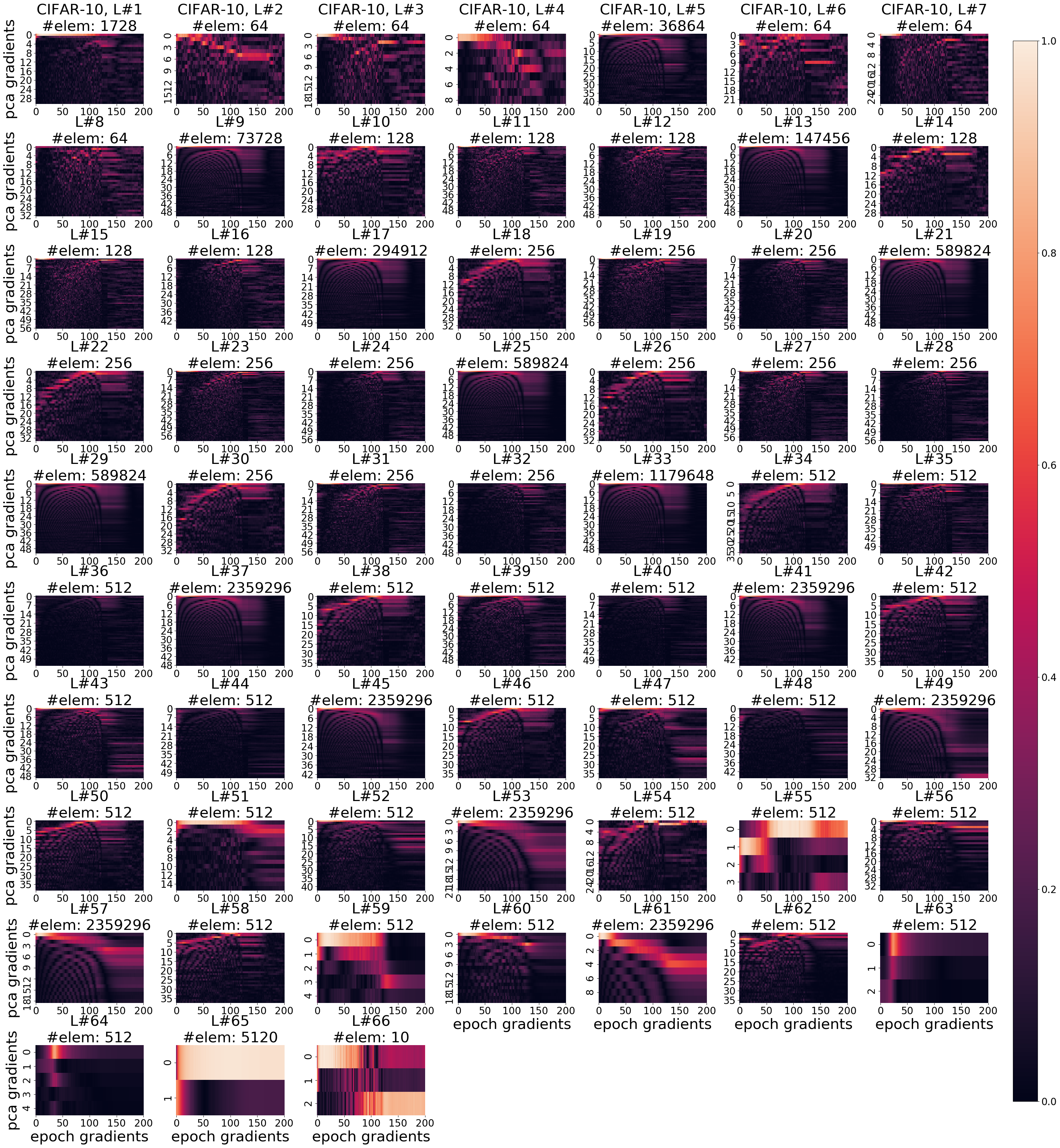}}
  \caption{\small{\textit{PCA Components Overlap with Gradient}. Repeat of Fig.~\ref{fig:prelim_2} on \textbf{VGG19} trained on \textbf{CIFAR-10} dataset.}}
  \label{fig:prelim_2_cifar_vgg19}
\end{figure}

\begin{figure}[h!]
  \centering
  \centerline{\includegraphics[width=1.2\textwidth]{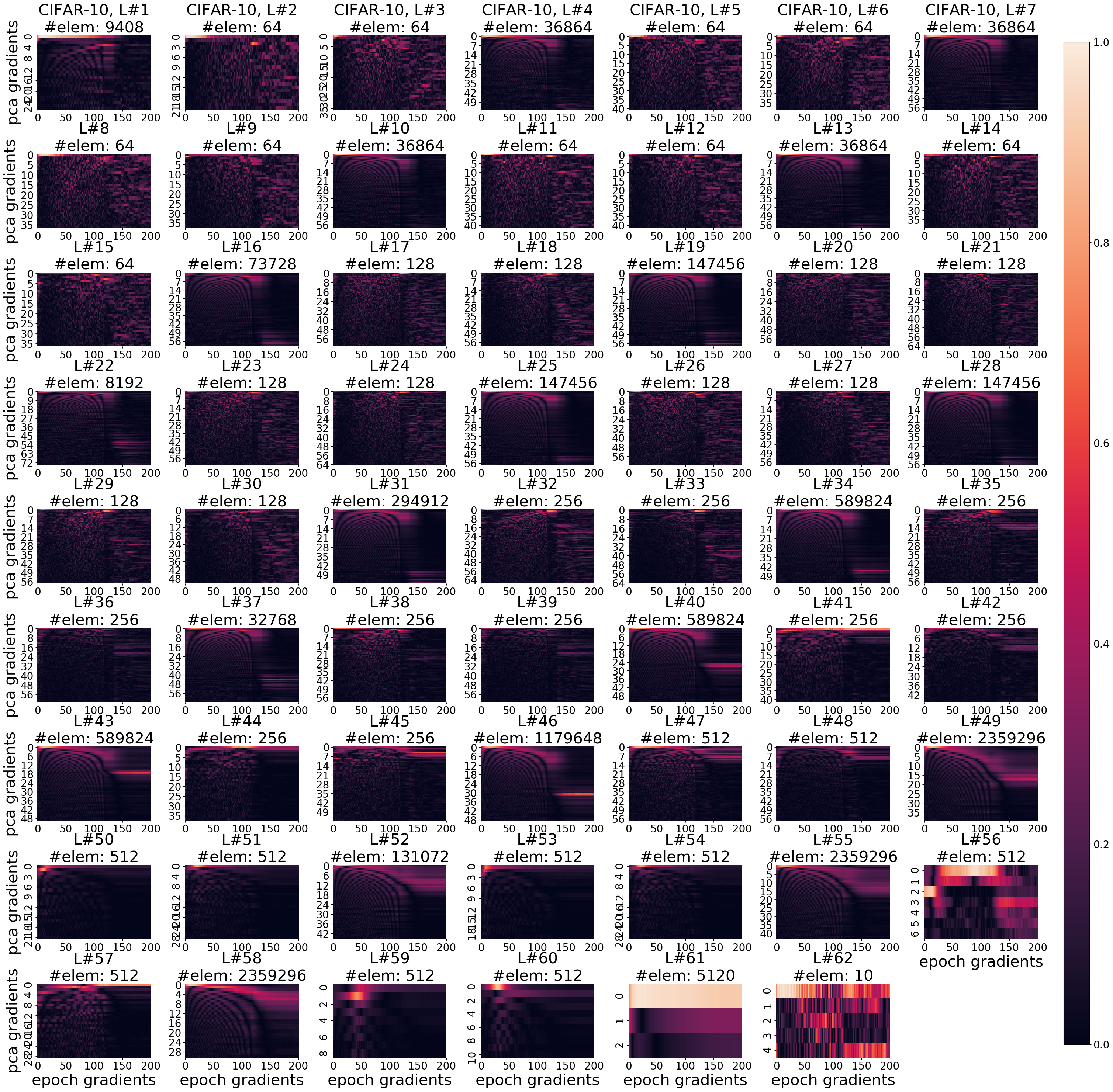}}
  \caption{\small{\textit{PCA Components Overlap with Gradient}. Repeat of Fig.~\ref{fig:prelim_2} on \textbf{ResNet18} trained on \textbf{CIFAR-10} dataset.}}
  \label{fig:prelim_2_cifar_resnet18}
\end{figure}

\begin{figure}[h!]
  \centering
  \centerline{\includegraphics[width=0.7\textwidth]{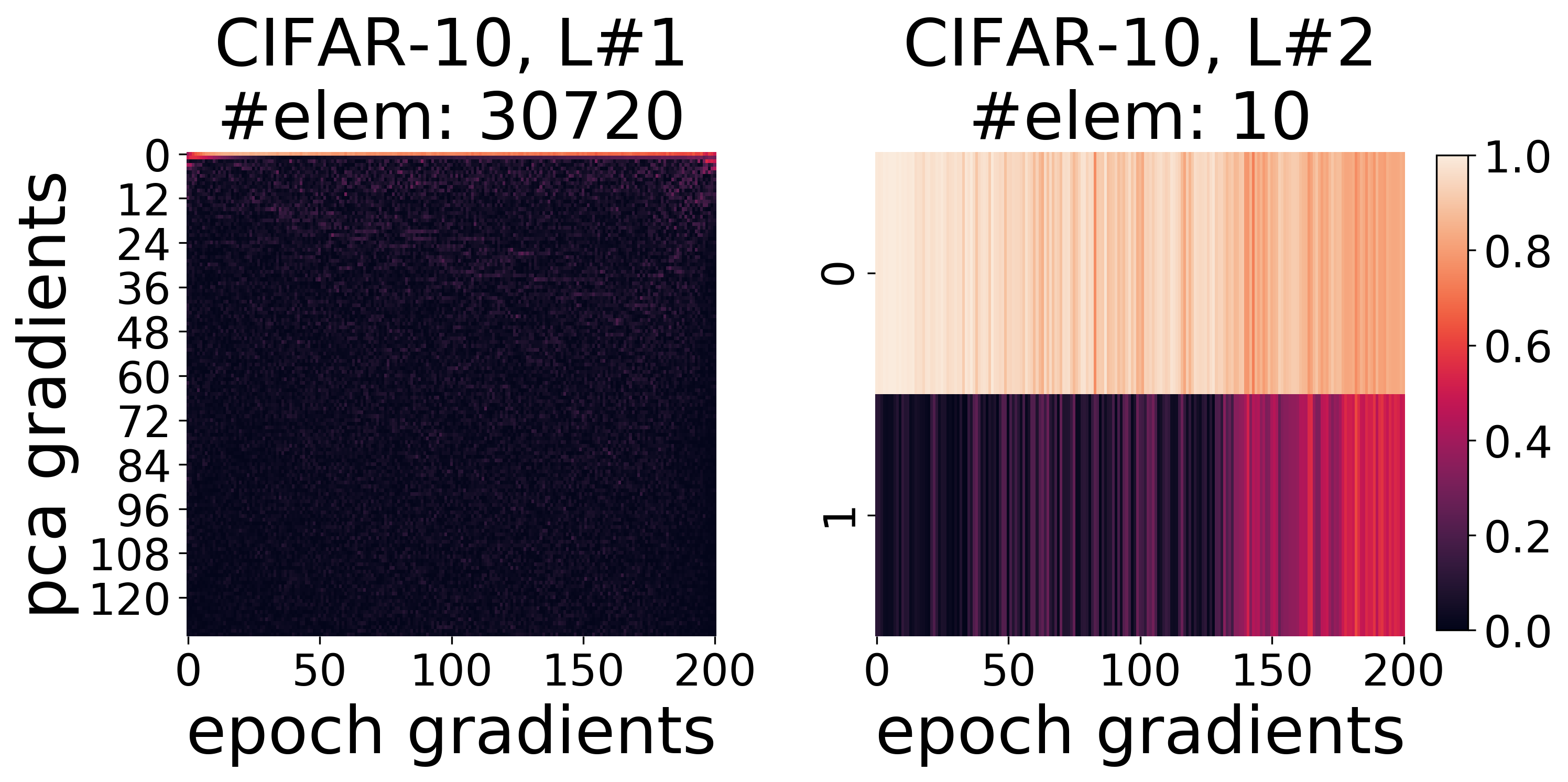}}
  \caption{\small{\textit{PCA Components Overlap with Gradient}. Repeat of Fig.~\ref{fig:prelim_2} on \textbf{FCN} trained on \textbf{CIFAR-10} dataset.}}
  \label{fig:prelim_2_cifar_fcn}
\end{figure}

\begin{figure}[h!]
  \centering
  \centerline{\includegraphics[width=1.0\textwidth]{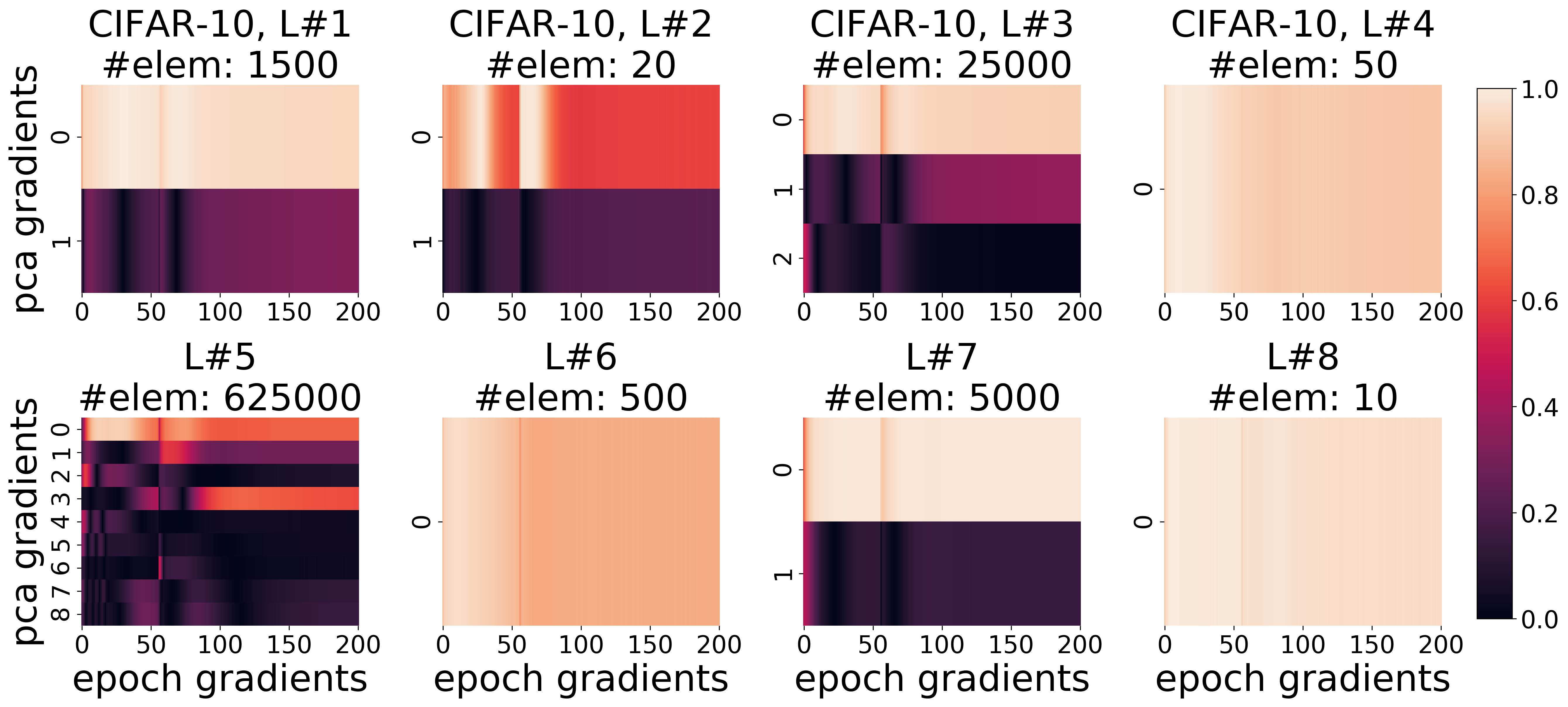}}
  \caption{\small{\textit{PCA Components Overlap with Gradient}. Fig.~\ref{fig:prelim_2} on \textbf{CNN} trained on \textbf{CIFAR-10} dataset.}}
  \label{fig:prelim_2_cifar_cnn}
\end{figure}

\clearpage

\begin{figure}[h!]
  \centering
  \centerline{\includegraphics[width=1.2\textwidth]{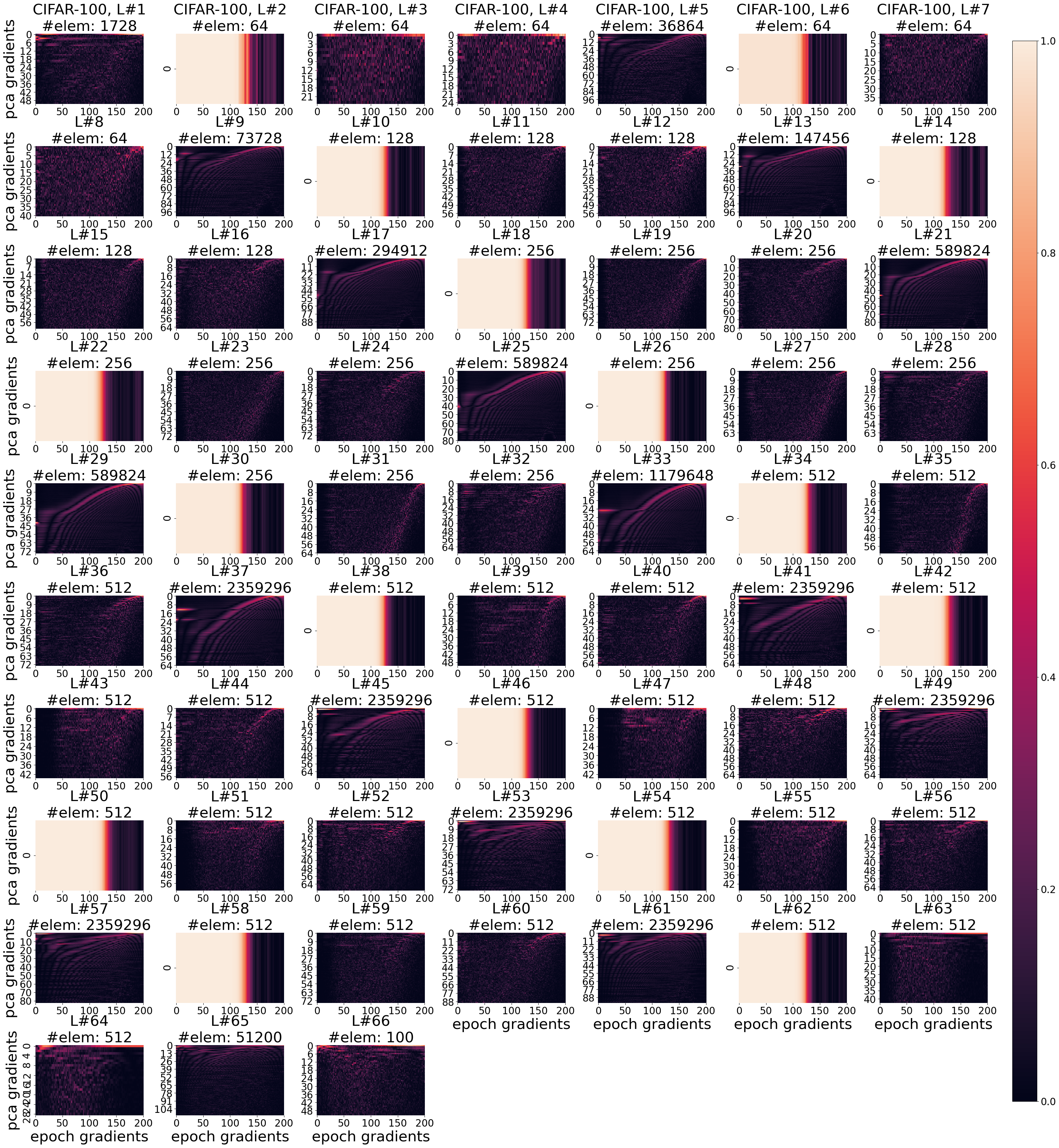}}
  \caption{\small{\textit{PCA Components Overlap with Gradient}. Repeat of Fig.~\ref{fig:prelim_2} on \textbf{VGG19} trained on \textbf{CIFAR-100} dataset.}}
  \label{fig:prelim_2_cifar100_vgg19}
\end{figure}

\begin{figure}[h!]
  \centering
  \centerline{\includegraphics[width=1.2\textwidth]{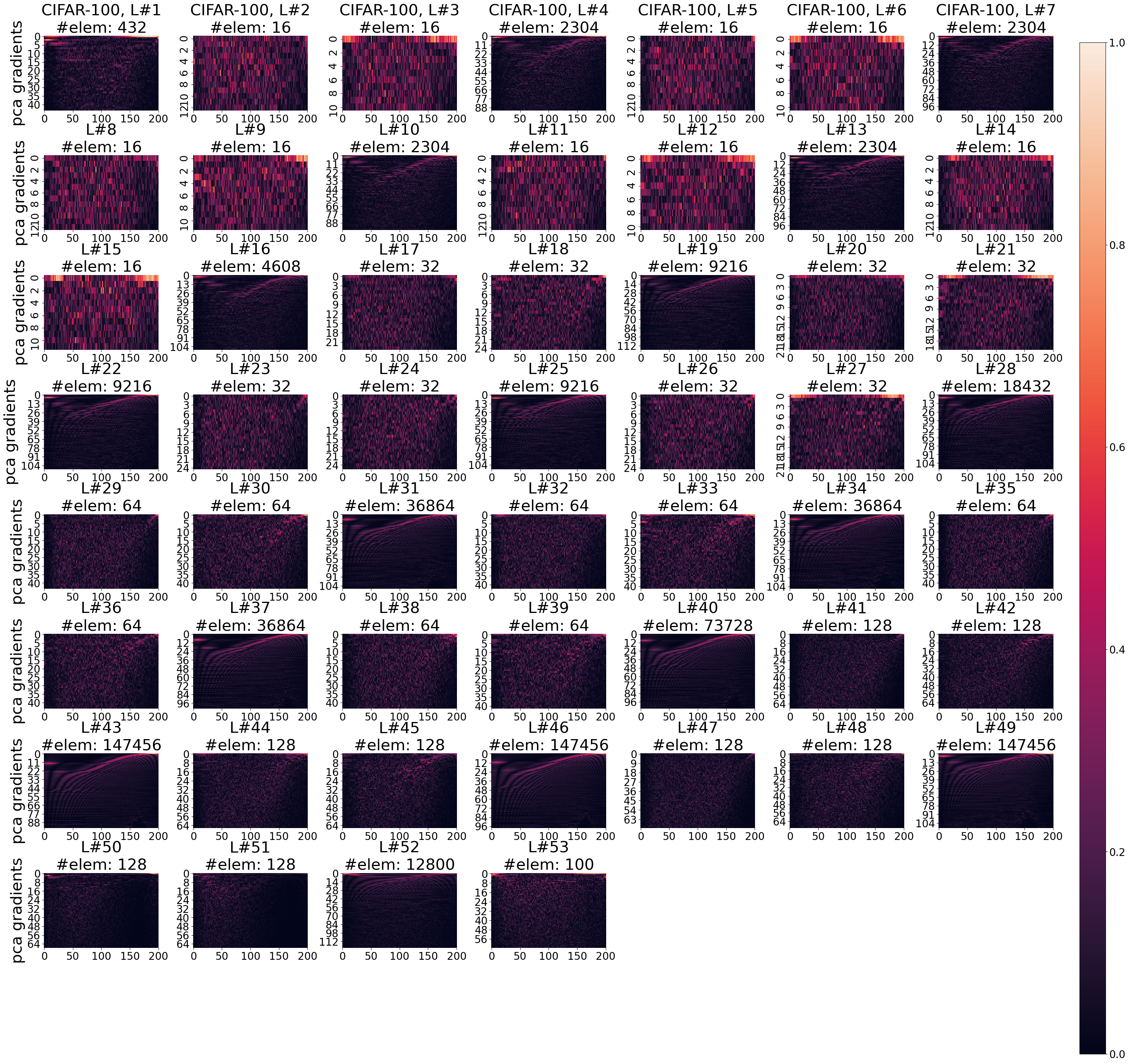}}
  \caption{\small{\textit{PCA Components Overlap with Gradient}. Repeat of Fig.~\ref{fig:prelim_2} on \textbf{ResNet18} trained on \textbf{CIFAR-100} dataset.}}
  \label{fig:prelim_2_cifar100_resnet18}
\end{figure}
\clearpage
\begin{figure}[h!]
  \centering
  \centerline{\includegraphics[width=0.7\textwidth]{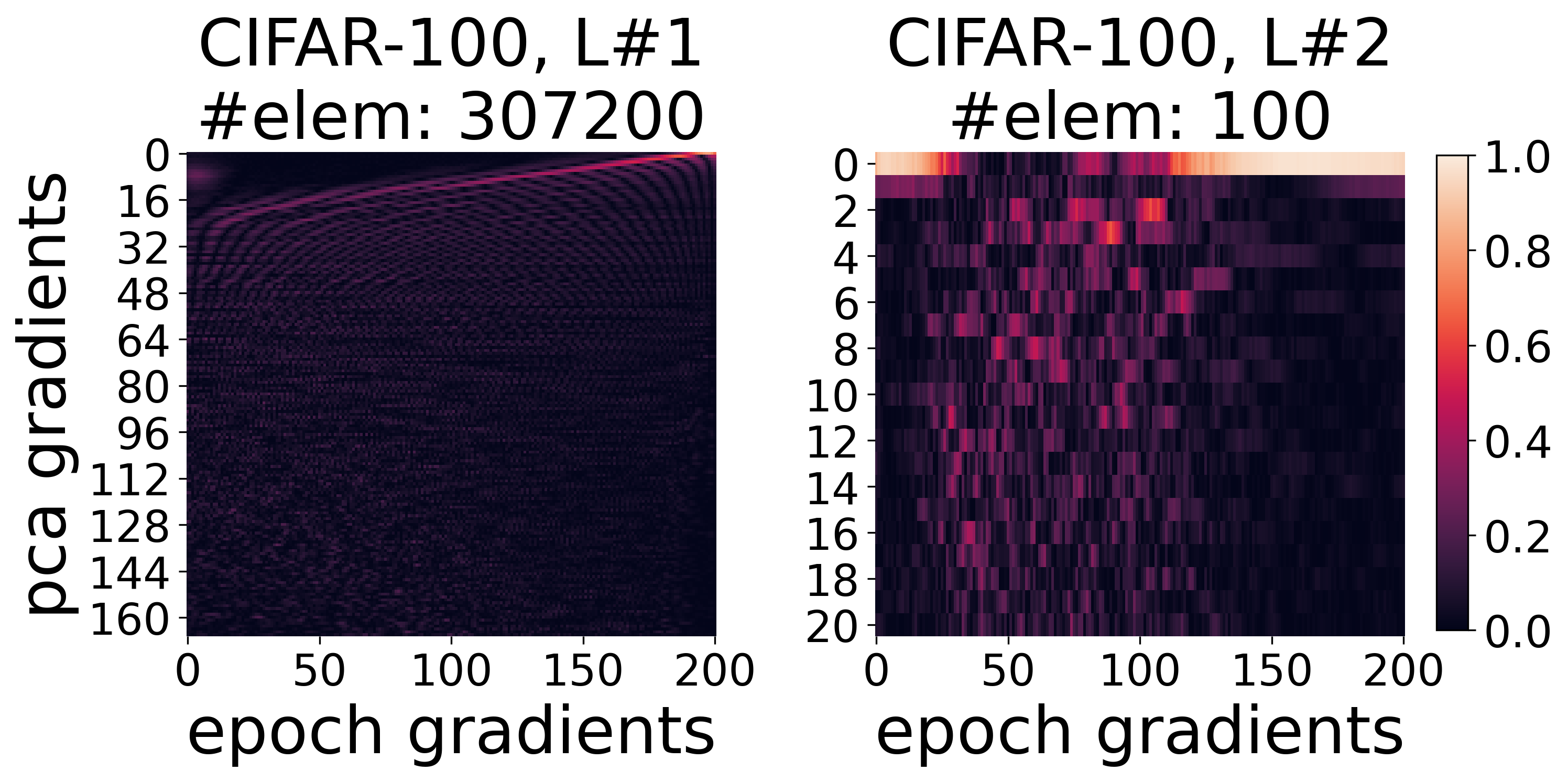}}
  \caption{\small{\textit{PCA Components Overlap with Gradient}. Repeat of Fig.~\ref{fig:prelim_2} on \textbf{FCN} trained on \textbf{CIFAR-100} dataset.}}
  \label{fig:prelim_2_cifar100_fcn}
\end{figure}

\begin{figure}[h!]
  \centering
  \centerline{\includegraphics[width=1.0\textwidth]{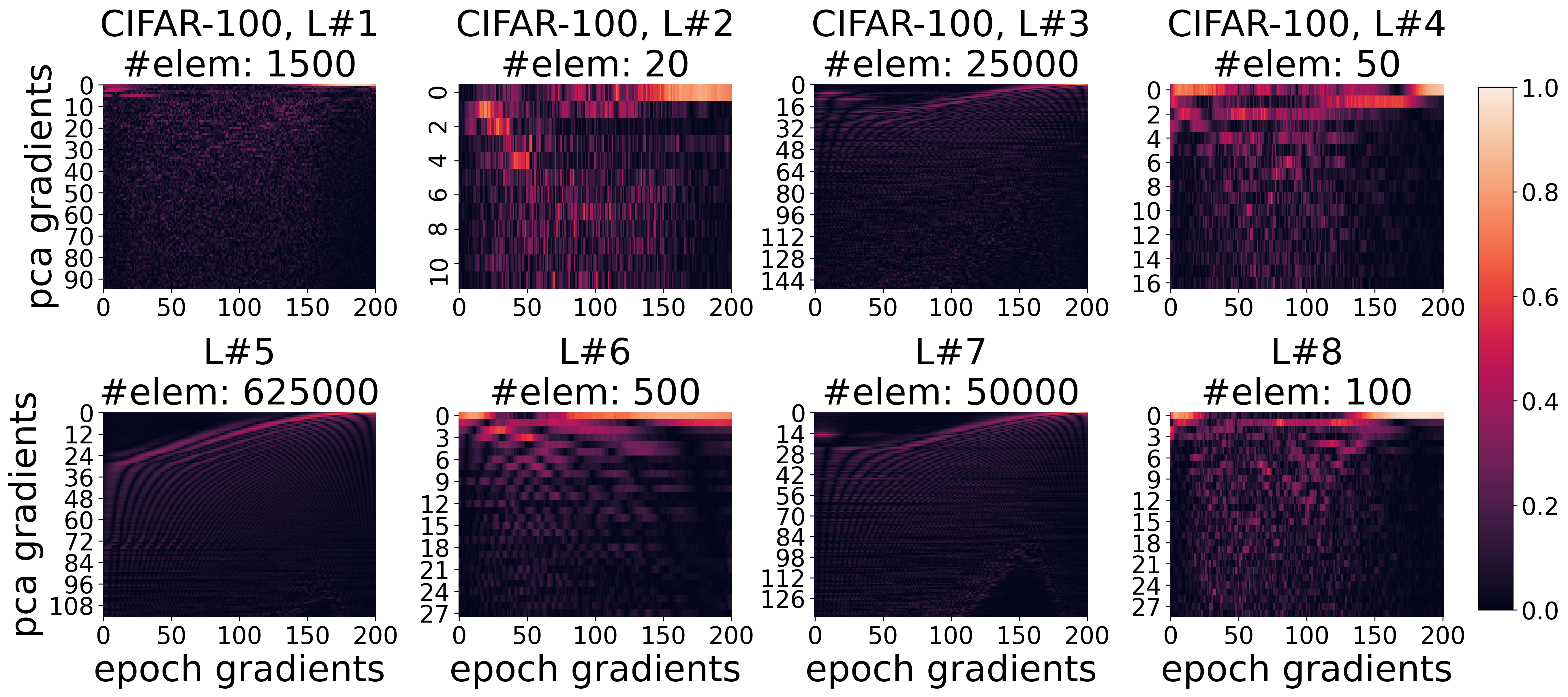}}
  \caption{\small{\textit{PCA Components Overlap with Gradient}. Fig.~\ref{fig:prelim_2} on \textbf{CNN} trained on \textbf{CIFAR-100} dataset.}}
  \label{fig:prelim_2_cifar100_cnn}
\end{figure}

\clearpage

\begin{figure}[h!]
  \centering
  \centerline{\includegraphics[width=1.2\textwidth]{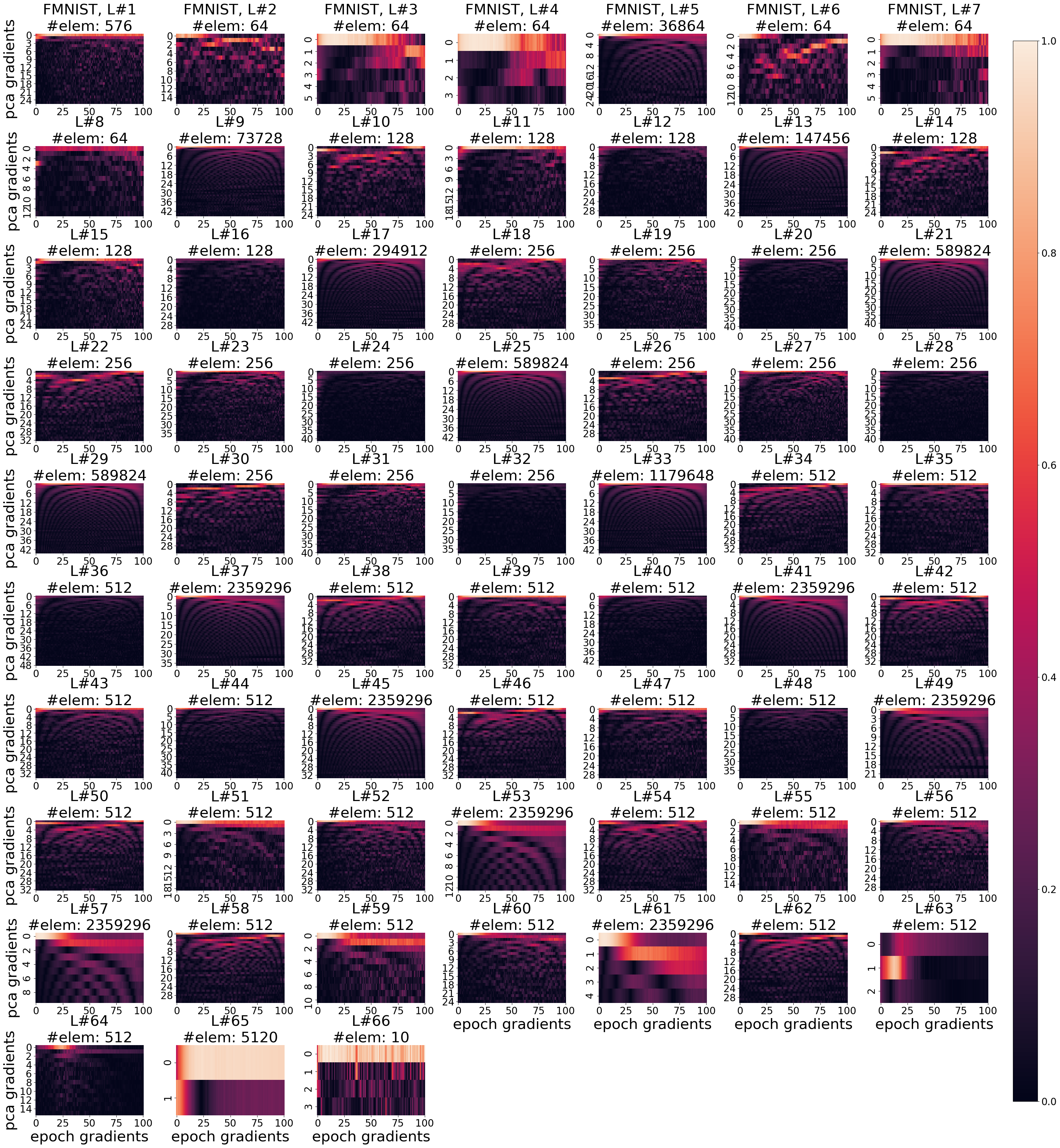}}
  \caption{\small{\textit{PCA Components Overlap with Gradient}. Repeat of Fig.~\ref{fig:prelim_2} on \textbf{VGG19} trained on \textbf{FMNIST} dataset.}}
  \label{fig:prelim_2_fmnist_vgg19}
\end{figure}

\begin{figure}[h!]
  \centering
  \centerline{\includegraphics[width=1.2\textwidth]{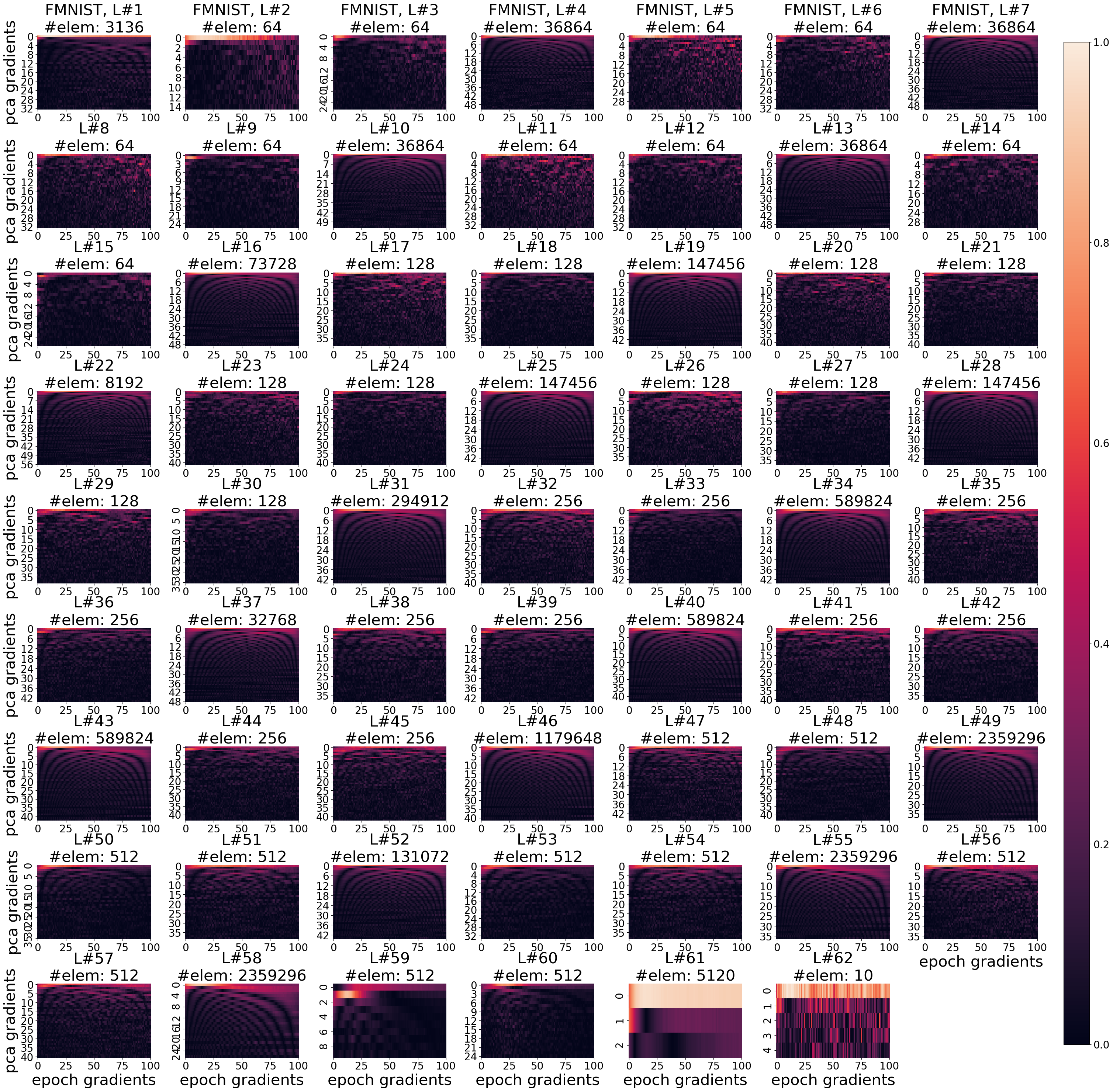}}
  \caption{\small{\textit{PCA Components Overlap with Gradient}. Repeat of Fig.~\ref{fig:prelim_2} on \textbf{ResNet18} trained on \textbf{FMNIST} dataset.}}
  \label{fig:prelim_2_fmnist_resnet18}
\end{figure}

\begin{figure}[h!]
  \centering
  \centerline{\includegraphics[width=0.7\textwidth]{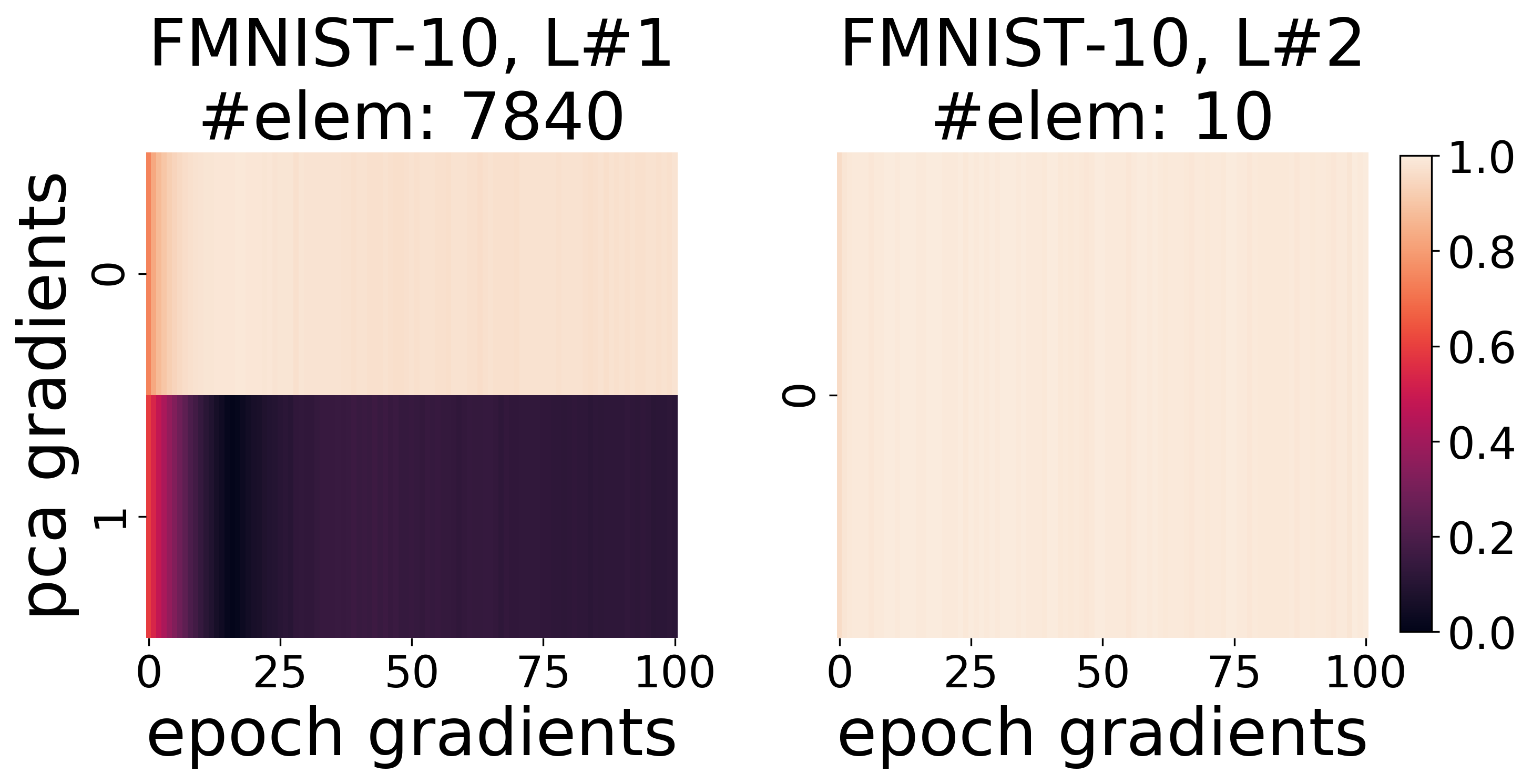}}
  \caption{\small{\textit{PCA Components Overlap with Gradient}. Repeat of Fig.~\ref{fig:prelim_2} on \textbf{FCN} trained on \textbf{FMNIST} dataset.}}
  \label{fig:prelim_2_fmnist_fcn}
\end{figure}

\begin{figure}[h!]
  \centering
  \centerline{\includegraphics[width=1.0\textwidth]{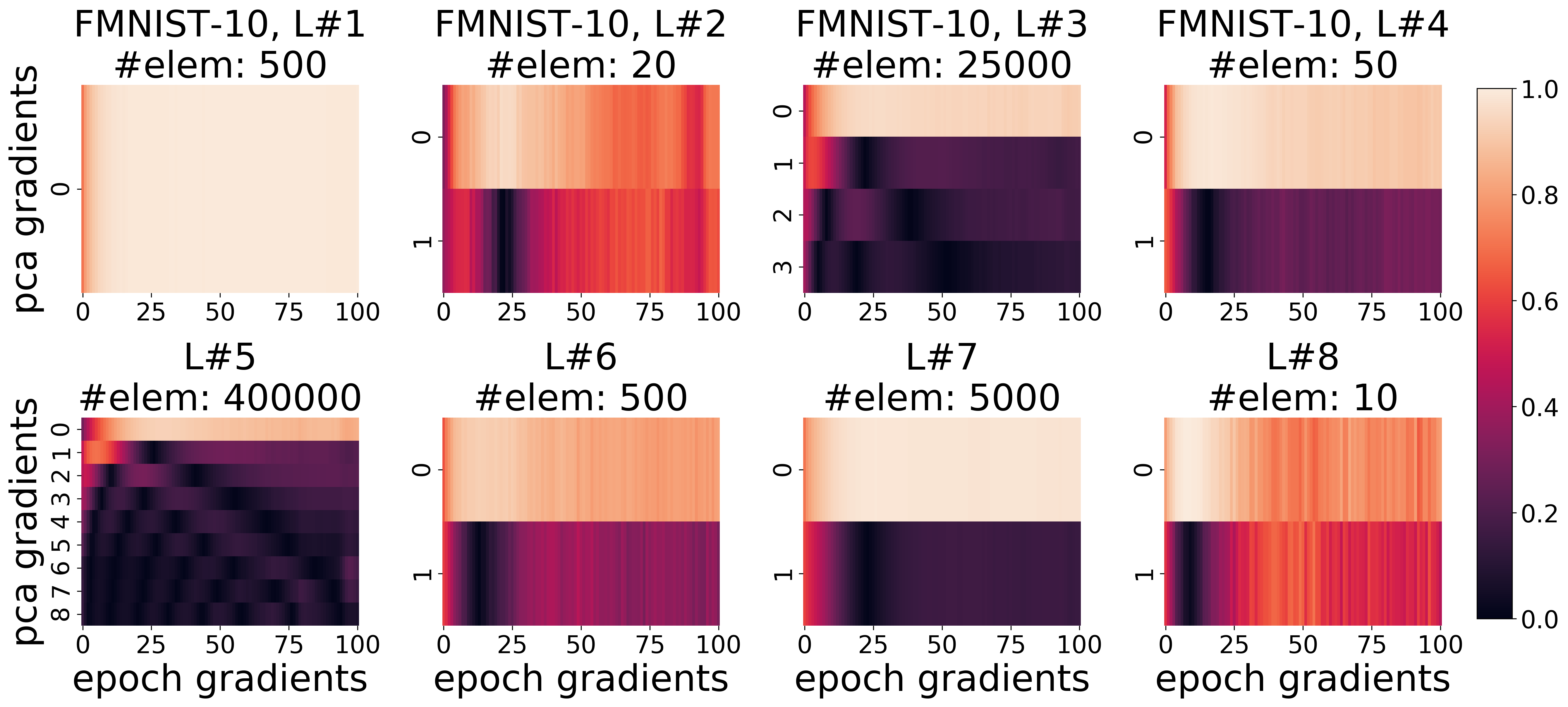}}
  \caption{\small{\textit{PCA Components Overlap with Gradient}. Fig.~\ref{fig:prelim_2} on \textbf{CNN} trained on \textbf{FMNIST} dataset.}}
  \label{fig:prelim_2_fmnist_cnn}
\end{figure}

\clearpage

\begin{figure}[h!]
  \centering
  \centerline{\includegraphics[width=1.2\textwidth]{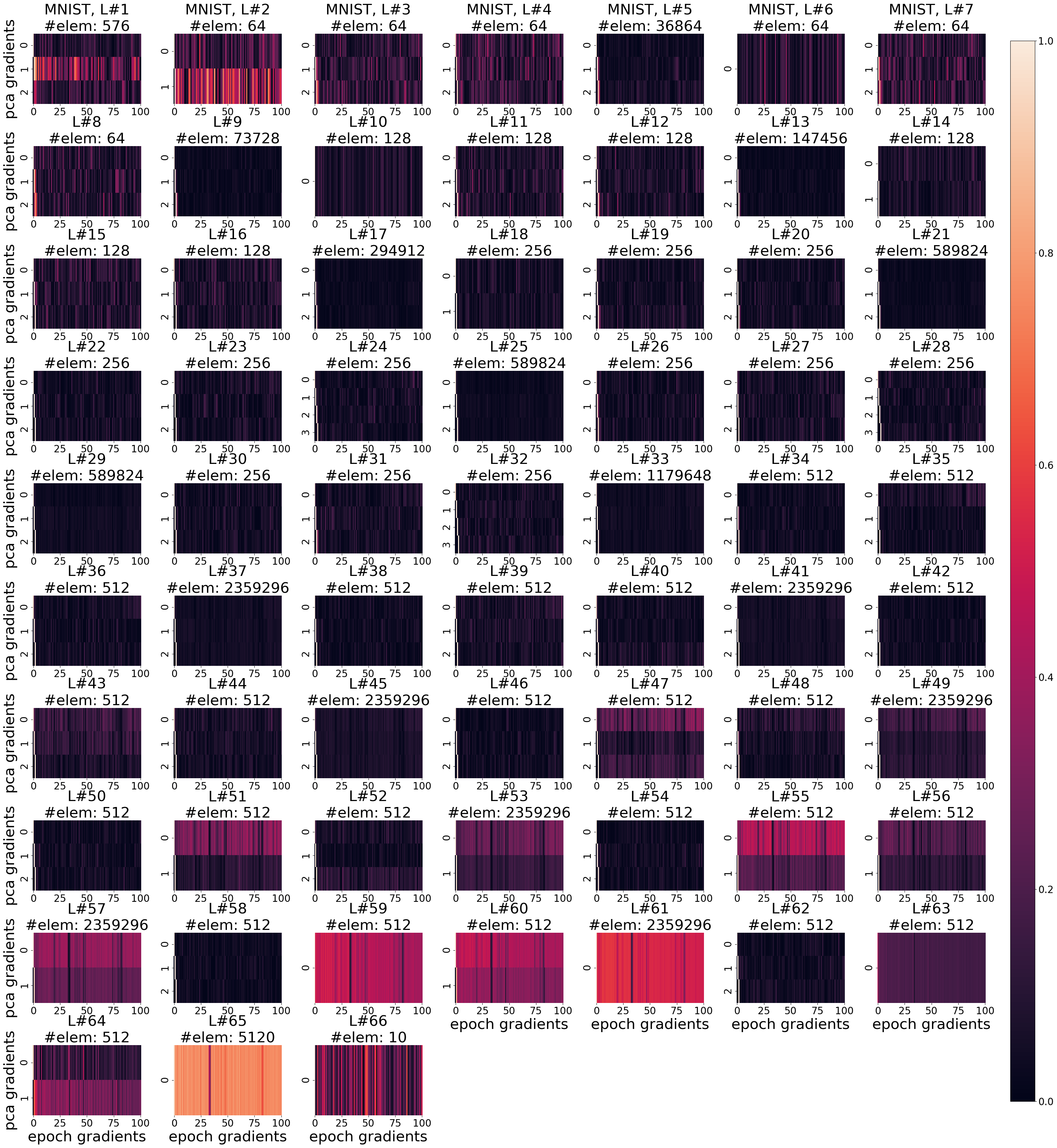}}
  \caption{\small{\textit{PCA Components Overlap with Gradient}. Repeat of Fig.~\ref{fig:prelim_2} on \textbf{VGG19} trained on \textbf{MNIST} dataset.}}
  \label{fig:prelim_2_mnist_vgg19}
\end{figure}

\begin{figure}[h!]
  \centering
  \centerline{\includegraphics[width=1.2\textwidth]{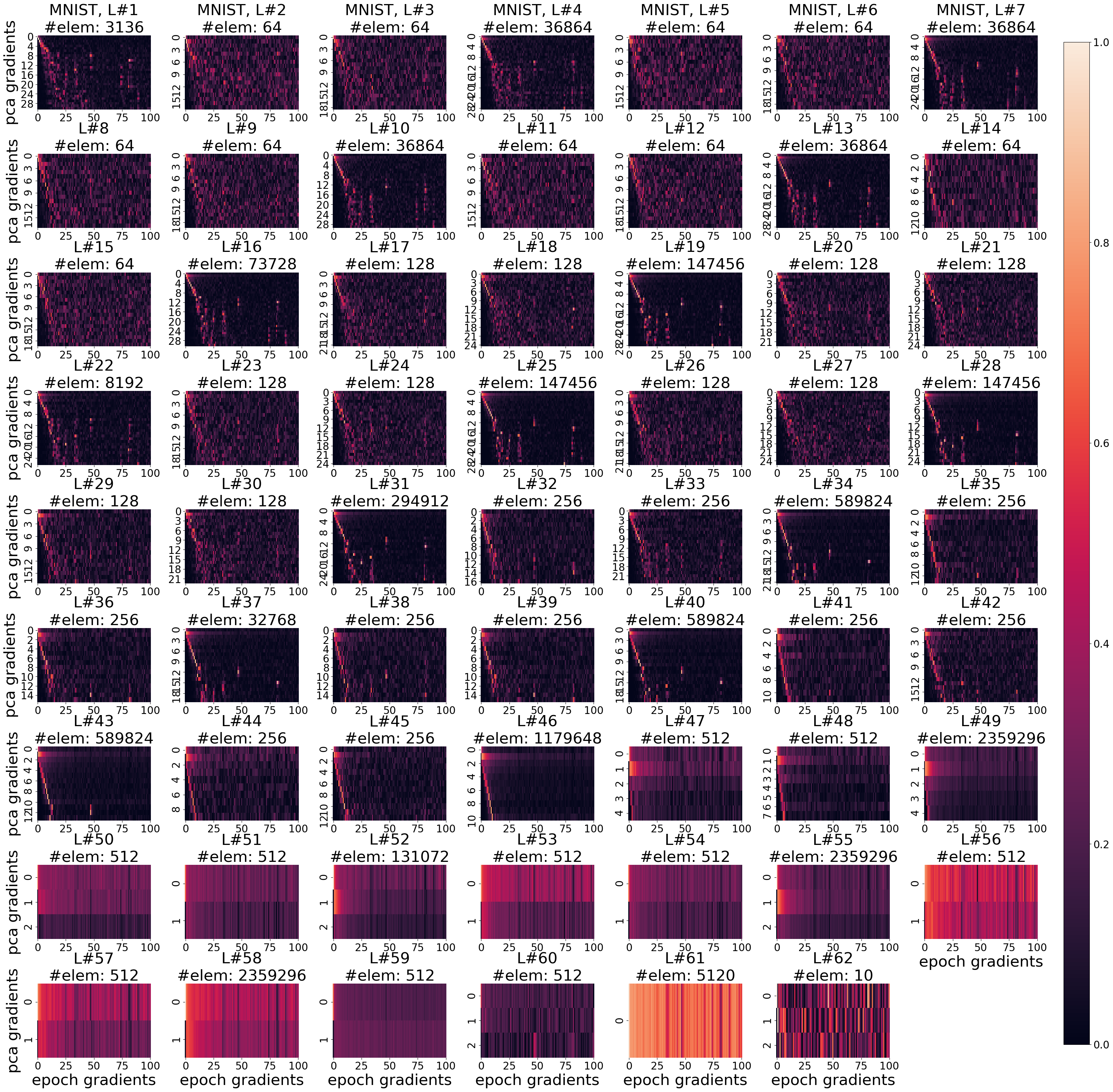}}
  \caption{\small{\textit{PCA Components Overlap with Gradient}. Repeat of Fig.~\ref{fig:prelim_2} on \textbf{ResNet18} trained on \textbf{MNIST} dataset.}}
  \label{fig:prelim_2_mnist_resnet18}
\end{figure}

\begin{figure}[h!]
  \centering
  \centerline{\includegraphics[width=0.7\textwidth]{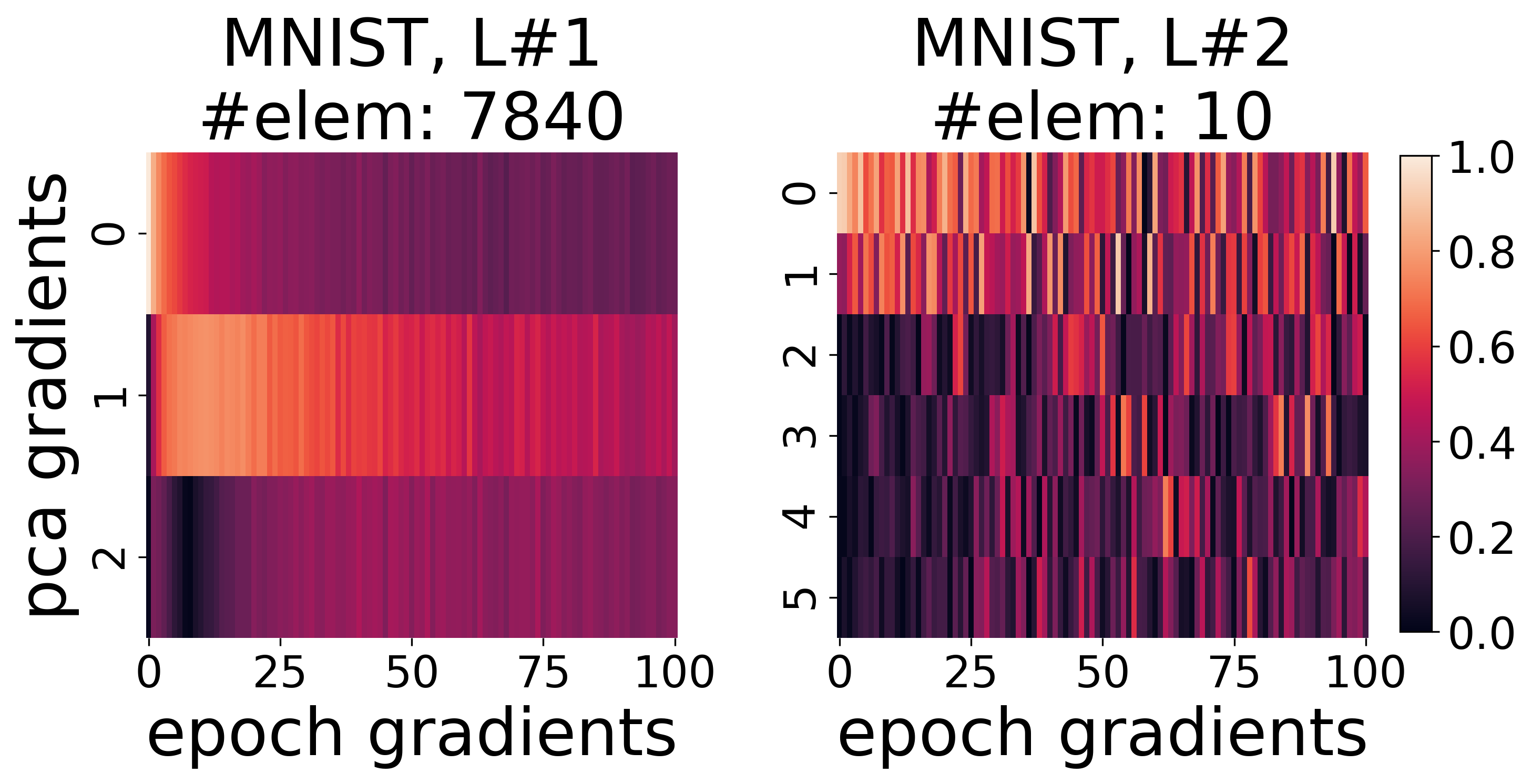}}
  \caption{\small{\textit{PCA Components Overlap with Gradient}. Repeat of Fig.~\ref{fig:prelim_2} on \textbf{FCN} trained on \textbf{MNIST} dataset.}}
  \label{fig:prelim_2_mnist_fcn}
\end{figure}

\begin{figure}[h!]
  \centering
  \centerline{\includegraphics[width=1.0\textwidth]{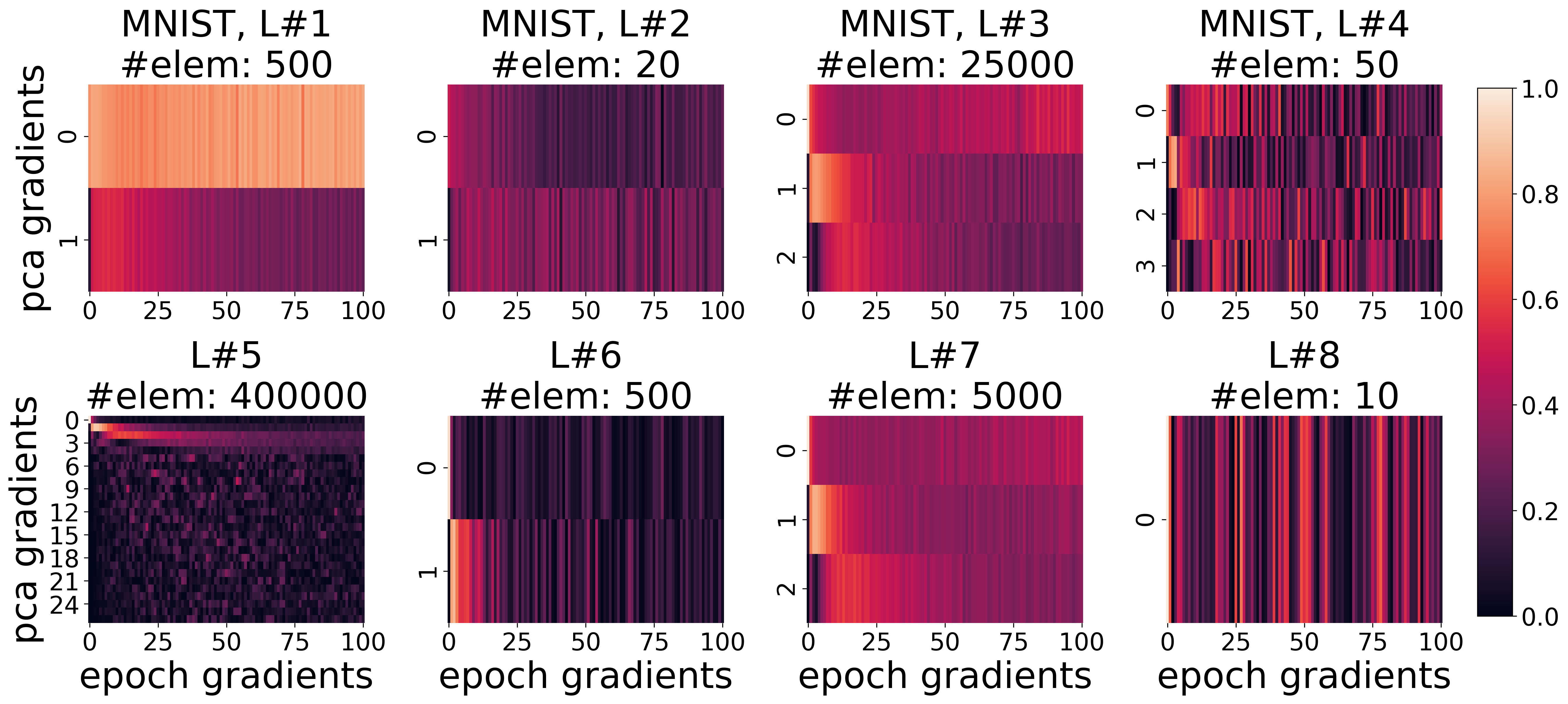}}
  \caption{\small{\textit{PCA Components Overlap with Gradient}. Fig.~\ref{fig:prelim_2} on \textbf{CNN} trained on \textbf{MNIST} dataset.}}
  \label{fig:prelim_2_mnist_cnn}
\end{figure}

\clearpage

\begin{figure}[h!]
  \centering
  \centerline{\includegraphics[width=1.2\textwidth]{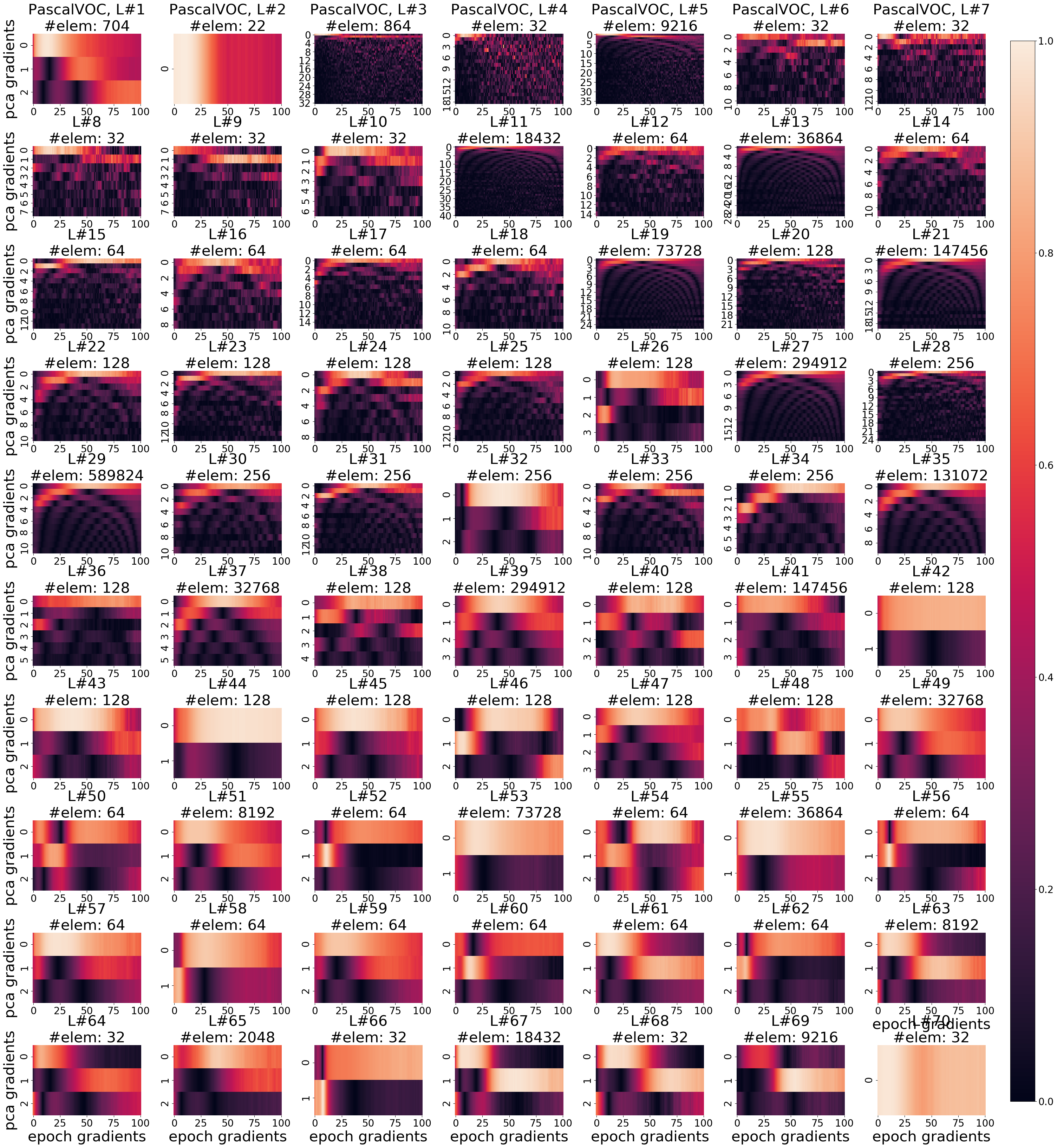}}
  \caption{\small{\textit{PCA Components Overlap with Gradient}. Repeat of Fig.~\ref{fig:prelim_2} on \textbf{U-Net} trained on \textbf{PascalVOC} dataset.}}
  \label{fig:prelim_2_voc_unet}
\end{figure}

\begin{figure}[h!]
  \centering
  \centerline{\includegraphics[width=1.2\textwidth]{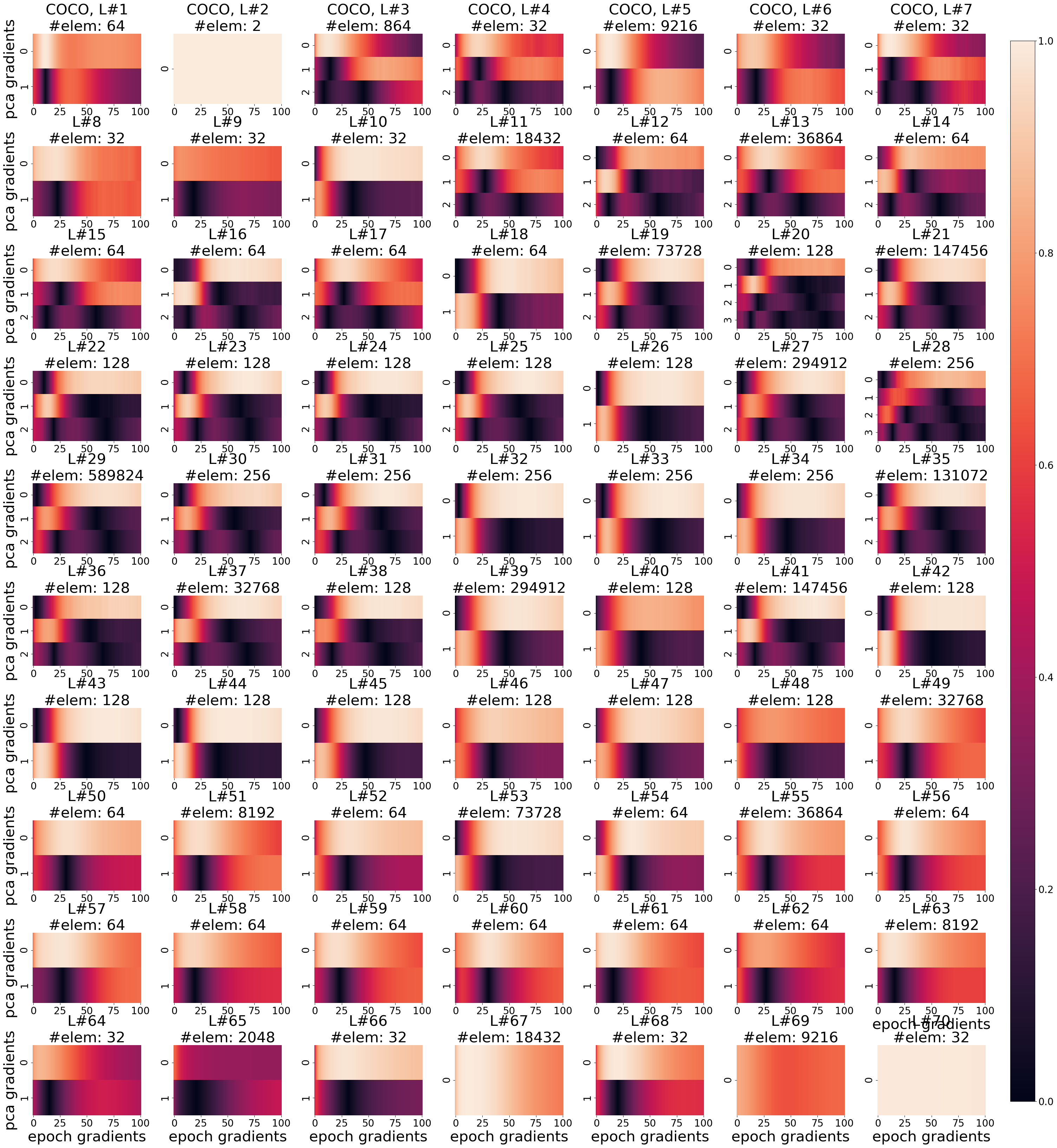}}
  \caption{\small{\textit{PCA Components Overlap with Gradient}. Repeat of Fig.~\ref{fig:prelim_2} on \textbf{U-Net} trained on \textbf{COCO} dataset.}}
  \label{fig:prelim_2_coco_unet}
\end{figure}

\newpage

\begin{figure}[h!]
  \centering
  \centerline{\includegraphics[width=1.2\textwidth]{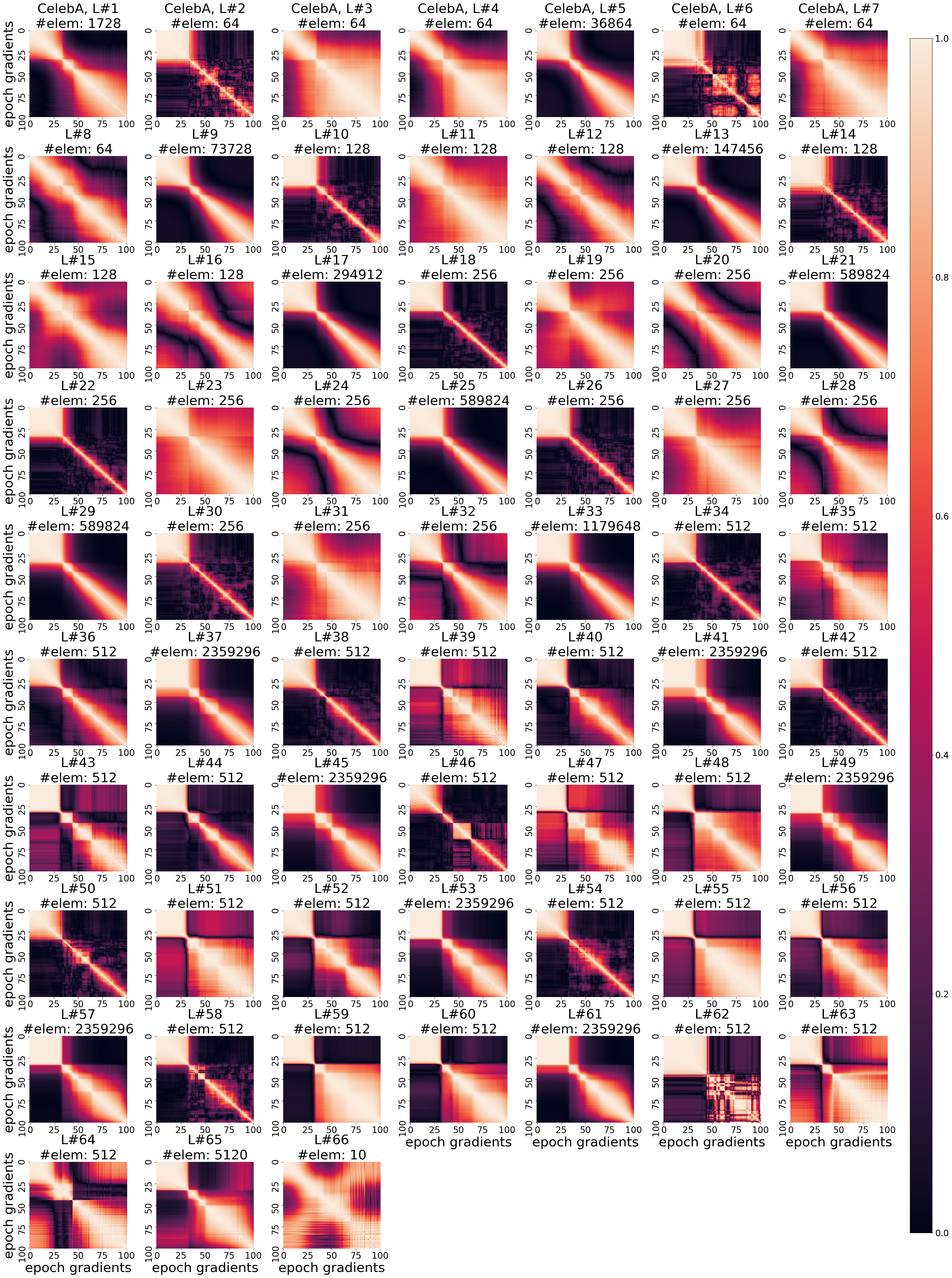}}
  \caption{\small{\textit{PCA Components Overlap with Gradient}. Repeat of Fig.~\ref{fig:prelim_3} on \textbf{VGG19} trained on \textbf{CelebA} dataset.}}
  \label{fig:prelim_3_celeba_vgg19}
\end{figure}

\begin{figure}[h!]
  \centering
  \centerline{\includegraphics[width=1.2\textwidth]{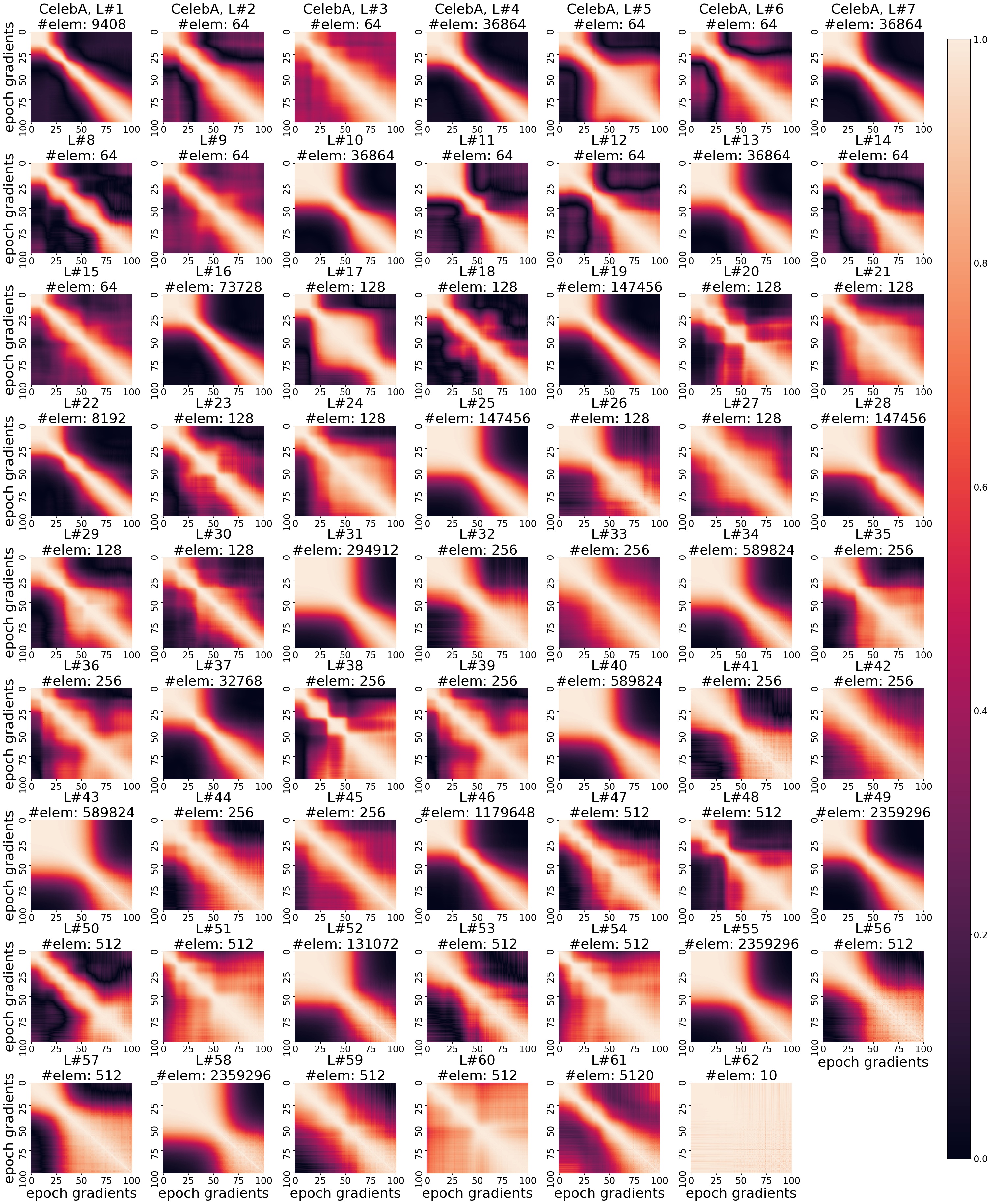}}
  \caption{\small{\textit{PCA Components Overlap with Gradient}. Repeat of Fig.~\ref{fig:prelim_3} on \textbf{ResNet18} trained on \textbf{CelebA} dataset.}}
  \label{fig:prelim_3_celeba_resnet18}
\end{figure}

\begin{figure}[h!]
  \centering
  \centerline{\includegraphics[width=0.7\textwidth]{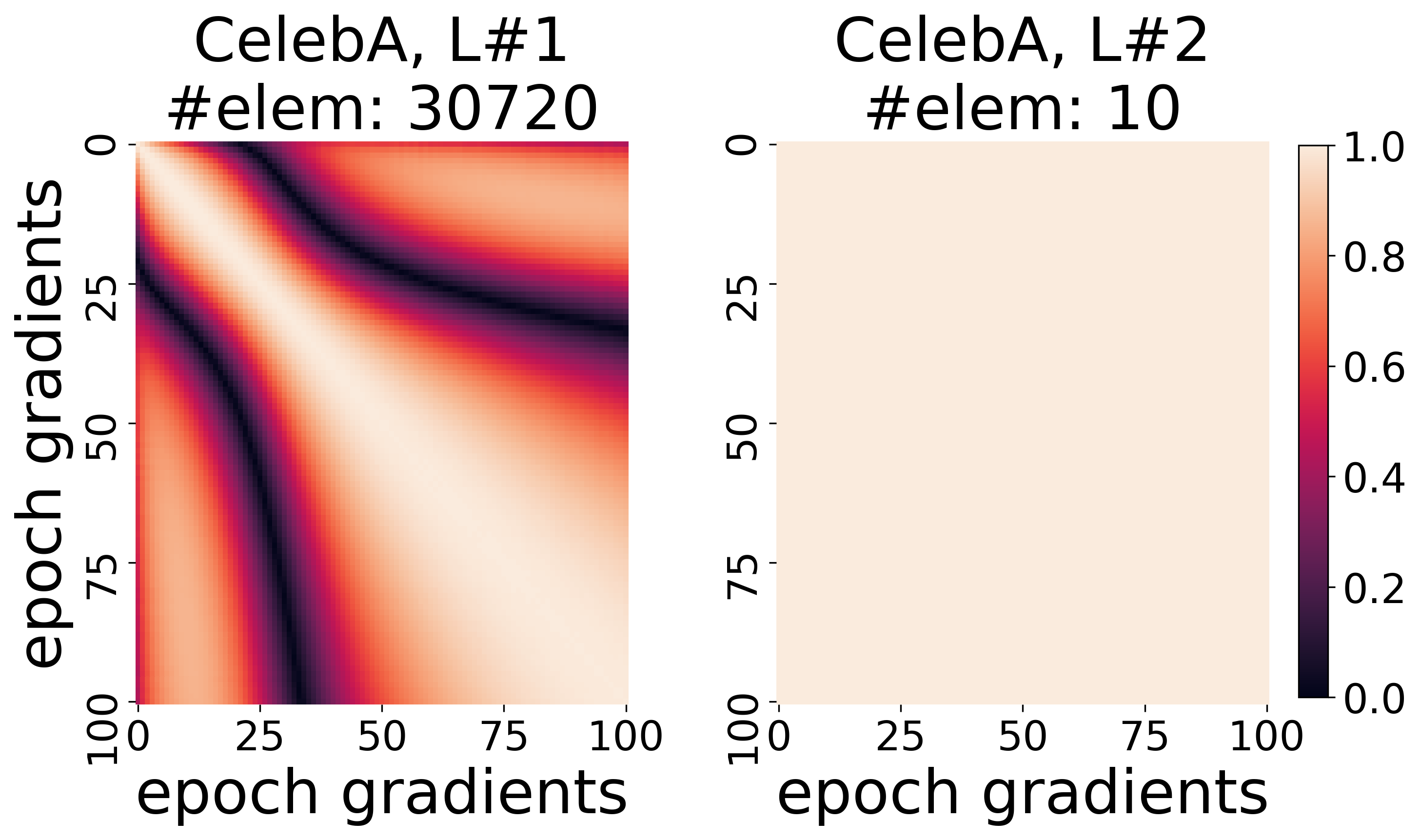}}
  \caption{\small{\textit{PCA Components Overlap with Gradient}. Repeat of Fig.~\ref{fig:prelim_3} on \textbf{FCN} trained on \textbf{CelebA} dataset.}}
  \label{fig:prelim_3_celeba_fcn}
\end{figure}

\begin{figure}[h!]
  \centering
  \centerline{\includegraphics[width=1.0\textwidth]{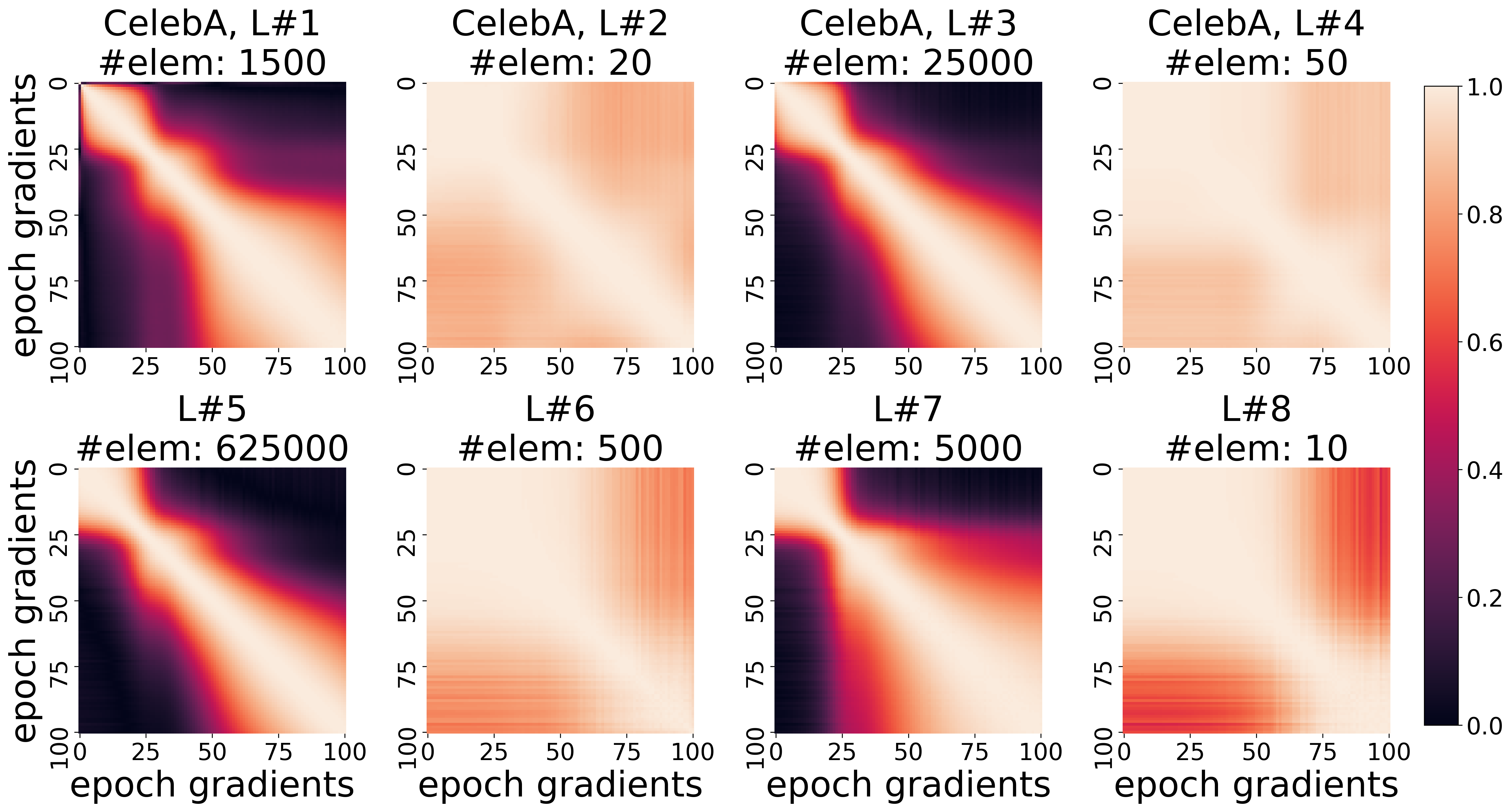}}
  \caption{\small{\textit{PCA Components Overlap with Gradient}. Fig.~\ref{fig:prelim_3} on \textbf{CNN} trained on \textbf{CelebA} dataset.}}
  \label{fig:prelim_3_celeba_cnn}
\end{figure}


\begin{figure}[h!]
  \centering
  \centerline{\includegraphics[width=1.2\textwidth]{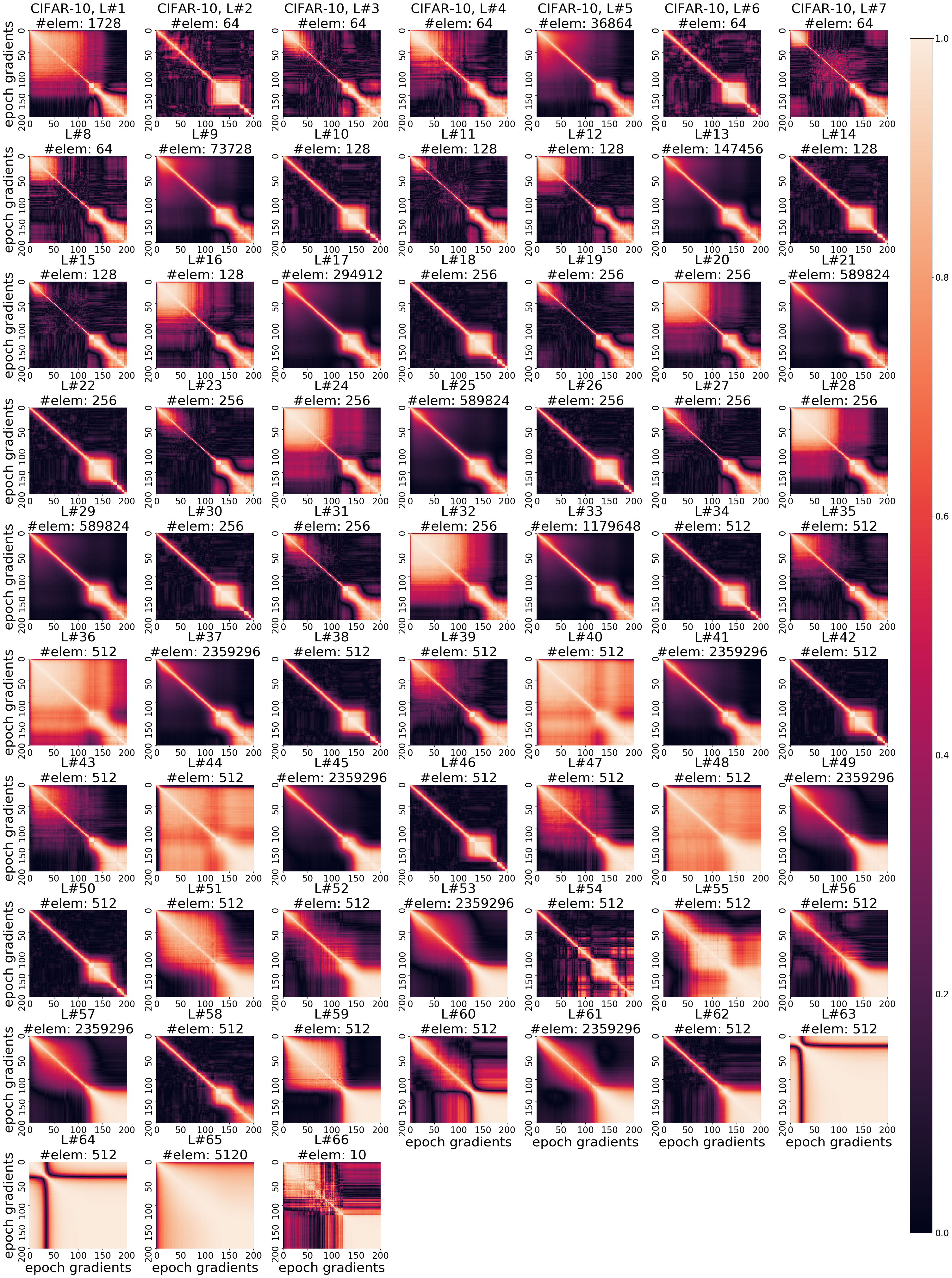}}
  \caption{\small{\textit{PCA Components Overlap with Gradient}. Repeat of Fig.~\ref{fig:prelim_3} on \textbf{VGG19} trained on \textbf{CIFAR-10} dataset.}}
  \label{fig:prelim_3_cifar_vgg19}
\end{figure}

\begin{figure}[h!]
  \centering
  \centerline{\includegraphics[width=1.2\textwidth]{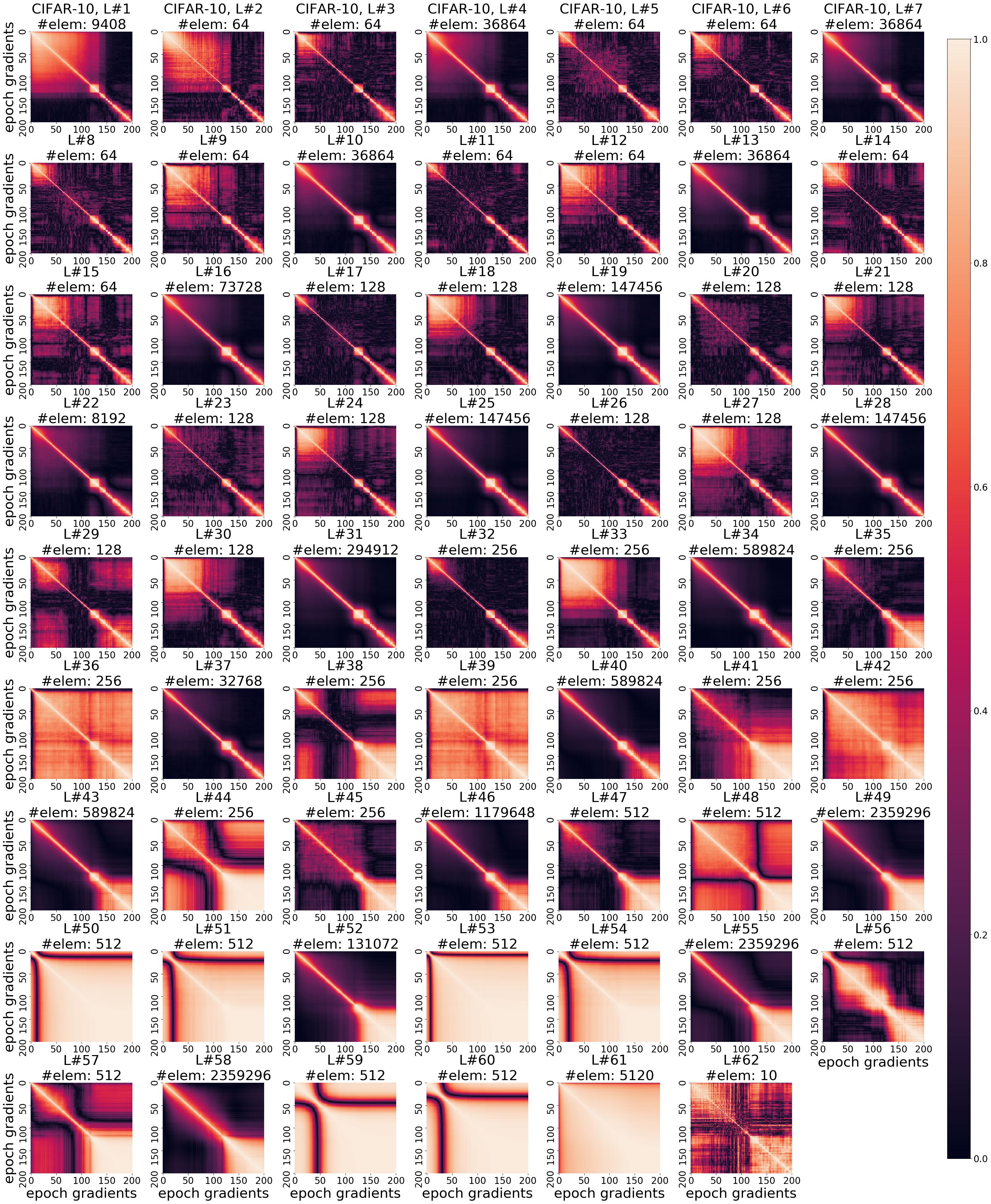}}
  \caption{\small{\textit{PCA Components Overlap with Gradient}. Repeat of Fig.~\ref{fig:prelim_3} on \textbf{ResNet18} trained on \textbf{CIFAR-10} dataset.}}
  \label{fig:prelim_3_cifar_resnet18}
\end{figure}

\begin{figure}[h!]
  \centering
  \centerline{\includegraphics[width=0.7\textwidth]{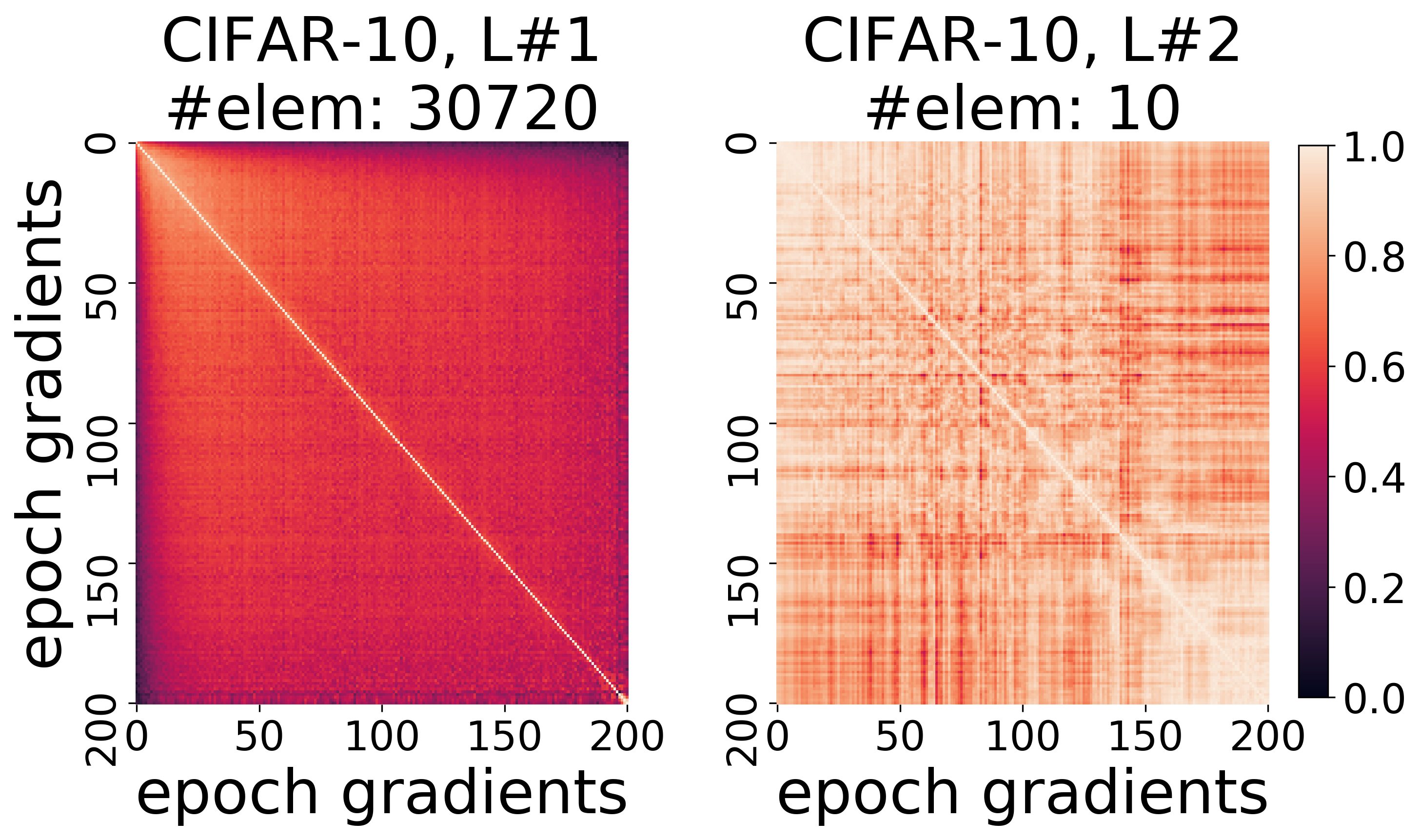}}
  \caption{\small{\textit{PCA Components Overlap with Gradient}. Repeat of Fig.~\ref{fig:prelim_3} on \textbf{FCN} trained on \textbf{CIFAR-10} dataset.}}
  \label{fig:prelim_3_cifar_fcn}
\end{figure}

\begin{figure}[h!]
  \centering
  \centerline{\includegraphics[width=1.0\textwidth]{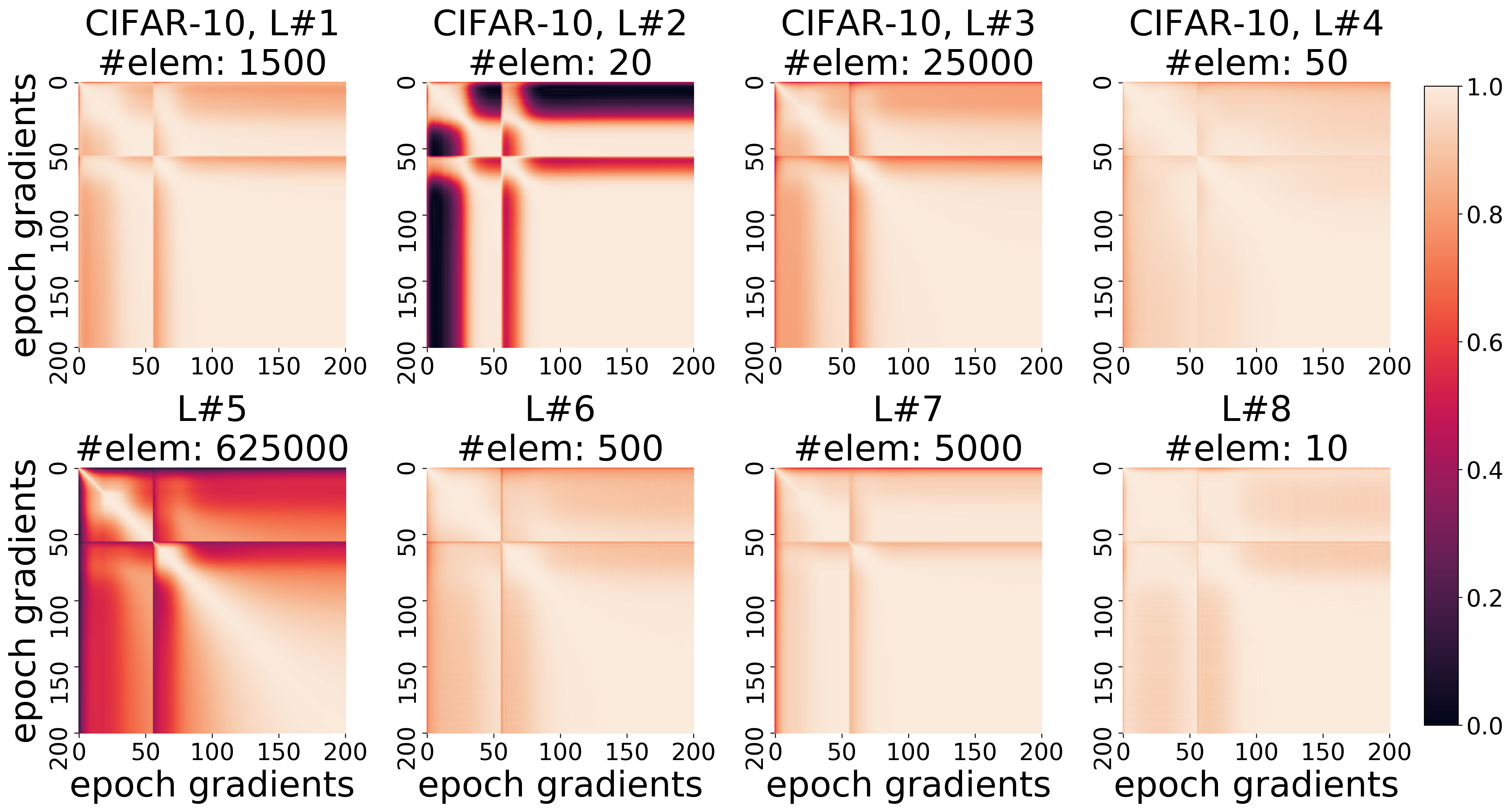}}
  \caption{\small{\textit{PCA Components Overlap with Gradient}. Fig.~\ref{fig:prelim_3} on \textbf{CNN} trained on \textbf{CIFAR-10} dataset.}}
  \label{fig:prelim_3_cifar_cnn}
\end{figure}


\begin{figure}[h!]
  \centering
  \centerline{\includegraphics[width=1.2\textwidth]{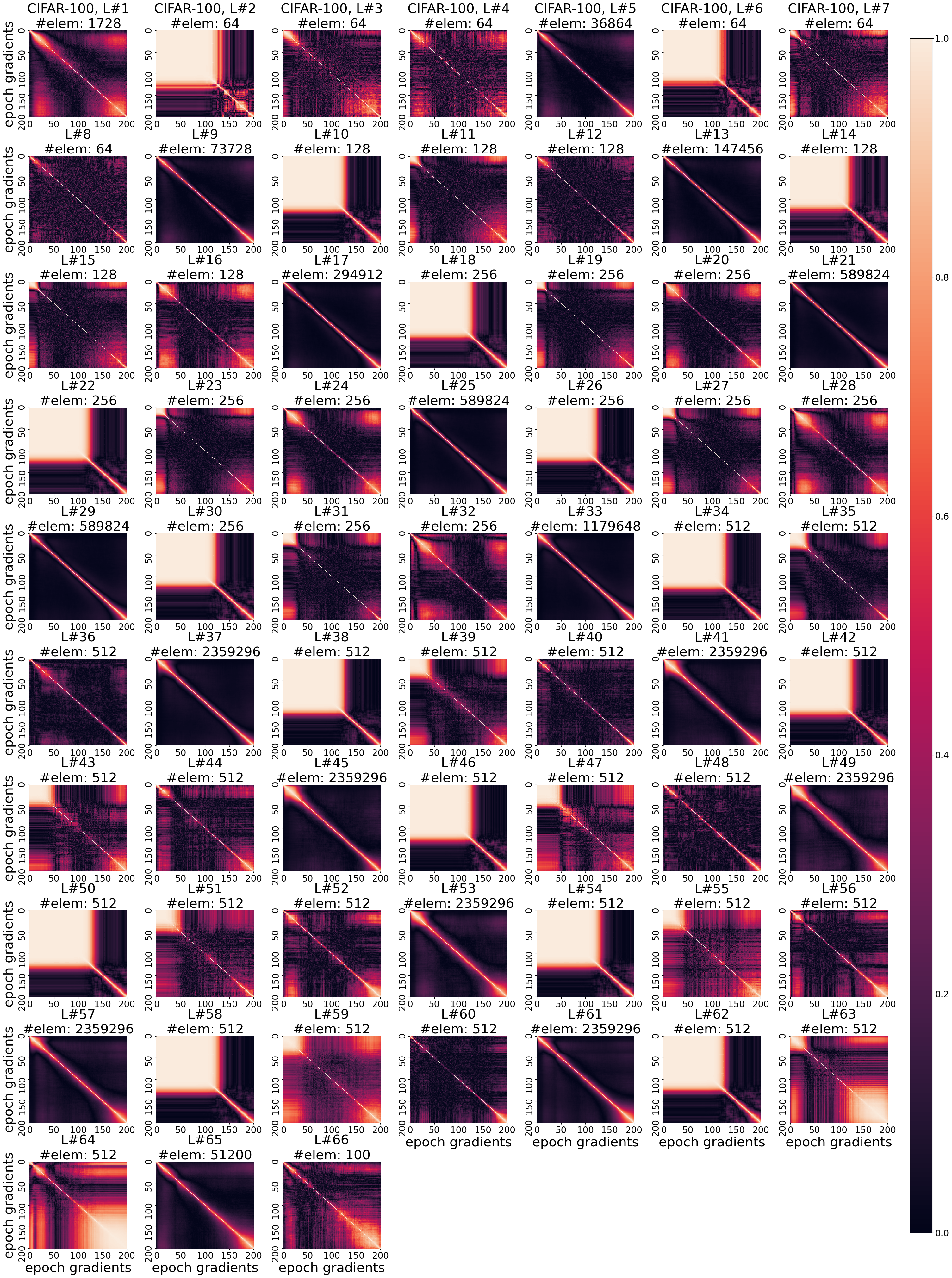}}
  \caption{\small{\textit{PCA Components Overlap with Gradient}. Repeat of Fig.~\ref{fig:prelim_3} on \textbf{VGG19} trained on \textbf{CIFAR-100} dataset.}}
  \label{fig:prelim_3_cifar100_vgg19}
\end{figure}

\begin{figure}[h!]
  \centering
  \centerline{\includegraphics[width=1.2\textwidth]{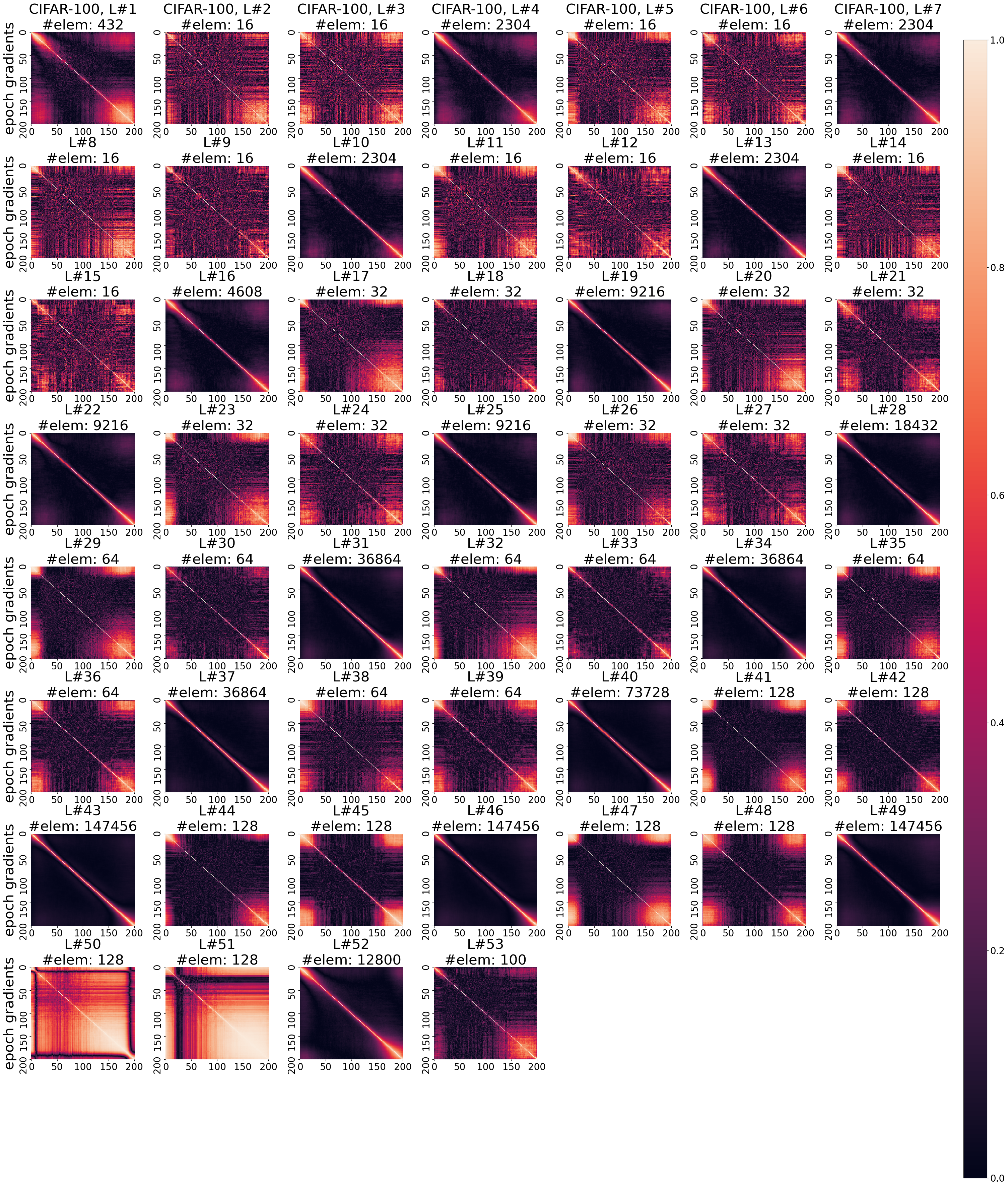}}
  \caption{\small{\textit{PCA Components Overlap with Gradient}. Repeat of Fig.~\ref{fig:prelim_3} on \textbf{ResNet18} trained on \textbf{CIFAR-100} dataset.}}
  \label{fig:prelim_3_cifar100_resnet18}
\end{figure}

\begin{figure}[h!]
  \centering
  \centerline{\includegraphics[width=0.7\textwidth]{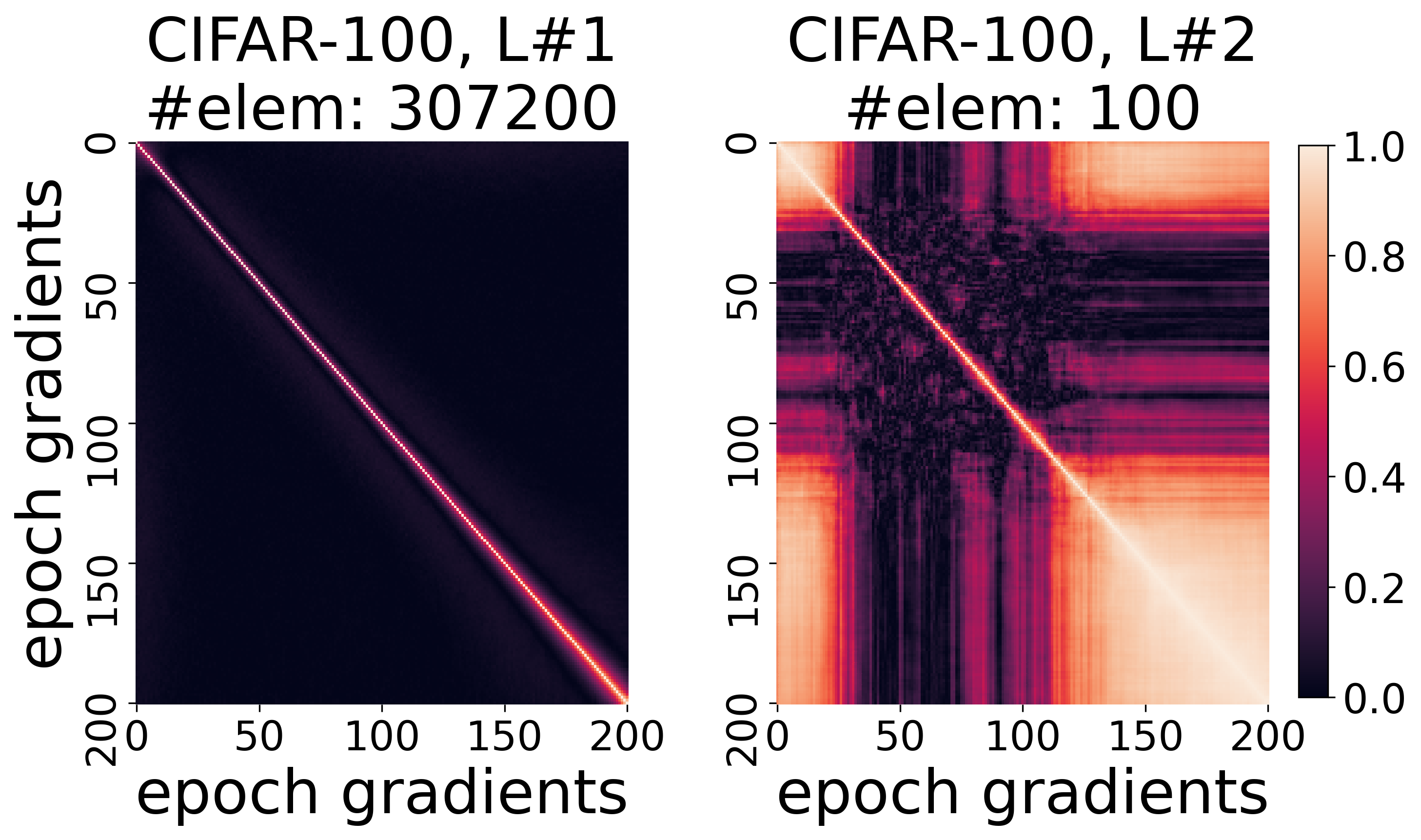}}
  \caption{\small{\textit{PCA Components Overlap with Gradient}. Repeat of Fig.~\ref{fig:prelim_3} on \textbf{FCN} trained on \textbf{CIFAR-100} dataset.}}
  \label{fig:prelim_3_cifar100_fcn}
\end{figure}

\begin{figure}[h!]
  \centering
  \centerline{\includegraphics[width=1.0\textwidth]{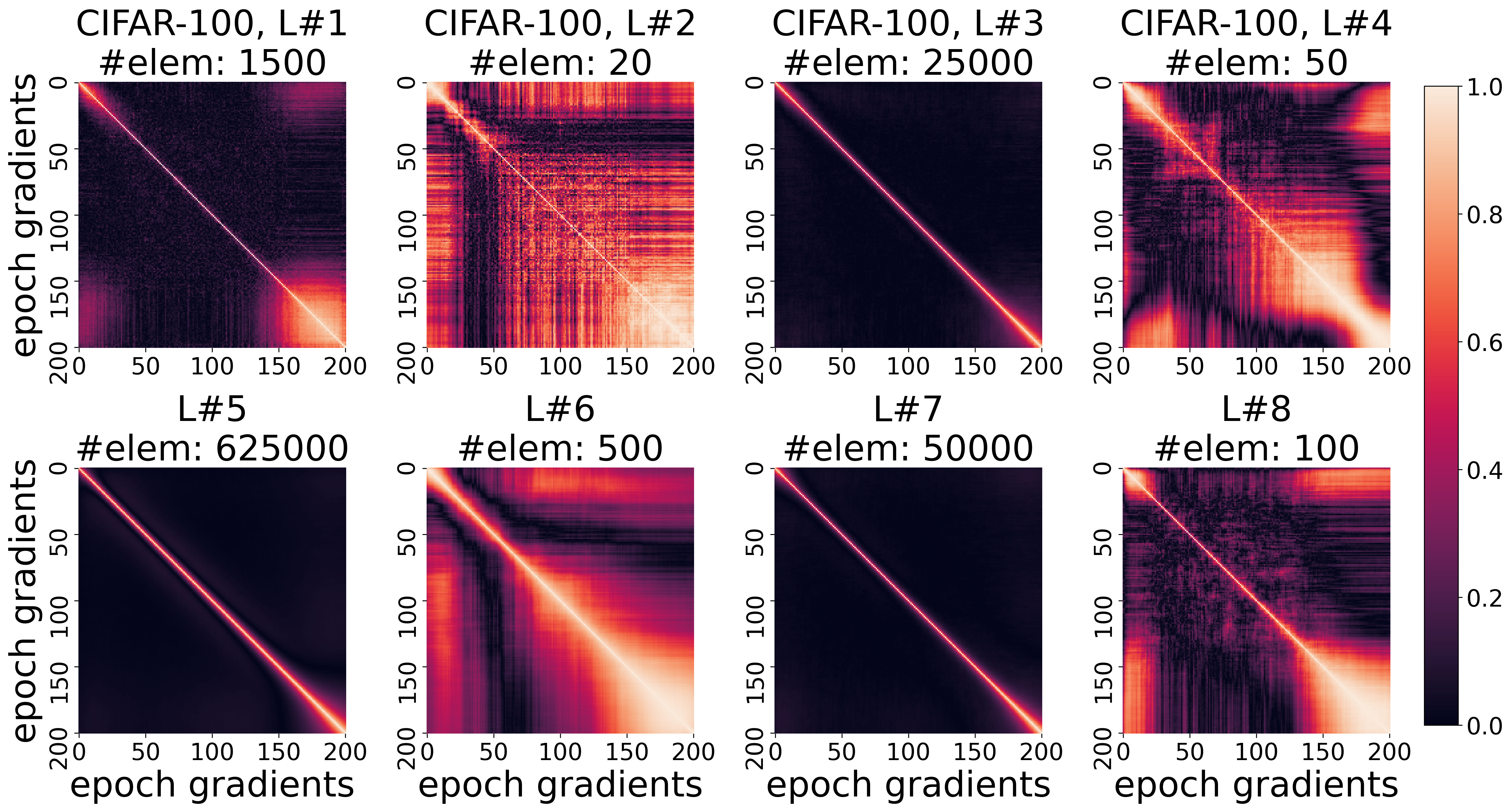}}
  \caption{\small{\textit{PCA Components Overlap with Gradient}. Fig.~\ref{fig:prelim_3} on \textbf{CNN} trained on \textbf{CIFAR-100} dataset.}}
  \label{fig:prelim_3_cifar100_cnn}
\end{figure}


\begin{figure}[h!]
  \centering
  \centerline{\includegraphics[width=1.2\textwidth]{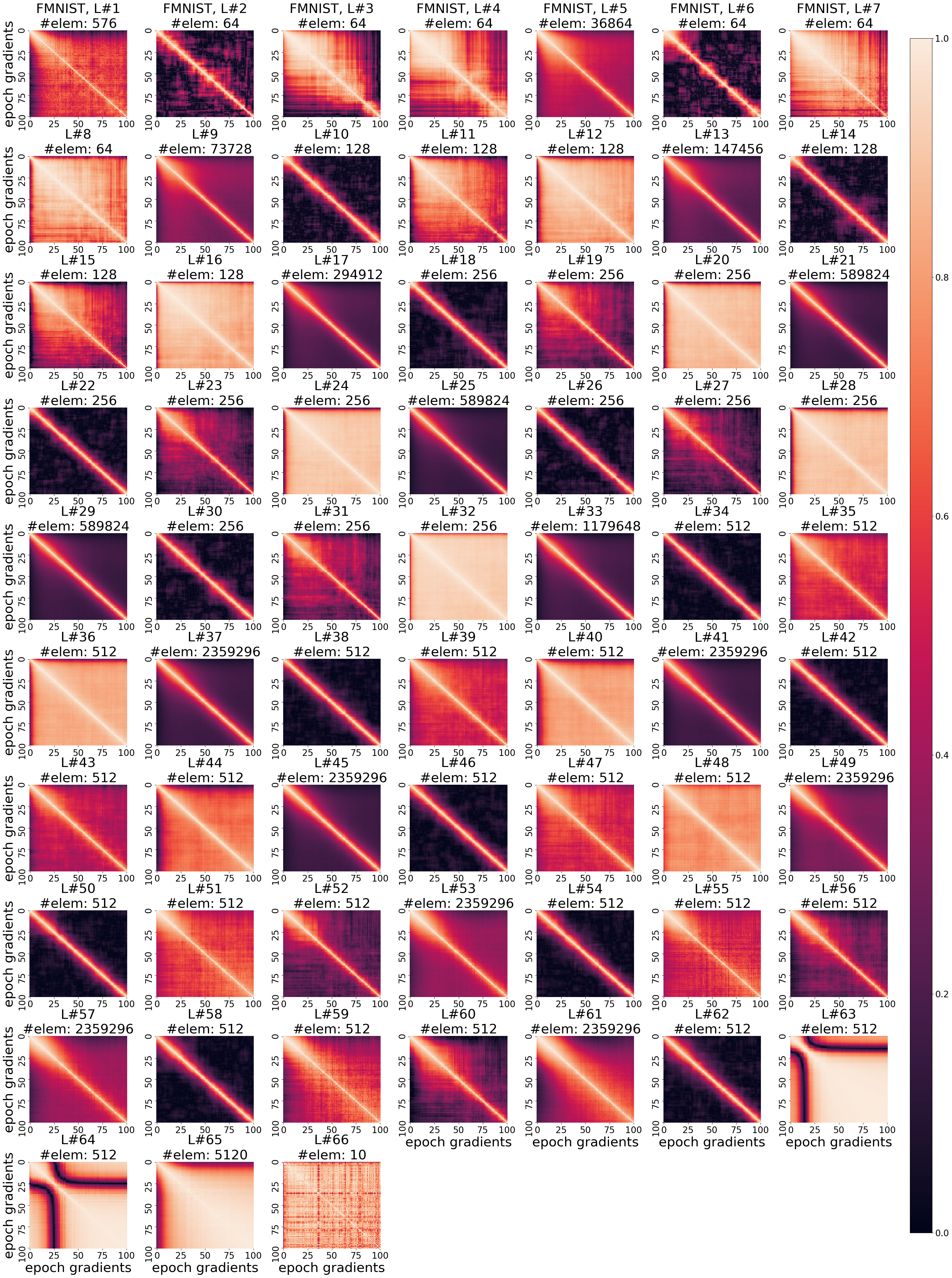}}
  \caption{\small{\textit{PCA Components Overlap with Gradient}. Repeat of Fig.~\ref{fig:prelim_3} on \textbf{VGG19} trained on \textbf{FMNIST} dataset.}}
  \label{fig:prelim_3_fmnist_vgg19}
\end{figure}

\begin{figure}[h!]
  \centering
  \centerline{\includegraphics[width=1.2\textwidth]{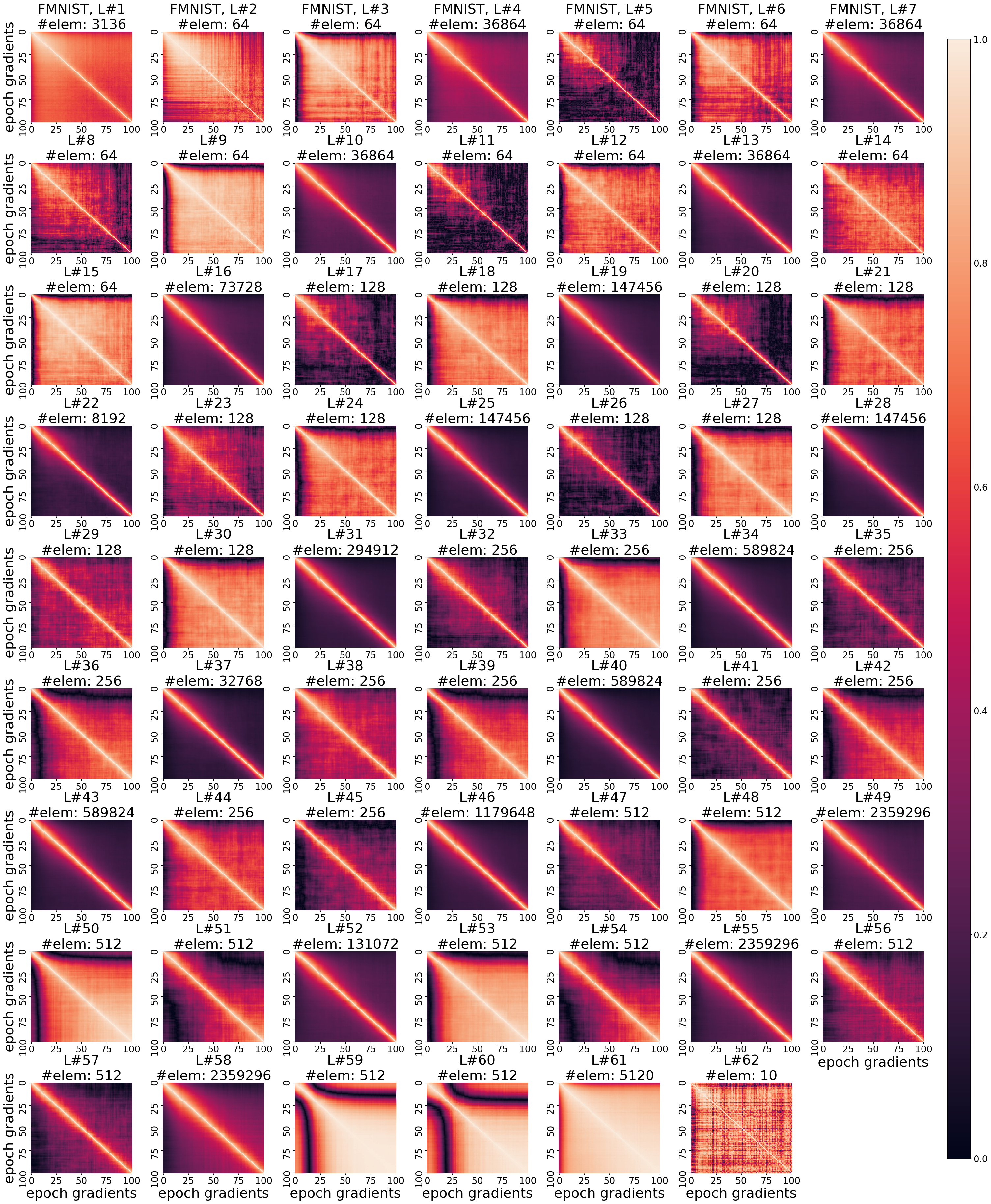}}
  \caption{\small{\textit{PCA Components Overlap with Gradient}. Repeat of Fig.~\ref{fig:prelim_3} on \textbf{ResNet18} trained on \textbf{FMNIST} dataset.}}
  \label{fig:prelim_3_fmnist_resnet18}
\end{figure}

\begin{figure}[h!]
  \centering
  \centerline{\includegraphics[width=0.7\textwidth]{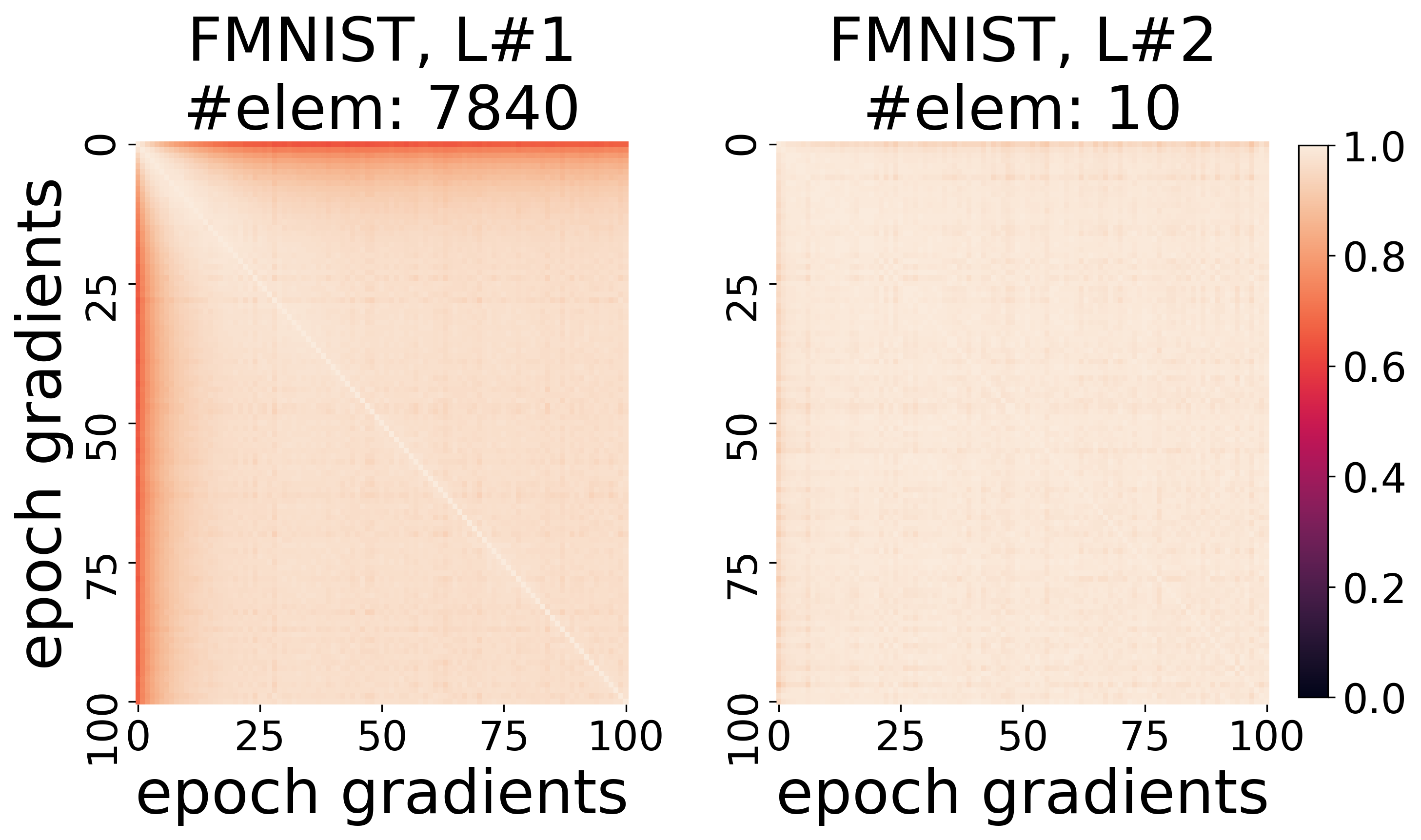}}
  \caption{\small{\textit{PCA Components Overlap with Gradient}. Repeat of Fig.~\ref{fig:prelim_3} on \textbf{FCN} trained on \textbf{FMNIST} dataset.}}
  \label{fig:prelim_3_fmnist_fcn}
\end{figure}

\begin{figure}[h!]
  \centering
  \centerline{\includegraphics[width=1.0\textwidth]{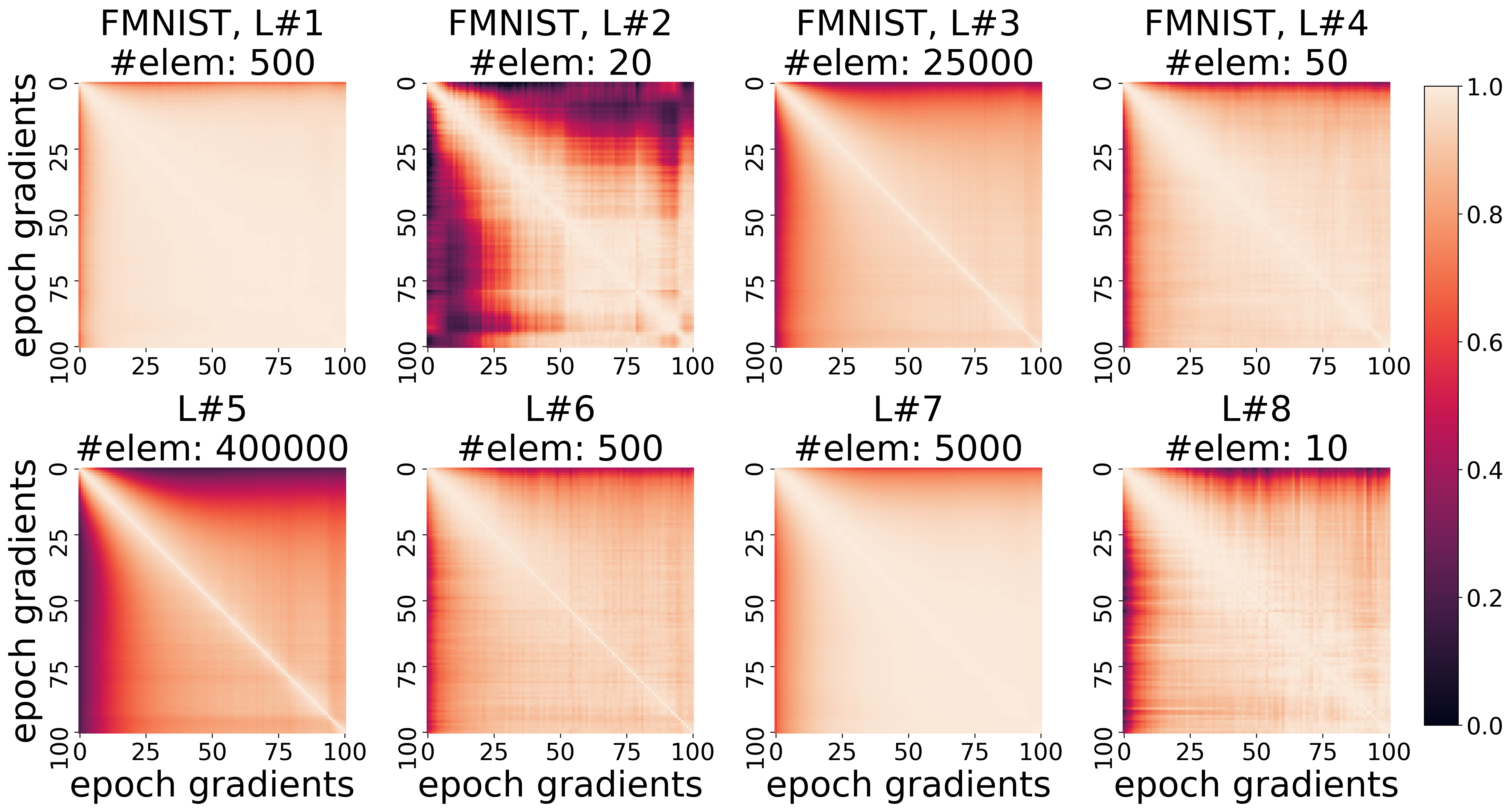}}
  \caption{\small{\textit{PCA Components Overlap with Gradient}. Fig.~\ref{fig:prelim_3} on \textbf{CNN} trained on \textbf{FMNIST} dataset.}}
  \label{fig:prelim_3_fmnist_cnn}
\end{figure}


\begin{figure}[h!]
  \centering
  \centerline{\includegraphics[width=1.2\textwidth]{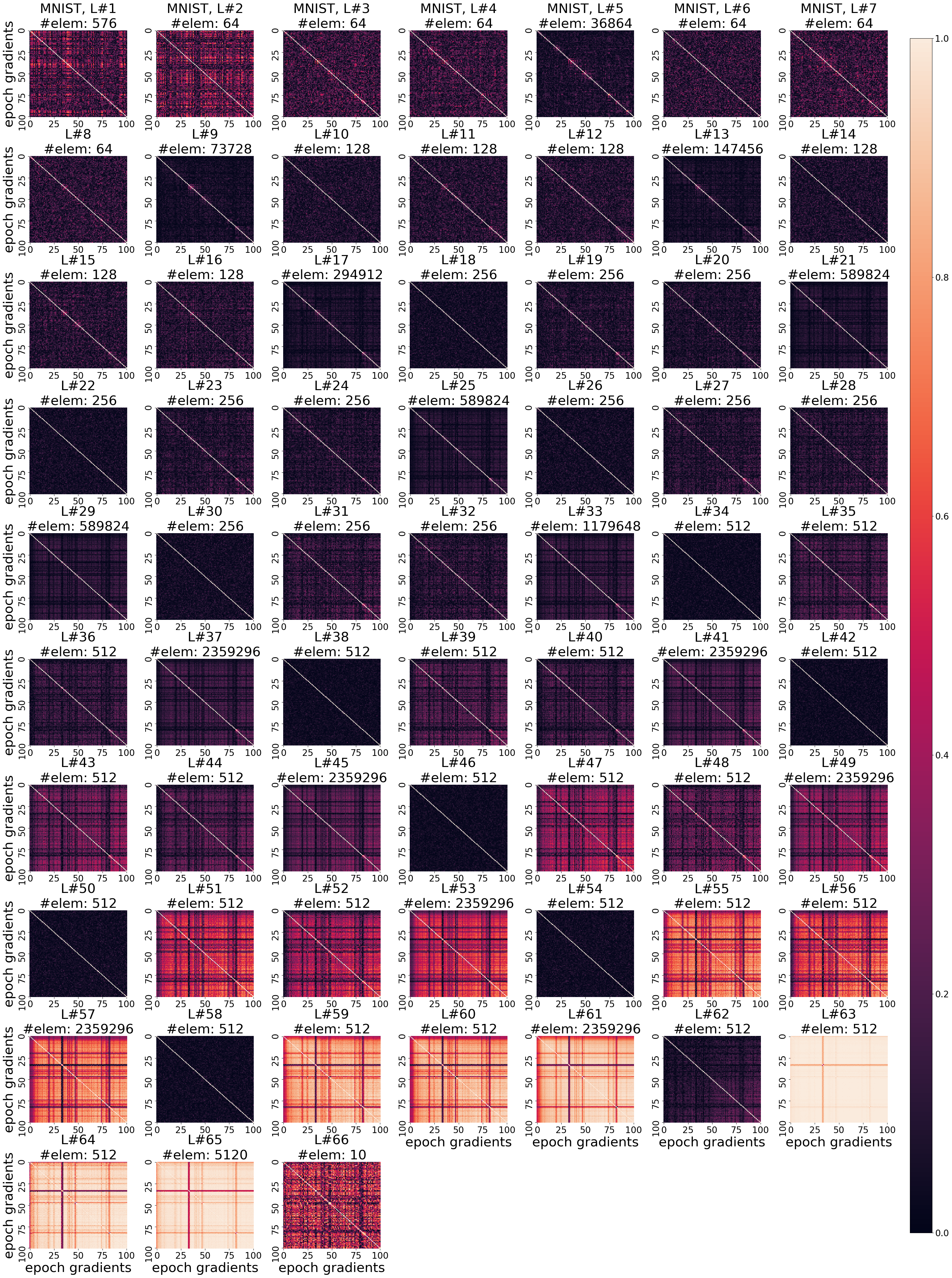}}
  \caption{\small{\textit{PCA Components Overlap with Gradient}. Repeat of Fig.~\ref{fig:prelim_3} on \textbf{VGG19} trained on \textbf{MNIST} dataset.}}
  \label{fig:prelim_3_mnist_vgg19}
\end{figure}

\begin{figure}[h!]
  \centering
  \centerline{\includegraphics[width=1.2\textwidth]{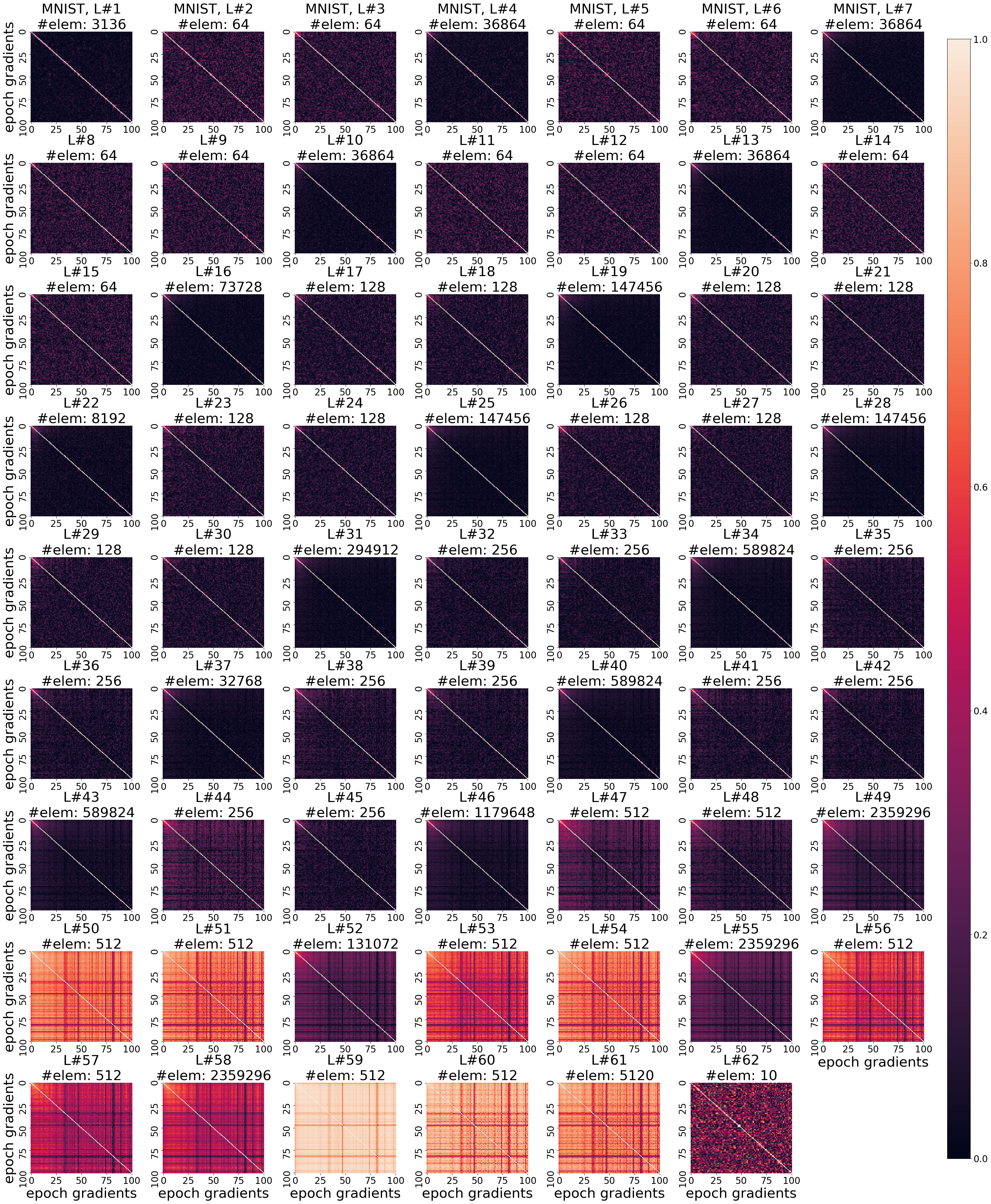}}
  \caption{\small{\textit{PCA Components Overlap with Gradient}. Repeat of Fig.~\ref{fig:prelim_3} on \textbf{ResNet18} trained on \textbf{MNIST} dataset.}}
  \label{fig:prelim_3_mnist_resnet18}
\end{figure}

\begin{figure}[h!]
  \centering
  \centerline{\includegraphics[width=0.7\textwidth]{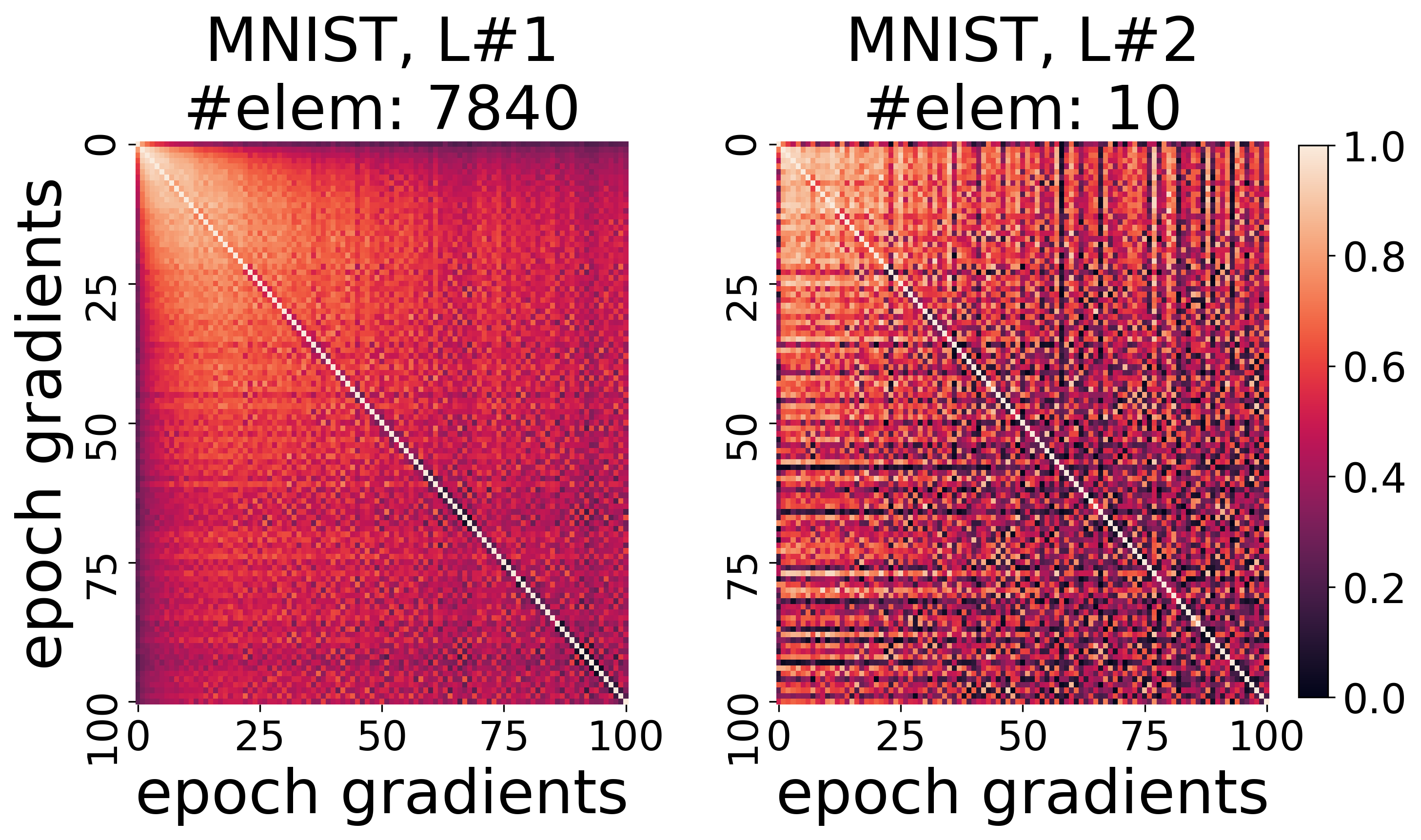}}
  \caption{\small{\textit{PCA Components Overlap with Gradient}. Repeat of Fig.~\ref{fig:prelim_3} on \textbf{FCN} trained on \textbf{MNIST} dataset.}}
  \label{fig:prelim_3_mnist_fcn}
\end{figure}

\begin{figure}[h!]
  \centering
  \centerline{\includegraphics[width=1.0\textwidth]{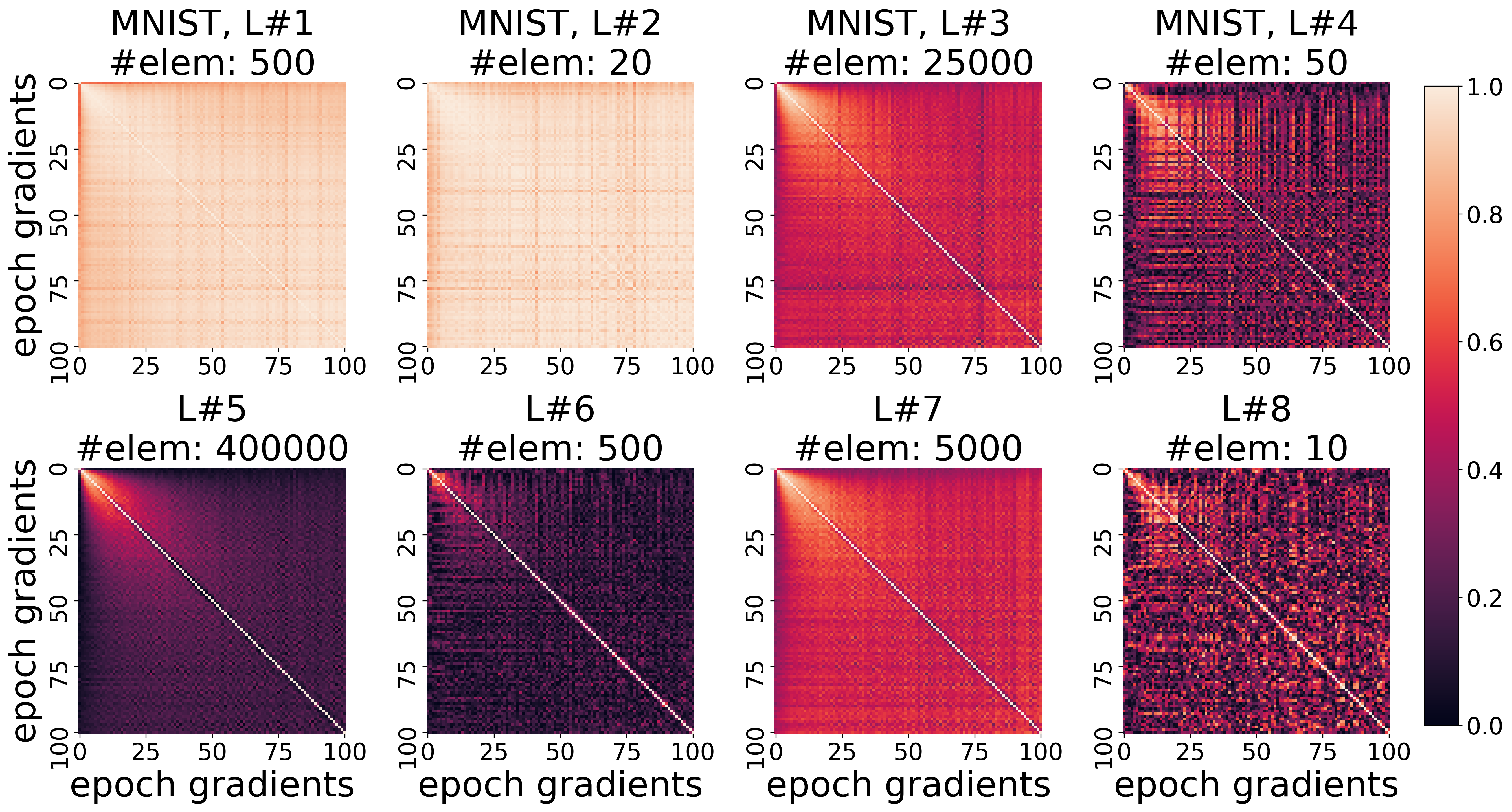}}
  \caption{\small{\textit{PCA Components Overlap with Gradient}. Fig.~\ref{fig:prelim_3} on \textbf{CNN} trained on \textbf{MNIST} dataset.}}
  \label{fig:prelim_3_mnist_cnn}
\end{figure}


\begin{figure}[h!]
  \centering
  \centerline{\includegraphics[width=1.2\textwidth]{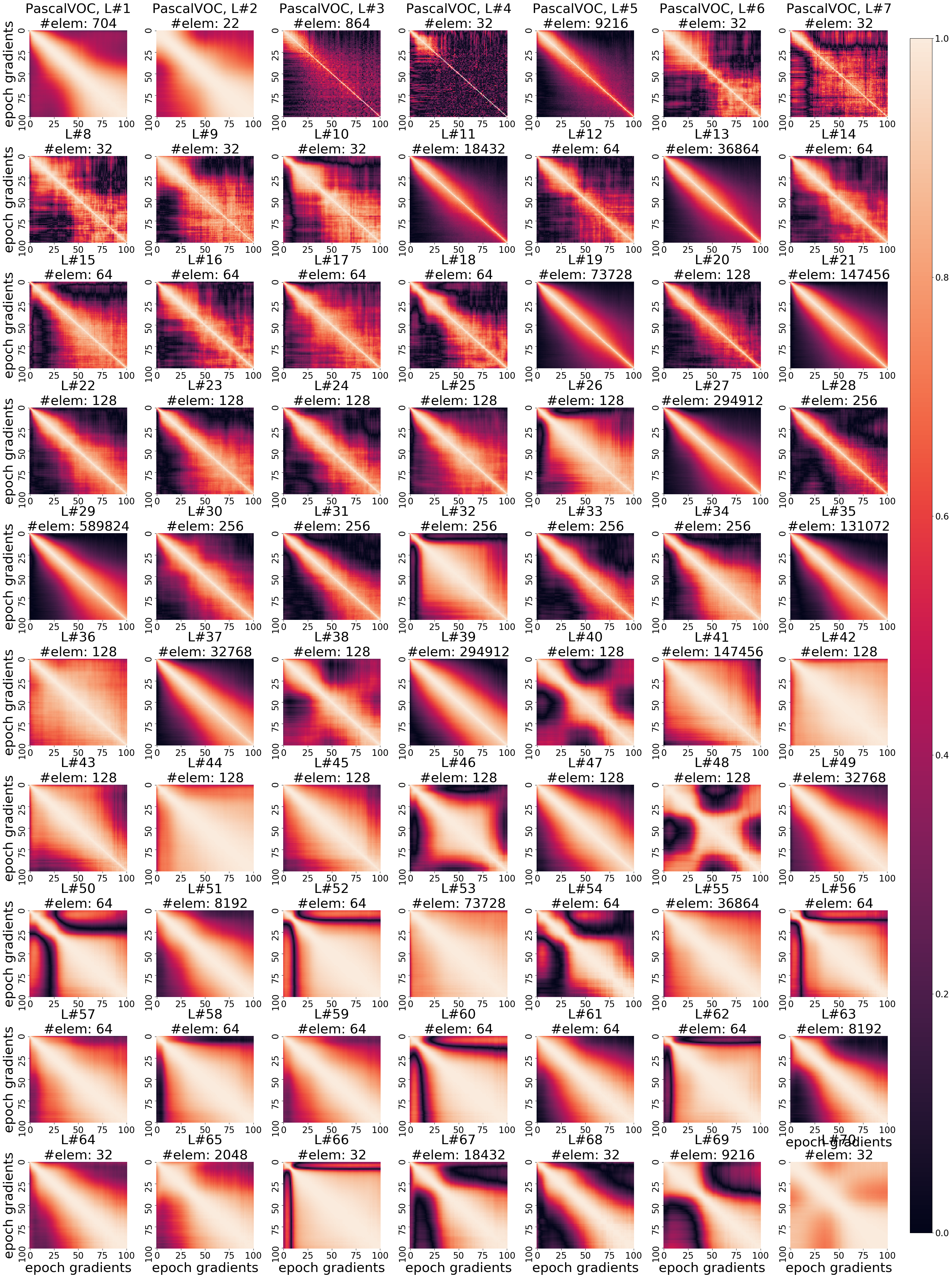}}
  \caption{\small{\textit{PCA Components Overlap with Gradient}. Repeat of Fig.~\ref{fig:prelim_3} on \textbf{U-Net} trained on \textbf{PascalVOC} dataset.}}
  \label{fig:prelim_3_voc_unet}
\end{figure}

\begin{figure}[h!]
  \centering
  \centerline{\includegraphics[width=1.2\textwidth]{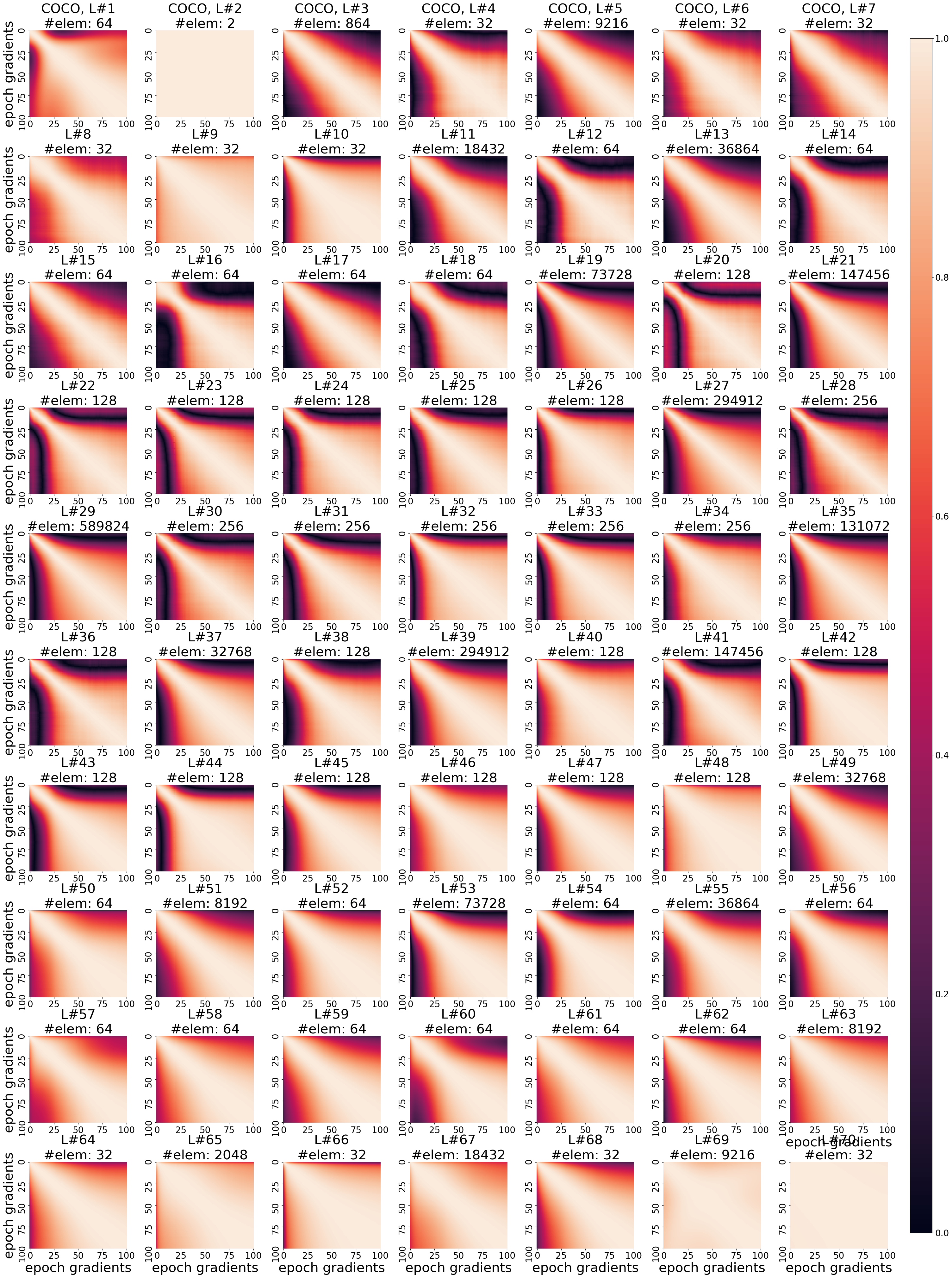}}
  \caption{\small{\textit{PCA Components Overlap with Gradient}. Repeat of Fig.~\ref{fig:prelim_3} on \textbf{U-Net} trained on \textbf{COCO} dataset.}}
  \label{fig:prelim_3_coco_unet}
\end{figure}

\newpage

\begin{figure}[t]
  \centering
    \centerline{\includegraphics[width=0.9\textwidth]{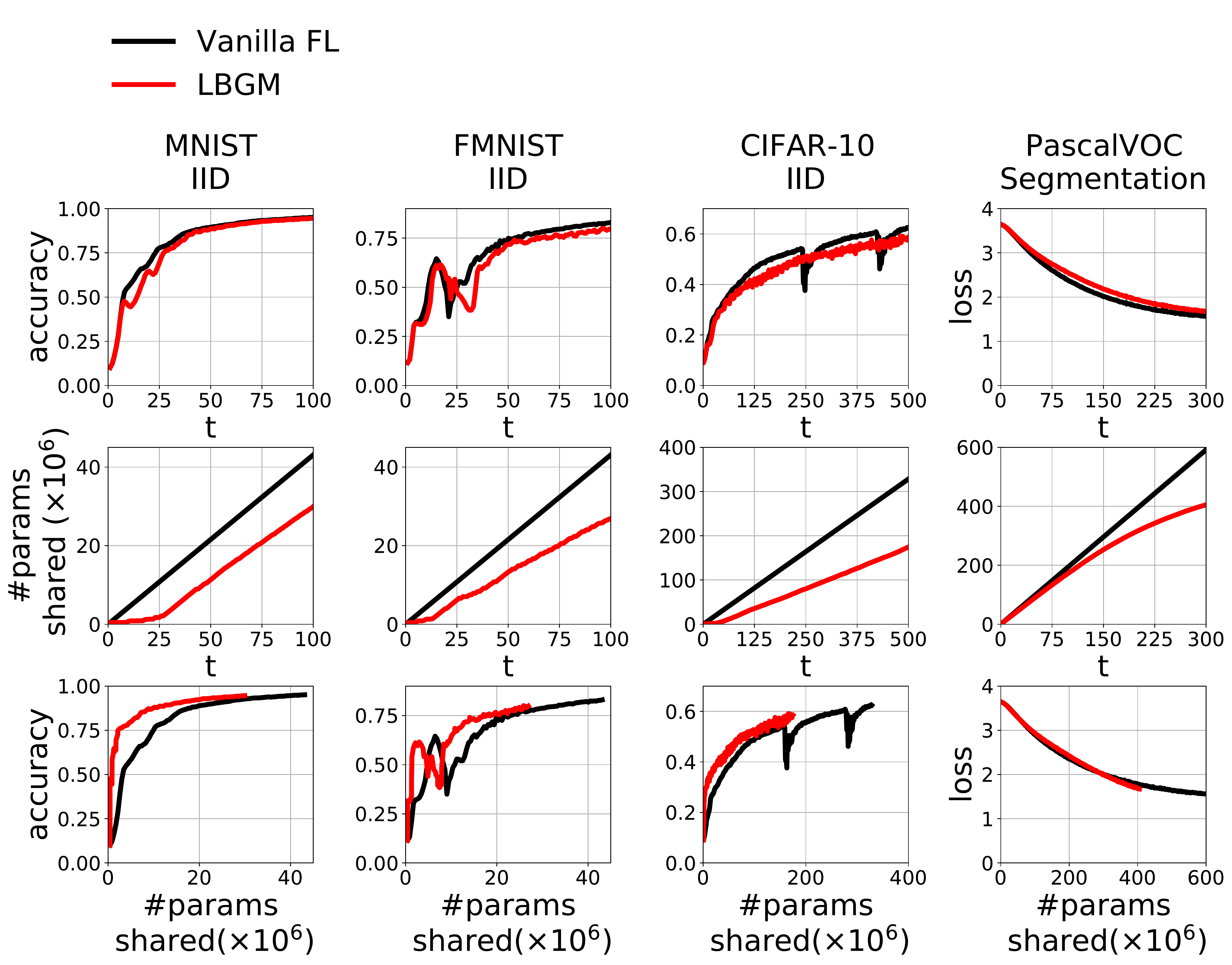}}
  \vspace{-1mm}
  \caption{\small{\textit{{\tt \algName} as a Standalone Algorithm}. Experimental results in Fig.~\ref{fig:standalone} repeated on datasets: \textbf{CIFAR-10}, \textbf{FMNIST}, and \textbf{MNIST} (iid data distribution) using classifier: \textbf{CNN}, and dataset: \textbf{PascalVOC} using \textbf{U-Net} architecture.}}
  \label{fig:standalone_cnn}
\end{figure}

\begin{figure}[t]
  \centering
  \centerline{\includegraphics[width=0.9\textwidth]{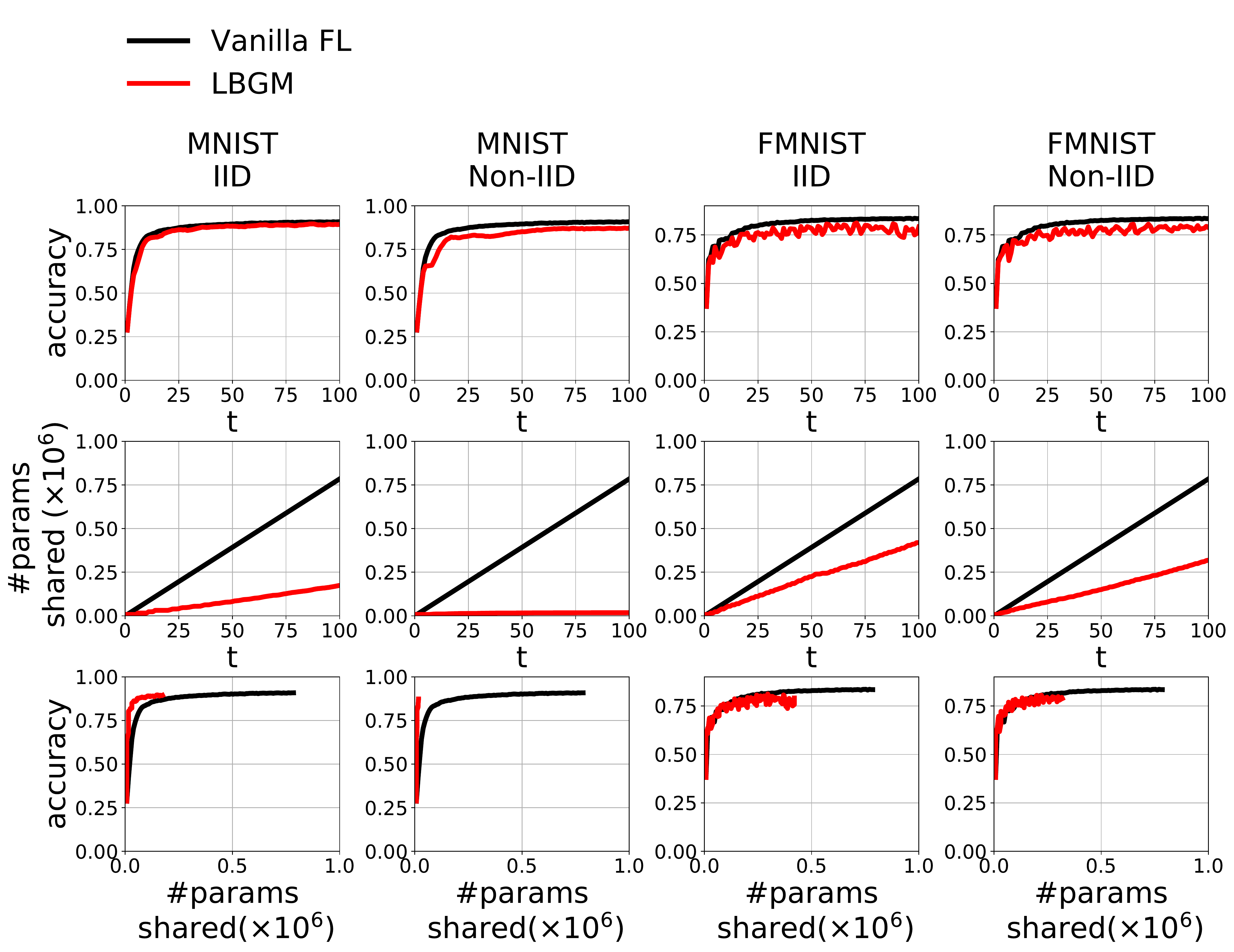}}
  \caption{\small{\textit{{\tt \algName} as a Standalone Algorithm}. Experimental results in Fig.~\ref{fig:standalone} repeated for datasets: \textbf{FMNIST} and \textbf{MNIST} (both iid and non-iid data distribution) using classifier: \textbf{FCN}.}}
  \label{fig:standalone_fcn}
\end{figure}

\begin{figure}[t]
  \centering
    \centerline{\includegraphics[width=0.9\textwidth]{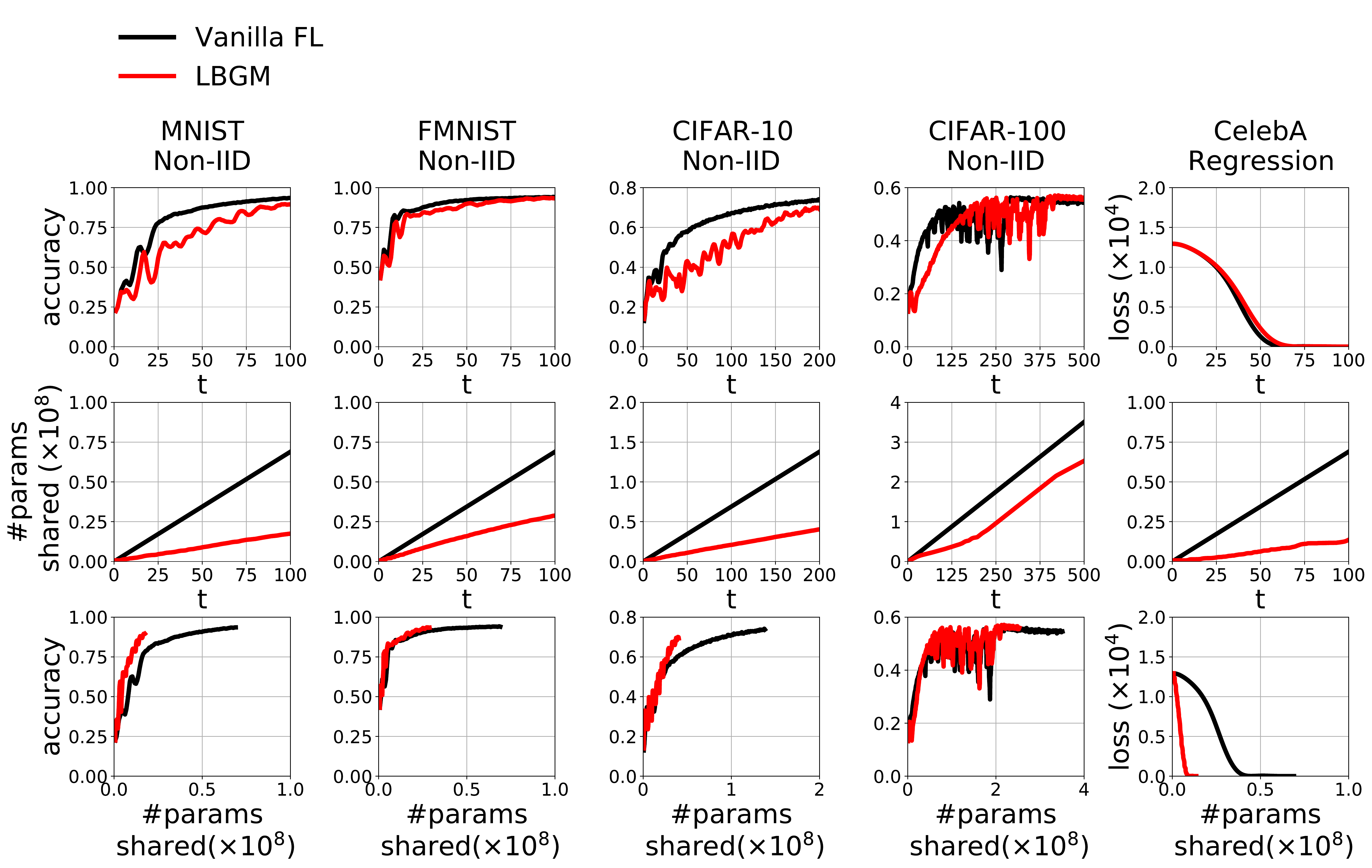}}
  \caption{\small{\textit{{\tt \algName} as a Standalone Algorithm}. Experimental results in Fig.~\ref{fig:standalone} repeated for datasets: \textbf{CIFAR-10}, \textbf{CIFAR-100}, \textbf{FMNIST}, \textbf{MNIST} (non-iid data distribution), and \textbf{CelebA} (face landmark regression task) using classifier: \textbf{Resnet18}.}}
  \label{fig:standalone_resnet18}
\end{figure}


\begin{figure}[t]
  \centering
    \centerline{\includegraphics[width=0.75\textwidth]{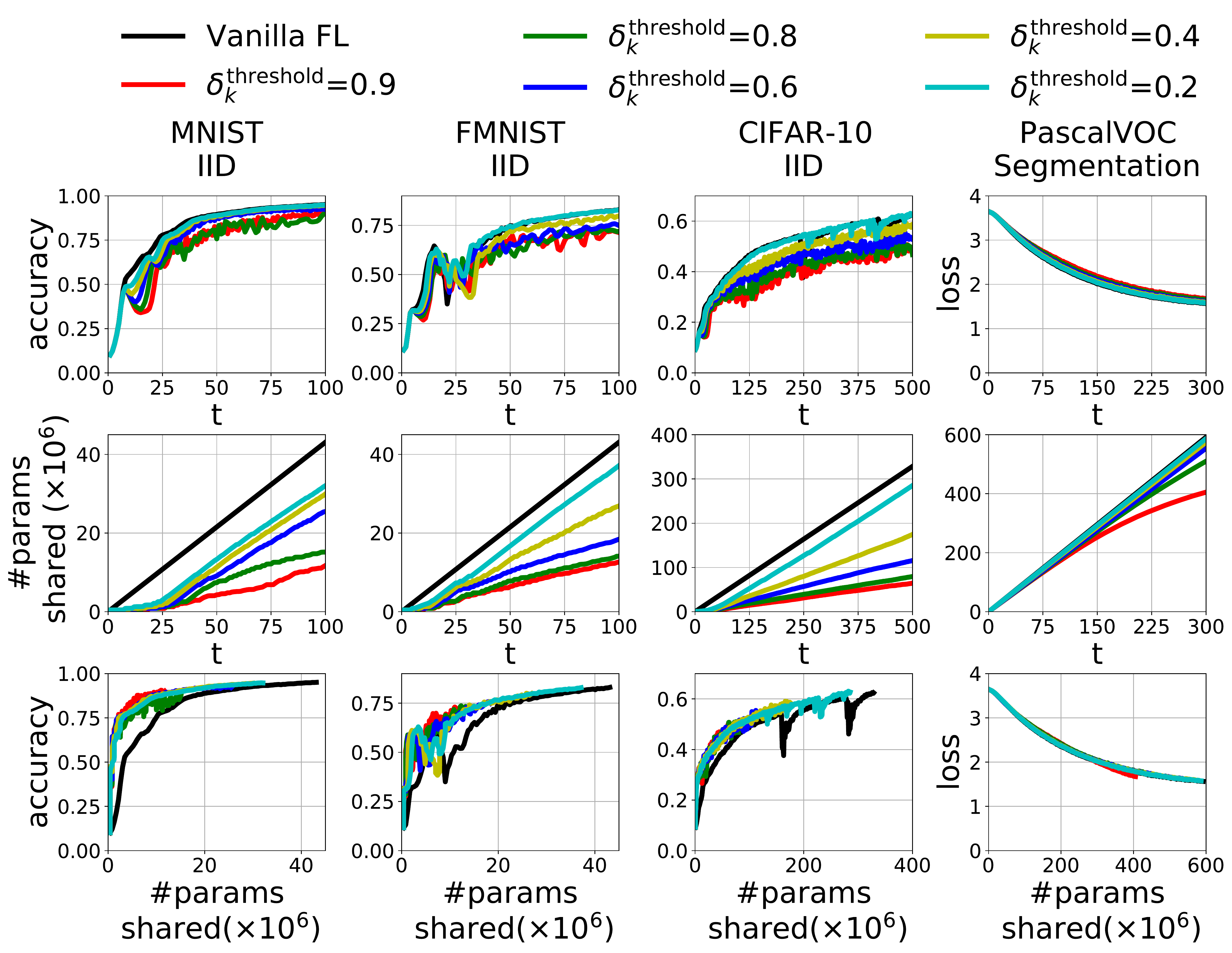}}
  \vspace{-1mm}
  \caption{\small{\textit{Effect of $\subsup{\delta}{k}{\mathsf{threshold}}$ on {\algName}}. Experimental results in Fig.~\ref{fig:rho_effect} repeated for datasets: \textbf{CIFAR-10}, \textbf{FMNIST}, and \textbf{MNIST} (iid data distribution) using classifier: \textbf{CNN}, and dataset: \textbf{PascalVOC} using \textbf{U-Net} architecture.}}
  \label{fig:rho_effect_cnn}
\end{figure}

\begin{figure}[t]
  \centering
  \centerline{\includegraphics[width=0.9\textwidth]{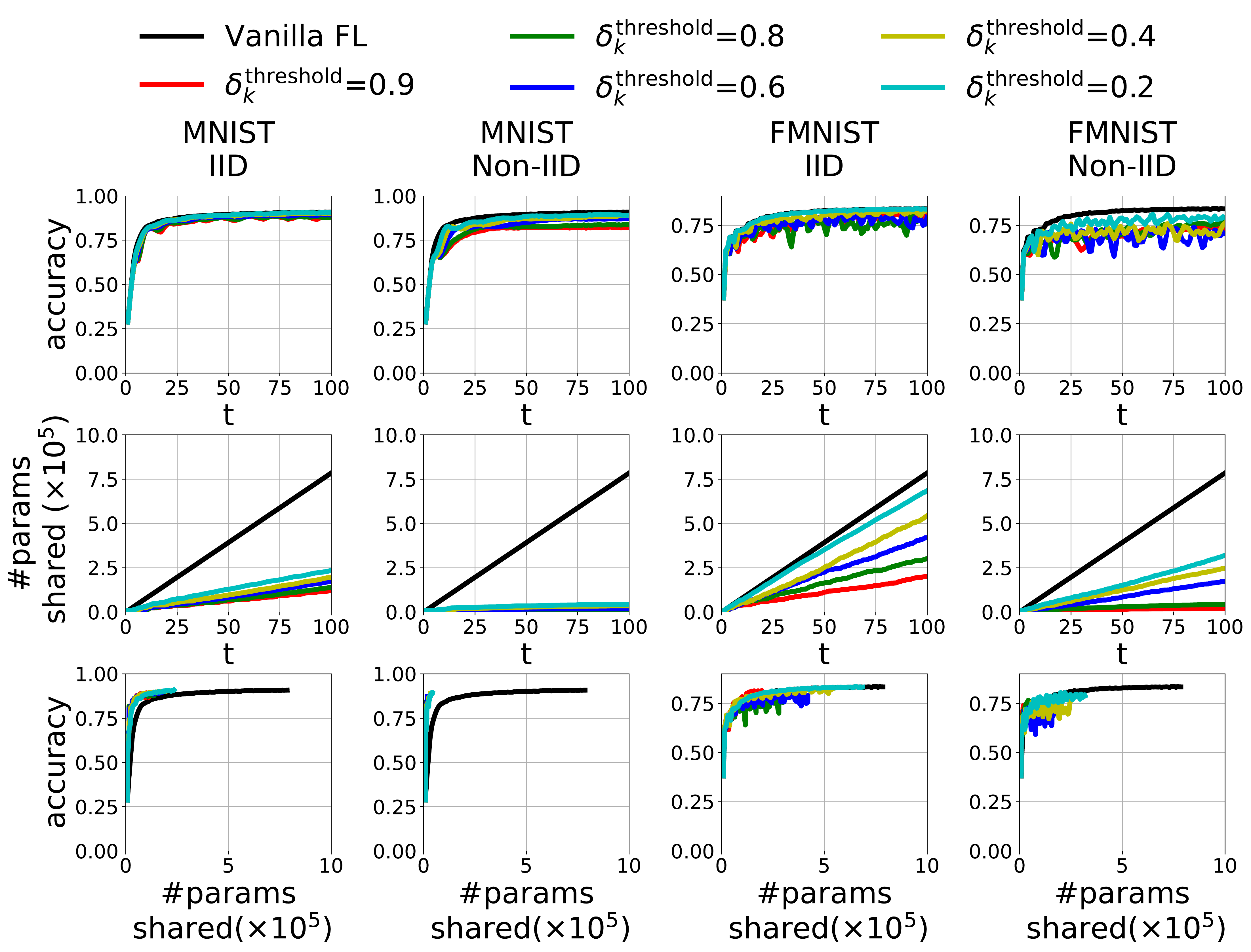}}
  \caption{\small{\textit{Effect of $\subsup{\delta}{k}{\mathsf{threshold}}$ on {\algName}}. Experimental results in Fig.~\ref{fig:rho_effect} repeated for datasets: \textbf{FMNIST} and \textbf{MNIST} (both iid and non-iid data distribution) using classifier: \textbf{FCN}.}}
      \label{fig:rho_effect_fcn}
\end{figure}

\begin{figure}[t]
  \centering
    \centerline{\includegraphics[width=0.75\textwidth]{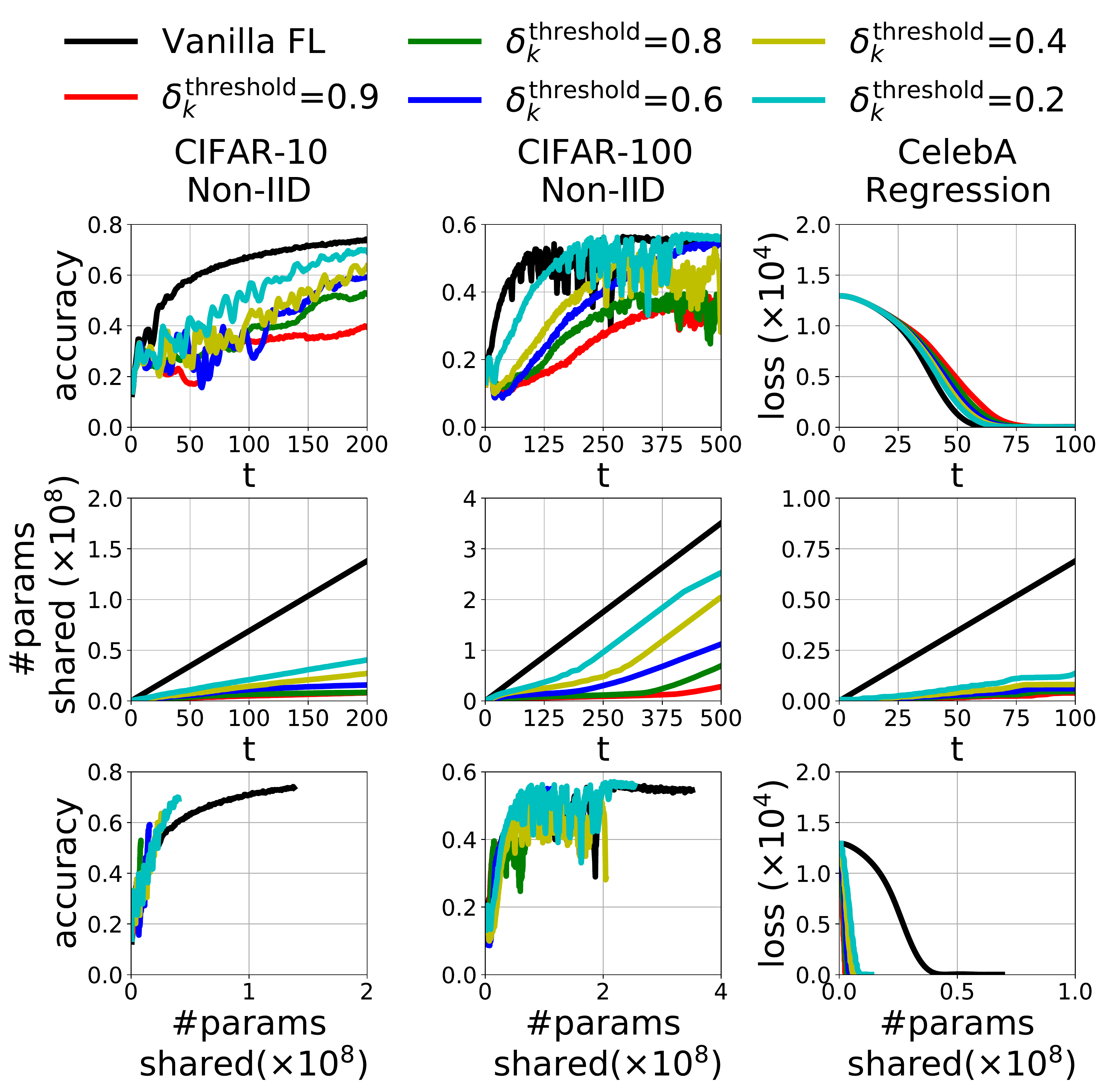}}
  \caption{\small{\textit{Effect of $\subsup{\delta}{k}{\mathsf{threshold}}$ on {\algName}}. Experimental results in Fig.~\ref{fig:rho_effect} repeated for datasets: \textbf{CIFAR-10}, \textbf{CIFAR-100} (non-iid data distribution), and \textbf{CelebA} (face landmark regression task) using classifier: \textbf{Resnet18}.}}
  \label{fig:rho_effect_resnet18}
\end{figure}


\begin{figure}[t]
  \centering
    \centerline{\includegraphics[width=1.0\textwidth]{images/acc_comm_curve/expt_2_plugnplay_main_iid_with_acc_comm_curve.pdf}}
  \caption{\small{\textit{{\algName} as a Plug-and-Play Algorithm.}. Experimental results in Fig.~\ref{fig:plugnplay} repeated for dataset: \textbf{CIFAR-10}, \textbf{FMNIST}, and \textbf{MNIST} (iid data distribution) using classifier: \textbf{CNN}.}}
  \label{fig:plugnplay_cnn}
\end{figure}

\begin{figure}[t]
  \centering
    \centerline{\includegraphics[width=1.2\textwidth]{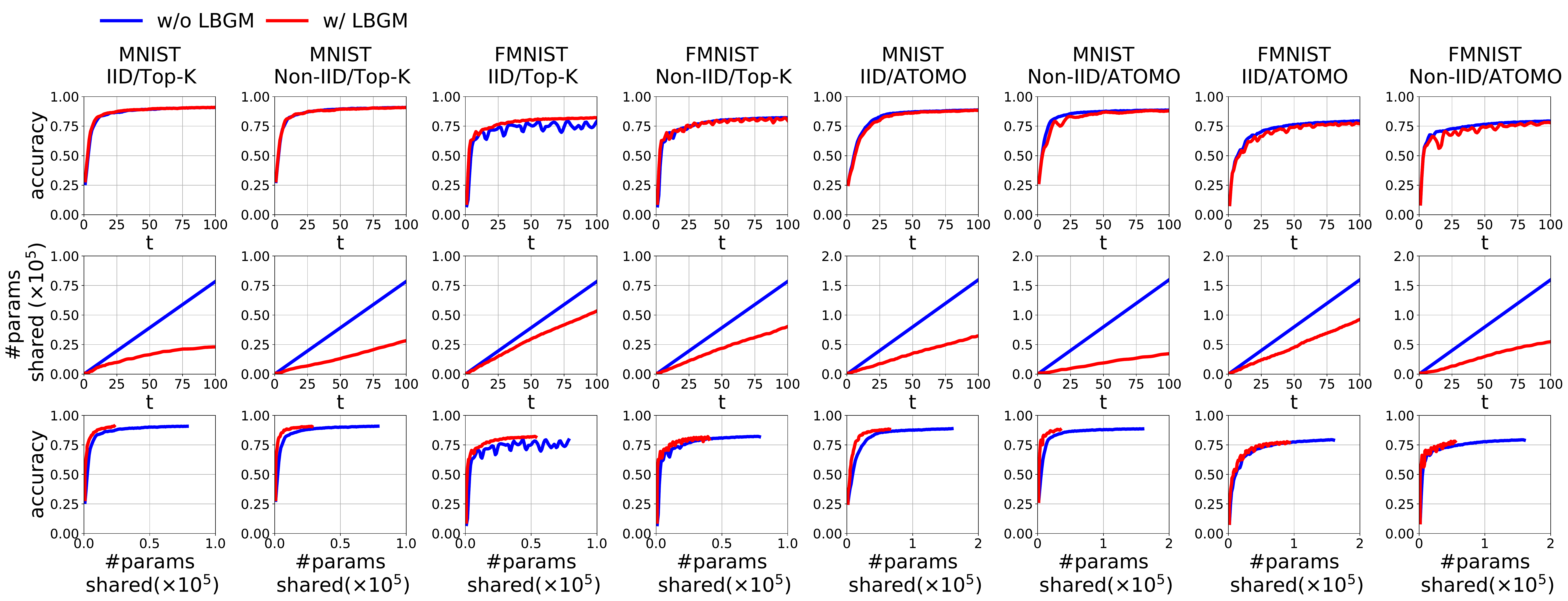}}
    \caption{\small{\textit{{\algName} as a Plug-and-Play Algorithm.}. Experimental results in Fig.~\ref{fig:plugnplay} repeated for datasets: \textbf{FMNIST} and \textbf{MNIST} (both iid and non-iid data distribution) using classifier: \textbf{FCN}.}}
  \label{fig:plugnplay_fcn}
  \vspace{-4mm}
\end{figure}

\begin{figure}[t]
  \centering
    \centerline{\includegraphics[width=1.0\textwidth]{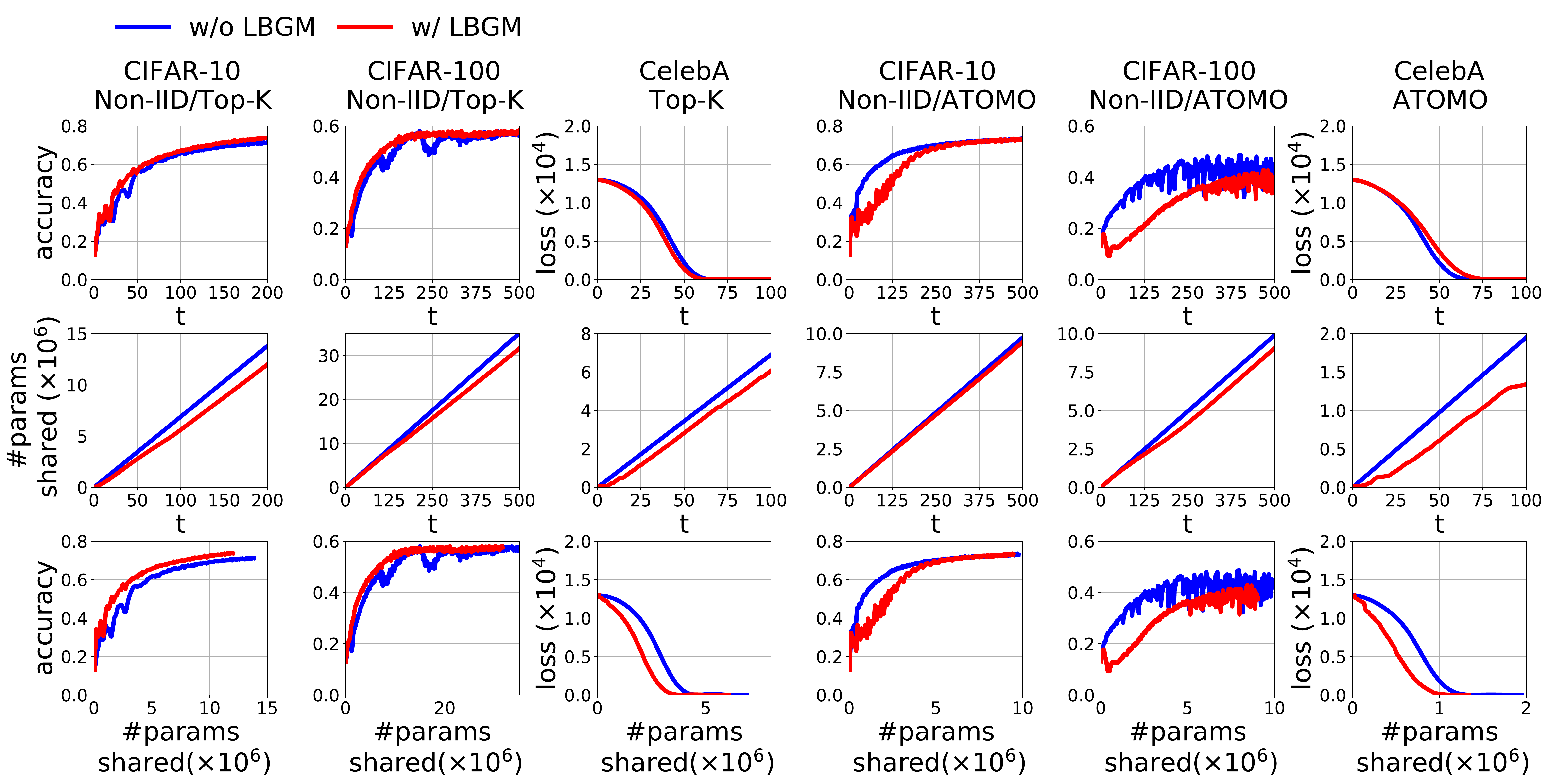}}
  \caption{\small{\textit{Effect of $\subsup{\delta}{k}{\mathsf{threshold}}$ on {\algName}}. Experimental results in Fig.~\ref{fig:plugnplay} repeated for datasets: \textbf{CIFAR-10}, \textbf{CIFAR-100} (non-iid data distribution), and \textbf{CelebA} (face landmark regression task) using classifier: \textbf{Resnet18}.}}
  \label{fig:plugnplay_resnet18}
\end{figure}


\begin{figure}[t]
  \centering
    \centerline{\includegraphics[width=0.75\textwidth]{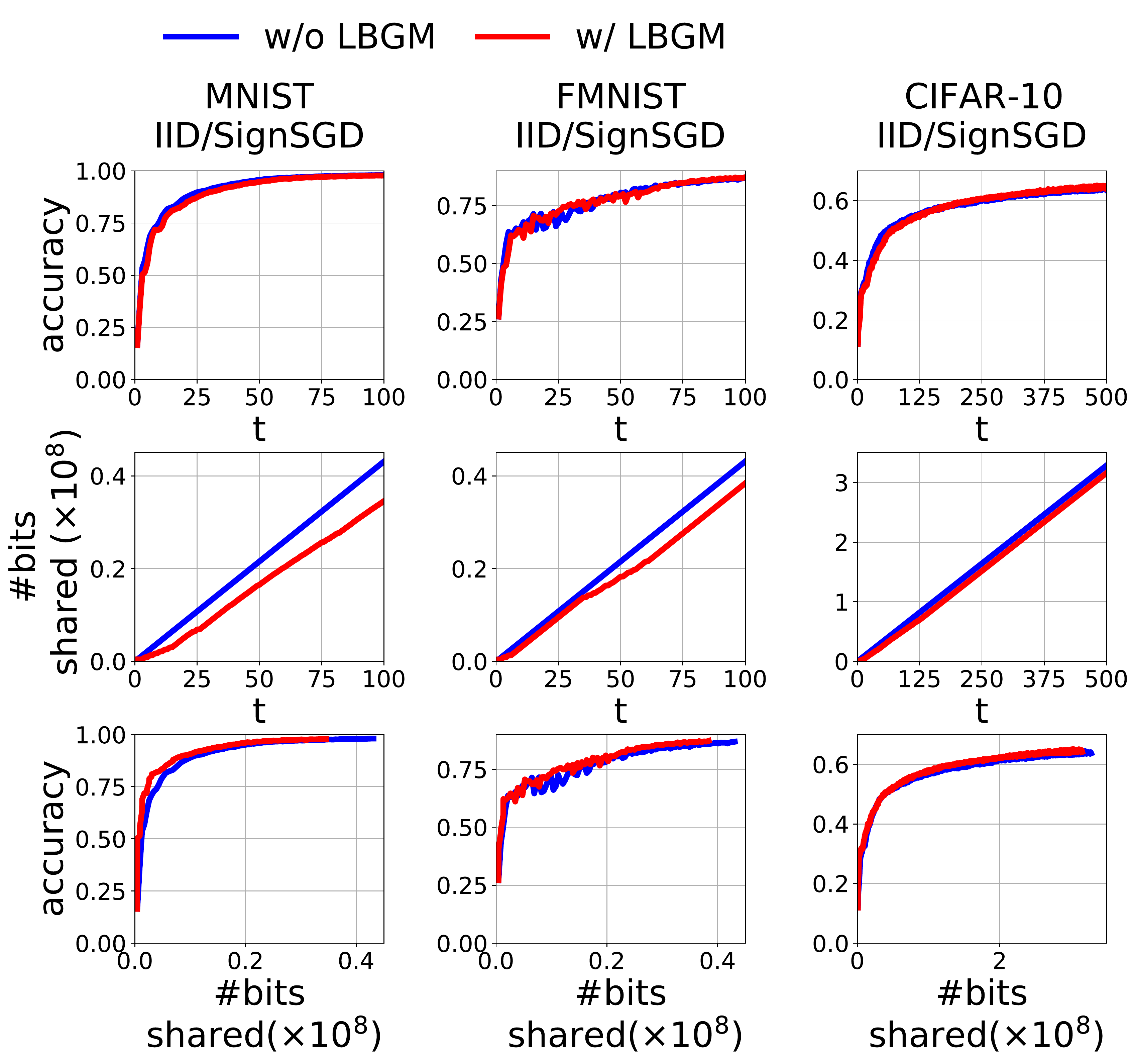}}
  \caption{\small{Experimental results in Fig.~\ref{fig:dist_training} repeated for dataset: \textbf{CIFAR-10}, \textbf{FMNIST}, \textbf{FMNIST} (iid data distribution) using classifier: \textbf{CNN}.}}
  \label{fig:dist_training_cnn}
\end{figure}

\begin{figure}[t]
  \centering
    \centerline{\includegraphics[width=0.9\textwidth]{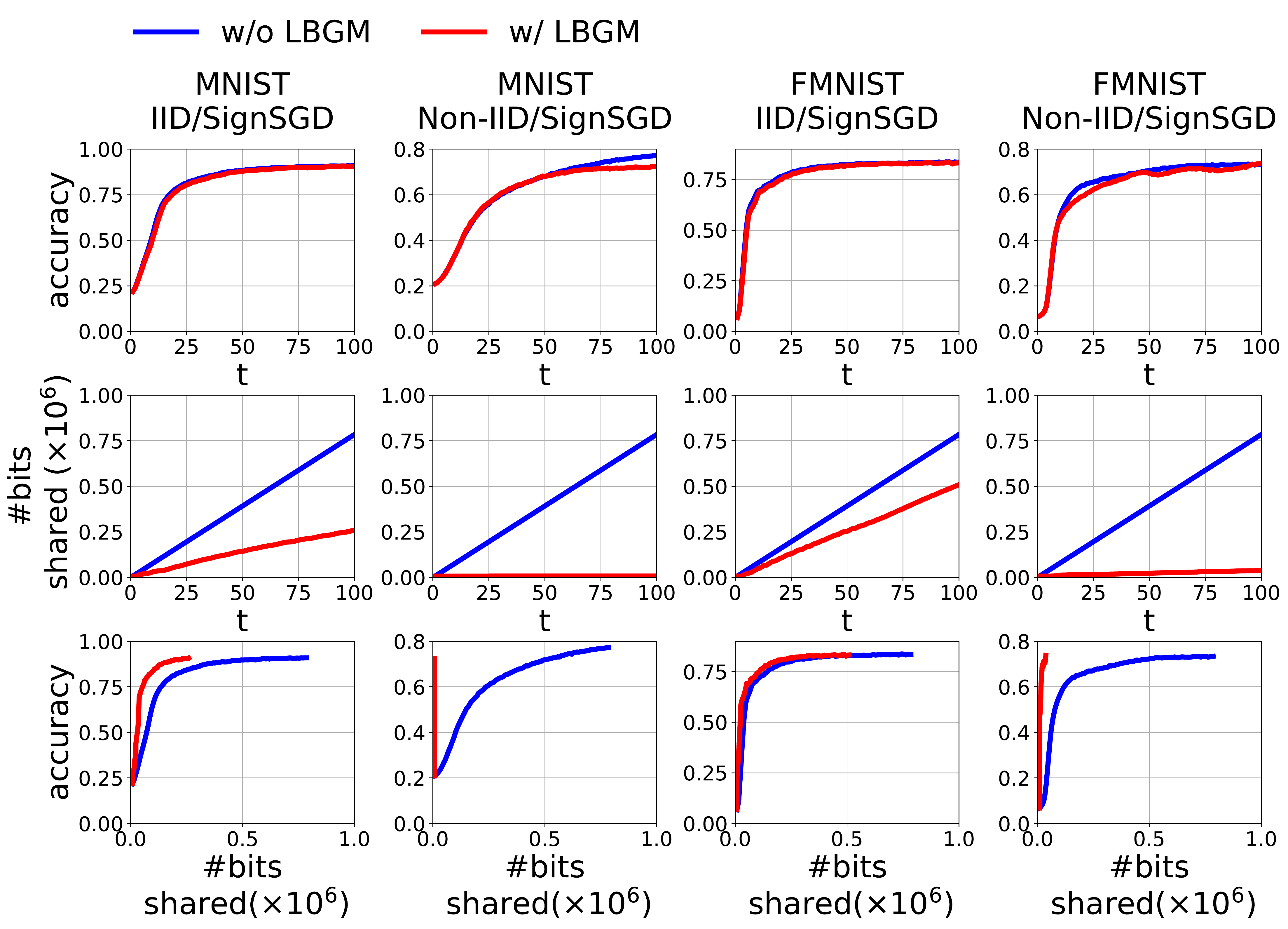}}
  \caption{\small{Experimental results in Fig.~\ref{fig:dist_training} repeated for dataset: \textbf{FMNIST} and \textbf{MNIST} (both iid and non-iid data distribution) using classifier: \textbf{FCN}.}}
  \label{fig:dist_training_fcn}
\end{figure}

\begin{figure}[t]
  \centering
    \centerline{\includegraphics[width=0.75\textwidth]{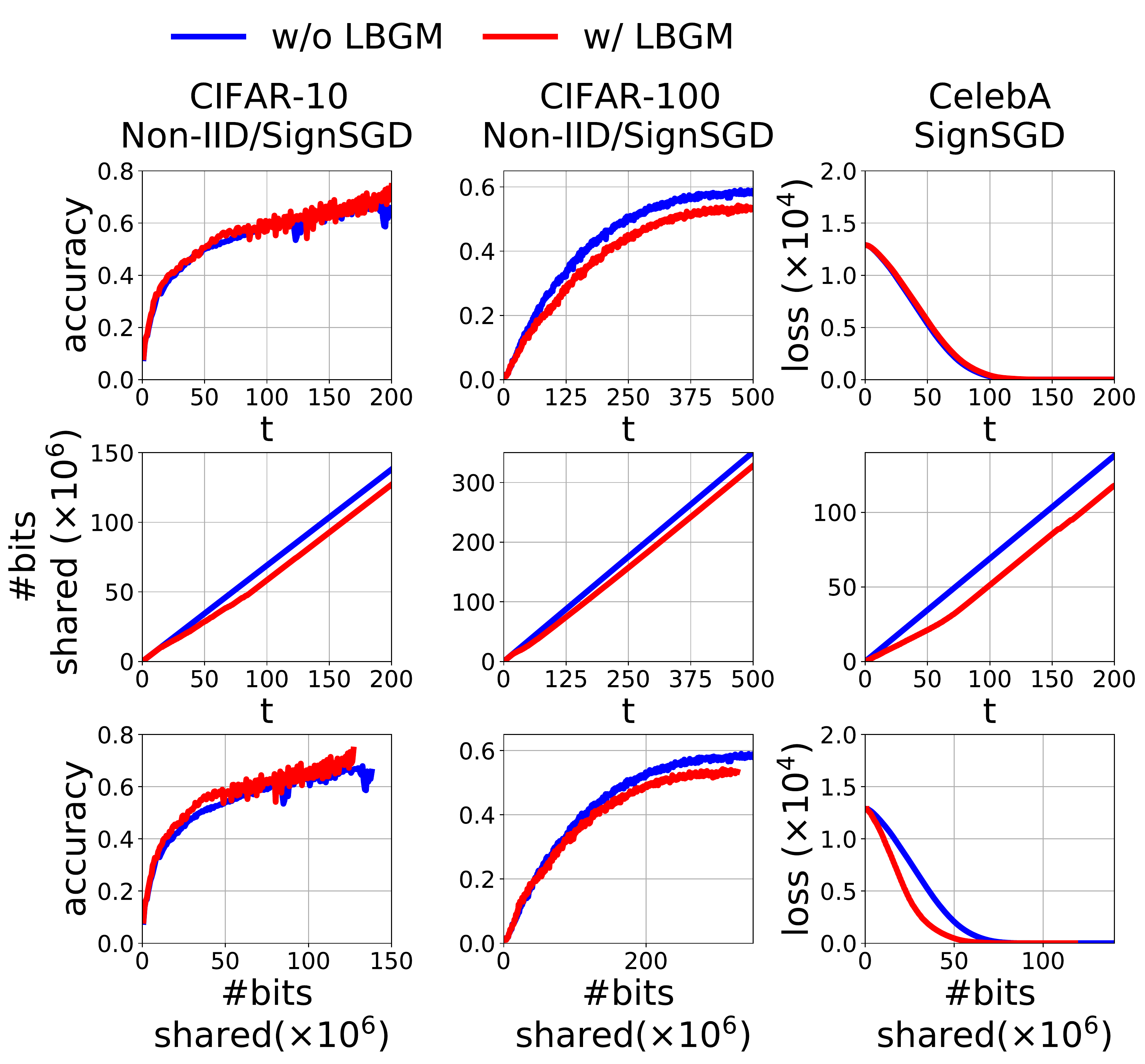}}
  \caption{\small{\textit{Effect of $\subsup{\delta}{k}{\mathsf{threshold}}$ on {\algName}}. Experimental results in Fig.~\ref{fig:dist_training} repeated for datasets: \textbf{CIFAR-10}, \textbf{CIFAR-100} (non-iid data distribution), and \textbf{CelebA} (face landmark regression task) using classifier: \textbf{Resnet18}.}}
  \label{fig:dist_training_resnet18}
\end{figure}


\begin{figure}[t]
  \centering
    \centerline{\includegraphics[width=0.9\textwidth]{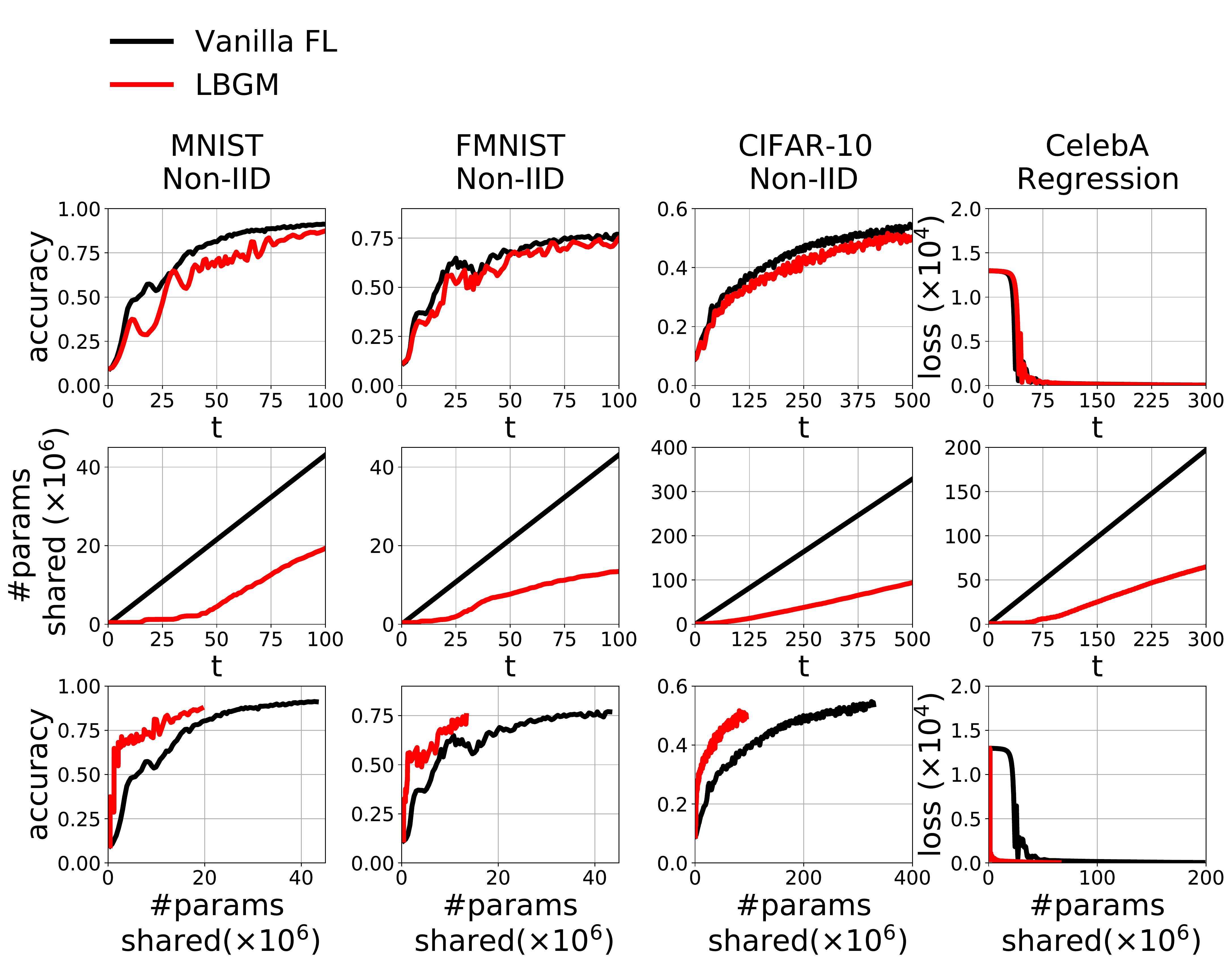}}
  \caption{\small{Experimental results in Fig.~\ref{fig:standalone} repeated for dataset: \textbf{CIFAR-10}, \textbf{FMNIST}, \textbf{FMNIST} (non-iid data distribution) and \textbf{CelebA} (face landmark regression taks) using classifier: \textbf{CNN} under \textbf{$50\%$ client sampling} for both Vanilla FL and LBGM.}}
  \label{fig:sampled_training_cnn_non_iid}
\end{figure}

\begin{figure}[t]
  \centering
    \centerline{\includegraphics[width=0.7\textwidth]{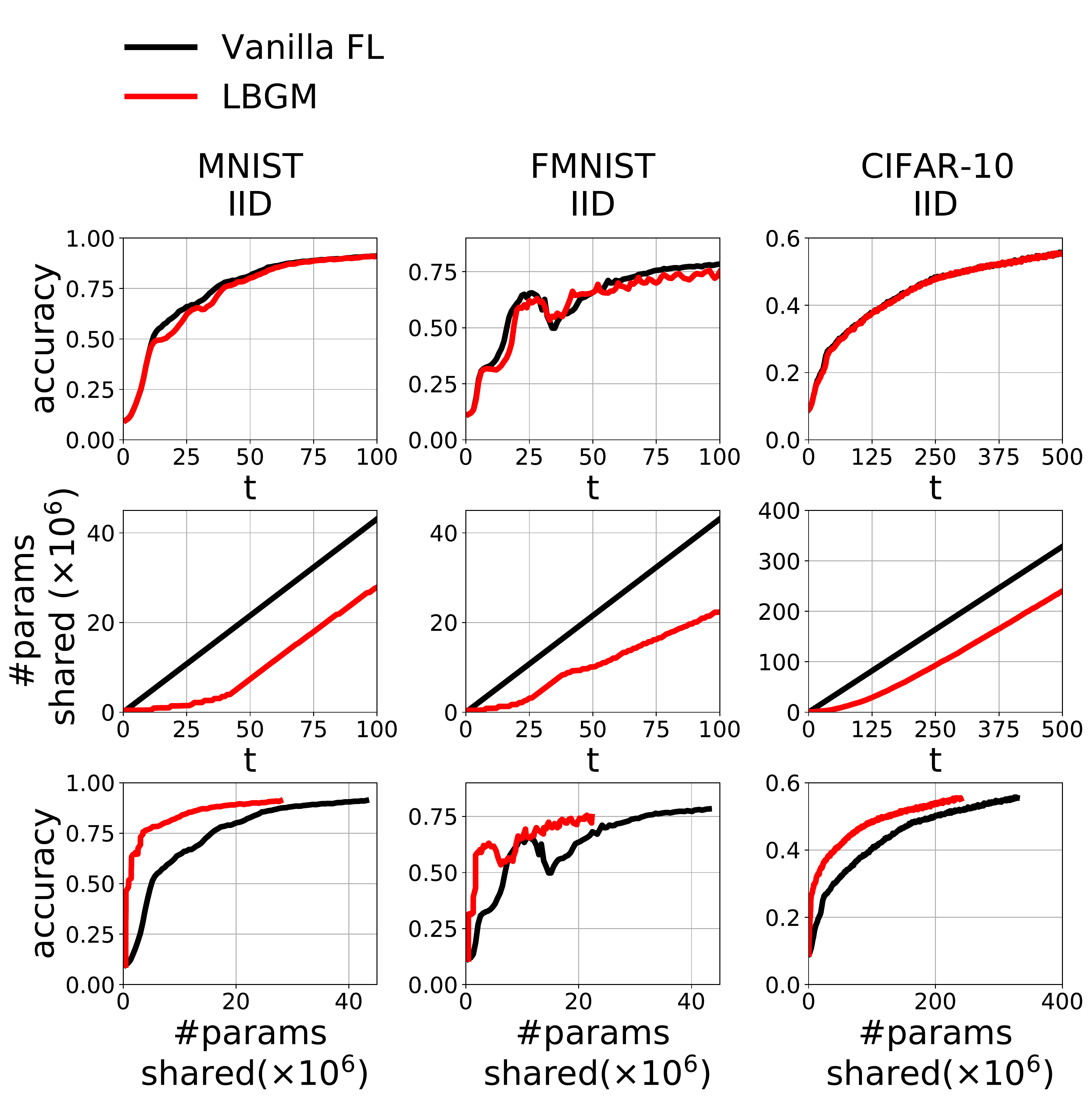}}
  \caption{\small{Experimental results in Fig.~\ref{fig:standalone} repeated for dataset: \textbf{CIFAR-10}, \textbf{FMNIST}, \textbf{FMNIST} (iid data distribution) using classifier: \textbf{CNN} under \textbf{$50\%$ client sampling} for both Vanilla FL and LBGM.}}
  \label{fig:sampled_training_cnn_iid}
\end{figure}

\end{document}